\documentclass[english]{article}

\usepackage{geometry}
\usepackage[T1]{fontenc}
\usepackage{amsmath,amssymb,amsthm,graphicx,bm,dsfont}
\usepackage{epsfig, fontawesome}
\usepackage[authoryear]{natbib} 
\usepackage{psfrag}
\usepackage{enumerate} \usepackage{setspace}
\usepackage{float} 
\usepackage{color}
\usepackage{array}
\usepackage{microtype, wasysym, comment,mathtools}
\usepackage{subfigure}
\usepackage{booktabs} 
\usepackage{enumitem,multirow}
\usepackage[colorlinks=true,allcolors=blue]{hyperref}%
\usepackage{graphicx} 
\usepackage{epstopdf}
\usepackage[ruled,vlined]{algorithm2e}
\usepackage{algorithmic}

\definecolor{yuting}{RGB}{255,69,0}
\definecolor{gen}{RGB}{199,21,133} 
\definecolor{yxc}{RGB}{21,199,133} 
\definecolor{yc}{RGB}{21,0,255}

\DeclareMathOperator{\ind}{\mathds{1}}

\newcommand{\mymid}{\,|\,}

\newcommand{\overalpha}{\overline{\alpha}}

\newcommand{\Pdata}{p_{\mathsf{data}}}
\newcommand{\diff}{\mathrm{d}}

\newcommand{\score}{\mathsf{score}}
\newcommand{\Jacobi}{\mathsf{Jacobi}}

\newtheorem{assumption}{\textbf{Assumption}}
\newtheorem{definition}{\textbf{Definition}}
\newtheorem{claim}{\textbf{Claim}}
\newtheorem{remark}{\textbf{Remark}}

\theoremstyle{plain}

\newtheorem{theo}{Theorem}[section]

\newtheorem{lem}{Lemma}[section]
\newtheorem{prop}{Proposition}[section]
\newtheorem{cor}{Corollary}[section]

\theoremstyle{definition} 

\newtheorem{nota}{Notation}[section]
\newtheorem{de}{Definition}[section]
\newtheorem{exa}{Example}[section]
\newtheorem{as}{Assumption}[section]
\newtheorem{alg}{Algorithm}[section]

\newcommand{\btheo}{\begin{theo}}
\newcommand{\bde}{\begin{de}}
\newcommand{\ble}{\begin{lem}}
\newcommand{\bpr}{\begin{prop}}
\newcommand{\bno}{\begin{nota}}
\newcommand{\bex}{\begin{exa}}
\newcommand{\bcor}{\begin{cor}}
\newcommand{\spro}{\begin{proof}}
\newcommand{\bas}{\begin{as}}
\newcommand{\balg}{\begin{alg}}

\newcommand{\etheo}{\end{theo}}
\newcommand{\ede}{\end{de}}
\newcommand{\ele}{\end{lem}}
\newcommand{\epr}{\end{prop}}
\newcommand{\eno}{\end{nota}}
\newcommand{\eex}{\end{exa}}
\newcommand{\ecor}{\end{cor}}
\newcommand{\fpro}{\end{proof}}
\newcommand{\eas}{\end{as}}
\newcommand{\ealg}{\end{alg}}

\theoremstyle{plain}

\newtheorem{theos}{Theorem}
\newtheorem{props}{Proposition}
\newtheorem{lems}{Lemma}
\newtheorem{cors}{Corollary}

\theoremstyle{definition}
\newtheorem{exas}{Example}
\newtheorem{algs}{Algorithm}
\newtheorem{asss}{Assumption}
\newtheorem{defns}{Definition}

\newcommand{\btheos}{\begin{theos}}
\newcommand{\etheos}{\end{theos}}
\newcommand{\bprops}{\begin{props}}
\newcommand{\eprops}{\end{props}}
\newcommand{\bdes}{\begin{defns}}
\newcommand{\edes}{\end{defns}}
\newcommand{\blems}{\begin{lems}}
\newcommand{\elems}{\end{lems}}
\newcommand{\bcors}{\begin{cors}}
\newcommand{\ecors}{\end{cors}}
\newcommand{\bexs}{\begin{exas}}
\newcommand{\eexs}{\end{exas}}
\newcommand{\balgs}{\begin{algs}}
\newcommand{\ealgs}{\end{algs}}
\newcommand{\bass}{\begin{asss}}
\newcommand{\eass}{\end{asss}}

\newcommand{\ltwo}[1]{\|#1\|_2}

\newcommand{\real}{\ensuremath{\mathbb{R}}}

\newcommand{\defn}{\coloneqq}

\newcommand{\Exs}{\ensuremath{\mathbb{E}}}

\long\def\comment#1{}

\newcommand{\HACKPROOF}{\begin{proof}}
\newcommand{\HACKENDPROOF}{\end{proof}}

\newlength{\widebarargwidth}
\newlength{\widebarargheight}
\newlength{\widebarargdepth}

\makeatletter
\long\def\@makecaption#1#2{
        \vskip 0.8ex
        \setbox\@tempboxa\hbox{\small {\bf #1:} #2}
        \parindent 1.5em  
        \dimen0=\hsize
        \advance\dimen0 by -3em
        \ifdim \wd\@tempboxa >\dimen0
                \hbox to \hsize{
                        \parindent 0em
                        \hfil 
                        \parbox{\dimen0}{\def\baselinestretch{0.96}\small
                                {\bf #1.} #2
                                } 
                        \hfil}
        \else \hbox to \hsize{\hfil \box\@tempboxa \hfil}
        \fi
        }
\makeatother

\allowdisplaybreaks

\begin{document}

\title{Towards Faster Non-Asymptotic Convergence for \\ Diffusion-Based Generative Models}

\author{Gen Li\thanks{Department of Statistics, The Chinese University
of Hong Kong, Hong Kong.} \and  Yuting Wei\thanks{Department of Statistics and Data Science, Wharton School, University
of Pennsylvania, Philadelphia, PA 19104, USA.} 
\and Yuxin Chen\footnotemark[2] \thanks{Department of Electrical and Systems Engineering, University
of Pennsylvania, Philadelphia, PA 19104, USA.} 
\and Yuejie Chi\thanks{Department of Electrical and Computer Engineering, Carnegie Mellon University, Pittsburgh, PA 15213, USA.} }

\date{June 2023;~ Revised: February 2024}

\maketitle

\medskip

\begin{abstract}

	Diffusion models, which convert noise into new data instances by learning to reverse a Markov diffusion process, have become a cornerstone in contemporary generative modeling. While their practical power has now been widely recognized, the theoretical underpinnings remain far from mature.  In this work, we develop a suite of non-asymptotic theory towards understanding the data generation process of diffusion models in discrete time, assuming access to $\ell_2$-accurate estimates of the (Stein) score functions. For a popular deterministic sampler (based on the probability flow ODE), we establish a convergence rate proportional to $1/T$ (with $T$ the total number of steps), improving upon past results; for another mainstream stochastic sampler (i.e., a type of the denoising diffusion probabilistic model), we derive a convergence rate proportional to $1/\sqrt{T}$, matching the state-of-the-art theory. Imposing only minimal assumptions on the target data distribution (e.g., no smoothness assumption is imposed), our results characterize how $\ell_2$ score estimation errors affect the quality of the data generation processes.  In contrast to prior works, our theory is developed based on an elementary yet versatile non-asymptotic approach without resorting to toolboxes for SDEs and ODEs. Further, we design two accelerated variants, improving the convergence to $1/T^2$ for the ODE-based sampler and $1/T$ for the DDPM-type sampler, which might be of independent theoretical and empirical interest.

 
\end{abstract}


\noindent \textbf{Keywords:} diffusion models, score-based generative modeling, non-asymptotic theory, probability flow ODE, reverse SDE, denoising diffusion probabilistic model 

\setcounter{tocdepth}{2}
\tableofcontents

\section{Introduction}

Diffusion models have emerged as a cornerstone in contemporary generative modeling, a task that learns to generate new data instances (e.g., images, text, audio) that look similar in distribution to the training data \citep{ho2020denoising,sohl2015deep,song2019generative,dhariwal2021diffusion,jolicoeur2021adversarial,chen2021wavegrad,kong2021diffwave,austin2021structured}.  
Originally proposed by \citet{sohl2015deep} and later popularized by \citet{song2019generative,ho2020denoising}, 
the mainstream diffusion generative models --- e.g., denoising diffusion probabilistic models (DDPMs) \citep{ho2020denoising} and  
denoising diffusion implicit models (DDIMs) \citep{song2020denoising} --- have underpinned major successes in content generators like DALL$\cdot$E 2 \citep{ramesh2022hierarchical}, Stable Diffusion \citep{rombach2022high} 
and Imagen \citep{saharia2022photorealistic}, 
claiming state-of-the-art performance in the now broad field of generative artificial intelligence (AI). 
See  \citet{yang2022diffusion,croitoru2023diffusion} for overviews of recent development.

In a nutshell, a diffusion generative model is based upon two stochastic processes in $\real^d$: 
\begin{itemize}
	\item[1)] a forward process 
		\begin{equation}
			X_0 \rightarrow X_1 \rightarrow \cdots \rightarrow X_T
			\label{eq:forward-process-informal}
		\end{equation}
		that starts from a sample drawn from the target data distribution (e.g., of natural images) 
		and gradually diffuses it into a noise-like distribution (e.g., standard Gaussians); 
	\item[2)] a reverse process 
		\begin{equation}
			Y_T \rightarrow Y_{T-1} \rightarrow \cdots \rightarrow Y_0
			\label{eq:reverse-process-informal}
		\end{equation}
		that starts from pure noise (e.g., standard Gaussians) and successively converts it into new samples sharing similar distributions as the target data distribution. 
\end{itemize}
\noindent 
Transforming data into noise in the forward process is straightforward, often hand-crafted by increasingly injecting more noise into the data at hand.   
What is challenging is the construction of the reverse process: 
how to generate the desired information out of pure noise? 
To do so, a diffusion model learns to build a reverse process \eqref{eq:reverse-process-informal}
 that imitates the dynamics of the forward process \eqref{eq:forward-process-informal} in a time-reverse fashion; 
 more precisely, the design goal is to ascertain  distributional proximity\footnote{Two random vectors $X$ and $Y$ are said to obey $X \overset{\mathrm{d}}{=} Y$ (resp.~$X \overset{\mathrm{d}}{\approx} Y$) if they are equivalent (resp.~close) in distribution.} 
\begin{equation}
	Y_t \,\overset{\mathrm{d}}{\approx}\, X_t, \qquad t = T,\cdots,1	
\end{equation}
through proper learning based on how the training data propagate in the forward process. 
Encouragingly, there often exist feasible strategies to achieve this goal  
as long as faithful estimates about the (Stein) score functions --- the gradients of the log marginal density of the forward process --- are available, 
an intriguing fact that can be illuminated by the existence and construction of reverse-time stochastic differential equations (SDEs) \citep{anderson1982reverse,haussmann1986time} (see Section~\ref{sec:det-stochastic-samplers} for more precise discussions). 
Viewed in this light, a diverse array of diffusion models are frequently referred to as {\em score-based generative modeling (SGM)}. 
The popularity of SGM was initially motivated by, and has since further inspired,  numerous recent studies on the problem of learning score functions, 
a subroutine that also goes by the name of score matching (e.g., \citet{hyvarinen2005estimation,hyvarinen2007some,vincent2011connection,song2020sliced,koehler2022statistical}).

Nonetheless, despite the mind-blowing empirical advances, a mathematical theory for diffusion generative models is still in its infancy. 
Given the complexity of developing a full-fledged end-to-end theory, a divide-and-conquer approach has been advertised, 
decoupling the score learning phase (i.e., how to estimate score functions reliably from training data) 
and the generative sampling phase (i.e., how to generate new data instances given the score estimates). 
In particular, 
the past two years have witnessed growing interest and remarkable progress from the theoretical community
towards understanding the generative sampling phase 
\citep{block2020generative,de2021diffusion,liu2022let,de2022convergence,lee2023convergence,pidstrigach2022score,chen2022sampling,chen2022improved,chen2023restoration,tang2023diffusion,tang2024score,li2024accelerating}. 
For instance, 
polynomial-time convergence guarantees have been established for stochastic samplers (e.g., \citet{chen2022sampling,chen2022improved,benton2023linear,tang2023diffusion}) and deterministic samplers (e.g., \citet{chen2023restoration,benton2023error}),  both of which accommodated a fairly general family of data distributions.


\paragraph{This paper.} 
The present paper contributes to this growing list of theoretical endeavors by 
developing a new suite of non-asymptotic theory for several score-based generative modeling algorithms.  
We concentrate on two types of samplers \citep{song2020score} in discrete time:   
(i) a deterministic sampler based on a sort of ordinary differential equations (ODEs) called probability flow ODEs (which is closely related to the DDIM sampler);  
and (ii) a DDPM-type stochastic sampler motivated by reverse-time SDEs. 
We impose only minimal assumptions on the target data distribution (e.g., no smoothness condition is needed), 
and would like to quantify the impact of $\ell_2$ score estimation errors.   
In comparisons to past works, our main contributions are three-fold.

\begin{itemize}
	\item {\em Non-asymptotic convergence guarantees.} 
		For a popular deterministic sampler, we demonstrate that the number of steps needed to yield $\varepsilon$-accuracy 
		--- meaning that the total variation (TV) distance between the distribution of $X_1$ and that of $Y_1$ is no larger than $\varepsilon$ --- is proportional to $1/\varepsilon$ (in addition to other polynomial dimension dependency). This improves upon prior convergence guarantees considerably \citep{chen2023restoration} 
		and does not exhibit exponential dependency on the smoothness or regularity conditions as in \citet{chen2023restoration,benton2023error} (e.g., the regularity parameter used in \citet{benton2023error} might even scale with the dimension $d$). 
		For another DDPM-type stochastic sampler, 
		we establish an iteration complexity proportional to $1/\varepsilon^2$ via a new non-asymptotic analysis framework, matching existing theory \citet{chen2022sampling,chen2022improved,benton2023linear} 
		in terms of the $\varepsilon$-dependency. 

	\item {\em $\ell_2$ score estimation errors for the determinstic sampler.} 
		Our theory for the deterministic sampler reveals that the TV distance between $X_1$ and $Y_1$ are shown to be proportional to the $\ell_2$ score estimation error as well as the associated mean Jacobian errors. 
		As far as we know, this is the first result for this deterministic sampler that accounts for score estimation errors in discrete time. 
		In comparison, existing theoretical results that accommodate score errors for the probability flow ODE approach 
		either study stochastic variations of this deterministic sampler \citep{chen2023probability} (so that the samplers are no longer the original deterministic sampler) or fall short of accommodating discretization errors \citep{benton2023error}.

	\item {\em An elementary non-asymptotic analysis framework.}  
		From the technical point of view, the analysis framework laid out in this paper is fully non-asymptotic in nature. 
		In contrast to prior theoretical analyses that take a detour to study the continuum limits and then control the discretization error, 
		our approach tackles the discrete-time processes directly using elementary analysis strategies.  
		No knowledge of SDEs or ODEs is required for establishing our theory, 
		thereby resulting in a more versatile framework and sometimes lowering the technical barrier towards understanding diffusion models. 
		
	\item {\em Accelerating data generation processes.} 
		In order to further speed up the sampling processes, 
		we develop an accelerated variant for each of the above two samplers, taking advantage of estimates of a small number of additional quantities. 
		As it turns out, these variants achieve more rapid convergence, 
		with the deterministic (resp.~stochastic) variant exhibiting a $1/\sqrt{\varepsilon}$ (resp.~$1/\varepsilon$) scaling in the accuracy level $\varepsilon$ (again measured in terms of the TV distance). 

\end{itemize}

%
%
%
%
%

\paragraph{Notation.}
Before proceeding, we introduce a couple of notation to be used throughout. 
For any two functions $f(d,T)$ and $g(d,T)$, 
we adopt the notation $f(d,T)\lesssim g(d,T)$ or $f(d,T)=O( g(d,T) )$ (resp.~$f(d,T)\gtrsim g(d,T)$)
to mean that there exists some universal constant $C_1>0$ such that $f(d,T)\leq C_1 g(d,T)$ (resp.~$f(d,T)\geq C_1 g(d,T)$) for all $d$ and $T$; moreover, the notation  $f(d,T)\asymp g(d,T)$ indicates that $f(d,T)\lesssim g(d,T)$ and $f(d,T)\gtrsim g(d,T)$ hold at once. 
The notation $\widetilde{O}(\cdot)$ is defined similar to $O(\cdot)$ except that it hides the logarithmic dependency. 
Additionally, the notation  $f(d,T)=o\big( g(d,T) \big)$ means that $f(d,T)/g(d,T) \rightarrow 0$ as $d,T$ tend to infinity. 
We shall often use capital letters to denote random variables/vectors/processes, and lowercase letters for deterministic variables. 
For any two probability measures $P$ and $Q$, the total variation (TV) distance between them is defined to be $\mathsf{TV}(P,Q)\coloneqq \frac{1}{2}\int |\mathrm{d}P - \mathrm{d}Q|$. 
Throughout the paper, $p_X(\cdot)$ (resp.~$p_{X\mymid Y}(\cdot \mymid \cdot)$) denotes the probability density function of $X$ (resp.~$X$ given $Y$). 
For any matrix $A$, we denote by $\|A\|$ (resp.~$\|A\|_{\mathrm{F}}$) the spectral norm (resp.~Frobenius norm) of $A$. 
Also, for any vector-valued function $f$, we let $J_f$ or $\frac{\partial f}{\partial x}$ represent the Jacobian matrix of $f$.





\section{Preliminaries} 
\label{sec:preliminaries}

In this section, we introduce the basics of diffusion generative models. 
The ultimate goal of a generative model can be concisely stated: 
given  data samples drawn from an unknown distribution of interest $\Pdata$ in $\real^d$, 
we wish to generate new samples whose distributions closely resemble $\Pdata$.



\subsection{Diffusion generative models}

Towards achieving the above goal, 
a diffusion generative model typically encompasses two Markov processes: a forward process and a reverse process, as described below.

\paragraph{The forward process.}
In the forward chain, one progressively injects noise into the data samples to diffuse and obscure the data. 
The distributions of the injected noise are often hand-picked, 
with the standard Gaussian distribution receiving widespread adoption. 
More specifically, the forward Markov process produces a sequence of $d$-dimensional random vectors $X_1\rightarrow X_2\rightarrow \cdots\rightarrow X_T $ as follows:   
\begin{subequations}
\label{eq:forward-process}
\begin{align}
	X_0 &\sim \Pdata,\\
	X_t &= \sqrt{1-\beta_t}X_{t-1} + \sqrt{\beta_t}\,W_{t}, \qquad 1\leq t\leq T,
\end{align}
\end{subequations}
where $\{W_t\}_{1\leq t\leq T}$ 
indicates a sequence of independent noise vectors drawn from $W_{t} \overset{\mathrm{i.i.d.}}{\sim} \mathcal{N}(0, I_d)$. 
The hyper-parameters $\{\beta_t \in (0,1)\}$ represent prescribed learning rate schedules 
that control the variance of the noise injected in each step. 
If we define 
\begin{align}
	\alpha_t\coloneqq 1 - \beta_t, 
	\qquad \overline{\alpha}_t \coloneqq \prod_{k = 1}^t \alpha_k ,\qquad 1\leq t\leq T,
\end{align}
then it can be straightforwardly verified that for every $1\leq t\leq T$, 
\begin{align}
\label{eqn:Xt-X0}
	X_t = \sqrt{\overalpha_t} X_{0} + \sqrt{1-\overalpha_t} \,\overline{W}_{t}
	\qquad \text{for some } \overline{W}_{t}\sim \mathcal{N}(0,I_d) .
\end{align}
%
Clearly, if the covariance of $X_0$ is also equal to $I_d$, then the covariance of $X_t$ is preserved throughout the forward process; 
for this reason, this forward process \eqref{eq:forward-process} is sometimes referred to as {\em variance-preserving} \citep{song2020score}. 
Throughout this paper, we employ the notation
\begin{equation}
	q_t \coloneqq \mathsf{law}\big( X_t \big)
	\label{eq:defn-qt}
\end{equation}
to denote the distribution of $X_t$. 
As long as  $\overalpha_{T}$ is vanishingly small, 
one has the following property for a general family of data distributions: 
\begin{equation}
	q_T \approx \mathcal{N}(0,I_d). 
\end{equation}
%






\paragraph{The reverse process.} 
The reverse chain $Y_T\rightarrow Y_{T-1}\rightarrow \ldots\rightarrow Y_1 $ is designed to (approximately) revert the forward process, 
allowing one to transform pure noise into new samples with matching distributions as the original data. 
To be more precise, by initializing it as
\begin{subequations}
	\label{eq:goal-reverse-process}
\begin{equation}
	Y_T \sim \mathcal{N}(0,I_d), 
\end{equation}
we seek to design a reverse-time Markov process with nearly identical marginals as the forward process,  namely, 
\begin{equation}
	(\text{goal})\qquad\qquad
	Y_t \,\overset{\mathrm{d}}{\approx}\, X_t, \qquad t = T, T-1,\cdots, 1.  
\end{equation}
\end{subequations}
Throughout the paper, we shall often employ the following notation to indicate the distribution of $Y_t$: 
\begin{equation}
	p_t \coloneqq \mathsf{law}\big( Y_t \big) .
	\label{eq:defn-pt}
\end{equation}

\subsection{Deterministic vs.~stochastic samplers: a continuous-time interpretation}
\label{sec:det-stochastic-samplers}

Evidently, the most crucial step of the diffusion model lies in effective design of the reverse process. 
Two mainstream approaches stand out: 
\begin{itemize}

	\item {\em Deterministic samplers.} Starting from $Y_T \sim \mathcal{N}(0,I_d)$, this approach selects a set of functions $\{\Phi_t(\cdot)\}_{1\leq t\leq T}$ and computes:
	\begin{equation}
		Y_{t-1}=\Phi_{t}\big(Y_{t}\big),\qquad t= T,\cdots,1.\label{eq:deterministic-sampler}
	\end{equation}
	Clearly, the sampling process is fully deterministic except for the initialization $Y_T$.  

	\item {\em Stochastic samplers.} 
		Initialized again at $Y_T \sim \mathcal{N}(0,I_d)$, this approach computes another collection of functions $\{\Psi_t(\cdot,\cdot)\}_{1\leq t\leq T}$ and performs the updates:
	\begin{equation}
		Y_{t-1}=\Psi_{t}\big(Y_{t},Z_{t}\big),\qquad t= T,\cdots,1,\label{eq:stochastic-sampler}
	\end{equation}
	where the $Z_t$'s are independent noise vectors obeying $Z_t \overset{\mathrm{i.i.d.}}{\sim} \mathcal{N}(0,I_d)$. 
	
\end{itemize}

In order to elucidate the feasibility of the above two approaches, 
we find it helpful to look at the continuum limit through the lens of SDEs and ODEs. It is worth emphasizing, however, that the development of our main theory does {\em not} rely on any knowledge of SDEs and ODEs.   
\begin{itemize}
	\item {\em The forward process.} 
		A continuous-time analog of the forward diffusion process can be modeled as 
		\begin{equation}
			\mathrm{d}X_{t}=f(X_{t},t)\mathrm{d}t+g(t)\mathrm{d}W_{t} \quad (0\leq t\leq T),
			\qquad X_{0}\sim\Pdata
			\label{eq:forward-SDE-general}
		\end{equation}	
		for some functions $f(\cdot,\cdot)$ and $g(\cdot)$ (denoting respectively the drift and diffusion coefficient), 
		where $W_t$ denotes a $d$-dimensional standard Brownian motion. 
		As a special example, the continuum limit of \eqref{eq:forward-process} takes the following form\footnote{To see its connection with \eqref{eq:forward-process}, it suffices to derive from \eqref{eq:forward-process} that 
$X_{t}-X_{t-\mathrm{d}t}=\sqrt{1-\beta_{t}}X_{t-\mathrm{d}t}-X_{t-\mathrm{d}t}+\sqrt{\beta_{t}}W_{t}\approx-\frac{1}{2}\beta_{t}X_{t-\mathrm{d}t}+\sqrt{\beta_{t}}W_{t}$.} \citep{song2020score}
		\begin{equation}
			\mathrm{d}X_{t}= - \frac{1}{2} \beta(t) X_t \mathrm{d}t+ \sqrt{\beta(t)}\,\mathrm{d}W_{t} \quad (0\leq t\leq T),
			\qquad X_{0}\sim\Pdata
			\label{eq:forward-SDE}
		\end{equation}	
		for some function $\beta(t)$.  
		As before, we denote by $q_t$ the distribution of $X_t$ in \eqref{eq:forward-SDE-general}.

	\item {\em The reverse process.} 
		As it turns out, the following two reverse processes are both capable of reconstructing the distribution of the forward process, 		
		motivating the design of two distinctive samplers. 
		Here and throughout, we use $\nabla \log q_{t}(X)$ to abbreviate $\nabla_X \log q_{t}(X)$ for notational simplicity. 
		\begin{itemize}

			\item One feasible approach is to resort to the so-called {\em probability flow ODE} \citep{song2020score}
				\begin{align}
					\mathrm{d}Y_{t}^{\mathsf{ode}} 
					& =\Big(-f\big(Y_{t}^{\mathsf{ode}},T-t\big)+\frac{1}{2}g(T-t)^{2}\nabla\log q_{T-t}\big(Y_{t}^{\mathsf{ode}}\big)\Big)
					\mathrm{d}t\quad(0\leq t\leq T),\qquad Y_{0}^{\mathsf{ode}}\sim q_{T},
					\label{eq:prob-flow-ODE}
				\end{align}
				which exhibits matching distributions as follows: 
				\begin{align*}
					Y_{T-t}^{\mathsf{ode}}  \, \overset{\mathrm{d}}{=} \, X_t, \qquad 0\leq t\leq T. 
				\end{align*}
				The deterministic nature of this approach often enables faster sampling. 
				It has been shown that this family of deterministic samplers is closely related to the DDIM sampler \citep{karraselucidating,song2020score}.

			\item In view of the classical results \citet{anderson1982reverse,haussmann1986time}, 
			      one can also construct a ``reverse-time'' SDE 
				\begin{align}
					\mathrm{d}Y_{t}^{\mathsf{sde}} 
					& =\Big(-f\big(Y_{t}^{\mathsf{sde}},T-t\big)+g(T-t)^{2}\nabla\log q_{T-t}\big(Y_{t}^{\mathsf{sde}}\big)\Big)\mathrm{d}t
					+g(T-t)\mathrm{d}Z_{t}^{\mathsf{sde}}\quad(0\leq t\leq T) 	
				\label{eq:reverse-SDE}
				\end{align}
				with $Y_{0}^{\mathsf{sde}}\sim q_{T}$ and $Z_t^{\mathsf{sde}}$ being a standard Brownian motion. 
			 	Strikingly, this process also satisfies
				\begin{align*}
					Y_{T-t}^{\mathsf{sde}}  \, \overset{\mathrm{d}}{=} \, X_t, \qquad 0\leq t\leq T. 
				\end{align*}
				The popular DDPM sampler \citep{ho2020denoising,nichol2021improved} falls under this category.

		\end{itemize}

\end{itemize}

\noindent 
Interestingly, in addition to the functions $f$ and $g$ that define the forward process, construction of both \eqref{eq:prob-flow-ODE} and \eqref{eq:reverse-SDE} relies only upon the knowledge of the gradient of the log density $\nabla\log q_{t}(\cdot)$ of the intermediate steps of the forward diffusion process --- often referred to as the (Stein) score function.  
Consequently, a key enabler of the above paradigms lies in reliable learning of the score function, 
and hence the name {\em score-based generative modeling}.

\section{Algorithms and main results}
\label{sec:main-results}

In this section, we analyze a couple of diffusion generative models, including both deterministic and stochastic samplers.  
While the proofs for our main theory are all postponed to the appendix, 
it is worth emphasizing upfront that our analysis framework directly tackles the discrete-time processes without resorting to any toolbox of SDEs and ODEs 
tailored to the continuous-time limits. 
This elementary approach might potentially be versatile for analyzing a broad class of variations of these samplers. 
For instance, prior ODE-based theory (e.g., \citet{chen2023probability,chen2023restoration}) encountered certain technical challenges when analyzing the deterministic sampler directly,  and our elementary approach is able to shed new light on the convergence of this important sampler.

\subsection{Assumptions and learning rates}

Before proceeding, we impose some assumptions on the score estimates and the target data distributions, and specify the hypter-parameters $\{\alpha_t\}$
 that shall be adopted throughout all cases.

\paragraph{Score estimates.}
Given that the score functions are an essential component in score-based generative modeling, 
we assume access to faithful estimates of the score functions $\nabla \log q_t(\cdot)$ across all intermediate steps $t$, 
thus disentangling the score learning phase and the data generation phase. 
Towards this end, let us first formally introduce the true score function as follows. 
\begin{definition}[Score function]
\label{defition:score}
The score function, denoted by $s_t^{\star}: \real^d\rightarrow \real^d$ ($1\leq t\leq T$), is defined as 
\begin{align}
\label{eqn:training-score}
	s_{t}^{\star}(X) \coloneqq 
	\nabla \log q_{t}(X) , 
	\qquad 1\leq t\leq T. 
\end{align}
\end{definition}
As has been pointed out by previous works concerning score matching (e.g., \citet{hyvarinen2005estimation,vincent2011connection,chen2022sampling}), 
the score function $s_{t}^{\star}$ admits an alternative form as follows (owing to properties of Gaussian distributions): 
\begin{equation}
s_{t}^{\star}\coloneqq\arg\min_{s:\real^{d}\rightarrow\real^{d}}\mathop{\mathbb{E}}_{W\sim\mathcal{N}(0,I_{d}),X_{0}\sim\Pdata}\Bigg[\bigg\| s\big(\sqrt{\overline{\alpha}_{t}}X_{0}+\sqrt{1-\overline{\alpha}_{t}}W\big)+\frac{1}{\sqrt{1-\overline{\alpha}_{t}}}W\bigg\|_{2}^{2}\Bigg],
	\label{eqn:training-score-equiv}
\end{equation}
which takes the form of the minimum mean square error estimator for $-\frac{1}{\sqrt{1-\overline{\alpha}_{t}}}W$ 
given $\sqrt{\overline{\alpha}_{t}}X_{0}+\sqrt{1-\overline{\alpha}_{t}}W$
and is often more amenable to training.

With Definition~\ref{defition:score} in place, 
we can readily introduce the following assumptions that capture the quality of the score estimate $\{s_t\}_{1\leq t\leq T}$ we have available. 
\begin{assumption}
\label{assumption:score-estimate}
Suppose that the score function estimate $\{s_t\}_{1\leq t\leq T}$ obeys
\begin{align}
\label{eqn:score-estimate}
	\frac{1}{T}\sum_{t = 1}^T \mathop{\mathbb{E}}_{X\sim q_t}\Big[ \big\| s_t(X) - s_t^{\star}(X) \big\|_2^2\Big] \le \varepsilon_{\mathsf{score}}^2.
\end{align}
\end{assumption}
\begin{assumption}
\label{assumption:score-estimate-Jacobi}
	For each $1\leq t\leq T$, assume that $s_t(\cdot)$ is continuously differentiable, and 
	denote by $J_{s_t^{\star}} = \frac{\partial s_t^{\star}}{\partial x}$ and $J_{s_t} = \frac{\partial s_t}{\partial x}$ 
	the Jacobian matrices of $s_t^{\star}(\cdot)$ and $s_t(\cdot)$, respectively. 
	Assume that the score function estimate $\{s_t\}_{1\leq t\leq T}$ obeys
\begin{align}
\label{eqn:score-estimate-Jacobi}
\qquad\qquad
	\frac{1}{T}\sum_{t = 1}^T \mathop{\mathbb{E}}_{X\sim q_t}\Big[\big\| J_{s_t}(X) - J_{s_t^{\star}} (X)  \big\|\Big] \le \varepsilon_{\mathsf{Jacobi}}.
\end{align}
\end{assumption}
In a nutshell, 
Assumption~\ref{assumption:score-estimate} reflects the $\ell_2$ score estimation error, 
whereas Assumption~\ref{assumption:score-estimate-Jacobi} is concerned with the estimation error in terms of the corresponding Jacobian matrix (so as to ensure certain continuity of the score estimator).  
Both assumptions consider the {\em average} estimation errors over all $T$ steps. 
As we shall see momentarily, our theory for the deterministic sampler relies on both Assumptions~\ref{assumption:score-estimate} and \ref{assumption:score-estimate-Jacobi}, 
while the theory for the stochastic sampler requires only Assumption~\ref{assumption:score-estimate}. 
We shall discuss in Section~\ref{sec:ODE-basic} the insufficiency of Assumption~\ref{assumption:score-estimate} alone for the deterministic sampler.

\paragraph{Target data distributions.}  
Our goal is to uncover the effectiveness of diffusion models in generating a broad family of data distributions. 
Throughout this paper, the only assumptions we need to impose on the target data distribution $\Pdata$ are the following: 
\begin{itemize}
	\item $X_0$ is an absolutely continuous random vector,
		and
\begin{equation}
	\mathbb{P}\big(\|X_0\|_2 \leq R = T^{c_R} \mid X_0\sim \Pdata \big)=1
	\label{eq:assumption-data-bounded}
\end{equation}
for some arbitrarily large constant $c_R>0$. 
\end{itemize}
This assumption allows the radius of the support of $\Pdata$ to be exceedingly large (given that the exponent $c_R$ can be arbitrarily large).


\paragraph{Learning rate schedule. }
Let us also take a moment to specify the learning rates to be used for our theory and analyses. 
For some large enough numerical constants $c_0,c_1 > 0$, we set
\begin{subequations}
\label{eqn:alpha-t}
\begin{align}
	\beta_{1} & =1-\alpha_{1} = \frac{1}{T^{c_0}};\\
\beta_{t} & =1-\alpha_{t} = \frac{c_{1}\log T}{T}\min\bigg\{\beta_{1}\Big(1+\frac{c_{1}\log T}{T}\Big)^{t},\,1\bigg\}.
\end{align}
\end{subequations}
\begin{remark}
	As will be seen in the analysis, 
	in general the discretization error depends crucially on the quantity $\frac{1-\alpha_t}{1-\overline{\alpha}_t}$, 
	whereas the initialization error relies on $\overline{\alpha}_1$ and $\overline{\alpha}_T$. 
	Our learning rates \eqref{eqn:alpha-t} are designed to make $\frac{1-\alpha_t}{1-\overline{\alpha}_t}$ as small as possible, while guaranteeing $\overline{\alpha}_1$ (resp.~$\overline{\alpha}_T$) is close to $1$ (resp.~$0$); 
	see the properties in \eqref{eqn:properties-alpha-proof}. 
	In addition, note that our theoretical framework can readily accommodate much broader choices of learning rates, 
	although the resultant convergence rates might vary. 

\end{remark}



%


\subsection{Deterministic samplers} 

We begin by analyzing a deterministic sampler: a discrete-time version of the probability flow ODE. 

\subsubsection{An ODE-based deterministic sampler} 
\label{sec:ODE-basic}
Armed with the score estimates $\{s_t\}_{1\leq t\leq T}$, 
a discrete-time version of the probability flow ODE approach (cf.~\eqref{eq:prob-flow-ODE}) adopts the following update rule: 
\begin{subequations}
\label{eqn:ode-sampling}
\begin{align}
\label{eqn:ode-sampling-Y}
	Y_T\sim \mathcal{N}(0,I_d),\qquad
	Y_{t-1} = \Phi_t\big(Y_{t}\big)
	\quad~ \text{for } t = T,\cdots, 1, 
\end{align}
where $\Phi_t(\cdot)$ is taken to be
\begin{align}
\label{eqn:phi-func}
	\Phi_t(x) \coloneqq 
	\frac{1}{\sqrt{\alpha_t}} \bigg( x + \frac{1-\alpha_{t}}{2}s_{t}(x) \bigg). 
\end{align}
\end{subequations}
This approach, based on the probability flow ODE \eqref{eq:prob-flow-ODE}, often achieves faster sampling compared to the stochastic counterpart \citep{song2020score}.  
Despite the empirical advances, however, 
the theoretical understanding of this type of deterministic samplers remained far from mature.


We first derive non-asymptotic convergence guarantees --- measured by the total variation distance between the forward and the reverse processes --- for the above deterministic sampler \eqref{eqn:ode-sampling}. 
The proof of this result is postponed to Section~\ref{sec:pf-theorem-ode}.


\begin{theos}
\label{thm:main-ODE}
Suppose that \eqref{eq:assumption-data-bounded} holds true. 
Assume that the score estimates $s_t(\cdot)$ $(1 \le t \le T)$ satisfy Assumptions~\ref{assumption:score-estimate} and \ref{assumption:score-estimate-Jacobi}.
Then the sampling process \eqref{eqn:ode-sampling} with the learning rate schedule \eqref{eqn:alpha-t} satisfies
\begin{align}
\mathsf{TV}\big(q_1, p_1\big) \leq C_1 \frac{d^2\log^4 T}{T} + C_1\frac{d^{6}\log^{6} T}{T^2} 
	+C_1\sqrt{d\log^{3}T}\varepsilon_{\score}+ C_1d(\log T)\varepsilon_{\Jacobi}
\label{eq:ratio-ODE}
\end{align}
for some universal constants $C_1>0$, where we recall that $p_1$ (resp.~$q_1$) represents the distribution of $Y_1$ (resp.~$X_1$). 
\end{theos}
\begin{remark}
Note that our theory is concerned with convergence to $q_1$ (the first step of the forward process). 
Given that $X_1 \sim q_1$ and $X_0 \sim q_0$ are exceedingly close due to the choice of $\alpha_1$, 
focusing on the convergence w.r.t.~$q_1$ instead of $q_0$ remains practically relevant.
\end{remark}
Let us remark on the main implications of Theorem~\ref{thm:main-ODE}, as well as several points worth discussing. 
Before proceeding, we shall note that our theory is concerned with convergence to $q_1$. 
Given that $X_1 \sim q_1$ and $X_0 \sim q_0$ are very close due to the choice of $\alpha_1$, 
focusing on the convergence w.r.t.~$q_1$ instead of $q_0$ remains practically relevant.

\paragraph{Iteration complexity.} Consider first the scenario that has access to perfect score estimates (i.e., $\varepsilon_{\score}=0$). 
		In order to achieve $\varepsilon$-accuracy (in the sense that $\mathsf{TV}(q_1, p_1)\leq \varepsilon$), 
the number of steps $T$ only needs to exceed\footnote{ 
As a technical note, the suboptimal $d$-dependency in our theory for the deterministic sampler comes mainly from Lemma~\ref{lem:main-ODE}; 
the main difficulty to improve Lemma~\ref{lem:main-ODE} lies in obtaining tighter control of some quantities regarding the conditional distribution of $x_0$ given $x_t$ (e.g., the Jacobian matrix of $s_t$). 
}
\begin{equation}
	\widetilde{O}\bigg( \frac{d^2}{\varepsilon} + \frac{d^3}{\sqrt{\varepsilon}}\bigg). 
	\label{eq:iteration-complexity-ODE}
\end{equation}

\paragraph{Stability.} Turning to the more general case with imperfect score estimates (i.e., $\varepsilon_{\score}>0$), 
		the deterministic sampler \eqref{eqn:ode-sampling} yields a distribution 
		whose distance to the target distribution (measured again by the TV distance) scales proportionally with $\varepsilon_{\score}$ and $\varepsilon_{\Jacobi}$. 
		It is noteworthy that in addition to the $\ell_2$ score estimation errors, 
		we are in need of an assumption on the stability of the associated Jacobian matrices, 
		which plays a pivotal in ensuring that the reverse-time deterministic process does not deviate considerably from the desired process. 
		
\paragraph{Insufficiency of the score estimation error assumption alone.}
The careful reader might wonder why we are in need of additional assumptions beyond the $\ell_2$ score error stated in 
Assumption~\ref{assumption:score-estimate}. 
To answer this question, we find it helpful to look at a simple example below. 
\begin{itemize}
	\item {\bf Example.}   
		Consider the case where $X_{0}\sim\mathcal{N}(0,1)$, and hence $X_{1}\sim\mathcal{N}(0,1)$. Suppose that the reverse process for time $t=2$ can lead to the desired distribution if exact score function is employed, namely, 
\[
Y_{1}^{\star}\coloneqq\frac{1}{\sqrt{\alpha_{2}}}\left(Y_{2}-\frac{1-\alpha_{2}}{2}s_{2}^{\star}(Y_{2})\right)\sim\mathcal{N}(0,1). 
\]
Now, suppose that the score estimate $s_{2}(\cdot)$ we have available obeys
\[
	s_{2}(y_{2})=s_{2}^{\star}(y_{2})+\frac{2\sqrt{\alpha_{2}}}{1-\alpha_{2}}\left\{ y_{1}^{\star}-L\left\lfloor \frac{y_{1}^{\star}}{L}\right\rfloor \right\} 
		\qquad \text{with } y_{1}^{\star}\coloneqq\frac{1}{\sqrt{\alpha_{2}}}\left(y_{2}-\frac{1-\alpha_{2}}{2}s_{2}^{\star}(y_{2})\right)
\]
for some $L>0$, where $\lfloor z \rfloor$ is the greatest  integer not exceeding $z$. It follows that
\[
	Y_{1}=Y_{1}^{\star}+\frac{1-\alpha_{2}}{2\sqrt{\alpha_{2}}}\big[ s_{2}^{\star}(Y_{2})-s_{2}(Y_{2})\big]  =L\left\lfloor \frac{Y_{1}^{\star}}{L}\right\rfloor .
\]
Clearly, the score estimation error $\mathbb{E}_{X_2\sim \mathcal{N}(0,1)}\big[|s_{2}(X_{2})-s_{2}^{\star}(X_{2})|^2\big]$
can be made arbitrarily small by taking $L$ to be sufficiently small. 
However, the discrete nature of $Y_{1}$ forces the TV distance to be
\[
	\mathsf{TV}(Y_{1},X_{1}) = 1. 
\]
%
\end{itemize}
The above example demonstrates that, for the deterministic sampler, the TV distance between $Y_1$ and $X_1$ might not improve as the score error decreases. 
As we shall see in Section~\ref{sec:DDPM-theory}, this is in stark contrast to the stochastic sampler. 
If we wish to eliminate the need of imposing Assumption~\ref{assumption:score-estimate-Jacobi}, 
one potential way is to resort to other metrics (e.g., the Wasserstein distance) instead of the TV distance between $Y_1$ and $X_1$.

\paragraph{Relaxing the boundedness assumption on $X_0$.}
As it turns out, the assumption~\eqref{eq:assumption-data-bounded} can also be relaxed. 
Supposing that $\mathbb{P}\big(\|X_0\|_2 \leq  B \mid X_0\sim \Pdata \big)=1$ for some quantity $B>0$ (which is allowed to grow faster than a polynomial in $T$), we can readily extend our analysis to obtain
{\small
\begin{align*}
\mathsf{TV}\big(q_1, p_1\big) \leq C_1 \frac{d^2\log^4 T\log^2 B}{T} + C_1\frac{d^{6}\log^{6} T\log^3 B}{T^2} 
	+C_1\sqrt{d\log^{3}T\log B}\varepsilon_{\score}+ C_1d(\log T)\varepsilon_{\Jacobi}.
\end{align*}
}
Importantly, the convergence rate depends only logarithmically in $B$.


\paragraph{Comparisons to previous works.}
Next, let us compare our results with past works. 
To the best of our knowledge, the only non-asymptotic analysis for the discretized probability flow ODE approach in
prior literature was derived by a very recent work \citet{chen2023restoration}, which established the first non-asymptotic convergence guarantees that exhibit polynomial dependency in both $d$ and $1/\varepsilon$ (see, e.g.,  \citet[Theorem~4.1]{chen2023restoration}). However, it fell short of providing concrete polynomial dependency in $d$ and $1/\varepsilon$,  suffered from exponential dependency in the Lipschitz constant of the score function, and relied on exact score estimates.  
In contrast, our result in Theorem~\ref{thm:main-ODE} uncovers a concrete $d^2/\varepsilon$ scaling (ignoring lower-order and logarithmic terms) without imposing any smoothness assumption on the target data distribution, 
and makes explicit the effect of $\ell_2$ score estimation errors, both of which were previously unavailable for such discrete-time deterministic samplers. 
Another recent work \citet{benton2023error} studied the convergence of the probability flow ODE approach  
without accounting for the discretization error; the result therein also exhibited exponential dependency on a certain Lipschitz constant w.r.t.~the forward flow and a regularity parameter (denoted by $\lambda$ therein, which might scale with the dimension $d$).  
Finally, while we were wrapping up the current paper, we became aware of the independent work \citet{chen2023probability} establishing improved polynomial dependency for two variants of the probability flow ODE.  By inserting an additional stochastic corrector step --- based on overdamped (resp.~underdamped) Langevin diffusion --- in each iteration of the probability flow ODE (so strictly speaking, these variations are no longer deterministic samplers), 
\citet{chen2023probability} showed that $\widetilde{O}(L^3d/\varepsilon^2)$ (resp.~$\widetilde{O}(L^2\sqrt{d}/\varepsilon)$) steps are sufficient, where $L$ denotes the Lipschitz constant of the score function. 
In comparison, our result demonstrates for the first time that the plain probability flow ODE already achieves the $1/\varepsilon$ scaling without requiring either a corrector step;  one limitation of our result, however, is the sub-optimal $d$-dependency compared to the variants studied in \citet{chen2023probability}. 


\subsubsection{An accelerated deterministic sampler}

Thus far, we have demonstrated that
the iteration complexity of the deterministic sampler \eqref{eqn:ode-sampling} 
is proportional to $1/\varepsilon$ (for small enough $\varepsilon$). 
A natural question is whether this convergence rate can be further improved.

As it turns out, if we have access to perfect estimates of two additional quantities in addition to exact score estimates, 
then a modified version of the sampler \eqref{eqn:ode-sampling} is able to achieve much improved convergence guarantees.  
These estimates are made precise in the following assumption. 
\begin{assumption}
\label{assumption:ODE-estimate}
Suppose that we have access to the estimates $w_{t} : \real^d \rightarrow \real^d$ ($1\leq t\leq T$) defined as follows: 
\begin{align}\label{eqn:training-ODE}
w_{t} &\coloneqq \arg\min_{w:\real^d\rightarrow\real^{d}}\mathop{\mathbb{E}}\Big[\Big\| \|W\|_{2}^{2}\Big(\frac{1}{\sqrt{1-\overline{\alpha}_{t}}}W + s_t^{\star}\big(X_t\big)\Big) + WW^{\top}s_t^{\star}\big(X_t\big) - w\big(X_t\big) \Big\|_{2}^{2}\Big],
\end{align}
	where  $X_t= \sqrt{\overline{\alpha}_{t}}X_{0}+\sqrt{1-\overline{\alpha}_{t}}W$.
Here, the expectation is with respect to $W \sim \mathcal{N}(0, I_d)$ and $X_0 \sim \Pdata$.
\end{assumption}

Armed with the score estimate in Assumption~\ref{assumption:score-estimate} and the additional estimates in Assumption~\ref{assumption:ODE-estimate}, 
we are ready to introduce an accelerated variant of \eqref{eqn:ode-sampling} as follows:
\begin{subequations}
\label{eqn:ode-sampling-R}
\begin{align}
\label{eqn:ode-sampling-R-Y}
	Y_T\sim \mathcal{N}(0,I_d),\qquad
	Y_{t-1} = \Phi_t\big(Y_{t}\big)
	\quad \text{for } t = T,\cdots, 1, 
\end{align}
where the mapping $\Phi_t(\cdot)$ is chosen to be
\begin{align}
\Phi_{t}(x)=\frac{1}{\sqrt{\alpha_{t}}}\bigg(x+\Big(\frac{1-\alpha_{t}}{2} + \frac{(1-\alpha_{t})^2}{8(1-\overline{\alpha}_{t})} - \frac{(1-\alpha_{t})^2}{8}\big\|s_{t}^{\star}(x)\big\|_2^2\Big)s_{t}^{\star}(x) + \frac{(1-\alpha_{t})^2}{8(1-\overline{\alpha}_{t})}w_{t}(x)\bigg).
\label{eqn:ode-sampling-R-Phi}
\end{align}
\end{subequations}
Notably, this new variant \eqref{eqn:ode-sampling-R} 
is closely related to the original sampler \eqref{eqn:ode-sampling}; 
in fact, they both move along the direction specified by the score estimate $s_t$, 
except that the accelerated variant includes a proper correction term chosen based on higher-order expansion.

Encouragingly, our non-asymptotic analysis framework can be extended to derive enhanced convergence guarantees for the sampler \eqref{eqn:ode-sampling-R}, 
assuming access to exact score functions.  
The proof of our convergence result below is postponed to Section~\ref{sec:pf-theorem-ode-fast}.
\begin{theos}
\label{thm:main-ODE-fast}
	Suppose that \eqref{eq:assumption-data-bounded} holds true and that the score estimates are perfect (i.e., $s_t=s^{\star}_t$).  
Equipped with the estimates in Assumptions~\ref{assumption:ODE-estimate} and the learning rate schedule \eqref{eqn:alpha-t}, 
the sampling process \eqref{eqn:ode-sampling-R} obeys
\begin{align} \label{eq:ratio-ODE-R}
\mathsf{TV}\big(q_1, p_1\big) \le C_1\frac{d^{6}\log^{6} T}{T^2}
\end{align}
for some universal constants $C_1>0$, where  $p_1$ (resp.~$q_1$) is the distribution of $Y_1$ (resp.~$X_1$).
\end{theos}
Theorem~\ref{thm:main-ODE-fast} reveals that: in order to achieve $\mathsf{TV}(q_1, p_1)\leq \varepsilon$, 
the accelerated deterministic sampler \eqref{eqn:ode-sampling-R} only requires the number of steps $T$ to be on the order of
\begin{equation}
	\widetilde{O}\bigg( \frac{d^3}{\sqrt{\varepsilon}}\bigg),  
	\label{eq:iteration-complexity-ODE-fast}
\end{equation}
thus improving the dependency on $\varepsilon$ from $\widetilde{O}(1/\varepsilon)$ (cf.~\eqref{eq:iteration-complexity-ODE}) to $\widetilde{O}(1/\sqrt{\varepsilon})$ 
for small enough $\varepsilon$. 
Consequently, the improved convergence result underscores the crucial role of bias correction when selecting the search direction. 


\subsection{Stochastic samplers}

\subsubsection{A DDPM-type stochastic sampler} 
\label{sec:DDPM-theory}
 
Armed with the score estimates $\{s_t\}$, 
we can readily introduce the following stochastic sampler that operates in discrete time, motivated by the reverse-time SDE \eqref{eq:reverse-SDE}:
\begin{subequations}
\label{eqn:sde-sampling}
\begin{align} 
	Y_T \sim \mathcal{N}(0, I_d),
	\qquad Y_{t-1} &= \Psi_t(Y_t,Z_t) \quad~ \text{for } t= T,\cdots,1
\end{align}
where $Z_t \overset{\mathrm{i.i.d.}}{\sim} \mathcal{N}(0,I_d)$, and
\begin{align}
	\Psi_t(y,z) = \frac{1}{\sqrt{\alpha_t}}\Big(y + (1-\alpha_t)s_{t}(y)\Big) + \sigma_t z
	\qquad 
	\text{with }\sigma_t^{2} = \frac{1}{\alpha_t} - 1.
\end{align} 
\end{subequations}
The key difference between this sampler and the deterministic sampler \eqref{eqn:ode-sampling} is that: (i) there exists an additional pre-factor of $1/2$ on $s_t$ in the deterministic sampler; and (ii) the stochastic sampler injects additional noise $Z_t$ in each step.

In contrast to deterministic samplers, 
the stochastic samplers have received more theoretical attention, 
with the state-of-the-art results established by \citet{chen2022sampling,chen2022improved} as well as a very recent paper \citet{benton2023linear}. 
The elementary approach developed in the current paper is also applicable towards understanding this type of samplers, 
leading to the following non-asymptotic theory.

%
\begin{theos}
\label{thm:main-SDE}
Suppose that \eqref{eq:assumption-data-bounded} holds true. 
Equipped with the estimates in Assumption~\ref{assumption:score-estimate} and the learning rate schedule \eqref{eqn:alpha-t}, 
the stochastic sampler \eqref{eqn:sde-sampling} achieves
\begin{align} \label{eq:ratio-SDE}
\mathsf{TV}\big(q_1, p_1\big) \le \sqrt{\frac{1}{2}\mathsf{KL}\big(q_1 \parallel p_1)} \le C_1\frac{d^2\log^3 T}{\sqrt{T}} + C_1\sqrt{d}\varepsilon_{\score}\log^2 T
\end{align}
for some universal constants $C_1>0$, provided that $T\geq C_2 d^4 \log^6 T$ for some large enough constant $C_2>0$.
\end{theos}
Theorem~\ref{thm:main-SDE} establishes non-asymptotic convergence guarantees for the stochastic sampler \eqref{eqn:sde-sampling}. 
As asserted by the theorem, 
if we have access to perfect score estimates, 
then the number of steps needed to attain $\varepsilon$-accuracy (measured by the TV distance between $p_1$ and $q_1$) 
is proportional to $1/\varepsilon^2$, matching the state-of-the-art $\varepsilon$-dependency derived in \citet{chen2022improved}, albeit exhibiting a worse dimensional dependency.  
In addition, in the presence of score estimation error, 
the sampler achieves a TV distance proportional to $\varepsilon_{\score}$, again consistent with prior results.  
Our analysis follows a completely different path compared with the SDE-based approach in \citet{chen2022improved}, 
thus offering complementary interpretations for this important sampler. 
In order to further illustrate the versatility of our analysis approach, 
we shall demonstrate how it can be applied to study an accelerated version in the next subsection.

\subsubsection{An accelerated stochastic sampler}

In this subsection, we come up with a potential strategy to speed up the stochastic sampler \eqref{eqn:sde-sampling}, 
assuming access to reliable estimates of additional objects as described below. 
\begin{assumption}
\label{assumption:variance-estimate}
Suppose that we have access to the estimates $v_t: \real^d\times \real^d \rightarrow \real^d$ ($1\leq t\leq T$) as follows: 
\begin{align}
\label{eqn:training-variance}
	v_{t}\coloneqq
	\arg\min_{v:\real^d\times \real^d\rightarrow\real^{d}}\mathop{\mathbb{E}}\Big[\Big\| WW^{\top}Z - v\big(\sqrt{\overline{\alpha}_t} X + \sqrt{1-\overline{\alpha}_t}W, Z\big) \Big\|_{2}^{2}\Big],
	\qquad 1\leq t\leq T,
\end{align}
where $X,W,Z$ are independently generated obeying $X \sim \Pdata$, $W \sim \mathcal{N}(0, I_d)$, and $Z \sim \mathcal{N}(0, I_d)$. 
\end{assumption}

With perfect score estimates as well as the additional estimates in Assumption~\ref{assumption:score-estimate} and Assumption~\ref{assumption:variance-estimate} in place,  
we are positioned to introduce the proposed accelerated sampler as follows: 
\begin{subequations}
\label{eqn:sde-sampling-R}
\begin{align} \label{eqn:sde-sampling-R-Y}
	Y_T \sim \mathcal{N}(0, I_d), \qquad Y_{t-1} &= \Psi_t(Y_t,Z_t)\quad~ \text{for }t = T,\cdots, 1,
\end{align}
where we choose the mapping $\Psi_t(\cdot,\cdot)$ as follows
\begin{align}
\Psi_{t}(y,z) 
&=\frac{1}{\sqrt{\alpha_{t}}}\Big(y + (1-\alpha_t)s_{t}(y)\Big) 
+ \sigma_{t}\left\{ z-\frac{1-\alpha_{t}}{2(1-\overline{\alpha}_{t})}\left[z+(1-\overline{\alpha}_{t})s_{t}(y)s_{t}(y)^{\top}z-v_{t}(y, z)\right]\right\} 
	\label{eqn:sde-sampling-R-Psi}
\end{align}
with 
\begin{align}
\sigma_t^{2} = \frac{1}{\alpha_t} - 1.
\end{align} 
\end{subequations}
Clearly, the modified update mapping \eqref{eqn:sde-sampling-R-Psi} 
is still mainly a linear combination of the score estimate $s_{t-1}$ and the additive noise $Z_t$, 
except that a correction term $v_t$ (learned by solving \eqref{eqn:training-variance}) needs to be included for acceleration purposes.

We now apply our analysis strategy to establish performance guarantees for the above stochastic sampler. 
\begin{theos}
\label{thm:main-SDE-R}
	Suppose that \eqref{eq:assumption-data-bounded} holds true and that the score estimates are exact (i.e., $s_t=s^{\star}_t$).  
Equipped with the estimates in Assumption~\ref{assumption:score-estimate},~\ref{assumption:variance-estimate} and the learning rate schedule \eqref{eqn:alpha-t}, 
the sampling process \eqref{eqn:sde-sampling-R} satisfies
\begin{align} \label{eq:ratio-SDE-R}
	\mathsf{TV}\big(q_1, p_1\big) \le \sqrt{\frac{1}{2}\mathsf{KL}\big(q_1 \parallel p_1)} \leq C_1\frac{d^3\log^{4.5} T}{T}
\end{align}
	for some universal constants $C_1>0$, provided that $T\geq C_2 d^3 \log^{4.5} T$ for some large enough constant $C_2>0$.
\end{theos}
\noindent The proof of this result is provided in Section~\ref{sec:pf-thm-sde-r}.
In comparison to the stochastic sampler \eqref{eqn:sde-sampling}, Theorem~\ref{thm:main-SDE-R} asserts that the iteration complexity of the sampler \eqref{eqn:sde-sampling-R} is at most
\begin{equation}
	\widetilde{O}\bigg( \frac{d^3}{\varepsilon}\bigg),  
	\label{eq:iteration-complexity-SDE-fast}
\end{equation}
thus significantly reducing the scaling $\widetilde{O}(1/\varepsilon^2)$ for the original sampler \eqref{eqn:sde-sampling} to $\widetilde{O}(1/\varepsilon)$ regarding the $\varepsilon$-dependency.  All in all, our theory reveals that having information about a small number of additional objects might substantially speed up the data generation process.

\section{Other related works}
\label{sec:related-works}

\medskip

\paragraph{Convergence theory for diffusion models.}
Early theoretical efforts in understanding the convergence of score-based stochastic samplers suffered from being either not quantitative \citep{de2021diffusion,liu2022let,pidstrigach2022score}, or the curse of dimensionality (e.g., exponential dependencies in the convergence guarantees) \citep{block2020generative,de2022convergence}.
The recent work \citet{lee2022convergence} provided the first polynomial convergence guarantee in the presence of $L_2$-accurate score estimates, for any smooth distribution satisfying the log-Sobelev inequality.
\citet{chen2022sampling,lee2023convergence,chen2022improved} subsequently lifted such a stringent data distribution assumption. 
More concretely,  \citet{chen2022sampling} accommodated a broad family of data distributions under the premise that the score functions over the entire trajectory of the forward process are Lipschitz; \citet{lee2023convergence} only required certain smoothness assumptions but came with worse dependence on the problem parameters; and more recent results in \citet{chen2022improved} applied to literally any data distribution with bounded second-order moment. In addition, \citet{wibisono2022convergence} also established a convergence theory for score-based generative models, assuming that the error of the score estimator has a bounded moment generating function and that the data distribution satisfies the log-Sobelev inequality.
Turning attention to samplers based on the probability flow ODE, \citet{chen2023restoration} derived the first non-asymptotic bounds for this type of samplers.  
Improved convergence guarantees have recently been provided by a concurrent work \citet{chen2023probability}, with the assistance of additional corrector steps inerspersed in each iteration of the probability flow ODE. 
It is worth noting that the corrector steps proposed therein are based on Langevin-type diffusion and inject additive noise, and hence the resulting sampling processes are not deterministic. 
Additionally,  theoretical justifications for DDPM in the context of image in-painting have been developed by \citet{rout2023theoretical}.
Moreover, convergence results based on the Wasserstein distance have recently  been derived as well (e.g., \citet{tang2023diffusion,benton2023error}), although these results typically exhibit exponential dependency on the Lipschitz constants of the score functions. 
Theoretical guarantees have also recently been extended to accommodate other popular methods like consistency models \citep{song2023consistency,li2024towards} and diffusion guidance \citep{ho2022classifier,wu2024theoretical}.




\paragraph{Score matching.} \citet{hyvarinen2005estimation} showed that the score function can be estimated via integration by parts, 
a result that was further extended in \citet{hyvarinen2007some}. \citet{song2020sliced} proposed sliced score matching to tame the computational complexity in high dimension. The consistency of the score matching estimator was studied in \citet{hyvarinen2005estimation}, with asymptotic normality established in \citet{forbes2015linear}. Optimizing the score matching loss has been shown to be intimately connected to minimizing upper bounds on the Kullback-Leibler divergence \citep{song2021maximum} and Wasserstein distance \citep{kwon2022score} between the generated distribution and the target data distribution. Furthermore, the recent work \citet{koehler2022statistical} studied the statistical efficiency of score matching by connecting it with the isoperimetric properties of the target data distribution.



\paragraph{Other theory for diffusion models.} \citet{oko2023diffusion} studied the approximation and generalization capabilities of diffusion modeling for distribution estimation. 
Assuming that the data are supported on a low-dimensional linear subspace, 
 \citet{chen2023score} developed a sample complexity bound for diffusion models. Moreover, \citet{ghimire2023geometry} adopted a geometric perspective and showed that the forward and backward processes of diffusion models are essentially Wasserstein gradient flows operating in the space of probability measures.
Recently, the idea of stochastic localization, which is closely related to diffusion models, is adopted to sample from posterior distributions \citep{montanari2023posterior,el2022sampling}, which has been implemented using the approximate message passing algorithm (\cite{donoho2009message,li2022non}).

\section{Analysis}
\label{sec:analysis}

In this section, we describe our non-asymptotic proof strategies for two simpler samplers (i.e., \eqref{eqn:ode-sampling} and \eqref{eqn:sde-sampling}). 
The analyses for the two accelerated variants follow similar arguments as their non-accelerated counterparts, and are hence postponed to the appendices.

\subsection{Preliminary facts}
\label{sec:preliminary-facts}

Before proceeding, 
we gather a couple of facts that will be useful for  the proof, with most proofs postponed to Appendix~\ref{sec:proof-preliminary}.

\paragraph{Properties related to the score function.}
First of all, in view of the alternative expression \eqref{eqn:training-score-equiv} for the score function and the property of the minimum mean square error (MMSE) estimator (e.g., \citet[Section~3.3.1]{hajek2015random}), we know that the true score function $s_t^{\star}$ is given by the conditional expectation
\begin{align}
s_{t}^{\star}(x) & =\mathbb{E}\left[-\frac{1}{\sqrt{1-\overline{\alpha_{t}}}}W\,\bigg|\,\sqrt{\overline{\alpha_{t}}}X_{0}+\sqrt{1-\overline{\alpha}_{t}}W=x\right]=\frac{1}{1-\overline{\alpha}_{t}}\mathbb{E}\left[\sqrt{\overline{\alpha_{t}}}X_{0}-x\,\big|\,\sqrt{\overline{\alpha_{t}}}X_{0}+\sqrt{1-\overline{\alpha}_{t}}W=x\right] \notag\\
	& = - \frac{1}{1-\overline{\alpha}_{t}} \underset{\eqqcolon \, g_t(x) }{\underbrace{ {\displaystyle \int}_{x_{0}}\big(x-\sqrt{\overline{\alpha_{t}}}x_{0}\big)p_{X_{0}\mid X_{t}}(x_{0}\,|\, x)\mathrm{d}x_{0} }}.
	\label{eq:st-MMSE-expression}
\end{align}
%
%
%
Let us also introduce the Jacobian matrix associated with $g_t(\cdot)$ as follows:
\begin{equation}
J_{t}(x) \coloneqq 
	\frac{\partial g_t(x)}{\partial x}, \label{eq:Jacobian-Thm4}
\end{equation}
which can be equivalently rewritten as
%
\begin{align}
J_{t}(x) & =I_{d}+\frac{1}{1-\overline{\alpha}_{t}}\bigg\{\mathbb{E}\big[X_{t}-\sqrt{\overline{\alpha}_{t}}X_{0}\mid X_{t}=x\big]\Big(\mathbb{E}\big[X_{t}-\sqrt{\overline{\alpha}_{t}}X_{0}\mid X_{t}=x\big]\Big)^{\top}\notag\\
 & \qquad\qquad\qquad-\mathbb{E}\Big[\big(X_{t}-\sqrt{\overline{\alpha}_{t}}X_{0}\big)\big(X_{t}-\sqrt{\overline{\alpha}_{t}}X_{0}\big)^{\top}\mid X_{t}=x\Big]\bigg\}.
	\label{eq:Jt-x-expression-ij-23}
\end{align}
%

\paragraph{Properties about the learning rates.}
Next, we isolate a few useful properties about the learning rates as specified by $\{\alpha_t\}$ in \eqref{eqn:alpha-t}:  
\begin{subequations}
\label{eqn:properties-alpha-proof}
\begin{align}
	\alpha_t &\geq1-\frac{c_{1}\log T}{T}  \ge \frac{1}{2},\qquad\qquad\quad~~ 1\leq t\leq T \label{eqn:properties-alpha-proof-00}\\
	\frac{1}{2}\frac{1-\alpha_{t}}{1-\overline{\alpha}_{t}} \leq \frac{1}{2}\frac{1-\alpha_{t}}{\alpha_t-\overline{\alpha}_{t}}
	&\leq \frac{1-\alpha_{t}}{1-\overline{\alpha}_{t-1}}  \le \frac{4c_1\log T}{T},\qquad\quad~~ 2\leq t\leq T  \label{eqn:properties-alpha-proof-1}\\
	1&\leq\frac{1-\overline{\alpha}_{t}}{1-\overline{\alpha}_{t-1}} \leq1+\frac{4c_{1}\log T}{T} ,\qquad2\leq t\leq T  \label{eqn:properties-alpha-proof-3} \\
	\overline{\alpha}_{T} & \le \frac{1}{T^{c_2}}, \label{eqn:properties-alpha-proof-alphaT} 
\end{align}
provided that $T$ is large enough. 
Here, $c_1$ is defined in \eqref{eqn:alpha-t}, and $c_2\geq 1000$ is some large numerical constant. 
In addition, if $\frac{d(1-\alpha_{t})}{\alpha_{t}-\overline{\alpha}_{t}}\lesssim 1$, then one has
\begin{align}
\Big(\frac{1-\overline{\alpha}_{t}}{\alpha_{t}-\overline{\alpha}_{t}}\Big)^{d/2} 
	& =1+\frac{d(1-\alpha_{t})}{2(\alpha_{t}-\overline{\alpha}_{t})}+\frac{d(d-2)(1-\alpha_{t})^{2}}{8(\alpha_{t}-\overline{\alpha}_{t})^{2}}+O\bigg(d^{3}\Big(\frac{1-\alpha_{t}}{\alpha_{t}-\overline{\alpha}_{t}}\Big)^{3}\bigg).
	\label{eq:expansion-ratio-1-alpha}	
\end{align}
\end{subequations}
The proof of these properties is postponed to Appendix~\ref{sec:proof-properties-alpha}.

\paragraph{Properties of the forward process.}
Additionally,  recall that the forward process satisfies $X_t \overset{\mathrm{d}}{=} \sqrt{\overline{\alpha}_t} X_0 + \sqrt{1-\overline{\alpha}_t} W$ with $W\sim \mathcal{N}(0,I_d)$.  
We have the following tail bound concerning the random vector $X_0$ conditional on $X_t$, 
whose proof can be found in Appendix~\ref{sec:proof-lem:x0}.
\begin{lems} \label{lem:x0}
Suppose that there exists some numerical constant $c_R>0$ obeying 
\begin{equation}
	\mathbb{P}\big( \|X_0 \|_2 \leq R\big) = 1
	\qquad \text{and} \qquad 
	R = T^{c_R}.  
	\label{eq:cR-defn-lem}
\end{equation}
Consider any $y \in \real$, and let
\begin{align}
\label{eqn:choice-y-prelim}
	\theta(y) \coloneqq \max\bigg\{ \frac{-\log {p_{X_t}(y)}}{d\log T} , c_6 \bigg\}
\end{align}
for some large enough constant $c_6\geq 2c_R+c_0$. Then for any quantity $c_5 \ge 2$, 
conditioned on $X_t=y$ one has
\begin{align}
	\big\|\sqrt{\overline{\alpha}_{t}}X_0 - y \big\|_2 \leq 5c_5\sqrt{\theta(y) d(1-\overline{\alpha}_{t})\log T} 
	\label{eq:P-xt-X0-124}
\end{align} 
with probability at least $1 - \exp\big(-c_5^2\theta(y) d\log T \big)$.
In addition, it holds that
\begin{subequations}
\begin{align}
	\mathbb{E}\left[\big\| \sqrt{\overline{\alpha}_{t}}X_{0} - y \big\|_{2}\,\big|\,X_{t}=y\right] &\leq 12\sqrt{\theta(y) d(1-\overline{\alpha}_{t})\log T},\label{eq:E-xt-X0} \\
	\mathbb{E}\left[\big\| \sqrt{\overline{\alpha}_{t}}X_{0} - y \big\|^2_{2}\,\big|\,X_{t}=y\right] &\leq 120\theta(y) d(1-\overline{\alpha}_{t})\log T,\label{eq:E2-xt-X0} \\
	\mathbb{E}\left[\big\| \sqrt{\overline{\alpha}_{t}}X_{0} - y \big\|^3_{2}\,\big|\,X_{t}=y\right] &\leq 1040\big(\theta(y) d(1-\overline{\alpha}_{t})\log T\big)^{3/2},\label{eq:E3-xt-X0}\\
	\mathbb{E}\left[\big\| \sqrt{\overline{\alpha}_{t}}X_{0} - y \big\|^4_{2}\,\big|\,X_{t}=y\right] &\leq 10080\big(\theta(y) d(1-\overline{\alpha}_{t})\log T\big)^{2}.\label{eq:E4-xt-X0}
\end{align}
\end{subequations}
\end{lems}
\noindent
In order to interpret Lemma~\ref{lem:x0}, let us look at the case with $\theta(y)=c_6$, 
corresponding to the scenario where $p_{X_t}(y)\geq \exp(-c_6d\log T)$ (so that $p_{X_t}(y)$ is not exceedingly small). 
In this case, Lemma~\ref{lem:x0} implies that conditional on $X_t=y$ taking on a ``typical'' value,  
the vector $\sqrt{\overline{\alpha}_{t}}X_{0} - X_t  = \sqrt{1-\overline{\alpha}_t} \,\overline{W}_t$ (see \eqref{eqn:Xt-X0}) might still follow a sub-Gaussian tail, 
whose expected norm remains on the same order of that of an unconditional Gaussian vector $\mathcal{N}(0, (1-\overline{\alpha}_t)I_d)$.

The next lemma singles out another useful fact that controls the tail of $p_{X_t}$ of the forward process; the proof is postponed to Appendix~\ref{sec:proof-lem:river}. 
\begin{lems}
	\label{lem:river}
	Consider any two points $x_t, x_{t-1}\in \mathbb{R}^d$ obeying
\begin{align}
	-\log p_{X_t}(x_t) \leq \frac{1}{2} c_6 d\log T, 
	\quad\text{and}\quad \bigg\|x_{t-1} -   \frac{x_t}{ \sqrt{\alpha_t}} \bigg\|_2 \leq c_3 \sqrt{d(1 - \alpha_t)\log T}
	\label{eq:assumption-lem:river}
\end{align}
for some large constants $c_6, c_3>0$. 
If we define $x_t(\gamma) \coloneqq \gamma x_{t-1} + (1-\gamma) x_t / \sqrt{\alpha_t}$ for any $\gamma \in [0,1]$, then
\begin{align}
\label{eqn:river}
	-\log p_{X_{t-1}}\big(x_{t}(\gamma)\big) & \leq  c_6 d\log T, \qquad \forall \gamma \in [0,1].
\end{align}
\end{lems}
In other words, if $x_t$ falls within a typical range of $X_t$ and if the point $x_{t-1}$ is not too far away from $x_t/\sqrt{\alpha_t}$, 
then $x_{t-1}$ is also a typical value of the previous point $X_{t-1}$. 
As an immediate consequence, 
combining Lemma~\ref{lem:river} with Lemma~\ref{lem:x0} reveals that: 
if the assumption \eqref{eq:assumption-lem:river} holds, then conditional on $X_{t-1} = x_t(\gamma)$ for any $\gamma \in [0,1]$, one has
\begin{subequations}
\label{eqn:condi-g}
\begin{align}
	\big\| \sqrt{\overline{\alpha}_{t-1}}X_0 - x_t(\gamma) \big\|_2 &\leq 5c_5 \sqrt{c_6d(1-\overline{\alpha}_{t-1})\log T} \\
	\bigg\| \sqrt{\overline{\alpha}_{t-1}}X_0 - \frac{x_t}{ \sqrt{\alpha_t}} \bigg\|_2 
	&\leq \big\| \sqrt{\overline{\alpha}_{t-1}}X_0 - x_t(\gamma) \big\|_2 + \bigg\| x_t(\gamma) - \frac{x_t}{ \sqrt{\alpha_t}} \bigg\|_2 
	\leq (5c_5 + c_3) \sqrt{c_6d(1-\overline{\alpha}_{t-1})\log T}
\end{align}
\end{subequations}
with probability exceeding $1 - \exp\big(-c_5^2c_6 d\log T \big)$, 
where the last inequality invokes the property \eqref{eqn:properties-alpha-proof-1}.

\paragraph{Distance between $p_T$ and $q_T$.} 
We now record a simple result that demonstrates the proximity between $p_T$ and $q_T$, whose proof is provided in Appendix~\ref{sec:proof-lem-KL-T}. 
\begin{lems}
	\label{lem:KL-T}
	For any large enough $T$, it holds that
	\begin{align}
		\big( \mathsf{TV}(p_{X_{T}}\parallel p_{Y_{T}}) \big)^2 \leq \frac{1}{2}\mathsf{KL}(p_{X_{T}}\parallel p_{Y_{T}}) \lesssim \frac{1}{T^{200}}. 
	\end{align}
\end{lems}

\paragraph{Additional notation about score errors.}  
For any vector $x\in \mathbb{R}^d$ and any $1 < t\leq T$, let us define
%
\begin{align}
	\varepsilon_{\score, t}(x) \coloneqq \big\|s_t(x) - s_t^{\star}(x) \big\|_2
\qquad\text{and}\qquad
	\varepsilon_{\Jacobi, t}(x) \coloneqq \big\| J_{s_t}(x)  -  J_{s_t^{\star}} (x) \big\|,
	\label{eq:pointwise-epsilon-score-J}
\end{align} 
with $J_{s_t}$ and $J_{s_t^{\star}}$ the Jacobian matrices of $s_t(\cdot)$ and $s_t^{\star}(\cdot)$, respectively.  
Under Assumption~\ref{assumption:score-estimate}, 
we have
\begin{subequations}
	\label{eq:score-assumptions-equiv}
\begin{align}
\frac{1}{T}\sum_{t=1}^{T}\mathbb{E}_{X\sim q_{t}}\big[\varepsilon_{\score,t}(X)\big] & \leq\bigg(\frac{1}{T}\sum_{t=1}^{T}\mathbb{E}_{X\sim q_{t}}\left[\varepsilon_{\score,t}(X)^{2}\right]\bigg)^{1/2}\leq\varepsilon_{\score}.
\end{align}
Also, Assumption~\ref{assumption:score-estimate-Jacobi} says that
\begin{align}
	\frac{1}{T}\sum_{t=1}^{T}\mathbb{E}_{X\sim q_{t}}\big[\varepsilon_{\Jacobi,t}(X)\big] & \leq\varepsilon_{\Jacobi}.
\end{align}
\end{subequations}

\subsection{Analysis for the sampler based on probability flow ODE (Theorem~\ref{thm:main-ODE})}
\label{sec:pf-theorem-ode}



We now present the proof for our main result (i.e., Theorem~\ref{thm:main-ODE}) tailored to the deterministic sampler \eqref{eqn:ode-sampling} based on the probability flow ODE. 
Given that the total variation distance is always bounded above by 1, it suffices to assume 
\begin{subequations}
	\label{eq:assumption-T-score-Jacob}
\begin{align}
	T &\geq C_1 d^2 \log^4 T + \sqrt{C_1} d^3 \log^3 T \\
	\varepsilon_{\mathsf{score}} & \leq\frac{1}{C_{1}\sqrt{d}\log^{2}T} \label{eq:assumption-T-score-Jacob-score}\\
	\varepsilon_{\mathsf{Jacobi}} & \leq\frac{1}{C_{1}d\log^{2}T} \label{eq:assumption-T-score-Jacob-Jacobi}
\end{align}
\end{subequations}
throughout the proof; otherwise the claimed result \eqref{eq:ratio-ODE} becomes trivial. 

\paragraph{Preparation.}
Before proceeding, we find it convenient to introduce a function 
\begin{subequations}
	\label{defn:phit-x}
\begin{align}
	\phi_t^{\star}(x) &= x + \frac{1-\alpha_{t}}{2}s_t^{\star}(x) 
	= x - \frac{1-\alpha_{t}}{2(1-\overline{\alpha}_{t})} \displaystyle \int_{x_{0}}\big(x-\sqrt{\overline{\alpha_{t}}}x_{0}\big)p_{X_{0}\mid X_{t}}(x_{0}\,|\, x)\mathrm{d}x_{0}, \\
	\phi_t(x) &= x + \frac{1-\alpha_{t}}{2}s_t(x),
\end{align}
\end{subequations}
where the first line follows from \eqref{eq:st-MMSE-expression}. 
The update rule \eqref{eqn:ode-sampling} can then be expressed as follows: 
\begin{equation}
	Y_{t-1} = \Phi_t(Y_t) = \frac{1}{\sqrt{\alpha_t}} \phi_t(Y_t). 
	\label{eq:Yt-phi-ODE}
\end{equation}
%
%

Moreover, for any point $y_T\in \mathbb{R}^d$ (resp.~$y_T'\in \mathbb{R}^d$), let us define the corresponding deterministic sequence
		\begin{equation}
			y_{t-1} = \frac{1}{\sqrt{\alpha_t}} \phi_t(y_t),
			\qquad 
			y_{t-1}' = \frac{1}{\sqrt{\alpha_t}} \phi_t(y_t'),
			\qquad t=T, T-1,\cdots 
			\label{eq:defn-yt-sequence-proof}
		\end{equation}
In other words, $\{y_{T-1},\ldots,y_1\}$ (resp.~$\{y_{T-1}',\ldots,y_1'\}$) is the (reverse-time) sequence generated by the probability flow ODE (cf.~\eqref{eq:Yt-phi-ODE}) when initialized to $Y_T=y_T$ (resp.~$Y_T=y_T'$). 
We also define the following quantities for any point $y_T\in \mathcal{R}^d$ and its associated sequence   $\{y_{T-1},\ldots,y_1\}$: 
\begin{subequations}
	\label{eq:defn-xik-Stk-proof}
\begin{align}
	\xi_t(y_T) &\coloneqq \frac{\log T}{T}\big(d\varepsilon_{\Jacobi, t}(y_t) + \sqrt{d\log T}\varepsilon_{\score, t}(y_t)\big); \\ 
	S_{t}(y_T) &\coloneqq \sum_{1 < k \le t} \xi_k(y_k), \quad \text{for }t\geq 2,
	\qquad \text{ and } \qquad
	S_{1}(y_T) = 0. 
\end{align}
\end{subequations}
In words, for any given starting point $y_T$, 
$\xi_t(y_t)$ captures the (properly weighted) score error incurred in the $t$-th iteration, 
whereas $S_{t}(y_T)$ quantifies the aggregate weighted score error up to the $t$-th iteration.

With the above notation in place, we can readily proceed to our proof, which consists of several steps.

\paragraph{Step 1: bounding the density ratios of interest.} 
To begin with, we note that for any vectors $y_{t-1}$ and $y_t$, elementary properties about transformation of probability distributions give
\begin{align}
\frac{p_{Y_{t-1}}(y_{t-1})}{p_{X_{t-1}}(y_{t-1})} & =\frac{p_{\sqrt{\alpha_{t}}Y_{t-1}}(\sqrt{\alpha_{t}}y_{t-1})}{p_{\sqrt{\alpha_{t}}X_{t-1}}(\sqrt{\alpha_{t}}y_{t-1})}\notag\\
 & =\frac{p_{\sqrt{\alpha_{t}}Y_{t-1}}(\sqrt{\alpha_{t}}y_{t-1})}{p_{Y_{t}}(y_{t})}\cdot\bigg(\frac{p_{\sqrt{\alpha_{t}}X_{t-1}}(\sqrt{\alpha_{t}}y_{t-1})}{p_{X_{t}}(y_{t})}\bigg)^{-1}\cdot\frac{p_{Y_{t}}(y_{t})}{p_{X_{t}}(y_{t})},\label{eq:recursion}
\end{align}
thus converting the density ratio of interest into the product of three other density ratios. 
Noteworthily, this observation \eqref{eq:recursion} connects the target density ratio $\frac{p_{Y_{t-1}}}{p_{X_{t-1}}}$ at the $(t-1)$-th step with its counterpart $\frac{p_{Y_{t}}}{p_{X_{t}}}$ at the $t$-th step, 
motivating us to look at the density changes within adjacent steps in both the forward and the reverse processes (i.e., $p_{X_{t-1}}$ vs.~$p_{X_{t}}$ and $p_{Y_{t-1}}$ vs.~$p_{Y_{t}}$). 
In light of this expression, we develop a key lemma related to some of these density ratios,  
which plays a central role in establishing Theorem~\ref{thm:main-ODE}. 
The proof of this lemma is postponed to Appendix~\ref{sec:proof-lem:main-ODE}. 

\begin{lems} 
\label{lem:main-ODE}
For any $x \in \real^d$, let
\begin{align}
\label{eqn:choice-y}
	\theta_t(x) \coloneqq  \max\bigg\{ -\frac{\log {p_{X_t}(x)}}{d\log T} , c_6 \bigg\}
\end{align}
for some large enough constant $c_6 \geq 2c_R+c_0$, and suppose that 
$
	\frac{40c_{1}\varepsilon_{\score,t}(x)\log^{\frac{3}{2}}T}{T}\leq\sqrt{\theta_t(x)d}.
$
%
Then one has
\begin{align} 
	\frac{p_{\sqrt{\alpha_{t}}X_{t-1}}\big(\phi_{t}(x)\big)}{p_{X_{t}}(x)}&\leq2\exp\bigg(\Big(5\varepsilon_{\score,t}(x)\sqrt{\theta_t(x)d\log T}+60\theta_t(x)d\log T\Big)\frac{1-\alpha_{t}}{\alpha_{t}-\overline{\alpha}_{t}}\bigg).
	\label{eq:xt_up}
\end{align}
If, in addition, we have
$
C_{10}\frac{\theta_t(x)d\log^{2}T+\varepsilon_{\score,t}(x)\sqrt{\theta_t(x)d\log^{3}T}}{T}\leq1
$
for some large enough constant $C_{10}>0$, then it holds that 
\begin{subequations}
\label{eq:ODE}
\begin{align} 
&\frac{p_{\sqrt{\alpha_{t}}X_{t-1}}(\phi_{t}(x))}{p_{X_{t}}(x)} \notag \\
&=1+\frac{d(1-\alpha_{t})}{2(\alpha_{t}-\overline{\alpha}_{t})}
+\frac{(1-\alpha_{t})\Big(\big\|\int\big(x-\sqrt{\overline{\alpha}_{t}}x_{0}\big)p_{X_{0}\mymid X_{t}}(x_{0}\mymid x)\mathrm{d}x_{0}\big\|_{2}^{2}-\int\big\| x-\sqrt{\overline{\alpha}_{t}}x_{0}\big\|_{2}^{2}p_{X_{0}\mymid X_{t}}(x_{0}\mymid x)\mathrm{d}x_{0}\Big)}{2(\alpha_{t}-\overline{\alpha}_{t})(1-\overline{\alpha}_{t})}
 \notag\\
 & \quad
+O\bigg(\theta_t(x)^2d^{2}\Big(\frac{1-\alpha_{t}}{\alpha_{t}-\overline{\alpha}_{t}}\Big)^{2}\log^{2}T + \varepsilon_{\score, t}(x)\sqrt{\theta_t(x) d\log T}\Big(\frac{1-\alpha_{t}}{\alpha_{t}-\overline{\alpha}_{t}}\Big)\bigg).
	\label{eq:xt}
\end{align}

Moreover, for any random vector $Y$, one has 
\begin{align} 
 & \frac{p_{\phi_{t}(Y)}(\phi_{t}(x))}{p_{Y}(x)} \notag \\
 &=1 + \frac{d(1-\alpha_{t})}{2(\alpha_{t}-\overline{\alpha}_{t})}+\frac{(1-\alpha_{t})\Big(\big\|\int\big(x-\sqrt{\overline{\alpha}_{t}}x_{0}\big)p_{X_{0}\mymid X_{t}}(x_{0}\mymid x)\mathrm{d}x_{0}\big\|_{2}^{2}-\int\big\| x-\sqrt{\overline{\alpha}_{t}}x_{0}\big\|_{2}^{2}p_{X_{0}\mymid X_{t}}(x_{0}\mymid x)\mathrm{d}x_{0}\Big)}{2(\alpha_{t}-\overline{\alpha}_{t})(1-\overline{\alpha}_{t})} \notag \\
&\quad+ O\bigg(\theta_t(x)^2d^{2}\Big(\frac{1-\alpha_{t}}{\alpha_{t}-\overline{\alpha}_{t}}\Big)^{2}\log^{2}T+\theta_t(x)^3d^{6}\log^{3}T\Big(\frac{1-\alpha_{t}}{\alpha_{t}-\overline{\alpha}_{t}}\Big)^{3} + (1-\alpha_{t})d\varepsilon_{\Jacobi, t}(x)\bigg),
\label{eq:yt}
\end{align}
\end{subequations}
provided that 
$
	C_{11}\frac{d^{2}\log^{2}T+d\varepsilon_{\Jacobi,t}(x)\log T}{T}  \leq 1
$
 for some large enough constant $C_{11}>0$. 
\end{lems}

\begin{remark}
Combining Lemma~\ref{lem:main-ODE} with Lemma~\ref{lem:x0} and \eqref{eqn:properties-alpha-proof},  
gives: if $C_{10}\frac{\theta_t(x)d\log^{2}T+\varepsilon_{\score,t}(x)\sqrt{\theta_t(x)d\log^{3}T}}{T}\leq1$ 
and if $\theta_t(x)\lesssim 1$, then \eqref{eq:xt} tells us that
\begin{align}
	\log\frac{p_{\sqrt{\alpha_{t}}X_{t-1}}(\phi_{t}(x))}{p_{X_{t}}(x)}\leq\frac{4c_{1}d\log T}{T}
	+C_{10}\left\{ \frac{d^{2}\log^{4}T}{T^{2}}+\frac{d^{6}\log^{6}T}{T^{3}}+\frac{\varepsilon_{\score,t}(x)\sqrt{d\log^{3}T}}{T}\right\} 
	\label{eq:crude-ratio-qt-1-qt}
\end{align}
under our sample size assumption~\eqref{eq:assumption-T-score-Jacob}, 
where $C_{10}>0$ is some large enough constant. Here, we have made use of the fact that the second-to-last term in \eqref{eq:xt} is non-positive due to Jensen's inequality. 
\end{remark}

\paragraph{Step 2: decomposing the TV distance based on ``typical'' points.}  

To bound the TV distance of interest, it is helpful to isolate the following sets
%
\begin{align}
	\mathcal{E} &\coloneqq \Big\{y : q_{1}(y) > \max\big\{ p_{1}(y),\, \exp\big(- c_{6} d\log T \big) \big\} \Big\},
\end{align}
%
where $c_{6}>0$ is some large enough universal constant introduced in Lemma~\ref{lem:main-ODE}. 
In words, this set $\mathcal{E}$ contains all $y$ that can be viewed as ``typical'' values under the distribution~$q_1$ (meaning that
$q_1(y)$ is not exceedingly small), while at the same time obeying $q_1(y)>p_1(y)$.

In view of the basic properties about the TV distance, we can derive
\begin{align}
\mathsf{TV}\big(q_{1},p_{1}\big) &= \int_{y : q_{1}(y) > p_{1}(y)}\big(q_{1}(y) - p_{1}(y)\big)\mathrm{d} y \notag\\
	&= \int_{y \in \mathcal{E}}\big(q_{1}(y) - p_{1}(y)\big)\mathrm{d} y +  \int_{y:p_{1}(y)<q_{1}(y)\le\exp(-c_{6}d\log T)}\big(q_{1}(y)-p_{1}(y)\big)\mathrm{d}y . 
	\label{eqn:ode-tv-123}
\end{align}
In order to bound the second term on the right-hand side of \eqref{eqn:ode-tv-123},  
we make note of a basic fact: 
since $X_{t}\overset{\mathrm{(d)}}{=}\sqrt{\overline{\alpha}_{t}}X_0+\sqrt{1-\overline{\alpha}_{t}}W$
with $W\sim\mathcal{N}(0,I_{d})$ and $\mathbb{P}(\|X_{0}\|_{2}\leq T^{c_{R}})=1$,
it holds that
\begin{equation}
\mathbb{P}\left\{ \|X_{t}\|_{2}\geq T^{c_{R}+2}\right\} \leq\mathbb{P}\left\{ \|W\|_{2}\geq T^{2}\right\} < \exp\left(-c_{6}d\log T\right)
	\label{eq:Xt-2range-ODE}
\end{equation}
under our assumption \eqref{eq:assumption-T-score-Jacob} on $T$, thereby indicating that 
\begin{equation}
\int_{y:\|y\|_{2}\geq T^{c_{R}+2}}q_t(y)\mathrm{d}y < \exp\left(-c_{6}d\log T\right).
	\label{eq:y_norm-qy-UB}
\end{equation}
This basic fact in turn reveals that
\begin{align*}
\int_{y:p_{1}(y)<q_{1}(y)\le\exp(-c_{12}d\log T)}\big(q_{1}(y)-p_{1}(y)\big)\mathrm{d}y & \le
\int_{y:q_{1}(y)\le\exp(-c_{6}d\log T)}q_{1}(y)\mathrm{d}y \\
	& \leq\exp(-c_{6}d\log T)\int_{y:\|y\|_{2}\leq T^{c_{R}+2}}\mathrm{d}y+\exp\left(-c_{6}d\log T\right)\\
 & \leq\exp(-c_{6}d\log T)\big(2T^{c_{R}+2}\big)^{d}+\exp\left(-c_{6}d\log T\right)\\
 & \le\exp\big(-0.5c_{6}d\log T\big), 
\end{align*}
provided that $c_{6}\geq 4(c_R+2)$. Substitution into \eqref{eqn:ode-tv-123} then yields
\begin{align}
\mathsf{TV}\big(q_{1},p_{1}\big) 
	&\le \mathbb{E}_{Y_{1}\sim p_{1}}\bigg[\Big(\frac{q_{1}(Y_{1})}{p_{1}(Y_{1})}-1\Big)\ind\left\{ Y_{1}\in\mathcal{E}\right\} \bigg] + \exp\big(-c_{6}d\log T\big),
	\label{eqn:ode-tv-10}
\end{align}
with the proviso that $c_{6}\geq 4(c_R+2)$.

To proceed, let us isolate the following set   
%
\begin{align}
	\mathcal{I}_{1}\coloneqq\Big\{ y_T \mid S_{T}\big(y_{T}\big)\leq c_{14}\Big\}
	\label{eq:defn-I1-proof-ode}
\end{align}
for some small enough constant $c_{14}>0$. 
In words, $\mathcal{I}_{1}$ is composed of a set of points whose aggregate score error along the backward trajectory is well-controlled; 
in fact, these are points that exhibit ``typical'' behavior under the assumptions \eqref{eq:assumption-T-score-Jacob-score} and \eqref{eq:assumption-T-score-Jacob-Jacobi}.
As a result, we can decompose the first term of \eqref{eqn:ode-tv-10} into the influence of ``typical'' points and that of the remaining points as follows: 
\begin{align}
 & \mathop{\mathbb{E}}_{Y_{1}\sim p_{1}}\bigg[\Big(\frac{q_{1}(Y_{1})}{p_{1}(Y_{1})}-1\Big)\ind\left\{ Y_{1}\in\mathcal{E}\right\} \bigg]=\mathop{\mathbb{E}}_{Y_{T}\sim p_{T}}\bigg[\Big(\frac{q_{1}(Y_{1})}{p_{1}(Y_{1})}-1\Big)\ind\left\{ Y_{1}\in\mathcal{E}\right\} \bigg]\notag\\
 & \quad= \mathop{\mathbb{E}}_{Y_{T}\sim p_{T}}\bigg[\Big(\frac{q_{1}(Y_{1})}{p_{1}(Y_{1})}-1\Big)\ind\left\{ Y_{1}\in\mathcal{E}, Y_{T}\in \mathcal{I}_1\right\} \bigg]
	+ \mathop{\mathbb{E}}_{Y_{T}\sim p_{T}}\bigg[\frac{q_{1}(Y_{1})}{p_{1}(Y_{1})}\ind\left\{ Y_{1}\in\mathcal{E},Y_{T}\notin \mathcal{I}_1\right\} \bigg],
 	\label{eq:decompose-I1-I1c}
\end{align}
where the first identity holds since $Y_{1}$ is determined purely by $Y_{T}$ via deterministic update rules. 
%
The decomposition \eqref{eq:decompose-I1-I1c} leaves us with two terms to control, which we accomplish in the next two steps.

\paragraph{Step 3: controlling the first term on the right-hand side of \eqref{eq:decompose-I1-I1c}.}
This step analyzes the first term on the right-hand side of \eqref{eq:decompose-I1-I1c}. 
We would like to make the analysis in this step slightly more general than needed, given that it will be useful for the subsequent analysis as well. 

To begin with, let us introduce the following quantity: 
\begin{equation}
	\tau(y_T)
	\coloneqq
	\max\Big\{2\le t\le T+1:S_{t-1}\big(y_{T}\big)\leq c_{14}\Big\},\label{eq:defn-tao-i}
\end{equation}
meaning that the score errors exhibit ``typical'' behavior up to the $\big(\tau(y_T)-1\big)$-th iteration.   
As can be clearly seen from the definition \eqref{eq:defn-I1-proof-ode} of $\mathcal{I}_1$,  
\begin{equation}
	\tau(y_T) = T+1, \qquad \forall y_T \in \mathcal{I}_1.
	\label{eq:tau-T-I1}
\end{equation}
In the sequel, we first single out the following lemma, whose proof is deferred to Appendix~\ref{sec:proof-lem:q1-large-qk-large}. 
\begin{lems}
	\label{lem:q1-large-qk-large}
	Consider any $y_{T}$ and its associated sequence $\{y_{T-1},\cdots,y_1\}$ (see \eqref{eq:defn-yt-sequence-proof}).  
	If $-\log q_1(y_1)\leq c_{6}d\log T$,
	then one has
	\begin{align}
		-\log q_k(y_k)\leq 2c_{6}d\log T
		\label{eq:q_k_yk_UB}
	\end{align}
	for any $1\leq k<\tau(y_T)$ (cf.~\eqref{eq:defn-tao-i}), provided that $c_6\geq 3c_1$. 
\end{lems}

As a consequence of Lemma~\ref{lem:q1-large-qk-large}, we are able to control the density ratio $q_t/p_t$ up to the $\big(\tau(y_T)-1\big)$-th iteration, 
as stated in the following lemma. The proof can be found in Appendix~\ref{sec:proof-lem-density-ratio-tau}. 
\begin{lems}
	\label{lem:density-ratio-tau}
	Consider any $y_T$, along with the deterministic sequence $\{y_{T-1},\cdots,y_1\}$ (cf.~\eqref{eq:defn-yt-sequence-proof})), and set $\tau=\tau(y_T)$  (cf.~\eqref{eq:defn-tao-i}). Then one has 
\begin{subequations}
	\label{eq:pt-qt-equiv-ODE-St}
\begin{align}
	\frac{q_{1}(y_{1})}{p_{1}(y_{1})}  = &\left\{ 1+O\Bigg(\frac{d^{2}\log^{4}T}{T}+\frac{d^{6}\log^{6}T}{T^{2}}+S_{\tau-1}(y_{\tau-1})\Bigg)\right\} \frac{q_{\tau-1}(y_{\tau-1})}{p_{\tau-1}(y_{\tau-1})},	
	\label{eq:pt-qt-equiv-ODE-St-taui} \\
	&\text{and}
	\qquad \frac{q_{k}(y_{k})}{2p_{k}(y_{k})} \leq \frac{q_{1}(y_{1})}{p_{1}(y_{1})} \leq 2 \frac{q_{k}(y_{k})}{p_{k}(y_{k})}, \qquad \forall k < \tau. 
	\label{eq:pt-qt-equiv-ODE-St-k}
\end{align}
\end{subequations}
\end{lems}

Now let us look at the set $\mathcal{I}_1$. Taking $\tau(y_T)=T+1$ (cf.~\eqref{eq:tau-T-I1}) in Lemma~\ref{lem:density-ratio-tau} yields
\begin{align}
 &  \mathop{\mathbb{E}}_{Y_{T}\sim p_{T}}\bigg[\Big(\frac{q_{1}(Y_{1})}{p_{1}(Y_{1})}-1\Big)\ind\left\{ Y_{1}\in\mathcal{E},Y_{T}\in\mathcal{I}_1 \right\} \bigg]\nonumber\\
 & = \mathop{\mathbb{E}}_{Y_{T}\sim p_{T}}\left[\left(\left\{ 1+O\Bigg(\frac{d^{2}\log^{4}T}{T}+\frac{d^{6}\log^{6}T}{T^{2}}+S_{T}(y_{T})\Bigg)\right\} \frac{q_{T}(Y_{T})}{p_{T}(Y_{T})}-1\right)\ind\left\{ Y_{1}\in\mathcal{E},Y_{T}\in\mathcal{I}_1 \right\} \right]\nonumber\\
 & = {\displaystyle \int}\left\{ \left(1+O\Bigg(\frac{d^{2}\log^{4}T}{T}+\frac{d^{6}\log^{6}T}{T^{2}}+S_{T}(y_{T})\Bigg)\right)q_{T}(y_{T})-p_{T}(y_{T})\right\} \ind\left\{ y_{1}\in\mathcal{E},y_{T}\in\mathcal{I}_1 \right\} \mathrm{d}y_{T}\nonumber\\
 & \leq 
	{\displaystyle \int}\big|q_{T}(y_{T})-p_{T}(y_{T})\big|\mathrm{d}y_{T}
	+ O\Bigg(\frac{d^{2}\log^{4}T}{T}+\frac{d^{6}\log^{6}T}{T^{2}}\Bigg){\displaystyle \int}q_{T}(y_{T})\mathrm{d}y_{T}+O\left(\sqrt{d\log^{3}T}\varepsilon_{\score}+(d\log T)\varepsilon_{\Jacobi}\right)\nonumber\\
 & \lesssim\frac{d^{2}\log^{4}T}{T}+\frac{d^{6}\log^{6}T}{T^{2}}+\sqrt{d\log^{3}T}\varepsilon_{\score}+(d\log T)\varepsilon_{\Jacobi}. 
	\label{eq:I1-expectation-UB-ode}
\end{align}
Here, the last line holds since $\mathsf{TV}(p_T,q_T)\lesssim T^{-100}$ (according to Lemma~\ref{lem:main-ODE}), and the penultimate line follows from the observation below: 
\begin{align*}
 & {\displaystyle \int}S_{T}(y_{T})q_{T}(y_{T})\ind\left\{ y_{1}\in\mathcal{E},y_{T}\in\mathcal{I}_1\right\} \mathrm{d}y_{T}\\
 & \quad=\frac{\log T}{T}\sum_{t=1}^{T}{\displaystyle \int}\left(d\varepsilon_{\Jacobi,t}(y_{t})+\sqrt{d\log T}\varepsilon_{\score,t}(y_{t})\right)q_{T}(y_{T})\ind\left\{ y_{1}\in\mathcal{E},y_{T}\in\mathcal{I}_1\right\} \mathrm{d}y_{T}\\
 & \quad\leq\frac{4\log T}{T}\sum_{t=1}^{T}{\displaystyle \int}\left(d\varepsilon_{\Jacobi,t}(y_{t})+\sqrt{d\log T}\varepsilon_{\score,t}(y_{t})\right)\frac{q_{t}(y_{t})}{p_{t}(y_{t})}p_{T}(y_{T})\ind\left\{ y_{1}\in\mathcal{E},y_{T}\in\mathcal{I}_1\right\} \mathrm{d}y_{T}\\
 & \quad\leq\frac{4\log T}{T}\sum_{t=1}^{T}\mathbb{E}_{Y_{T}\sim p_{T}}\left[\left(d\varepsilon_{\Jacobi,t}(Y_{t})+\sqrt{d\log T}\varepsilon_{\score,t}(Y_{t})\right)\frac{q_{t}(Y_{t})}{p_{t}(Y_{t})}\right]\\
 & \quad=\frac{4\log T}{T}\sum_{t=1}^{T}\mathbb{E}_{Y_{t}\sim p_{t}}\left[\left(d\varepsilon_{\Jacobi,t}(Y_{t})+\sqrt{d\log T}\varepsilon_{\score,t}(Y_{t})\right)\frac{q_{t}(Y_{t})}{p_{t}(Y_{t})}\right]\\
 & \quad=\frac{4\log T}{T}\sum_{t=1}^{T}\mathbb{E}_{Y_{t}\sim q_{t}}\left[d\varepsilon_{\Jacobi,t}(Y_{t})+\sqrt{d\log T}\varepsilon_{\score,t}(Y_{t})\right]\\
 & \quad\lesssim (d\log T)\varepsilon_{\Jacobi} + \sqrt{d\log^{3}T}\varepsilon_{\score},
\end{align*}
where the first inequality is due to \eqref{eq:pt-qt-equiv-ODE-St}, and the last relation comes from \eqref{eq:score-assumptions-equiv}.

%
%
%
%
%
%

\paragraph{Step 4: controlling the second term on the right-hand side of \eqref{eq:decompose-I1-I1c}.}
In this step,  we find it helpful to introduce the following sets (in addition to $\mathcal{I}_1$ defined in \eqref{eq:defn-I1-proof-ode}), 
where we again abbreviate $\tau=\tau(y_T)$ as long as it is clear from the context: 
\begin{subequations}
	\label{eq:defn-I2-I3-I4-ode}
\begin{align}
\mathcal{I}_{2} & \coloneqq\Big\{ y_T : c_{14}\leq S_{\tau}\big(y_{T}\big)\leq2c_{14}\Big\},
	\label{eq:defn-I2-I3-I4-ode-I2}\\
\mathcal{I}_{3} & \coloneqq\bigg\{ y_T : S_{\tau-1}\big(y_{T}\big)\leq c_{14},\xi_{\tau}\big(y_T\big)\geq c_{14},\frac{q_{\tau-1}(y_{\tau-1})}{p_{\tau-1}(y_{\tau-1})}\leq\frac{8q_{\tau}(y_{\tau})}{p_{\tau}(y_{\tau})}\bigg\},
	\label{eq:defn-I2-I3-I4-ode-I3}\\
	\mathcal{I}_{4} & \coloneqq\bigg\{ y_T : S_{\tau-1}\big(y_{T}\big)\leq c_{14},\xi_{\tau}\big(y_T\big)\geq c_{14},\frac{q_{\tau-1}(y_{\tau-1})}{p_{\tau-1}(y_{\tau-1})}>\frac{8q_{\tau}(y_{\tau})}{p_{\tau}(y_{\tau})}\bigg\}.
\label{eq:defn-I2-I3-I4-ode-I4}
\end{align}
\end{subequations}
%
It follows immediately from the definition that $\mathcal{I}_{1} \cup \mathcal{I}_{2} \cup \mathcal{I}_{3} \cup \mathcal{I}_{4} = \mathbb{R}^d.$
%
%
In words, 
for any point $y_T$ in $\mathcal{I}_2$, the resulting score error remains well-controlled in the $\tau$-th iteration; 
in comparison, the points in $\mathcal{I}_3$ and $\mathcal{I}_4$ might incur large score errors in the $\tau$-th iteration. 
The difference between $\mathcal{I}_3$ and $\mathcal{I}_4$ then lies in the comparison between the density ratios $q_t/p_t$ 
in the $(\tau-1)$-th and the $\tau$-th iteration. 

We shall tackle each of these sets separately, with the combined result summarized in the lemma below. 
\begin{lems}
	\label{lem:I2-I3-I4-bound}
	It holds that
	\begin{align}
		\mathop{\mathbb{E}}_{Y_{T}\sim p_{T}}\bigg[\frac{q_{1}(Y_{1})}{p_{1}(Y_{1})}\ind\left\{ Y_{1}\in\mathcal{E},Y_{T}\in\mathcal{I}_2\cup \mathcal{I}_3 \cup \mathcal{I}_4\right\} \bigg]
		&\lesssim \frac{d^{2}\log^{4}T}{T}+\frac{d^{6}\log^{6}T}{T^{2}}+\sqrt{d\log^{3}T}\varepsilon_{\score}+(d\log T)\varepsilon_{\Jacobi}. 
		\label{eq:I2-I3-I4-bound}
	\end{align}
\end{lems}
\noindent See Appendix~\ref{sec:proof-lem:I2-I3-I4-bound} for the proof of this lemma.  

\paragraph{Step 5: putting all pieces together.}
Recall that $\mathcal{I}_1\cup \mathcal{I}_{2} \cup \mathcal{I}_{3} \cup \mathcal{I}_{4} = \mathbb{R}^d.$ 
Taking \eqref{eqn:ode-tv-1}, \eqref{eq:decompose-I1-I1c}, \eqref{eq:I1-expectation-UB-ode} and \eqref{eq:I2-I3-I4-bound} collectively, we conclude that
\begin{align*}
\mathsf{TV}(p_{1},q_{1}) & \leq \mathop{\mathbb{E}}_{Y_{T}\sim p_{T}}\bigg[\Big(\frac{q_{1}(Y_{1})}{p_{1}(Y_{1})}-1\Big)\ind\big\{ Y_{1}\in\mathcal{E},Y_{T}\in \mathcal{I}_1\big\} \bigg] \\
	&\qquad 
	+
	\mathop{\mathbb{E}}_{Y_{T}\sim p_{T}}\bigg[\frac{q_{1}(Y_{1})}{p_{1}(Y_{1})}\ind\left\{ Y_{1}\in\mathcal{E},Y_{T}\in\mathcal{I}_2\cup \mathcal{I}_3 \cup \mathcal{I}_4\right\} \bigg]+\exp(-c_{6}d\log T)\\
 & \lesssim\frac{d^{2}\log^{4}T}{T}+\frac{d^{6}\log^{6}T}{T^{2}}+\sqrt{d\log^{3}T}\varepsilon_{\score}+d\varepsilon_{\Jacobi}\log T
\end{align*}
as claimed.



%
%

\subsection{Analysis for the DDPM-type sampler (Theorem~\ref{thm:main-SDE})}

Turning attention to the DDPM-type stochastic sampler \eqref{eqn:sde-sampling}, 
we now present the main steps for the proof of Theorem~\ref{thm:main-SDE}.

\paragraph{Preparation.} Let us first introduce the following mapping
\begin{subequations}
\begin{align}
	\mu_t^{\star}(x_t) &\defn \frac{1}{\sqrt{\alpha_t}} \Big( x_t + ( 1-\alpha_t ) s_t^{\star}(x_t) \Big) = \frac{1}{\sqrt{\alpha_t}} x_t - 
	\frac{ 1-\alpha_t }{\sqrt{\alpha_t}(1-\overline{\alpha}_{t})} \int_{x_0}  p_{X_0\mymid X_{t}} (x_0\mymid x_t) \big( x_t - \sqrt{\overline{\alpha}_{t}} x_0 \big)\mathrm{d} x_0, \\
	\mu_t(x_t) &\defn \frac{1}{\sqrt{\alpha_t}} \Big( x_t + ( 1-\alpha_t ) s_t(x_t) \Big),
	\label{eqn:nu-t-2}
\end{align}
\end{subequations}
where the first line relies on the expression \eqref{eq:st-MMSE-expression}. 
For any $t$, let us also introduce an auxiliary vector
\begin{equation}
	Y_{t-1}^{\star} \coloneqq \frac{1}{\sqrt{\alpha_t}} \Big( Y_t + (1-\alpha_t)s_t^{\star}(Y_t)\Big) + \sigma_t Z_t,
	\label{eqn:sde-sampling-auxiliary}
\end{equation}
which applies the DDPM update rule from $Y_t$ using the true score function $s^{\star}_t$.  
From the update rule \eqref{eqn:sde-sampling} and \eqref{eqn:sde-sampling-auxiliary}, we can write 
\begin{subequations}
\begin{align}
	p_{Y_{t-1}\mymid Y_{t}}(x_{t-1}\mymid x_{t}) & =\frac{1}{\big(2\pi\frac{1-\alpha_{t}}{\alpha_{t}}\big)^{d/2}}\exp\bigg(-\frac{\alpha_{t}}{2(1-\alpha_{t})}\big\| x_{t-1}-\mu_{t}(x_t)\big\|_{2}^{2}\bigg)
	\label{eq:dist-Yt-Yt-1} \\
	p_{Y_{t-1}^{\star}\mymid Y_{t}}(x_{t-1}\mymid x_{t}) & =\frac{1}{\big(2\pi\frac{1-\alpha_{t}}{\alpha_{t}}\big)^{d/2}}\exp\bigg(-\frac{\alpha_{t}}{2(1-\alpha_{t})}\big\| x_{t-1}-\mu_{t}^{\star}(x_t)\big\|_{2}^{2}\bigg)
	\label{eq:dist-Ytstar}
\end{align}
\end{subequations}
for any two points $x_t,x_{t-1}\in \mathbb{R}^d$. 
For notational simplicity, we shall also use the following notation throughout: 
\begin{equation}
	\widehat{x}_t \coloneqq \frac{1}{\sqrt{\alpha_t}} x_t. 
	\label{eq:defn-xt-proof-thm3}
\end{equation}

Armed with this set of notation, 
we are ready to present the proof of Theorem~\ref{thm:main-SDE}, which consists of several steps below.

\paragraph{Step 1: decomposition of the KL divergence.}
The celebrated Pinsker inequality (see, e.g., \citet[Lemma~2.5]{tsybakov2009introduction}) tells us that
\begin{equation}
\mathsf{TV}(p_{X_{1}},p_{Y_{1}})\leq\sqrt{\frac{1}{2}\mathsf{KL}(p_{X_{1}}\parallel p_{Y_{1}})},\label{eq:Pinsker-thm3}
\end{equation}
and hence it suffices to work with the KL divergence. 
Recall that $X_1\rightarrow \cdots \rightarrow X_T$ and $Y_T\rightarrow \cdots \rightarrow Y_1$ are both Markov chains (so are their reverse processes). 
In order to compute the KL divergence between $p_{X_1}$ and $p_{Y_1}$, we make note of the following elementary relations:
\begin{align*}
	\mathsf{KL}(p_{X_{1},\ldots,X_{T}}\parallel p_{Y_{1},\ldots,Y_{T}}) & =\mathsf{KL}(p_{X_{1}}\parallel p_{Y_{1}})+\sum_{t=2}^{T} \mathop{\mathbb{E}}_{x\sim q_{t-1}}\Big[\mathsf{KL}\Big(p_{X_{t}\mymid X_{t-1}}(\cdot\mid x)\,\big\|\,p_{Y_{t}\mymid Y_{t-1}}(\cdot\mid x)\Big)\Big]\\
	& =\mathsf{KL}(p_{X_{T}}\parallel p_{Y_{T}})+\sum_{t=2}^{T} \mathop{\mathbb{E}}_{x\sim q_{t}}\Big[\mathsf{KL}\Big(p_{X_{t-1}\mymid X_{t}}(\cdot\mid x)\,\big\|\,p_{Y_{t-1}\mymid Y_{t}}(\cdot\mid x)\Big)\Big],
\end{align*}
where we recall that $q_t$ (resp.~$p_t$) denotes the distribution of $X_t$ (resp.~$Y_t$). 
This combined with the non-negativity of the KL divergence indicates that
\begin{align}
\mathsf{KL}(p_{X_{1}}\parallel p_{Y_{1}}) 
 & \leq \mathsf{KL}(p_{X_{T}}\parallel p_{Y_{T}})+
	\sum_{t=2}^{T}\mathop{\mathbb{E}}_{x\sim q_{t}}\Big[\mathsf{KL}\Big(p_{X_{t-1}\mymid X_{t}}(\cdot\mid x) \,\big\|\, p_{Y_{t-1}\mymid Y_{t}}(\cdot\mid x)\Big)\Big]. 
	\label{eqn:kl-decomp}
\end{align} 
This allows us to focus attention on the transition probabilities at each time instance $t$. 
On the right-hand side of \eqref{eqn:kl-decomp}, the term that is the easiest to bound is $\mathsf{KL}(p_{X_{T}}\parallel p_{Y_{T}})$. It has been shown in Lemma~\ref{lem:KL-T} that 
 \[
	 \mathsf{KL}(p_{X_{T}}\parallel p_{Y_{T}}) \lesssim \frac{1}{T^{200}}. 
\]
Thus, it suffices to focus attention on bounding 
$\mathsf{KL}\big(p_{X_{t-1}\mymid X_{t}}(\cdot\mid x) \parallel p_{Y_{t-1}\mymid Y_{t}}(\cdot\mid x)\big)$ for each $1<t\leq T$, 
which forms the main content of the subsequent proof.

\paragraph{Step 2: controlling the conditional distributions $p_{X_{t-1}\mymid X_{t}}$ and $p_{Y^{\star}_{t-1}\mymid Y_{t}}$.}  
In order to compute the KL divergence of interest in \eqref{eqn:kl-decomp}, 
one needs to calculate the two conditional distributions $p_{X_{t-1}\mymid X_{t}}$ and $p_{Y_{t-1}\mymid Y_{t}}$, 
which we study in this step. 
To do so, we find it helpful to first introduce the following set 
\begin{align}
\label{eqn:eset}
	\mathcal{E} \defn \bigg\{(x_t, x_{t-1}) \mymid -\log p_{X_t}(x_t) \leq \frac{1}{2}c_6 d\log T, ~\|x_{t-1} - \widehat{x}_t\|_2 \leq c_3 \sqrt{d(1 - \alpha_t)\log T} \bigg\},
\end{align}
where the two numerical constants $c_3,c_6>0$ are introduced in Lemma~\ref{lem:river}. 
Informally, $\mathcal{E}$ encompasses a typical range of the values of $(X_t,X_{t-1})$, and 
our analysis shall often proceed by studying the points in $\mathcal{E}$ and those outside $\mathcal{E}$ separately.

The first result below quantifies the conditional density $p_{X_{t-1}\mymid X_{t}}(x_{t-1}\mymid x_{t})$ for those points residing within $\mathcal{E}$, 
which plays a central role in comparing $p_{X_{t-1}\mymid X_{t}}$ against $p_{Y_{t-1}\mymid Y_{t}}$ (see \eqref{eq:dist-Yt-Yt-1}). 
The proof can be found in Appendix~\ref{sec:proof-lem:sde}.  
\begin{lems}
\label{lem:sde}
There exists some large enough numerical constant $c_{\zeta}>0$ such that: 
for every $(x_t, x_{t-1}) \in \mathcal{E}$,  
	\begin{align}
		p_{X_{t-1}\mymid X_{t}}(x_{t-1}\mymid x_{t})=
		\frac{1}{\big(2\pi\frac{1-\alpha_{t}}{\alpha_{t}}\big)^{d/2}} \exp\bigg(-\frac{\alpha_{t}\big\| x_{t-1}-\mu_{t}^{\star}(x_t)\big\|_{2}^{2}}{2(1-\alpha_{t})}+\zeta_{t}(x_{t-1},x_t)\bigg) 
		\label{eq:cond-dist-crude}
	\end{align}
holds for some residual term $\zeta_{t}(x_{t-1},x_t)$ obeying 
\begin{equation}
	\big|\zeta_{t}(x_{t-1},x_t) \big|\leq c_{\zeta} d^{2}\bigg(\frac{1-\alpha_{t}}{\alpha_{t}-\overline{\alpha}_{t}}\bigg)\log^{2}T. 
	\label{eq:eq:cond-dist-crude-xit}
\end{equation}
Here, we recall the definition of $\mu_t^{\star}(x_t)$ in \eqref{eqn:nu-t-2}.
%
%
%
%
\end{lems}
\noindent 
By comparing Lemma~\ref{lem:sde} with expression \eqref{eq:dist-Yt-Yt-1}, 
we see that when restricted to the set $\mathcal{E}$, 
the two conditional distributions $p_{X_{t-1}\mymid X_{t}}(x_{t-1}\mymid x_{t})$ and $p_{Y_{t-1}\mymid Y_{t}}(x_{t-1}\mymid x_{t})$ (i.e., informally, the time-reversed transition kernels) are fairly close to each other, 
a crucial observation that suggests the validity of the diffusion generative model.

Furthermore, we are also in need of bounding the ratio of the two conditional distributions when going beyond the set $\mathcal{E}$. 
As it turns out, it suffices to develop a crude bound on the logarithm of such ratios (which are used in defining the KL divergence), as stated in the following lemma. 
\begin{lems}
\label{lem:sde-full}
For all $(x_t, x_{t-1}) \in \real^d \times \real^d$, it holds that
\begin{align}
\log\frac{p_{X_{t-1}\mymid X_{t}}(x_{t-1}\mymid x_{t})}{p_{Y_{t-1}^{\star}\mymid Y_{t}}(x_{t-1}\mymid x_{t})}\leq2T\left(\|x_{t-1}-\widehat{x}_{t}\|_{2}^{2}+\|x_{t}\|_{2}^{2}+T^{2c_{R}}\right).
	\label{eq:SDE-ratio-crude}
\end{align}
\end{lems}
\noindent 
The proof of  Lemma~\ref{lem:sde-full} is provided in Appendix~\ref{sec:proof-lem:sde-full}.

%


\paragraph{Step 3: bounding the KL divergence between $p_{X_{t-1}\mymid X_t}$ and $p_{Y^{\star}_{t-1}\mymid Y_t}$.}
Equipped with the above two lemmas, we are positioned to first control the KL divergence between $p_{X_{t-1}\mymid X_t}$ and $p_{Y^{\star}_{t-1}\mymid Y_t}$; 
by doing so we are neglecting the impact of the score estimation error for the moment (see the definition \eqref{eqn:sde-sampling-auxiliary} of $Y^{\star}_{t-1}$). 
It is first seen from Lemma~\ref{lem:sde} and \eqref{eq:dist-Yt-Yt-1} that: for any $(x_t,x_{t-1})\in \mathcal{E}$, 
\begin{align}
	\frac{p_{X_{t-1}\mymid X_{t}}(x_{t-1}\mymid x_{t})}{p_{Y_{t-1}^{\star}\mymid Y_{t}}(x_{t-1}\mymid x_{t})} & =\exp\Big(O\Big(d^{2}\Big(\frac{1-\alpha_{t}}{\alpha_{t}-\overline{\alpha}_{t}}\Big)\log^{2}T\Big)\Big)=1+O\bigg(d^{2}\Big(\frac{1-\alpha_{t}}{\alpha_{t}-\overline{\alpha}_{t}}\Big)\log^{2}T\bigg) \notag\\
 & =1+O\bigg(d^{2}\frac{\log^{3}T}{T}\bigg)\in\Big[\frac{1}{2},2\Big],
	\label{eq:ratio-p-X-Y-range}
\end{align}
where the last line results from \eqref{eqn:properties-alpha-proof-1} and the assumption that $T\geq c_{10} d^2 \log^3T$ for some large enough constant $c_{10}>0$.     
We can then calculate 
\begin{align}
 & \mathbb{E}_{x_{t}\sim q_{t}}\Big[\mathsf{KL}\Big(p_{X_{t-1}\mymid X_{t}}(\cdot\mymid x_{t})\parallel p_{Y_{t-1}^{\star}\mymid Y_{t}}(\cdot\mymid x_{t})\Big)\Big]\notag\\
 & =\Big(\int_{\mathcal{E}}+\int_{\mathcal{E}^{\mathrm{c}}}\Big)p_{X_{t}}(x_{t})p_{X_{t-1}\mymid X_{t}}(x_{t-1}\mymid x_{t})\log\frac{p_{X_{t-1}\mymid X_{t}}(x_{t-1}\mymid x_{t})}{p_{Y_{t-1}^{\star}\mymid Y_{t}}(x_{t-1}\mymid x_{t})}\mathrm{d}x_{t-1}\mathrm{d}x_{t},\notag\\
 & \stackrel{(\text{i})}{=}\int_{\mathcal{E}}p_{X_{t}}(x_{t})\Bigg\{ p_{X_{t-1}\mymid X_{t}}(x_{t-1}\mymid x_{t})-p_{Y_{t-1}^{\star}\mymid Y_{t}}(x_{t-1}\mymid x_{t})\notag\\
 & \qquad\qquad\qquad+p_{X_{t-1}\mymid X_{t}}(x_{t-1}\mymid x_{t})\cdot O\Bigg(\bigg(\frac{p_{Y_{t-1}^{\star}\mymid Y_{t}}(x_{t-1}\mymid x_{t})}{p_{X_{t-1}\mymid X_{t}}(x_{t-1}\mymid x_{t})}-1\bigg)^{2}\Bigg)\Bigg\}\mathrm{d}x_{t-1}\mathrm{d}x_{t} \notag\\
 & \qquad+\int_{\mathcal{E}^{\mathrm{c}}}p_{X_{t}}(x_{t})p_{X_{t-1}\mymid X_{t}}(x_{t-1}\mymid x_{t})\log\frac{p_{X_{t-1}\mymid X_{t}}(x_{t-1}\mymid x_{t})}{p_{Y_{t-1}^{\star}\mymid Y_{t}}(x_{t-1}\mymid x_{t})}\mathrm{d}x_{t-1}\mathrm{d}x_{t}\notag%
\\
 & \stackrel{(\text{ii})}{=}\int_{\mathcal{E}}p_{X_{t}}(x_{t})\Bigg\{ p_{X_{t-1}\mymid X_{t}}(x_{t-1}\mymid x_{t})-p_{Y_{t-1}^{\star}\mymid Y_{t}}(x_{t-1}\mymid x_{t})+p_{X_{t-1}\mymid X_{t}}(x_{t-1}\mymid x_{t}) O\bigg(d^{4}\Big(\frac{1-\alpha_{t}}{\alpha_{t}-\overline{\alpha}_{t}}\Big)^{2}\log^{4}T\bigg)\Bigg\}\mathrm{d}x_{t-1}\mathrm{d}x_{t}\notag\\
 & \qquad+\int_{\mathcal{E}^{\mathrm{c}}}p_{X_{t}}(x_{t})p_{X_{t-1}\mymid X_{t}}(x_{t-1}\mymid x_{t})\Big\{2T\big(\|x_{t}\|_{2}^{2}+\|x_{t-1}-\widehat{x}_{t}\|_{2}^{2}+T^{2c_{R}}\big)\Big\}\mathrm{d}x_{t-1}\mathrm{d}x_{t}.%
\label{eqn:serenade}
\end{align}
Here,  $(\text{i})$ results from the elementary fact that: if $\big|\frac{p_Y(x)}{p_X(x)} - 1\big| < \frac{1}{2}$, then the Taylor expansion gives
\begin{align*}
p_X(x)\log \frac{p_X(x)}{p_Y(x)} &= -p_X(x)\log \Big(1 + \frac{p_Y(x) - p_X(x)}{p_X(x)}\Big) \notag \\
&= p_X(x) - p_Y(x) + p_X(x)O\bigg(\Big(\frac{p_Y(x)}{p_X(x)} - 1\Big)^2\bigg);
\end{align*}
regarding (ii), we invoke \eqref{eq:ratio-p-X-Y-range} and Lemma~\ref{lem:sde-full}.

To continue, let us bound each term on the right-hand side of \eqref{eqn:serenade} separately.  
From the definition of the set $\mathcal{E}$ (cf.~\eqref{eqn:eset}), direct calculations yield 
\begin{align}
\mathbb{P}\big((X_{t},X_{t-1})\notin\mathcal{E}\big) & =\int_{(x_{t},x_{t-1})\notin\mathcal{E}}p_{X_{t-1}}(x_{t-1})p_{X_{t}\mymid X_{t-1}}(x_{t}\mymid x_{t-1})\mathrm{d}x_{t-1}\mathrm{d}x_{t} \notag\\
 & =\int_{(x_{t},x_{t-1})\notin\mathcal{E}}p_{X_{t-1}}(x_{t-1})\frac{1}{\big(2\pi(1-\alpha_{t})\big)^{d/2}}\exp\bigg(-\frac{\|x_{t}-\sqrt{\alpha_{t}}x_{t-1}\|_{2}^{2}}{2(1-\alpha_{t})}\bigg)\mathrm{d}x_{t-1}\mathrm{d}x_{t} \notag\\
 & \le\exp\big(-c_3d\log T\big),
	\label{eq:P-Xt-X-t-1-not-E}
\end{align}
and similarly,
\begin{align}
\int_{(x_{t-1},x_{t})\notin\mathcal{E}}p_{X_{t}}(x_{t})p_{X_{t-1}\mymid X_{t}}(x_{t-1}\mymid x_{t})\Big(2T\big(\|x_{t}\|_{2}^{2}+\|x_{t-1}-\widehat{x}_{t}\|_{2}^{2}+T^{2c_{R}}\big)\Big)\mathrm{d}x_{t-1}\mathrm{d}x_{t} & \le\exp\big(-c_3d\log T\big).	
\label{eq:P-Xt-X-t-1-not-E-246}
\end{align}
In addition, for every $(x_t,x_{t-1})$ obeying $\|x_{t-1} - x_t/\sqrt{\alpha_t}\|_2 > c_3 \sqrt{d(1 - \alpha_t)\log T}$ and $-\log p_{X_t}(x_t) \leq \frac{1}{2}c_6 d\log T$, 
one can use the definition \eqref{eqn:nu-t-2} of $\mu_t^{\star}(\cdot)$ to obtain
\begin{align}
	\|x_{t-1}-\mu_{t}^{\star}(x_{t})\|_{2} & =\bigg\| x_{t-1}-\frac{1}{\sqrt{\alpha_{t}}}x_{t}-\frac{1-\alpha_{t}}{\sqrt{\alpha_{t}}(1-\overline{\alpha}_{t})}\mathbb{E}\Big[x_{t}-\sqrt{\overline{\alpha}_{t}}X_{0}\mid X_{t}=x_{t}\Big]\bigg\|_{2} \label{eq:x-tminus1-mu-2norm}\\
 & \geq\bigg\| x_{t-1}-\frac{1}{\sqrt{\alpha_{t}}}x_{t}\bigg\|_{2}-\frac{1-\alpha_{t}}{\sqrt{\alpha_{t}}(1-\overline{\alpha}_{t})}\mathbb{E}\Big[\big\| x_{t}-\sqrt{\overline{\alpha}_{t}}X_{0}\big\|_{2}\mid X_{t}=x_{t}\Big] \notag\\
 & \geq c_{3}\sqrt{d(1-\alpha_{t})\log T}-6\overline{c}_{5}\frac{1-\alpha_{t}}{\sqrt{\alpha_{t}(1-\overline{\alpha}_{t})}}\sqrt{d\log T} \notag\\
 & =\Bigg(c_{3}-6\overline{c}_{5}\frac{\sqrt{1-\alpha_{t}}}{\sqrt{\alpha_{t}(1-\overline{\alpha}_{t})}}\Bigg)\sqrt{d(1-\alpha_{t})\log T}\geq\frac{c_{3}}{2}\sqrt{d(1-\alpha_{t})\log T},
	\label{eq:x-tminus1-mu-2norm-479}
\end{align}
where the third line results from \eqref{eq:E-xt-X0} in Lemma~\ref{lem:x0},
and the last line applies \eqref{eqn:properties-alpha-proof}
and holds true as long as $c_{3}$ is large enough. 
In turn, this combined with \eqref{eq:dist-Yt-Yt-1} indicates that: for any $x_t$ obeying $-\log p_{X_t}(x_t) \leq \frac{1}{2}c_6 d\log T$, 
\begin{align}
	\int_{x_{t-1}:\|x_{t-1} - x_t/\sqrt{\alpha_t}\|_2 > c_3 \sqrt{d(1 - \alpha_t)\log T}} 
	p_{Y_{t-1}^{\star} \mymid Y_t}(x_{t-1}\mymid x_t)\mathrm{d}x_{t-1} &\le \exp\Big(- \frac{c_3}{2} d\log T \Big).
	\label{eq:P-Yt-X-t-1-not-E}
\end{align}
As a result, \eqref{eq:P-Xt-X-t-1-not-E} and \eqref{eq:P-Yt-X-t-1-not-E} taken collectively demonstrate that
\begin{align}
 & \left|\int_{\mathcal{E}}p_{X_{t}}(x_{t})\Big\{ p_{X_{t-1}\mymid X_{t}}(x_{t-1}\mymid x_{t})-p_{Y_{t-1}^{\star}\mymid Y_{t}}(x_{t-1}\mymid x_{t})\Big\}\mathrm{d}x_{t-1}\mathrm{d}x_{t}\right| \notag\\
 & \quad=\left|1-\mathbb{P}\big((X_{t},X_{t-1})\notin\mathcal{E}\big)-\int_{(x_{t},x_{t-1})\notin\mathcal{E}}p_{X_{t}}(x_{t})\Big\{1-p_{Y_{t-1}^{\star}\mymid Y_{t}}(x_{t-1}\mymid x_{t})\Big\}\mathrm{d}x_{t-1}\mathrm{d}x_{t}\right| \notag\\
 & \quad\leq2\exp\Big(- \frac{c_3}{2} d\log T \Big).
	\label{eq:X-Y-leak}
\end{align}

Substituting \eqref{eq:P-Xt-X-t-1-not-E-246} and \eqref{eq:X-Y-leak} into \eqref{eqn:serenade} yields: for each $t\geq 2$,  
\begin{align}
\label{eqn:each-kl}
	\mathop{\mathbb{E}}_{x_{t}\sim q_{t}}\Big[\mathsf{KL}\Big(p_{X_{t-1}\mymid X_{t}}(\cdot\mymid x_{t})\parallel p_{Y_{t-1}^{\star}\mymid Y_{t}}(\cdot\mymid x_{t})\Big)\Big] 
	 \lesssim
	 	d^4\Big(\frac{1-\alpha_t}{\alpha_t-\overline{\alpha}_{t}}\Big)^{2}\log^4 T  + 3\exp\Big(- \frac{c_3}{2} d\log T \Big)
		\lesssim \frac{d^{4}\log^{6}T}{T^{2}},
\end{align}
where the last inequality utilizes the properties \eqref{eqn:properties-alpha-proof} of the learning rates.

\paragraph{Step 4: quantifying the effect of score estimation errors.}
Thus far, we have quantified the KL divergence between $p_{X_{t-1}\mymid X_{t}}$ and $p_{Y_{t-1}^{\star}\mymid Y_{t}}$. 
Given that $Y_{t-1}^{\star}$ is obtained from $Y_t$ using the true score function, 
we still need to control the effect of the score estimation error, 
which is accomplished by means of the following lemma. The proof of this lemma can be found in Appendix~\ref{sec:proof-lem:influence-error-KL}. 
\begin{lems}
	\label{lem:influence-error-KL}
	For any $1<t\leq T$, one has
\begin{align}
	&\mathop{\mathbb{E}}_{x_{t}\sim q_{t}}\Big[\mathsf{KL}\Big(p_{X_{t-1}\mymid X_{t}}(\cdot\mymid x_{t})\parallel p_{Y_{t-1}\mymid Y_{t}}(\cdot\mymid x_{t})\Big)\Big] - 
	\mathop{\mathbb{E}}_{x_{t}\sim q_{t}}\Big[\mathsf{KL}\Big(p_{X_{t-1}\mymid X_{t}}(\cdot\mymid x_{t})\parallel p_{Y_{t-1}^{\star}\mymid Y_{t}}(\cdot\mymid x_{t})\Big)\Big] \notag\\
	&\qquad \lesssim \exp\big(- c_{20} d\log T \big) + \frac{d\log^3 T}{T} \mathop{\mathbb{E}}_{X_t\sim q_t}\big[\varepsilon_{\score, t}(X_t)^2\big] 
	\label{eq:H-upper-bound-DDPM-lem}
\end{align}
for some universal constant $c_{20}>0$. 
\end{lems}
%


\paragraph{Step 5: putting all this together.} 
To finish up, substitute \eqref{eqn:each-kl}, \eqref{eq:H-upper-bound-DDPM} and \eqref{eq:H-upper-bound-DDPM} into the decomposition~\eqref{eqn:kl-decomp} to obtain 
\begin{align*}
\mathsf{KL}(p_{X_{1}}\parallel p_{Y_{1}}) & \lesssim\mathsf{KL}(p_{X_{T}}\parallel p_{Y_{T}})+\sum_{t=1}^{T-1}\mathop{\mathbb{E}}_{x_{t}\sim q_{t}}\Big[\mathsf{KL}\Big(p_{X_{t-1}\mymid X_{t}}(\cdot\mymid x_{t})\parallel p_{Y_{t-1}^{\star}\mymid Y_{t}}(\cdot\mymid x_{t})\Big)\Big]\\
 & +\sum_{t=1}^{T-1}\left\{ \mathop{\mathbb{E}}_{x_{t}\sim q_{t}}\Big[\mathsf{KL}\Big(p_{X_{t-1}\mymid X_{t}}(\cdot\mymid x_{t})\parallel p_{Y_{t-1}\mymid Y_{t}}(\cdot\mymid x_{t})\Big)\Big]-\mathop{\mathbb{E}}_{x_{t}\sim q_{t}}\Big[\mathsf{KL}\Big(p_{X_{t-1}\mymid X_{t}}(\cdot\mymid x_{t})\parallel p_{Y_{t-1}^{\star}\mymid Y_{t}}(\cdot\mymid x_{t})\Big)\Big]\right\} \\
	& \lesssim\mathsf{KL}(p_{X_{T}}\parallel p_{Y_{T}})+\sum_{2\leq t\leq T}\frac{d^{4}\log^{6}T}{T^{2}}+\left(d\log^{3}T\right)\sum_{t=2}^{T}\mathop{\mathbb{E}}_{X_{t}\sim q_{t}}\left[\varepsilon_{\score,t}(X_{t})^{2}\right]\\
 & \asymp\frac{d^{4}\log^{6}T}{T}+d\varepsilon_{\score}^{2}\log^{3}T,
\end{align*}
where the last relation applies the bound on $\mathsf{KL}(p_{X_T} \parallel p_{Y_T})$ as in \eqref{eqn:KL-T-123}. 
This establishes Theorem~\ref{thm:main-SDE}. 


\section{Discussion}
\label{sec:discussion}

In this paper, we have developed a new suite of non-asymptotic theory for establishing the convergence and faithfulness of diffusion generative modeling, 
assuming access to reliable estimates of the (Stein) score functions. 
Our analysis framework seeks to track the dynamics of the reverse process directly using elementary tools, 
which eliminates the need to look at the continuous-time limit and invoke the SDE and ODE toolboxes. 
Only very minimal assumptions on the target data distribution are imposed. 
In addition to demonstrating the non-asymptotic iteration complexities of two mainstream discrete-time samplers --- a deterministic sampler based on the probability flow ODE, and a DDPM-type stochastic sampler --- we have discovered potential strategies to further accelerate the sampling processes, taking advantage of estimates of a small number of additional objects.  
The analysis framework laid out in the current paper might shed light on how to analyze other variants of score-based generative models as well.

Moving forward, there are plenty of questions that require in-depth theoretical understanding. 
For instance, the dimension dependency in our convergence results remains sub-optimal; can we further refine our theory in order to reveal tight dependency in this regard? 
Can we establish sharp convergence results in terms of the Wasserstein distance, which 
could sometimes be ``closer'' to how humans differentiate pictures and might potentially help relax Assumption~\ref{assumption:score-estimate-Jacobi} in the case of deterministic samplers?  
To what extent can we further accelerate the sampling process, without requiring much more information than the score functions? 
Ideally, one would hope to achieve accleration with the aid of the score functions only.    
It would also be of paramount interest to establish end-to-end performance guarantees that take into account both the score learning phase and the sampling phase.

\section*{Acknowledgements}

Y.~Wei is supported in part by the the NSF grants DMS-2147546/2015447, CAREER award DMS-2143215, CCF-2106778, and the Google Research Scholar Award. Y.~Chen is supported in part by the Alfred P.~Sloan Research Fellowship, the Google Research Scholar Award, the AFOSR grant FA9550-22-1-0198, 
the ONR grant N00014-22-1-2354,  and the NSF grants CCF-2221009 and CCF-1907661. Y.~Chi are supported in part by the grants ONR N00014-19-1-2404, NSF CCF-2106778, DMS-2134080 and CNS-2148212.

\appendix


\section{Proof for several preliminary facts}
\label{sec:proof-preliminary}

\subsection{Proof of~properties \eqref{eq:Jt-x-expression-ij-23}}
Elementary calculations reveal that: the $(i,j)$-th entry of $J_{t}(x)$ is given by
\begin{align}
\big[J_{t}(x)\big]_{i,j} & =\ind\{i=j\} + \frac{1}{1-\overline{\alpha}_{t}}\bigg\{\Big(\int_{x_{0}}p_{X_{0}\mymid X_{t}}(x_{0}\mymid x)\big(x_{i}-\sqrt{\overline{\alpha}_{t}}x_{0,i}\big)\mathrm{d}x_{0}\Big)\Big(\int_{x_{0}}p_{X_{0}\mymid X_{t}}(x_{0}\mymid x)\big(x_{j}-\sqrt{\overline{\alpha}_{t}}x_{0,j}\big)\mathrm{d}x_{0}\Big)\notag\nonumber \\
 & \qquad\qquad-\int_{x_{0}}p_{X_{0}\mymid X_{t}}(x_{0}\mymid x)\big(x_{i}-\sqrt{\overline{\alpha}_{t}}x_{0,i}\big)\big(x_{j}-\sqrt{\overline{\alpha}_{t}}x_{0,j}\big)\mathrm{d}x_{0}\bigg\}.
 \label{eqn:derivative-2}
\end{align}
This immediately establishes the matrix expression \eqref{eq:Jt-x-expression-ij-23}. 

\subsection{Proof of~properties \eqref{eqn:properties-alpha-proof} regarding the learning rates}

\label{sec:proof-properties-alpha}

\paragraph{Proof of property~\eqref{eqn:properties-alpha-proof-00}.} 
From the choice of $\beta_t$ in \eqref{eqn:alpha-t}, we have
\[
\alpha_{t}=1-\beta_{t}\geq1-\frac{c_{1}\log T}{T}\geq\frac{1}{2},\qquad t\geq2.
\]
The case with $t=1$ holds trivially since $\beta_1=1/T^{c_0}$ for some large enough constant $c_0>0$.

\paragraph{Proof of properties~\eqref{eqn:properties-alpha-proof-1} and \eqref{eqn:properties-alpha-proof-3}.}
We start by proving \eqref{eqn:properties-alpha-proof-1}. 
Let $\tau$ be an integer obeying 
\begin{equation}
	\beta_{1}\bigg(1+\frac{c_{1}\log T}{T}\bigg)^{\tau} \le 1 < \beta_{1}\bigg(1+\frac{c_{1}\log T}{T}\bigg)^{\tau+1},
	\label{eq:defn-tau-proof-alpha}
\end{equation}
and we divide into two cases based on $\tau$. 
\begin{itemize}
	\item Consider any $t$ satisfying $t\leq \tau$.  
		In this case, it suffices to prove that
		\begin{align}
			1-\overline{\alpha}_{t-1} \ge \frac{1}{3}\beta_{1}\bigg(1+\frac{c_{1}\log T}{T}\bigg)^{t}.
			\label{eq:induction-proof-alpha}
		\end{align}
		Clearly, if \eqref{eq:induction-proof-alpha} is valid, then any $t\leq \tau$ obeys
		\[
			\frac{1-\alpha_{t}}{1-\overline{\alpha}_{t-1}}
			= \frac{\beta_{t}}{1-\overline{\alpha}_{t-1}}
			\leq\frac{\frac{c_{1}\log T}{T}\beta_{1}\big(1+\frac{c_{1}\log T}{T}\big)^{t}}{\frac{1}{3}\beta_{1}\big(1+\frac{c_{1}\log T}{T}\big)^{t}}=\frac{3c_{1}\log T}{T}
		\]
as claimed. Towards proving \eqref{eq:induction-proof-alpha}, 
first note that the base case with $t=2$ holds true trivially since $1-\overline{\alpha}_{1} =  1-\alpha_1 = \beta_{1} \geq \beta_1 \big(1+\frac{c_{1}\log T}{T}\big)^{2} /3$. 
		Next, let $t_0>2$ be {\em the first time} that Condition \eqref{eq:induction-proof-alpha} fails to hold and suppose that $t_0\leq \tau$.  
		It then follows that
		\begin{align}
			1-\overline{\alpha}_{t_{0}-2}=1-\frac{\overline{\alpha}_{t_{0}-1}}{\alpha_{t_{0}-1}}\le1-\overline{\alpha}_{t_{0}-1}
			< \frac{1}{3}\beta_{1}\bigg(1+\frac{c_{1}\log T}{T}\bigg)^{t_{0}}\leq\frac{1}{2}\beta_{1}\bigg(1+\frac{c_{1}\log T}{T}\bigg)^{t_{0}-1} 
			< \frac{1}{2}, 			
		\end{align}
		where the last inequality result from \eqref{eq:defn-tau-proof-alpha} and the assumption $t_0\leq \tau$. 
		This taken together with the assumptions \eqref{eq:induction-proof-alpha} and $t_0\leq \tau$ implies that 
\[
\frac{(1-\alpha_{t_{0}-1})\overline{\alpha}_{t_{0}-1}}{1-\overline{\alpha}_{t_{0}-2}}\geq\frac{\frac{c_{1}\log T}{T}\beta_{1}\min\Big\{\big(1+\frac{c_{1}\log T}{T}\big)^{t_{0}-1},1\Big\}\cdot\big(1-\frac{1}{2}\big)}{\frac{1}{2}\beta_{1}\big(1+\frac{c_{1}\log T}{T}\big)^{t_{0}-1}}=\frac{\frac{c_{1}\log T}{T}\beta_{1}\big(1+\frac{c_{1}\log T}{T}\big)^{t_{0}-1}}{\beta_{1}\big(1+\frac{c_{1}\log T}{T}\big)^{t_{0}-1}}=\frac{c_{1}\log T}{T}.
\]
		As a result, we can further derive	
		\begin{align*}
1-\overline{\alpha}_{t_{0}-1} & =1-\alpha_{t_{0}-1}\overline{\alpha}_{t_{0}-2}=1-\overline{\alpha}_{t_{0}-2}+(1-\alpha_{t_{0}-1})\overline{\alpha}_{t_{0}-2}\\
 & =\bigg(1+\frac{(1-\alpha_{t_{0}-1})\overline{\alpha}_{t_{0}-2}}{1-\overline{\alpha}_{t_{0}-2}}\bigg)(1-\overline{\alpha}_{t_{0}-2})\\
 & \ge\bigg(1+\frac{c_{1}\log T}{T}\bigg)(1-\overline{\alpha}_{t_{0}-2})\ge\bigg(1+\frac{c_{1}\log T}{T}\bigg)\cdot\bigg\{\frac{1}{3}\beta_{1}\bigg(1+\frac{c_{1}\log T}{T}\bigg)^{t_{0}-1}\bigg\}\\
 & = \frac{1}{3}\beta_{1}\bigg(1+\frac{c_{1}\log T}{T}\bigg)^{t_{0}},			
		\end{align*}
		where the penultimate line holds since \eqref{eq:induction-proof-alpha} is first violated at $t=t_0$;  
		this, however, contradicts with the definition of $t_0$. 
		Consequently, one must have $t_0>\tau$, meaning that \eqref{eq:induction-proof-alpha} holds for all $t \le \tau$. 

	\item We then turn attention to those $t$ obeying $t>\tau$. 
		In this case, it suffices to make the observation that
		\begin{equation}
			1-\overline{\alpha}_{t-1} \ge 1 - \overline{\alpha}_{\tau-1} \geq  \frac{1}{3}\beta_{1}\bigg(1+\frac{c_{1}\log T}{T}\bigg)^{\tau} 
			= \frac{ \frac{1}{3}\beta_{1}  \big(1+\frac{c_{1}\log T}{T}\big)^{\tau+1}  }{ 1+ \frac{c_{1}\log T}{T} } 
			\ge \frac{1}{4},
		\end{equation}
		where the second and the third inequalities come from \eqref{eq:induction-proof-alpha}. 
		Therefore, one obtains
		\[
			\frac{1-\alpha_{t}}{1-\overline{\alpha}_{t-1}}\leq\frac{\frac{c_{1}\log T}{T}}{1/4}\leq\frac{4c_{1}\log T}{T}.
		\]

\end{itemize}
The above arguments taken together establish property \eqref{eqn:properties-alpha-proof-1}.

In addition, it comes immediately from \eqref{eqn:properties-alpha-proof-1} that
\[
1\leq\frac{1-\overline{\alpha}_{t}}{1-\overline{\alpha}_{t-1}}=1+\frac{\overline{\alpha}_{t-1}-\overline{\alpha}_{t}}{1-\overline{\alpha}_{t-1}}=1+\frac{\overline{\alpha}_{t-1}(1-\alpha_{t})}{1-\overline{\alpha}_{t-1}}\leq1+\frac{4c_{1}\log T}{T},
\]
thereby justifying property \eqref{eqn:properties-alpha-proof-3}.



%
%

\paragraph{Proof of property~\eqref{eqn:properties-alpha-proof-alphaT}.}
Turning attention to the second claim \eqref{eqn:properties-alpha-proof-alphaT}, 
we note that for any $t$ obeying $t \ge \frac{T}{2} \gtrsim \frac{T}{\log T}$, 
one has
\[
1-\alpha_{t}=\frac{c_{1}\log T}{T}\min\bigg\{\beta_{1}\Big(1+\frac{c_{1}\log T}{T}\Big)^{t},\,1\bigg\}=\frac{c_{1}\log T}{T}.
\]
This in turn allows one to deduce that
\begin{align*}
\overline{\alpha}_{T} \le \prod_{t: t \ge T/2} \alpha_t \le \Big(1-\frac{c_{1}\log T}{T}\Big)^{T/2} \le \frac{1}{T^{c_2}}
\end{align*}
for an arbitrarily large constant $c_2>0$.

\paragraph{Proof of property~\eqref{eq:expansion-ratio-1-alpha}.}
Finally, it is easily seen from the Taylor expansion that the learning rates $\{\alpha_t\}$ satisfy
\begin{align*}
\Big(\frac{1-\overline{\alpha}_{t}}{\alpha_{t}-\overline{\alpha}_{t}}\Big)^{d/2} & =\bigg(1+\frac{1-\alpha_{t}}{\alpha_{t}-\overline{\alpha}_{t}}\bigg)^{d/2}\nonumber \\
 & =1+\frac{d(1-\alpha_{t})}{2(\alpha_{t}-\overline{\alpha}_{t})}+\frac{d(d-2)(1-\alpha_{t})^{2}}{8(\alpha_{t}-\overline{\alpha}_{t})^{2}}+O\bigg(d^{3}\Big(\frac{1-\alpha_{t}}{\alpha_{t}-\overline{\alpha}_{t}}\Big)^{3}\bigg),
\end{align*}
provided that $\frac{d(1-\alpha_{t})}{\alpha_{t}-\overline{\alpha}_{t}}\lesssim 1$.

\subsection{Proof of Lemma~\ref{lem:x0}}
\label{sec:proof-lem:x0}


To establish this lemma, we first make the following claim, whose proof is deferred to the end of this subsection.  
\begin{claim}
\label{eq:claim-123}
Consider any $c_5 \geq 2$ and suppose that $c_6\geq 2c_R$. There exists some $x_0 \in \real^d$ such that
\begin{subequations}
\label{eqn:lemma-x0-claim}
\begin{align}
\label{eqn:lemma-x0-claim-1}
	\|\sqrt{\overline{\alpha}_{t}}x_0 - y\|_2 &\leq c_5 \sqrt{\theta(y) d(1-\overline{\alpha}_{t})\log T} 
	\qquad\quad \text{and} \\
	\mathbb{P}\big(\|X_0 - x_0\|_2 \leq \epsilon\big) &\ge \Big( \frac{\epsilon}{T^{2\theta(y)}} \Big)^d 
	\qquad\quad \text{with} \quad \epsilon = \frac{1}{T^{c_0/2}}
	\label{eqn:lemma-x0-claim-2}
\end{align}
\end{subequations}
hold simultaneously, where $c_0$ is defined in \eqref{eqn:alpha-t}.
\end{claim}
With the above claim in place, we are ready to prove Lemma~\ref{lem:x0}. 
For notational simplicity, we let $X$ represent a random vector whose distribution $p_X(\cdot)$ obeys 
\begin{align}
	p_X(x) = p_{X_0\mid X_t}( x \mymid  y). 
	\label{eq:pX-properties}
\end{align}

Consider the point $x_0$ in Claim~\ref{eq:claim-123}, and let us look at a set:
\begin{align*}
\mathcal{E} \coloneqq \Big\{x : \sqrt{\overline{\alpha}_{t}}\|x - x_0\|_2 > 4c_5 \sqrt{\theta(y) d(1-\overline{\alpha}_{t})\log T}\Big\},
\end{align*}
where $c_5\geq 2$ (see Claim~\ref{eq:claim-123}). 
Combining this with property \eqref{eqn:lemma-x0-claim-1} about $x_0$ results in
\begin{align}
	\mathbb{P}\Big(\ltwo{\sqrt{\overline{\alpha}_{t}}X-y}>5 c_5\sqrt{\theta(y) d(1-\overline{\alpha}_{t})\log T}\Big) & \leq\mathbb{P}(X\in\mathcal{E}). \label{eq:UB-P-set-E}
\end{align}
Consequently, everything boils down to bounding $\mathbb{P}(X\in\mathcal{E})$. 
Towards this, we first invoke the Bayes rule $p_{X_0 \mymid X_t}(x \mymid y) \propto p_{X_0}(x)p_{X_t \mymid X_0}(y \mymid x) $ to derive
%
%
\begin{align}
\mathbb{P}(X_{0}\in\mathcal{E}\mymid X_{t}=y) & =\frac{\int_{x\in\mathcal{E}}p_{X_{0}}(x)p_{X_{t}\mymid X_{0}}(y\mymid x)\diff x}{\int_{x}p_{X_{0}}(x)p_{X_{t}\mymid X_{0}}(y\mymid x)\diff x} \notag\\
 & \le\frac{\int_{x\in\mathcal{E}}p_{X_{0}}(x)p_{X_{t}\mymid X_{0}}(y\mymid x)\diff x}{\int_{x:\|x-x_{0}\|_{2}\leq\epsilon}p_{X_{0}}(x)p_{X_{t}\mymid X_{0}}(y\mymid x)\diff x} \notag\\
 & \le\frac{\sup_{x\in\mathcal{E}}p_{X_{t}\mymid X_{0}}(y\mymid x)}{\inf_{x:\|x-x_{0}\|_{2}\leq\epsilon}p_{X_{t}\mymid X_{0}}(y\mymid x)}\cdot\frac{\mathbb{P}(X_{0}\in\mathcal{E})}{\mathbb{P}(\|X_{0}-x_{0}\|_{2}\leq\epsilon)}. 
	\label{eq:UB1-P-X0-Xt}
\end{align}
To further bound this quantity, 
note that: in view of the definition of $\mathcal{E}$ and 
expression~\eqref{eqn:lemma-x0-claim-1}, one has
\begin{align*}
	\sup_{x\in\mathcal{E}}p_{X_{t}\mymid X_{0}}(y\mymid x) & =\sup_{x:\|\sqrt{\overline{\alpha}_{t}}x-\sqrt{\overline{\alpha}_{t}}x_{0}\|_{2}>4c_5\sqrt{\theta(y) d(1-\overline{\alpha}_{t})\log T}}p_{X_{t}\mymid X_{0}}(y\mymid x)\\
 & \leq\sup_{x:\|\sqrt{\overline{\alpha}_{t}}x-y\|_{2}>3c_5\sqrt{\theta(y) d(1-\overline{\alpha}_{t})\log T}}p_{X_{t}\mymid X_{0}}(y\mymid x)\\
 & \leq\frac{1}{\big(2\pi(1-\overline{\alpha}_{t})\big)^{d/2}}\exp\bigg(-\frac{9c_5^2\theta(y) d\log T}{2}\bigg)
\end{align*}
and
\begin{align*}
\inf_{x:\|x-x_{0}\|_{2}\leq\epsilon}p_{X_{t}\mymid X_{0}}(y\mymid x) & \geq\frac{1}{\big(2\pi(1-\overline{\alpha}_{t})\big)^{d/2}}\inf_{x:\|x-x_{0}\|_{2}\leq\epsilon}\exp\bigg(-\frac{\ltwo{y-\sqrt{\overline{\alpha}_{t}}x}^{2}}{2(1-\overline{\alpha}_{t})}\bigg)\\
 & \geq\frac{1}{\big(2\pi(1-\overline{\alpha}_{t})\big)^{d/2}}\inf_{x:\|x-x_{0}\|_{2}\leq\epsilon}\exp\bigg(-\frac{\ltwo{y-\sqrt{\overline{\alpha}_{t}}x_{0}}^{2}}{1-\overline{\alpha}_{t}}-\frac{\ltwo{\sqrt{\overline{\alpha}_{t}}x-\sqrt{\overline{\alpha}_{t}}x_{0}}^{2}}{1-\overline{\alpha}_{t}}\bigg)\\
 & \geq\frac{1}{\big(2\pi(1-\overline{\alpha}_{t})\big)^{d/2}}\exp\bigg(-\frac{\ltwo{y-\sqrt{\overline{\alpha}_{t}}x_{0}}^{2}}{1-\overline{\alpha}_{t}}-\frac{\epsilon^{2}}{1-\overline{\alpha}_{t}}\bigg)\\
 & \geq\frac{1}{\big(2\pi(1-\overline{\alpha}_{t})\big)^{d/2}}\exp\bigg(-c_5^2\theta(y) d\log T-\frac{1}{T^{c_{0}}}\frac{1}{1-\overline{\alpha}_{t}}\bigg)\\
 & \geq\frac{1}{\big(2\pi(1-\overline{\alpha}_{t})\big)^{d/2}}\exp\big(-2c_5^2\theta(y) d\log T\big),
\end{align*}
where the second line is due to the elementary inequality $\|a+b\|_2^2 \leq 2\|a\|_2^2+2\|b\|_2^2$,  
the penultimate line relies on \eqref{eqn:lemma-x0-claim}, 
and the last line holds true since $1-\overline{\alpha}_{t}\geq1-\alpha_{1}=1/T^{c_{0}}$ (see \eqref{eqn:alpha-t}).  
Substitution of the above two displays into \eqref{eq:UB1-P-X0-Xt}, we arrive at
\begin{align}
	\mathbb{P}(X_{0}\in\mathcal{E}\mymid X_{t}=y) & \le \exp\big(-2.5c_5^2\theta(y) d\log T \big)\cdot\frac{1}{\mathbb{P}(\|X_{0}-x_{0}\|_{2}\leq\epsilon)} \notag\\
 & \le\exp\big(-2.5c_5^2\theta d\log T \big)\cdot\left(T^{2\theta(y) +c_{0}/2}\right)^{d} \notag\\
 & \le\exp\big(-(2.5c_5^2\theta(y) -2\theta(y)-c_{0}/2)d\log T\big), 
\end{align}
where the second inequality invokes \eqref{eqn:lemma-x0-claim-2}. 
Substituting this into \eqref{eq:UB-P-set-E} and recalling the distribution \eqref{eq:pX-properties} of $X$, 
we arrive at
\begin{align*}
	\mathbb{P}\Big(\ltwo{\sqrt{\overline{\alpha}_{t}}X-y}>5 c_5\sqrt{\theta(y)d(1-\overline{\alpha}_{t})\log T}\Big) 
	& \leq \exp\big(-(2.5c_5^2\theta(y)-2\theta(y)-c_{0}/2)d\log T\big)\\
	& \leq \exp\big(- c_5^2\theta(y)d\log T\big),
\end{align*}
with the proviso that $c_5\geq 2$ and $c_6\geq c_0$ (so that $\theta(y)\geq c_6\geq c_0$). 
This concludes the proof of the advertised result \eqref{eq:P-xt-X0-124} when $c_5 \ge 2$ and $c_6\geq 2c_R+c_0$, 
as long as Claim~\ref{eq:claim-123} can be justified.

With the above result in place, it then follows that
\begin{align*}
&\mathbb{E}\left[\big\| x_{t}-\sqrt{\overline{\alpha}_{t}}X_{0}\big\|_{2}\,\big|\,X_{t}=x_{t}\right] \\
	&\quad \leq5c_5\sqrt{\theta(y) d(1-\overline{\alpha}_{t})\log T}+\mathbb{E}\left[\big\| x_{t}-\sqrt{\overline{\alpha}_{t}}X_{0}\big\|_{2}\ind\big\{\|x_{t}-\sqrt{\overline{\alpha}_{t}}X_{0}\|_{2}\geq5c_5\sqrt{\theta(y) d(1-\overline{\alpha}_{t})\log T}\big\}\,\Big|\,X_{t}=x_{t}\right]\\
 &\quad \leq5c_5\sqrt{\theta(y) d(1-\overline{\alpha}_{t})\log T}+\int_{5c_5\sqrt{\theta(y) d(1-\overline{\alpha}_{t})\log T}}^{\infty}\mathbb{P}\big(\|x_{t}-\sqrt{\overline{\alpha}_{t}}x_{0}\|_{2}\geq\tau\mymid X_{t}=x_{t}\big)\mathrm{d}\tau\\
 &\quad \leq5c_5\sqrt{\theta(y) d(1-\overline{\alpha}_{t})\log T}+\int_{5c_5\sqrt{\theta(y) d(1-\overline{\alpha}_{t})\log T}}^{\infty}\exp\bigg(-\frac{\tau^{2}}{25(1-\overline{\alpha}_{t})}\bigg)\mathrm{d}\tau\\
 &\quad \leq5c_5\sqrt{\theta(y) d(1-\overline{\alpha}_{t})\log T}+\exp\big(-c_5^{2} \theta(y) d\log T\big)\\
 &\quad \leq6c_5\sqrt{\theta(y) d(1-\overline{\alpha}_{t})\log T},
\end{align*}
as claimed in \eqref{eq:E-xt-X0} by taking $c_5 = 2$. 
The proofs for \eqref{eq:E2-xt-X0}, \eqref{eq:E3-xt-X0} and \eqref{eq:E4-xt-X0} follow from similar aguments and are hence omitted for the sake of brevity.


\paragraph{Proof of Claim~\ref{eq:claim-123}.}
We prove this claim by contradiction. 
Specifically, suppose instead that: 
for every $x$ obeying $\|\sqrt{\overline{\alpha}_{t}}x - y\|_2 \leq c_5 \sqrt{\theta(y) d(1-\overline{\alpha}_{t})\log T },$ we have
\begin{equation}
	\mathbb{P}(\|X_0 - x\|_2 \le \epsilon) \leq \bigg( \frac{\epsilon}{2T^{\theta(y)}R} \bigg)^d
	\qquad \text{with } 
	\epsilon = \frac{1}{T^{c_0/2}}.
	\label{eq:contradition-x-y}
\end{equation}
Clearly, the choice of $\epsilon$ ensures that $\epsilon< \frac{1}{2}\sqrt{d(1-\overline{\alpha}_t)\log T}$.  
In the following, we would like to show that this assumption leads to contradiction.

First of all, let us look at $p_{X_{t}}$, which obeys
\begin{align}
\notag p_{X_{t}}(y) & =\int_{x}p_{X_{0}}(x)p_{X_{t}\mymid X_{0}}(y\mymid x)\mathrm{d}x\\
\notag & =\int_{x:\,\|\sqrt{\overline{\alpha}_{t}}x-y\|_{2}\geq c_5\sqrt{\theta(y) d(1-\overline{\alpha}_{t})\log T}}p_{X_{0}}(x)p_{X_{t}\mymid X_{0}}(y\mymid x)\mathrm{d}x\\
 & \qquad+\int_{x:\,\|\sqrt{\overline{\alpha}_{t}}x-y\|_{2}< c_5\sqrt{\theta(y) d(1-\overline{\alpha}_{t})\log T}}p_{X_{0}}(x)p_{X_{t}\mymid X_{0}}(y\mymid x)\mathrm{d}x.
%
	\label{eqn:contradiction-12}
\end{align} 
To further control \eqref{eqn:contradiction-12}, 
we make two observations:
\begin{itemize}
	\item[1)] The first term on the right-hand side of \eqref{eqn:contradiction-12} can be bounded by
\begin{align}
	& \notag\int_{x:\,\|\sqrt{\overline{\alpha}_{t}}x-y\|_{2}\geq c_5\sqrt{\theta(y)d(1-\overline{\alpha}_{t})\log T}}p_{X_{0}}(x)p_{X_{t}\mymid X_{0}}(y\mymid x)\mathrm{d}x\\
 & \qquad\le \sup_{z:\,\|z\|_{2}\geq c_5\sqrt{\theta(y) d(1-\overline{\alpha}_{t})\log T}}\frac{1}{\big(2\pi(1-\overline{\alpha}_{t})\big)^{d/2}}\exp\bigg(-\frac{\ltwo{z}^{2}}{2(1-\overline{\alpha}_{t})}\bigg) \notag\\
	& \qquad < \frac{1}{2} \exp\big(-\theta(y) d\log T\big),
\end{align}
provided that $c_5 \ge 2$ and $c_6 > 0$ is large enough (note that $\theta(y) \ge c_6$). 
Here, we have used $X_t\overset{\mathrm{(i)}}{=}\sqrt{\overline{\alpha}_t}X_0+\sqrt{1-\overline{\alpha}_t}W$ with $W\sim \mathcal{N}(0,I_d)$
as well as standard properties about Gaussian distributions.

	\item[2)] Regarding the second term on the right-hand side of \eqref{eqn:contradiction-12}, 
		let us construct an epsilon-net 
		$\mathcal{N}_{\epsilon} = \{z_i\}$ for the following set
		$$
			\big\{x : \|\sqrt{\overline{\alpha}_{t}}x - y\|_2 \leq c_5 \sqrt{\theta(y) d(1-\overline{\alpha}_{t})\log T} \text{ and } \ltwo{x} \leq R \big\},
		$$
		so that for each $x$ in this set, one can find a vector $z_i\in  \mathcal{N}_{\epsilon}$ such that $\|x - z_i\|_2\leq \epsilon$. 
		Clearly, we can choose $\mathcal{N}_{\epsilon}$ so that its cardinality obeys $|\mathcal{N}_{\epsilon}|\leq (2R/\epsilon)^d$.  
		Define $\mathcal{B}_i \defn \{x \mid \ltwo{x - z_i}\leq \epsilon\}$ for each $z_i\in \mathcal{N}_{\epsilon}$.  
		Armed with these sets, we can derive
\begin{align*}
	\int_{x:\|\sqrt{\overline{\alpha}_{t}}x-y\|_{2}<c_5\sqrt{\theta(y)d(1-\overline{\alpha}_{t})\log T}}p_{X_{0}}(x)p_{X_{t}\mymid X_{0}}(y\mymid x) \mathrm{d}x
	&\le \big(2\pi(1-\overline{\alpha}_{t})\big)^{-d/2}\sum_{i=1}^{|\mathcal{N}_{\epsilon}|}\mathbb{P}(X_{0}\in\mathcal{B}_{i}) \\
	& \le \big(2\pi(1-\overline{\alpha}_{t})\big)^{-d/2}\bigg(\frac{\epsilon}{2T^{2\theta(y)}R}\bigg)^{d}\bigg(\frac{ 2R}{\epsilon}\bigg)^{d} \\
	& < \frac{1}{2} \exp\big(-\theta(y) d\log T \big), 
\end{align*}
where the penultimate step comes from the assumption \eqref{eq:contradition-x-y}. 

\end{itemize}
The above results taken collectively lead to 
\begin{align}
	 p_{X_{t}}(y)  & < \exp\big(-\theta(y) d\log T\big),\label{eqn:contradiction}
\end{align} 
thus contradicting the definition \eqref{eqn:choice-y-prelim} of $\theta(y)$.

Consequently, we have proven the existence of $x$ obeying 
$\|\sqrt{\overline{\alpha}_{t}}x - y\|_2 \leq c_5 \sqrt{\theta(y) d(1-\overline{\alpha}_{t})\log T }$ and 
\begin{equation*}
	\mathbb{P}(\|X_0 - x\|_2 \le \epsilon) > \bigg( \frac{\epsilon}{2T^{\theta(y)}R} \bigg)^d 
	\geq \bigg( \frac{\epsilon}{T^{2\theta(y)}} \bigg)^d,
\end{equation*}
provided that $\theta(y)\geq c_6 \geq 2c_R$. 
This completes the proof of Claim~\ref{eq:claim-123}. 
%


\subsection{Proof of Lemma~\ref{lem:river}}
\label{sec:proof-lem:river}
For notational convenience, let us denote 
$\widehat{x}_t= x_t / \sqrt{\alpha_t}$ throughout the proof. 
As a key step of the proof, we note that for any $x\in \mathbb{R}^d$,  
\begin{align}
\notag p_{X_{t-1}}(x) & =\int_{x_{0}}p_{X_{0}}(x_{0})p_{X_{t-1}\mymid X_{0}}(x\mymid x_{0})\mathrm{d}x_{0}\\
	\notag & = \int_{x_{0}} p_{X_{0}}(x_{0})p_{X_{t}\mymid X_{0}}(x_{t}\mymid x_{0})\cdot\frac{p_{X_{t-1}\mymid X_{0}}(x\mymid x_{0})}{p_{X_{t}\mymid X_{0}}(x_{t}\mymid x_{0})}\mathrm{d}x_{0}\\
 & = p_{X_{t}}(x_{t})\int_{x_{0}} p_{X_{0}\mymid X_{t}}(x_{0}\mymid x_{t})\cdot\frac{p_{X_{t-1}\mymid X_{0}}(x\mymid x_{0})}{p_{X_{t}\mymid X_{0}}(x_{t}\mymid x_{0})}\mathrm{d}x_{0}, 
	\label{eqn:rachmaninoff-mid}
\end{align}
thus establishing a link between $p_{X_{t-1}}(x)$ and $p_{X_{t}}(x_{t})$. 
Consequently, in order to control $p_{X_{t-1}}(x)$, 
it is helpful to first look at the density ratio $\frac{p_{X_{t-1}\mymid X_{0}}(x\mymid x_{0})}{p_{X_{t}\mymid X_{0}}(x_{t}\mymid x_{0})}$, 
which we accomplish in the sequel.

Recall that $X_t \overset{\mathrm{d}}{=} \sqrt{\overline{\alpha}}\, X_0 + \sqrt{1-\overline{\alpha}}\, W $ with $W\sim \mathcal{N}(0,I_d)$. 
In what follows, let us consider any $x_t, x_0\in \mathbb{R}^d$ and any $x$ obeying 
\begin{equation}
	\|\widehat{x}_t - x\|_2 \leq c_3 \sqrt{d(1-\alpha_t)\log T}
	\label{eq:defn-x-constraint-river}
\end{equation}
for some constant $c_3>0$.
The density ratio of interest satisfies
\begin{align}
 & \frac{p_{X_{t-1}\mymid X_{0}}(x\mymid x_{0})}{p_{X_{t}\mymid X_{0}}(x_{t}\mymid x_{0})} =\bigg(\frac{1-\overline{\alpha}_{t}}{1-\overline{\alpha}_{t-1}}\bigg)^{d/2}\exp\Bigg(\frac{\big\| x_{t}-\sqrt{\overline{\alpha}_{t}}x_{0}\big\|_{2}^{2}}{2(1-\overline{\alpha}_{t})}-\frac{\big\| x-\sqrt{\overline{\alpha}_{t-1}}x_{0}\big\|_{2}^{2}}{2(1-\overline{\alpha}_{t-1})}\Bigg) \notag\\
 & \qquad \leq\exp\Bigg(\frac{d(1-\alpha_{t})+\big\|\widehat{x}_{t}-x\big\|_{2}^{2}+2\big\|\widehat{x}_{t}-x\big\|_{2}\big\| x_{t}-\sqrt{\overline{\alpha}_{t}}x_{0}\big\|_{2}}{2(1-\overline{\alpha}_{t-1})}+\frac{(1-\alpha_{t})\big\| x_{t}-\sqrt{\overline{\alpha}_{t}}x_{0}\big\|_{2}^{2}}{(1-\overline{\alpha}_{t-1})^{2}}\Bigg),
	\label{eq:density-ratio-eqn-river}
\end{align}
where the last inequality follows from the two relations below: 
\[
\log\frac{1-\overline{\alpha}_{t}}{1-\overline{\alpha}_{t-1}}=\log\left(1+\frac{\overline{\alpha}_{t-1}(1-\alpha_{t})}{1-\overline{\alpha}_{t-1}}\right)\leq\frac{1-\alpha_{t}}{1-\overline{\alpha}_{t-1}}
\]
\begin{align*}
\text{and}\qquad & \left|\frac{\big\| x_{t}-\sqrt{\overline{\alpha}_{t}}x_{0}\big\|_{2}^{2}}{2(1-\overline{\alpha}_{t})}-\frac{\big\| x-\sqrt{\overline{\alpha}_{t-1}}x_{0}\big\|_{2}^{2}}{2(1-\overline{\alpha}_{t-1})}\right|\\
 & \qquad=\left|\frac{\big\| x_{t}-\sqrt{\overline{\alpha}_{t}}x_{0}\big\|_{2}^{2}}{2(1-\overline{\alpha}_{t})}-\frac{\big\|\widehat{x}_{t}-x\big\|_{2}^{2}+\big\|\widehat{x}_{t}-\sqrt{\overline{\alpha}_{t-1}}x_{0}\big\|_{2}^{2}-2\big\langle\widehat{x}_{t}-x,\widehat{x}_{t}-\sqrt{\overline{\alpha}_{t-1}}x_{0}\big\rangle}{2(1-\overline{\alpha}_{t-1})}\right|\\
 & \qquad\leq\frac{\big\|\widehat{x}_{t}-x\big\|_{2}^{2}}{2(1-\overline{\alpha}_{t-1})}+\frac{\big\|\widehat{x}_{t}-x\big\|_{2}\big\|\widehat{x}_{t}-\sqrt{\overline{\alpha}_{t-1}}x_{0}\big\|_{2}}{1-\overline{\alpha}_{t-1}}+\frac{(1-\alpha_{t})\big\| x_{t}-\sqrt{\overline{\alpha}_{t}}x_{0}\big\|_{2}^{2}}{2\alpha_{t}(1-\overline{\alpha}_{t})(1-\overline{\alpha}_{t-1})}\\
 & \qquad\leq\frac{\big\|\widehat{x}_{t}-x\big\|_{2}^{2}+2\big\|\widehat{x}_{t}-x\big\|_{2}\big\|\widehat{x}_{t}-\sqrt{\overline{\alpha}_{t-1}}x_{0}\big\|_{2}}{2(1-\overline{\alpha}_{t-1})}+\frac{(1-\alpha_{t})\big\| x_{t}-\sqrt{\overline{\alpha}_{t}}x_{0}\big\|_{2}^{2}}{(1-\overline{\alpha}_{t-1})^{2}}. 
\end{align*}
Next, let us define the following set given $x_t$: 
\begin{align}
\widetilde{\mathcal{E}} \coloneqq \Big\{x_0 : \big\|x_t - \sqrt{\overline{\alpha}_t}x_0 \big\|_2 \leq c_4 \sqrt{d(1 - \overline{\alpha}_t)\log T} \Big\}
	\label{eq:defn-E-river}
\end{align}
for some large enough numerical constant $c_4>0$, and we shall look at $\widetilde{\mathcal{E}}$ and $\widetilde{\mathcal{E}}^{\mathrm{c}}$ separately. 
Towards this, we make the following observations:

\begin{itemize}
	\item For any $x_0 \in \widetilde{\mathcal{E}}$, one can utilize \eqref{eq:defn-E-river} and \eqref{eq:defn-x-constraint-river} to deduce that 
\begin{align}
 & \frac{d(1-\alpha_{t})+\big\|\widehat{x}_{t}-x\big\|_{2}^{2}+2\big\|\widehat{x}_{t}-x\big\|_{2}\big\| x_{t}-\sqrt{\overline{\alpha}_{t}}x_{0}\big\|_{2}}{2(1-\overline{\alpha}_{t-1})}+\frac{(1-\alpha_{t})\big\| x_{t}-\sqrt{\overline{\alpha}_{t}}x_{0}\big\|_{2}^{2}}{(1-\overline{\alpha}_{t-1})^{2}} \notag\\
 & \qquad\lesssim\frac{d(1-\alpha_{t})}{1-\overline{\alpha}_{t-1}}+\frac{d(1-\alpha_{t})\log T+d\sqrt{(1-\overline{\alpha}_{t})(1-\alpha_{t})}\log T}{1-\overline{\alpha}_{t-1}}+\frac{(1-\alpha_{t})^{2}d(1-\overline{\alpha}_{t})\log T}{(1-\overline{\alpha}_{t-1})^{2}} \notag\\
 & \qquad\lesssim(d\log T)\left\{ \frac{1-\alpha_{t}}{1-\overline{\alpha}_{t-1}}+\sqrt{\frac{1-\alpha_{t}}{1-\overline{\alpha}_{t-1}}}\sqrt{\frac{1-\overline{\alpha}_{t}}{1-\overline{\alpha}_{t-1}}}+\left(\frac{1-\alpha_{t}}{1-\overline{\alpha}_{t-1}}\right)^{2}\right\} \notag\\
 & \qquad\lesssim d\sqrt{\frac{1-\alpha_{t}}{1-\overline{\alpha}_{t-1}}}\log T,
	\label{eq:ratio-d-sum-x-123}
\end{align}
where the last inequality makes use of the facts $\frac{1-\alpha_{t}}{1-\overline{\alpha}_{t-1}}\leq1$ (cf.~\eqref{eqn:properties-alpha-proof-1}) and
\begin{equation}
\frac{1-\overline{\alpha}_{t}}{1-\overline{\alpha}_{t-1}}=\frac{1-\alpha_{t}}{1-\overline{\alpha}_{t-1}}+\frac{\alpha_{t}-\overline{\alpha}_{t}}{1-\overline{\alpha}_{t-1}}
	= \frac{1-\alpha_{t}}{1-\overline{\alpha}_{t-1}}+\alpha_{t}\leq2.
	\label{eq:ratio-alpha-t-t1}
\end{equation}
Moreover, the properties \eqref{eqn:properties-alpha-proof} of the stepsizes tell us that
\[
d\sqrt{\frac{1-\alpha_{t}}{1-\overline{\alpha}_{t-1}}}\log T\lesssim d\sqrt{\frac{\log^3 T}{T}}\leq c_{10}
\]
for some small enough constant $c_{10}>0$, as long as $T\geq c_{11}d^{2}\log^3 T$ for some sufficiently large constant $c_{11}>0$. 
Taking this together with \eqref{eq:density-ratio-eqn-river} and \eqref{eq:ratio-d-sum-x-123} reveals that 
\begin{align}
\label{eqn:scriabin}
\frac{p_{X_{t-1} \mymid X_0}(x \mymid x_0)}{p_{X_{t} \mymid X_0}(x_t \mymid x_0)} &= 1 + O\bigg(d\sqrt{\frac{1-\alpha_t}{1-\overline{\alpha}_{t-1}}}\log T\bigg),
\end{align}
with the proviso that \eqref{eq:defn-x-constraint-river} holds and $x_0 \in \widetilde{\mathcal{E}}$.

	\item Instead, if $x_0 \notin \widetilde{\mathcal{E}}$, then one can obtain 
\begin{align*}
 & \frac{d(1-\alpha_{t})+\big\|\widehat{x}_{t}-x\big\|_{2}^{2}+2\big\|\widehat{x}_{t}-x\big\|_{2}\big\| x_{t}-\sqrt{\overline{\alpha}_{t}}x_{0}\big\|_{2}}{2(1-\overline{\alpha}_{t-1})}+\frac{(1-\alpha_{t})\big\| x_{t}-\sqrt{\overline{\alpha}_{t}}x_{0}\big\|_{2}^{2}}{(1-\overline{\alpha}_{t-1})^{2}}\\
 & \quad\overset{(\mathrm{i})}{\leq}\frac{d(1-\alpha_{t})+\big(1+\frac{1-\overline{\alpha}_{t-1}}{1-\alpha_{t}}\big)\big\|\widehat{x}_{t}-x\big\|_{2}^{2}+\frac{1-\alpha_{t}}{1-\overline{\alpha}_{t-1}}\big\| x_{t}-\sqrt{\overline{\alpha}_{t}}x_{0}\big\|_{2}^{2}}{2(1-\overline{\alpha}_{t-1})}+\frac{(1-\alpha_{t})\big\| x_{t}-\sqrt{\overline{\alpha}_{t}}x_{0}\big\|_{2}^{2}}{(1-\overline{\alpha}_{t-1})^{2}}\\
 & \quad\overset{(\mathrm{ii})}{\leq}\frac{d(1-\alpha_{t})+c_{3}\big(1+\frac{1-\overline{\alpha}_{t-1}}{1-\alpha_{t}}\big)d(1-\alpha_{t})\log T}{2(1-\overline{\alpha}_{t-1})}+\frac{2(1-\alpha_{t})\big\| x_{t}-\sqrt{\overline{\alpha}_{t}}x_{0}\big\|_{2}^{2}}{(1-\overline{\alpha}_{t-1})^{2}}\\
 & \quad\overset{(\mathrm{iii})}{\leq}\frac{c_{3}d\big(1-\alpha_{t}+1-\overline{\alpha}_{t-1}\big)\log T}{1-\overline{\alpha}_{t-1}}\cdot\frac{\big\| x_{t}-\sqrt{\overline{\alpha}_{t}}x_{0}\big\|_{2}^{2}}{c_{4}d(1-\overline{\alpha}_{t-1})\log T}+\frac{2(1-\alpha_{t})\big\| x_{t}-\sqrt{\overline{\alpha}_{t}}x_{0}\big\|_{2}^{2}}{(1-\overline{\alpha}_{t-1})^{2}}\\
 & \quad\leq\frac{2c_{3}}{c_{4}}\cdot\frac{\big\| x_{t}-\sqrt{\overline{\alpha}_{t}}x_{0}\big\|_{2}^{2}}{1-\overline{\alpha}_{t-1}}+\frac{8c_{1}\big\| x_{t}-\sqrt{\overline{\alpha}_{t}}x_{0}\big\|_{2}^{2}\log T}{T(1-\overline{\alpha}_{t-1})}\\
 & \quad\overset{(\mathrm{iv})}{\leq}\left(\frac{2c_{3}}{c_{4}}+\frac{8c_{1}\log T}{T}\right)\frac{2\big\| x_{t}-\sqrt{\overline{\alpha}_{t}}x_{0}\big\|_{2}^{2}}{1-\overline{\alpha}_{t}}\leq\frac{8c_{3}}{c_{4}}\frac{\big\| x_{t}-\sqrt{\overline{\alpha}_{t}}x_{0}\big\|_{2}^{2}}{1-\overline{\alpha}_{t}}, 
\end{align*}
where (i) comes from the Cauchy-Schwarz inequality, (ii) is valid
due to the assumption on $\|\widehat{x}_{t}-x\|_{2}$, (iii) follows
from the definition of $\widetilde{\mathcal{E}}$, and (iv) is a consequence
of \eqref{eq:ratio-alpha-t-t1}. 		
Substitution into \eqref{eq:density-ratio-eqn-river} leads to
\begin{align}
\frac{p_{X_{t-1}\mymid X_{0}}(x\mymid x_{0})}{p_{X_{t}\mymid X_{0}}(x_{t}\mymid x_{0})}\leq \exp\Bigg(\frac{8c_{3}}{c_{4}}\frac{\big\| x_{t}-\sqrt{\overline{\alpha}_{t}}x_{0}\big\|_{2}^{2}}{1-\overline{\alpha}_{t}}\Bigg), 
	\label{eq:ratio-swap-large}
\end{align}
provided that \eqref{eq:defn-x-constraint-river} holds and $x_0 \notin \widetilde{\mathcal{E}}$. 
\end{itemize}
In light of the above calculations, we can invoke \eqref{eqn:rachmaninoff} to demonstrate that: for any $x$ obeying \eqref{eq:defn-x-constraint-river}, 
\begin{align}
\notag p_{X_{t-1}}(x) 
%
%
	& =p_{X_{t}}(x_{t})\left(\int_{x_{0}\in\widetilde{\mathcal{E}}}+\int_{x_{0}\notin\widetilde{\mathcal{E}}}\right)p_{X_{0}\mymid X_{t}}(x_{0}\mymid x_{t})\cdot\frac{p_{X_{t-1}\mymid X_{0}}(x\mymid x_{0})}{p_{X_{t}\mymid X_{0}}(x_{t}\mymid x_{0})}\mathrm{d}x_{0}\\
\notag & =p_{X_{t}}(x_{t})\int_{x_{0}\in\widetilde{\mathcal{E}}}\Bigg(1+O\bigg(d\sqrt{\frac{1-\alpha_{t}}{1-\overline{\alpha}_{t-1}}}\log T\bigg)\Bigg)p_{X_{0}\mymid X_{t}}(x_{0}\mymid x_{t})\mathrm{d}x_{0}\\
 & \qquad+p_{X_{t}}(x_{t})\int_{x_{0}\notin\widetilde{\mathcal{E}}}
	O\left( \exp\Bigg(\frac{8c_{3}}{c_{4}}\frac{\big\| x_{t}-\sqrt{\overline{\alpha}_{t}}x_{0}\big\|_{2}^{2}}{1-\overline{\alpha}_{t}}\Bigg) \right) 
	p_{X_{0}\mymid X_{t}}(x_{0}\mymid x_{t})\mathrm{d}x_{0}.\label{eqn:rachmaninoff}
\end{align}

By virtue of Lemma~\ref{lem:x0}, if $-\log p_{X_t}(x_t) \leq \frac{1}{2}c_6 d\log T$ for some large constant $c_6>0$,  
then it holds that 
\begin{align}
\mathbb{P}\left\{ \big\|\sqrt{\overline{\alpha}_{t}}X_{0}-x_{t}\big\|_{2}\geq5c_{5}\sqrt{d(1-\overline{\alpha}_{t})\log T}\mid X_{t}=x_{t}\right\} \leq \exp(-c_{5}^{2}c_6d\log T)	
\end{align}
for any $c_5\ge 2$. 
Some elementary calculation then reveals that
\begin{align}
\int_{x_{0}\notin\widetilde{\mathcal{E}}}\exp\Bigg(\frac{8c_{3}}{c_{4}}\frac{\big\| x_{t}-\sqrt{\overline{\alpha}_{t}}x_{0}\big\|_{2}^{2}}{1-\overline{\alpha}_{t}}\Bigg)p_{X_{0}\mymid X_{t}}(x_{0}\mymid x_{t})\mathrm{d}x_{0} & \lesssim \frac{1}{T^{c_0}}
	\label{eqn:rachmaninoff-2}
\end{align}
with $c_0$ defined in \eqref{eqn:alpha-t}, 
provided that $c_4/c_3$ is sufficiently large. 
Hence, substituting it into \eqref{eqn:rachmaninoff} demonstrates that, for any $x$ satisfying \eqref{eq:defn-x-constraint-river}, 
\begin{align}
 p_{X_{t-1}}(x) & =\Bigg(1+O\bigg(d\sqrt{\frac{1-\alpha_{t}}{1-\overline{\alpha}_{t-1}}}\log T\bigg)\Bigg)\big(1-o(1)\big)p_{X_{t}}(x_{t})+O\bigg(\frac{1}{T^{c_{0}}}\bigg)p_{X_{t}}(x_{t})
	\label{eqn:prokofiev-0}\\
 & =\Bigg(1+O\bigg(d\sqrt{\frac{1-\alpha_{t}}{1-\overline{\alpha}_{t-1}}}\log T+\frac{1}{T^{c_{0}}}\bigg)\Bigg)p_{X_{t}}(x_{t})\notag\\
 & =\Bigg(1+O\bigg(d\sqrt{\frac{1-\alpha_{t}}{1-\overline{\alpha}_{t-1}}}\log T\bigg)\Bigg)p_{X_{t}}(x_{t})\label{eqn:prokofiev}\\
	& \in \left[ \frac{1}{2} p_{X_{t}}(x_{t}), \, \frac{3}{2} p_{X_{t}}(x_{t}) \right] , \label{eqn:prokofiev-2}
\end{align}
where the penultimate inequality holds since (according to \eqref{eqn:alpha-t})
\[
\sqrt{\frac{1-\alpha_{t}}{1-\overline{\alpha}_{t-1}}}\geq\sqrt{1-\alpha_{t}}=\sqrt{\beta_{t}}\geq\sqrt{\frac{c_{1}\log T}{T^{c_{0}+1}}},
\]
and the last inequality holds due to \eqref{eqn:properties-alpha-proof-1} and the condition that $T/ (d^2 \log ^3 T)$ is sufficiently large. 
In other words, \eqref{eqn:prokofiev-2} reveals that $ p_{X_{t-1}}(x)$ is sufficiently close to $p_{X_{t}}(x_{t})$.

We are now ready to establish our claim.  In view of the assumption \eqref{eq:assumption-lem:river}, we have
\[
\big\| x_{t}(\gamma)-\widehat{x}_{t}\big\|_{2}=\big\|\gamma x_{t-1}+(1-\gamma)\widehat{x}_{t}-\widehat{x}_{t}\big\|_{2}=\gamma\big\| x_{t-1}-\widehat{x}_{t}\big\|_{2}\leq c_{3}\sqrt{d(1-\alpha_{t})\log T}.
\]
Therefore, taking $x$ to be $x_{t}(\gamma)$ in \eqref{eqn:prokofiev-2} tells us that: if $-\log p_{X_t}(x_t) \leq \frac{1}{2}c_6 d\log T$, then
\[
	-\log p_{X_{t-1}}\big( x_{t}(\gamma) \big) \leq -\log p_{X_{t}}(x_{t}) + \log 2 \leq  c_6 d\log T
\]
as claimed.

\subsection{Proof of Lemma~\ref{lem:KL-T}}
\label{sec:proof-lem-KL-T}

Recognizing that $Y_T\sim \mathcal{N}(0, I_d)$ and that
$X_T \overset{\mathrm{d}}{=} \sqrt{\overline{\alpha}_T} X_0 + \sqrt{1 - \overline{\alpha}_T}\,\overline{W}_t$ with $\overline{W}_t\sim \mathcal{N}(0, I_d)$ (independent from $X_0$),  
one has 
\begin{align}
\notag\mathsf{KL}(p_{X_{T}}\parallel p_{Y_{T}}) & =\int p_{X_{T}}(x)\log\frac{p_{X_{T}}(x)}{p_{Y_{T}}(x)}\diff x\\
\notag & \stackrel{(\text{i})}{=}\int p_{X_{T}}(x)\log\frac{\int_{y:\|y\|_{2}\leq\sqrt{\overline{\alpha}_{T}}T^{c_{R}}}p_{\sqrt{\overline{\alpha}_{T}}X_{0}}(y)p_{\sqrt{1-\overline{\alpha}_{T}}\,\overline{W}_{t}}(x-y)\diff y}{p_{Y_{T}}(x)}\diff x\\
\notag & \leq\int p_{X_{T}}(x)\log\frac{\sup_{y:\|y\|_{2}\leq\sqrt{\overline{\alpha}_{T}}T^{c_{R}}}p_{\sqrt{1-\overline{\alpha}_{T}}\,\overline{W}_{t}}(x-y)}{p_{Y_{T}}(x)}\diff x\\
\notag & =\int p_{X_{T}}(x)\Bigg(-d/2\log(1-\overline{\alpha}_{T})+\sup_{y:\|y\|_{2}\leq\sqrt{\overline{\alpha}_{T}}T^{c_{R}}}\bigg(-\frac{\|x-y\|_{2}^{2}}{2(1-\overline{\alpha}_{T})}+\frac{\|x\|_{2}^{2}}{2}\Bigg)\diff x\\
\notag & \stackrel{(\text{ii})}{\leq}\int p_{X_{T}}(x)\bigg(-d/2\log(1-\overline{\alpha}_{T})+\|x\|_{2}\sup_{y:\|y\|_{2}\leq\sqrt{\overline{\alpha}_{T}}T^{c_{R}}}\frac{\|y\|_{2}}{1-\overline{\alpha}_{T}}\bigg)\diff x\\
\notag & \leq-d/2\log(1-\overline{\alpha}_{T})+\frac{\sqrt{\overline{\alpha}_{T}}T^{c_{R}}}{2(1-\overline{\alpha}_{T})}\mathbb{E}\left[\|X_{T}\|_{2}\right]\\
 & \stackrel{(\text{iii})}{\lesssim}\overline{\alpha}_{T}d+\frac{\sqrt{\overline{\alpha}_{T}}T^{c_{R}}}{2(1-\overline{\alpha}_{T})}\left(\sqrt{\overline{\alpha}_{T}}T^{c_{R}}+\sqrt{d}\right)\stackrel{(\text{iv})}{\lesssim}\frac{1}{T^{200}},\label{eqn:KL-T-123}
\end{align}
where $(\text{i})$ arises from the assumption that $\ltwo{X_0} \leq T^{c_R}$, 
(ii) applies the Cauchy-Schwarz inequality,  (iii) holds true since
\[
\mathbb{E}\left[\|X_{T}\|_{2}\right]\leq\sqrt{\overline{\alpha}_{T}}\|X_{0}\|_{2}+\mathbb{E}\left[\|\overline{W}_{t}\|_{2}\right]\leq\sqrt{\overline{\alpha}_{T}}T^{c_{R}}+\sqrt{\mathbb{E}\left[\|\overline{W}_{t}\|_{2}^{2}\right]}\leq\sqrt{\overline{\alpha}_{T}}T^{c_{R}}+\sqrt{d},
\]
and (iv) makes use of \eqref{eqn:properties-alpha-proof-alphaT} given that $c_2\geq 1000$.
The proof is thus completed by invoking the Pinsker inequality \citep[Lemma~2.5]{tsybakov2009introduction}.

\section{Proof of auxiliary lemmas for the ODE-based sampler}
\label{sec:proof-lem-auxiliary-ode}


\subsection{Proof of Lemma~\ref{lem:main-ODE}}
\label{sec:proof-lem:main-ODE}


\subsubsection{Proof of relations~\eqref{eq:xt_up} and \eqref{eq:xt}}

Recall the definition of $\phi_t$ and $\phi_t^{\star}$ in \eqref{defn:phit-x}, 
and introduce the following vector: 
\begin{align}
	u &\coloneqq x - \phi_t(x) = x - \phi_t^{\star}(x) + \phi_t^{\star}(x) - \phi_t(x)  \notag\\
	&= \frac{1-\alpha_{t}}{2(1-\overline{\alpha}_{t})} \int_{x_0} \big(x - \sqrt{\overline{\alpha}_{t}}x_0\big) p_{X_0 \mymid X_{t}}(x_0 \,|\, x) \mathrm{d} x_0 - 
	\frac{1-\alpha_{t}}{2}\big(s_t(x) - s_t^{\star}(x)\big).
	\label{eq:defn-u-Lemma-main-ODE}
\end{align}
The proof is composed of the following steps.

\paragraph{Step 1: decomposing $p_{\sqrt{\alpha_{t}}X_{t-1}}\big(\phi_{t}(x)\big)/p_{X_{t}}(x)$.} 
Recognizing that 
\begin{equation}
	X_t \overset{\mathrm{d}}{=} \sqrt{\overline{\alpha}_t}X_0+ \sqrt{1-\overline{\alpha}_t}\,W 
	\qquad \text{with }W\sim \mathcal{N}(0,I_d)
	\label{eq:Xt-dist-proof1}
\end{equation}
and making use of the Bayes rule, we can express the conditional distribution $p_{X_0 \mymid X_{t}}\big(\phi_{t}(x)\big)$ as
\begin{align}
\label{eqn:bayes}
p_{X_{0}\mymid X_{t}}(x_{0}\,|\,x)=\frac{p_{X_{0}}(x_{0})}{p_{X_{t}}(x)}p_{X_{t}\mymid X_{0}}(x\,|\,x_{0})=\frac{p_{X_{0}}(x_{0})}{p_{X_{t}}(x)}\cdot\frac{1}{\big(2\pi(1-\overline{\alpha}_{t})\big)^{d/2}}\exp\bigg(-\frac{\big\| x-\sqrt{\overline{\alpha}_{t}}x_{0}\big\|_{2}^{2}}{2(1-\overline{\alpha}_{t})}\bigg).	
\end{align}
Moreover, it follows from \eqref{eq:Xt-dist-proof1} that
\begin{equation}
	\sqrt{\alpha_t} X_{t-1} \overset{\mathrm{d}}{=} \sqrt{\alpha_t} \big(\sqrt{\overline{\alpha}_{t-1}}X_0+ \sqrt{1-\overline{\alpha}_{t-1}}\,W\big) 
	= \sqrt{\overline{\alpha}_{t}}X_0+ \sqrt{\alpha_t-\overline{\alpha}_{t}}\,W .
	\label{eq:Xt-1-rescale-expression}
\end{equation}
These taken together allow one to rewrite $p_{\sqrt{\alpha_t}X_{t-1}}$ such that: 
%
\begin{align}
\frac{p_{\sqrt{\alpha_{t}}X_{t-1}}\big(\phi_{t}(x)\big)}{p_{X_{t}}(x)} & \overset{\mathrm{(i)}}{=}\frac{1}{p_{X_{t}}(x)}\int_{x_{0}}p_{X_{0}}(x_{0})\frac{1}{\big(2\pi(\alpha_{t}-\overline{\alpha}_{t})\big)^{d/2}}\exp\bigg(-\frac{\big\|\phi_{t}(x)-\sqrt{\overline{\alpha}_{t}}x_{0}\big\|_{2}^{2}}{2(\alpha_{t}-\overline{\alpha}_{t})}\bigg)\mathrm{d}x_{0}\notag\\
 & \overset{\mathrm{(ii)}}{=}\frac{1}{p_{X_{t}}(x)}\int_{x_{0}}p_{X_{0}}(x_{0})\frac{1}{\big(2\pi(\alpha_{t}-\overline{\alpha}_{t})\big)^{d/2}}\exp\bigg(-\frac{\big\| x-\sqrt{\overline{\alpha}_{t}}x_{0}\big\|_{2}^{2}}{2(1-\overline{\alpha}_{t})}\bigg)\notag\\
 & \qquad\cdot\exp\bigg(-\frac{(1-\alpha_{t})\big\| x-\sqrt{\overline{\alpha}_{t}}x_{0}\big\|_{2}^{2}}{2(\alpha_{t}-\overline{\alpha}_{t})(1-\overline{\alpha}_{t})}-\frac{\|u\|_{2}^{2}-2u^{\top}\big(x-\sqrt{\overline{\alpha}_{t}}x_{0}\big)}{2(\alpha_{t}-\overline{\alpha}_{t})}\bigg)\mathrm{d}x_{0}\notag\\
 & \overset{\mathrm{(iii)}}{=}\Big(\frac{1-\overline{\alpha}_{t}}{\alpha_{t}-\overline{\alpha}_{t}}\Big)^{d/2}\cdot\int_{x_{0}}p_{X_{0}\mymid X_{t}}(x_{0}\mymid x)\cdot\notag\\
 & \qquad\qquad\qquad\exp\bigg(-\frac{(1-\alpha_{t})\big\| x-\sqrt{\overline{\alpha}_{t}}x_{0}\big\|_{2}^{2}}{2(\alpha_{t}-\overline{\alpha}_{t})(1-\overline{\alpha}_{t})}-\frac{\|u\|_{2}^{2}-2u^{\top}\big(x-\sqrt{\overline{\alpha}_{t}}x_{0}\big)}{2(\alpha_{t}-\overline{\alpha}_{t})}\bigg)\mathrm{d}x_{0}\notag\\
 & \overset{\mathrm{(iv)}}{=}\left\{ 1+\frac{d(1-\alpha_{t})}{2(\alpha_{t}-\overline{\alpha}_{t})}+O\bigg(d^{2}\Big(\frac{1-\alpha_{t}}{\alpha_{t}-\overline{\alpha}_{t}}\Big)^{2}\bigg)\right\} \cdot \notag\\
 & \quad\int_{x_{0}}p_{X_{0}\mymid X_{t}}(x_{0}\mymid x)\exp\bigg(-\frac{(1-\alpha_{t})\big\| x-\sqrt{\overline{\alpha}_{t}}x_{0}\big\|_{2}^{2}}{2(\alpha_{t}-\overline{\alpha}_{t})(1-\overline{\alpha}_{t})}-\frac{\|u\|_{2}^{2}-2u^{\top}\big(x-\sqrt{\overline{\alpha}_{t}}x_{0}\big)}{2(\alpha_{t}-\overline{\alpha}_{t})}\bigg)\mathrm{d}x_{0}.
	\label{eqn:fei}
\end{align}
Here,  identity (i) holds due to \eqref{eq:Xt-1-rescale-expression} 
and hence
\[
p_{\sqrt{\alpha_{t}}X_{t-1}}(x) =\int_{x_{0}}p_{X_{0}}(x_{0})p_{\sqrt{\alpha_{t}-\overline{\alpha}_{t}}W}\big(x-\sqrt{\overline{\alpha}_{t}}x_{0}\big)\mathrm{d}x_{0}; 
\]
identity (ii) follows from \eqref{eq:defn-u-Lemma-main-ODE} and elementary algebra; 
relation (iii) is a consequence of the Bayes rule \eqref{eqn:bayes}; 
and relation (iv) results from \eqref{eq:expansion-ratio-1-alpha}.

\paragraph{Step 2: controlling the integral in the decomposition~\eqref{eqn:fei}.}
In order to further control the right-hand side of expression~\eqref{eqn:fei}, we need to evaluate the integral in~\eqref{eqn:fei}. 
To this end, we make a few observations. 
\begin{itemize}
	\item To begin with,  Lemma~\ref{lem:x0} tells us that 
\begin{subequations}
\label{eqn:BBB}
\begin{align}
\label{eqn:brahms}
	\mathbb{P}\Big(\big\|\sqrt{\overline{\alpha}_{t}}X_0 - x\big\|_2 > 5c_5 \sqrt{\theta_t(x) d(1 - \overline{\alpha}_{t})\log T}\mymid X_t = x\Big) \leq \exp\big(-c_5^2\theta_t(x) d\log T\big) 
\end{align}
for any quantity $c_5 \ge 2$, provided that $c_6\geq 2c_R+c_0$. 

	\item A little algebra based on this relation allows one to bound $u$ (cf.~\eqref{eq:defn-u-Lemma-main-ODE}) as follows:  
\begin{align}
\|u\|_{2} & \le \frac{1-\alpha_{t}}{2}\varepsilon_{\score, t}(x) + \frac{1-\alpha_{t}}{2(1-\overline{\alpha}_{t})}\mathbb{E}\left[\big\| \sqrt{\overline{\alpha}_{t}}X_{0} - x \big\|_{2}\,\big|\,X_{t}=x\right] \notag\\
	& \leq \frac{1-\alpha_{t}}{2}\varepsilon_{\score, t}(x) + \frac{6(1-\alpha_{t})}{1-\overline{\alpha}_{t}} \sqrt{\theta_t(x) d(1 - \overline{\alpha}_{t})\log T},
	\label{eqn:johannes}
\end{align}
\end{subequations}
where the last inequality arises from Lemma~\ref{lem:x0}. 
\end{itemize}
Next, let us define
\begin{equation}
	\mathcal{E}_c^{\mathsf{typical}}
	\coloneqq\Big\{ x_0:\big\| x-\sqrt{\overline{\alpha}_{t}}x_{0}\big\|_{2}\leq 5c\sqrt{\theta_t(x) d(1-\overline{\alpha}_{t})\log T}\Big\}	 
\end{equation}
for any quantity $c>0$. 
Then for any $x_0 \in \mathcal{E}_c^{\mathsf{typical}}$, it is clearly seen from \eqref{eqn:BBB} and \eqref{eqn:properties-alpha-proof} that
\begin{subequations}
	\label{eq:long-UB-123}
\begin{align}
\frac{(1-\alpha_{t})\big\| x-\sqrt{\overline{\alpha}_{t}}x_{0}\big\|_{2}^{2}}{2(\alpha_{t}-\overline{\alpha}_{t})(1-\overline{\alpha}_{t})} & \leq\frac{25c^{2}}{2}\frac{(1-\alpha_{t})\theta_t(x)d\log T}{\alpha_{t}-\overline{\alpha}_{t}}\leq\frac{100c_{1}c^{2}\theta_t(x)d\log^{2}T}{T};\label{eq:long-UB-1}\\
\frac{\|u\|_{2}^{2}}{2(\alpha_{t}-\overline{\alpha}_{t})} & \leq\frac{(1-\alpha_{t})^{2}}{4(\alpha_{t}-\overline{\alpha}_{t})}\varepsilon_{\score,t}(x)^{2}+\frac{36(1-\alpha_{t})^{2}}{(1-\overline{\alpha}_{t})(\alpha_{t}-\overline{\alpha}_{t})}\theta_t(x)d\log T\label{eq:long-UB-2}\\
 & \leq\frac{2c_{1}^{2}\log^{2}T}{T^{2}}\varepsilon_{\score,t}(x)^{2}+\frac{2304c_{1}^{2}}{T^{2}}\theta_t(x)d\log^{3}T,\notag\\
\left|\frac{u^{\top}\big(x-\sqrt{\overline{\alpha}_{t}}x_{0}\big)}{\alpha_{t}-\overline{\alpha}_{t}}\right| & \leq\frac{\|u\|_{2}\big\| x-\sqrt{\overline{\alpha}_{t}}x_{0}\big\|_{2}}{\alpha_{t}-\overline{\alpha}_{t}}\nonumber\\
 & \leq\frac{5c(1-\alpha_{t})}{2(\alpha_{t}-\overline{\alpha}_{t})}\varepsilon_{\score,t}(x)\sqrt{\theta_t(x)d(1-\overline{\alpha}_{t})\log T}+\frac{30c(1-\alpha_{t})\theta_t(x)d\log T}{\alpha_{t}-\overline{\alpha}_{t}}\label{eq:long-UB-3}\\
 & \leq\frac{20cc_{1}}{T}\varepsilon_{\score,t}(x)\sqrt{\theta_t(x)d(1-\overline{\alpha}_{t})\log^{3}T}+\frac{240cc_{1}\theta_t(x)d\log^{2}T}{T}. \label{eq:long-UB-4}
\end{align}
\end{subequations}
As a consequence, for any $x_0\in \mathcal{E}_c^{\mathsf{typical}}$ for $c \ge 2$, we have seen from \eqref{eq:long-UB-4} and \eqref{eqn:properties-alpha-proof} that
\begin{align}
 & -\frac{(1-\alpha_{t})\big\| x-\sqrt{\overline{\alpha}_{t}}x_{0}\big\|_{2}^{2}}{2(\alpha_{t}-\overline{\alpha}_{t})(1-\overline{\alpha}_{t})}-\frac{\|u\|_{2}^{2}}{2(\alpha_{t}-\overline{\alpha}_{t})}+\frac{u^{\top}\big(x-\sqrt{\overline{\alpha}_{t}}x_{0}\big)}{\alpha_{t}-\overline{\alpha}_{t}} 
	\leq\frac{u^{\top}\big(x-\sqrt{\overline{\alpha}_{t}}x_{0}\big)}{\alpha_{t}-\overline{\alpha}_{t}}
	\notag\\
 &\qquad \leq \frac{5c(1-\alpha_{t})}{2(\alpha_{t}-\overline{\alpha}_{t})}\varepsilon_{\score,t}(x)\sqrt{\theta_t(x)d(1-\overline{\alpha}_{t})\log T}+\frac{30c(1-\alpha_{t})\theta_t(x)d\log T}{\alpha_{t}-\overline{\alpha}_{t}} \label{eq:multi-term-UB-45678}\\	
 & \qquad
 \le\frac{20cc_{1}}{T}\varepsilon_{\score,t}(x)\sqrt{\theta_t(x)d\log^3 T}+\frac{240cc_{1}}{T}\theta_t(x)d\log^{2}T 
	\le c\theta_t(x)d, \label{eq:multi-term-UB-456}
\end{align}
provided that 
\[
\frac{40c_{1}\varepsilon_{\score,t}(x)\log^{\frac{3}{2}}T}{T}\leq\sqrt{\theta_t(x)d}\qquad\text{and}\qquad T\geq480c_{1}\log^{2}T.
\]

\paragraph{Step 2(a): proof of relation \eqref{eq:xt_up}.}

Substituting \eqref{eq:multi-term-UB-45678} into \eqref{eqn:fei} and making use of \eqref{eqn:properties-alpha-proof} under our assumption on $T$ yield
\begin{align*}
\frac{p_{\sqrt{\alpha_{t}}X_{t-1}}\big(\phi_{t}(x)\big)}{p_{X_{t}}(x)} & \leq2\exp\bigg(\frac{5c(1-\alpha_{t})}{2(\alpha_{t}-\overline{\alpha}_{t})}\varepsilon_{\score,t}(x)\sqrt{\theta_t(x)d\log T}+\frac{30c(1-\alpha_{t})}{\alpha_{t}-\overline{\alpha}_{t}}\theta_t(x)d\log T\bigg)\int_{x_{0}}p_{X_{0}\mymid X_{t}}(x_{0}\mymid x)\mathrm{d}x_{0}\\
 & \leq2\exp\bigg(\frac{5c(1-\alpha_{t})}{2(\alpha_{t}-\overline{\alpha}_{t})}\varepsilon_{\score,t}(x)\sqrt{\theta_t(x)d\log T}+\frac{30c(1-\alpha_{t})}{\alpha_{t}-\overline{\alpha}_{t}}\theta_t(x)d\log T\bigg),
\end{align*}
thus establishing \eqref{eq:xt_up} by taking $c=2$.

\paragraph{Step 2(b): proof of relation \eqref{eq:xt}.}
Suppose now that
\begin{equation}
C_{10}\frac{\theta_t(x)d\log^{2}T+\varepsilon_{\score,t}(x)\sqrt{\theta_t(x)d\log^{3}T}}{T}\leq1
	\label{eq:Lemma3-strong-assump}
\end{equation}
holds for some large enough constant $C_{10}>0$. 
Under this additional condition, it can be easily verified that
\begin{align}
 & \left|-\frac{(1-\alpha_{t})\big\| x-\sqrt{\overline{\alpha}_{t}}x_{0}\big\|_{2}^{2}}{2(\alpha_{t}-\overline{\alpha}_{t})(1-\overline{\alpha}_{t})}-\frac{\|u\|_{2}^{2}}{2(\alpha_{t}-\overline{\alpha}_{t})}+\frac{u^{\top}\big(x-\sqrt{\overline{\alpha}_{t}}x_{0}\big)}{\alpha_{t}-\overline{\alpha}_{t}}\right|\notag\\
 & \qquad \leq c_{10}\left(\theta_t(x)d\log T+\varepsilon_{\score,t}(x)\sqrt{\theta_t(x)d\log T}\right)\frac{1-\alpha_{t}}{\alpha_{t}-\overline{\alpha}_{t}}
	\label{eq:exponent-UB-1357}
\end{align}
for any $x_0 \in \mathcal{E}_2^{\mathsf{typical}}$ (with $c=2$), where  $c_{10}>0$ is some sufficiently small constant. 
Therefore,   
the Taylor expansion  $e^{-z} = 1 - z + O(z^2)$ (for all $|z| < 1$) gives 
\begin{align}
 & \exp\bigg(-\frac{(1-\alpha_{t})\big\| x-\sqrt{\overline{\alpha}_{t}}x_{0}\big\|_{2}^{2}}{2(\alpha_{t}-\overline{\alpha}_{t})(1-\overline{\alpha}_{t})}-\frac{\|u\|_{2}^{2}-2u^{\top}\big(x-\sqrt{\overline{\alpha}_{t}}x_{0}\big)}{2(\alpha_{t}-\overline{\alpha}_{t})}\bigg)\notag\\
%
 & =1-\frac{(1-\alpha_{t})\big\| x-\sqrt{\overline{\alpha}_{t}}x_{0}\big\|_{2}^{2}}{2(\alpha_{t}-\overline{\alpha}_{t})(1-\overline{\alpha}_{t})}+\frac{u^{\top}\big(x-\sqrt{\overline{\alpha}_{t}}x_{0}\big)}{\alpha_{t}-\overline{\alpha}_{t}}+O\bigg(\big(\theta_t(x)^2d^{2}\log^{2}T + \varepsilon_{\score, t}(x)^2\theta_t(x) d\log T\big)\Big(\frac{1-\alpha_{t}}{\alpha_{t}-\overline{\alpha}_{t}}\Big)^{2}\bigg)
	\label{eq:exp-UB-13579}
\end{align}
for any $x_0 \in \mathcal{E}_2^{\mathsf{typical}}$, 
which invokes \eqref{eq:exponent-UB-1357} and \eqref{eq:long-UB-2} (under the assumption~\eqref{eq:Lemma3-strong-assump}).  
%
%
Combine \eqref{eq:exp-UB-13579} and \eqref{eq:multi-term-UB-456} to show that
%
\begin{align}
 & \int_{x_{0}}p_{X_{0}\mymid X_{t}}(x_{0}\mymid x)\exp\bigg(-\frac{(1-\alpha_{t})\big\| x-\sqrt{\overline{\alpha}_{t}}x_{0}\big\|_{2}^{2}}{2(\alpha_{t}-\overline{\alpha}_{t})(1-\overline{\alpha}_{t})}-\frac{\|u\|_{2}^{2}-2u^{\top}\big(x-\sqrt{\overline{\alpha}_{t}}x_{0}\big)}{2(\alpha_{t}-\overline{\alpha}_{t})}\bigg)\mathrm{d}x_{0}\notag\\
 & =\left(\int_{x_{0}\in\mathcal{E}_{2}}+\int_{x_{0}\notin\mathcal{E}_{2}}\right)p_{X_{0}\mymid X_{t}}(x_{0}\mymid x)\exp\bigg(-\frac{(1-\alpha_{t})\big\| x-\sqrt{\overline{\alpha}_{t}}x_{0}\big\|_{2}^{2}}{2(\alpha_{t}-\overline{\alpha}_{t})(1-\overline{\alpha}_{t})}-\frac{\|u\|_{2}^{2}-2u^{\top}\big(x-\sqrt{\overline{\alpha}_{t}}x_{0}\big)}{2(\alpha_{t}-\overline{\alpha}_{t})}\bigg)\mathrm{d}x_{0}\notag\\
 & =\int_{x_{0}\in\mathcal{E}_{2}}p_{X_{0}\mymid X_{t}}(x_{0}\mymid x)\bigg(1-\frac{(1-\alpha_{t})\big\| x-\sqrt{\overline{\alpha}_{t}}x_{0}\big\|_{2}^{2}}{2(\alpha_{t}-\overline{\alpha}_{t})(1-\overline{\alpha}_{t})}+\frac{u^{\top}\big(x-\sqrt{\overline{\alpha}_{t}}x_{0}\big)}{\alpha_{t}-\overline{\alpha}_{t}}\bigg)\mathrm{d}x_{0}\notag\\
 & \qquad+O\bigg(\big(\theta_t(x)^{2}d^{2}\log^{2}T+\varepsilon_{\score,t}(x)^{2}\theta_t(x)d\log T\big)\Big(\frac{1-\alpha_{t}}{\alpha_{t}-\overline{\alpha}_{t}}\Big)^{2}\bigg)+O\left(\sum_{c=3}^{\infty}\int_{x_{0}\in\mathcal{E}_{c}\backslash\mathcal{E}_{c-1}}p_{X_{0}\mymid X_{t}}(x_{0}\mymid x)\exp\left(c\theta_t(x)d\right)\mathrm{d}x_{0}\right)\notag\\
 & =1-\frac{(1-\alpha_{t})\big(\int_{x_{0}}p_{X_{0}\mymid X_{t}}(x_{0}\mymid x)\big\| x-\sqrt{\overline{\alpha}_{t}}x_{0}\big\|_{2}^{2}\mathrm{d}x_{0}-\big\|\int_{x_{0}}p_{X_{0}\mymid X_{t}}(x_{0}\mymid x)\big(x-\sqrt{\overline{\alpha}_{t}}x_{0}\big)\mathrm{d}x_{0}\big\|_{2}^{2}\big)}{2(\alpha_{t}-\overline{\alpha}_{t})(1-\overline{\alpha}_{t})}\notag\\
 & \qquad+O\bigg(\theta_t(x)^{2}d^{2}\Big(\frac{1-\alpha_{t}}{\alpha_{t}-\overline{\alpha}_{t}}\Big)^{2}\log^{2}T+\varepsilon_{\score,t}(x)\sqrt{\theta_t(x)d\log T}\Big(\frac{1-\alpha_{t}}{\alpha_{t}-\overline{\alpha}_{t}}\Big)\bigg)+O\Big(\exp\big(-\theta_t(x)d\log T\big)\Big)\notag\\
 & %
=1-\frac{(1-\alpha_{t})\big(\int_{x_{0}}p_{X_{0}\mymid X_{t}}(x_{0}\mymid x)\big\| x-\sqrt{\overline{\alpha}_{t}}x_{0}\big\|_{2}^{2}\mathrm{d}x_{0}-\big\|\int_{x_{0}}p_{X_{0}\mymid X_{t}}(x_{0}\mymid x)\big(x-\sqrt{\overline{\alpha}_{t}}x_{0}\big)\mathrm{d}x_{0}\big\|_{2}^{2}\big)}{2(\alpha_{t}-\overline{\alpha}_{t})(1-\overline{\alpha}_{t})}\notag\\
 & %
\qquad+O\bigg(\theta_t(x)^{2}d^{2}\Big(\frac{1-\alpha_{t}}{\alpha_{t}-\overline{\alpha}_{t}}\Big)^{2}\log^{2}T+\varepsilon_{\score,t}(x)\sqrt{\theta_t(x)d\log T}\Big(\frac{1-\alpha_{t}}{\alpha_{t}-\overline{\alpha}_{t}}\Big)\bigg),\label{eq:exp-UB-135702}
\end{align}
where the penultimate relation holds since, according to \eqref{eqn:brahms},
\begin{align*}
\sum_{c=3}^{\infty}\int_{x_{0}\in\mathcal{E}_{c}\backslash\mathcal{E}_{c-1}}p_{X_{0}\mymid X_{t}}(x_{0}\mymid x)\exp\left(c\theta_t(x)d\right)\mathrm{d}x_{0} & \leq\sum_{c=3}^{\infty}\exp\left(-c^{2}\theta_t(x)d\log T\right)\exp\left(c\theta_t(x)d\right)\\
 & \leq\sum_{c=3}^{\infty}\exp\left(-\frac{1}{2}c^{2}\theta_t(x)d\log T\right)\leq\exp\big(-\theta_t(x)d\log T\big),
\end{align*}
and the last line in \eqref{eq:exp-UB-135702} again utilizes \eqref{eqn:properties-alpha-proof} and the fact that $\theta_t(x)\geq c_6$ for some large enough constant $c_6>0$.

%
Putting \eqref{eq:exp-UB-135702} and \eqref{eqn:fei} together yields
\begin{align*}
\frac{p_{\sqrt{\alpha_{t}}X_{t-1}}\big(\phi_{t}(x)\big)}{p_{X_{t}}(x)} & =1+\frac{d(1-\alpha_{t})}{2(\alpha_{t}-\overline{\alpha}_{t})}+O\bigg(\theta_t(x)^2d^{2}\Big(\frac{1-\alpha_{t}}{\alpha_{t}-\overline{\alpha}_{t}}\Big)^{2}\log^{2}T + \varepsilon_{\score, t}(x)\sqrt{\theta_t(x) d\log T}\Big(\frac{1-\alpha_{t}}{\alpha_{t}-\overline{\alpha}_{t}}\Big)\bigg)-\notag\\
 & \quad\frac{(1-\alpha_{t})\big(\int_{x_{0}}p_{X_{0}\mymid X_{t}}(x_{0}\mymid x)\big\| x-\sqrt{\overline{\alpha}_{t}}x_{0}\big\|_{2}^{2}\mathrm{d}x_{0}-\big\|\int_{x_{0}}p_{X_{0}\mymid X_{t}}(x_{0}\mymid x)\big(x-\sqrt{\overline{\alpha}_{t}}x_{0}\big)\mathrm{d}x_{0}\big\|_{2}^{2}\big)}{2(\alpha_{t}-\overline{\alpha}_{t})(1-\overline{\alpha}_{t})}
\end{align*}
as claimed.

\subsubsection{Proof of relation~\eqref{eq:yt}} 

Consider any random vector $Y$. 
To understand the density ratio $p_{\phi_t(Y)}(\phi_t(x))/p_{Y}(x)$, we make note of the transformation
\begin{subequations}
\begin{align} \label{eq:dist_tranform}
	p_{\phi_t(Y)}\big(\phi_t(x)\big) &= \mathsf{det}\Big(\frac{\partial \phi_t(x)}{\partial x}\Big)^{-1}p_{Y}(x), \\
	p_{\phi_t^{\star}(Y)}\big(\phi^{\star}_t(x)\big) &= \mathsf{det}\Big(\frac{\partial \phi^{\star}_t(x)}{\partial x}\Big)^{-1}p_{Y}(x),
\end{align}
\end{subequations}
where $\frac{\partial \phi_t(x)}{\partial x}$ and $\frac{\partial \phi^{\star}_t(x)}{\partial x}$  denote the Jacobian matrices. 
It thus suffices to control the quantity $\mathsf{det}\Big(\frac{\partial \phi_t(x)}{\partial x}\Big)^{-1}$.

To begin with, recall from \eqref{defn:phit-x} and \eqref{eq:st-MMSE-expression} that
\[
	\phi_{t}^{\star}(x)=x-\frac{1-\alpha_{t}}{2(1-\overline{\alpha}_{t})}g_{t}(x).
\]
As a result, one can use \eqref{eq:Jacobian-Thm4} and \eqref{eq:Jt-x-expression-ij-23} to derive
\begin{align}
I-\frac{\partial\phi^{\star}_{t}(x)}{\partial x}=\frac{1-\alpha_{t}}{2(1-\overline{\alpha}_{t})}J_{t}(x) & =\frac{1-\alpha_{t}}{2(1-\overline{\alpha}_{t})}\Bigg\{ I+\frac{1}{1-\overline{\alpha}_{t}}\bigg\{\mathbb{E}\big[X_{t}-\sqrt{\overline{\alpha}_{t}}X_{0}\mid X_{t}=x\big]\Big(\mathbb{E}\big[X_{t}-\sqrt{\overline{\alpha}_{t}}X_{0}\mid X_{t}=x\big]\Big)^{\top}\nonumber \\
 & \quad\quad-\mathbb{E}\Big[\big(X_{t}-\sqrt{\overline{\alpha}_{t}}X_{0}\big)\big(X_{t}-\sqrt{\overline{\alpha}_{t}}X_{0}\big)^{\top}\mid X_{t}=x\Big]\bigg\}\Bigg\}\nonumber \\
 & \eqqcolon\frac{1-\alpha_{t}}{2(1-\overline{\alpha}_{t})}\bigg\{ I+\frac{1}{1-\overline{\alpha}_{t}}B\bigg\}.\label{eq:defn-B-intermediate}
\end{align}
This allows one to show that
\begin{subequations}
\begin{align}
 & \mathsf{Tr}\Big(I-\frac{\partial\phi^{\star}_{t}(x)}{\partial x}\Big)=\frac{d(1-\alpha_{t})}{2(1-\overline{\alpha}_{t})}+ \notag\\
 & \quad\quad\frac{(1-\alpha_{t})\big(\big\|\int_{x_{0}}p_{X_{0}\mymid X_{t}}(x_{0}\mymid x)\big(x-\sqrt{\overline{\alpha}_{t}}x_{0}\big)\mathrm{d}x_{0}\big\|_{2}^{2}-\int_{x_{0}}p_{X_{0}\mymid X_{t}}(x_{0}\mymid x)\big\| x-\sqrt{\overline{\alpha}_{t}}x_{0}\big\|_{2}^{2}\mathrm{d}x_{0}\big)}{2(1-\overline{\alpha}_{t})^{2}}. 
\end{align}
Moreover, the matrix $B$ defined in \eqref{eq:defn-B-intermediate} satisfies
\begin{align*}
\|B\|_{\mathrm{F}} 
	\leq \Big\| \mathbb{E}\Big[\big(X_{t}-\sqrt{\overline{\alpha}_{t}}X_{0}\big)\big(X_{t}-\sqrt{\overline{\alpha}_{t}}X_{0}\big)^{\top}\mid X_{t}=x\Big] \Big\|_{\mathrm{F}}
	\leq \int_{x_{0}}p_{X_{0}\mymid X_{t}}(x_{0}\mymid x)\big\| x-\sqrt{\overline{\alpha}_{t}}x_{0}\big\|_{2}^{2}\mathrm{d}x_{0}
\end{align*}
due to Jensen's inequality. 
Taking this together with \eqref{eq:defn-B-intermediate} and Lemma~\ref{lem:x0} reveals that
\begin{align}
	\Big\|\frac{\partial \phi^{\star}_t(x)}{\partial x} - I\Big\| \leq 
\Big\|\frac{\partial \phi^{\star}_t(x)}{\partial x} - I\Big\|_{\mathrm{F}} 
	&\lesssim \frac{1-\alpha_{t}}{1-\overline{\alpha}_{t}}
\bigg(\sqrt{d} + \frac{\int_{x_0}   p_{X_0 \mymid X_{t}}(x_0 \mymid x)\big\|x - \sqrt{\overline{\alpha}_{t}}x_0\big\|_2^2\mathrm{d} x_0}{1-\overline{\alpha}_{t}}\bigg) \notag \\
&\lesssim \frac{\theta_t(x) d(1-\alpha_{t})\log T}{1-\overline{\alpha}_{t}}.
\end{align}
\end{subequations}
Additionally, the Taylor expansion guarantees that for any $A$ and $\Delta$, 
\begin{subequations}
\begin{align}
\label{eqn:matrix-det}
	\mathsf{det}\big(I+A+\Delta\big) &= 1 + \mathsf{Tr}(A) + O\big((\mathsf{Tr}(A))^2 + \|A\|_{\mathrm{F}}^2 + d^3\|A\|^3 + d\|\Delta\|\big) \\
	\mathsf{det}\big(I+A+\Delta\big)^{-1} &= 1 - \mathsf{Tr}(A) + O\big((\mathsf{Tr}(A))^2 + \|A\|_{\mathrm{F}}^2 + d^3\|A\|^3 + d\|\Delta\|\big)
\end{align}
\end{subequations}
hold as long as  $d\|A\| + d\|\Delta\| \leq c_{11}$ for some small enough constant $c_{11}>0$. 
The above properties taken collectively with \eqref{defn:phit-x} and \eqref{eq:pointwise-epsilon-score-J} allow us to demonstrate that 
\begin{align}
\notag & \frac{p_{\phi_{t}(Y)}(\phi_{t}(x))}{p_{Y}(x)}=\mathsf{det}\Big(\frac{\partial\phi_{t}(x)}{\partial x}\Big)^{-1}
=\left(\mathsf{det}\bigg(\frac{\partial\phi_{t}^{\star}(x)}{\partial x}+\frac{1-\alpha_{t}}{2}\Big[J_{s_{t}}(x)-J_{s_{t}^{\star}}(x)\Big]\bigg)\right)^{-1}\notag\\
&\quad =\left(\mathsf{det}\bigg(I + \frac{ \partial\phi_{t}^{\star}(x)}{\partial x} - I +\frac{1-\alpha_{t}}{2}\Big[J_{s_{t}}(x)-J_{s_{t}^{\star}}(x)\Big]\bigg)\right)^{-1}\notag\\
 & \quad=1-\mathsf{Tr}\Big(\frac{\partial\phi_{t}^{\star}(x)}{\partial x}-I\Big)+O\bigg(\theta_t(x)^{2}d^{2}\Big(\frac{1-\alpha_{t}}{\alpha_{t}-\overline{\alpha}_{t}}\Big)^{2}\log^{2}T+\theta^{3}d^{6}\log^{3}T\Big(\frac{1-\alpha_{t}}{\alpha_{t}-\overline{\alpha}_{t}}\Big)^{3}+(1-\alpha_{t})d\varepsilon_{\Jacobi,t}(x)\bigg)\notag\\
 & \quad=1+\frac{d(1-\alpha_{t})}{2(\alpha_{t}-\overline{\alpha}_{t})}+\frac{(1-\alpha_{t})\Big(\big\|\int_{x_{0}}p_{X_{0}\mymid X_{t}}(x_{0}\mymid x)\big(x-\sqrt{\overline{\alpha}_{t}}x_{0}\big)\mathrm{d}x_{0}\big\|_{2}^{2}-\int_{x_{0}}p_{X_{0}\mymid X_{t}}(x_{0}\mymid x)\big\| x-\sqrt{\overline{\alpha}_{t}}x_{0}\big\|_{2}^{2}\mathrm{d}x_{0}\Big)}{2(\alpha_{t}-\overline{\alpha}_{t})(1-\overline{\alpha}_{t})}\notag\\
 & \quad\qquad+O\bigg(\theta_t(x)^{2}d^{2}\Big(\frac{1-\alpha_{t}}{\alpha_{t}-\overline{\alpha}_{t}}\Big)^{2}\log^{2}T+\theta^{3}d^{6}\log^{3}T\Big(\frac{1-\alpha_{t}}{\alpha_{t}-\overline{\alpha}_{t}}\Big)^{3}+(1-\alpha_{t})d\varepsilon_{\Jacobi,t}(x)\bigg),	 
\end{align}
with the proviso that 
\begin{align*}
\frac{d^{2}(1-\alpha_{t})\log T}{\alpha_{t}-\overline{\alpha}_{t}} & \leq\frac{8c_{1}d^{2}\log^{2}T}{T}\leq c_{12}\qquad\text{and}\qquad(1-\alpha_{t})d\varepsilon_{\Jacobi,t}(x)\leq\frac{c_{1}d\varepsilon_{\Jacobi,t}(x)\log T}{T}\leq c_{12}
\end{align*}
for some sufficiently small constant $c_{12}>0$ (see \eqref{eqn:properties-alpha-proof}).   
%

\subsection{Proof of Lemma~\ref{lem:q1-large-qk-large}}
\label{sec:proof-lem:q1-large-qk-large}

In view of the definition \eqref{eq:defn-tao-i}, 
one has 
\begin{align}
	S_k(y_T) \leq c_{14}, \qquad \text{for any }k < \tau(y_T). 
	\label{eq:Sk-yT-all-Omegai}
\end{align}
%
Suppose instead that \eqref{eq:q_k_yk_UB} does not hold true, namely, 
$-\log q_k(y_k)> 2c_{6}d\log T$ for some $k<\tau(y_T)$, and we would like to show that this leads to contradiction. 

Towards this, let $1 < t \le k$ be the smallest time step obeying
\begin{align}
	\theta_t(y_t) = \max\bigg\{ -\frac{\log q_t(y_t)}{d\log T}, c_6 \bigg\} > 2 c_6 = 2 \theta_1(y_1),
\end{align}
where the last identity holds since $-\log q_1(y_1)\leq c_{6}d\log T$ and hence $\theta_1(y_1)=\max\big\{ -\frac{\log q_1(y_1)}{d\log T}, c_6 \big\}=c_6$. 
We claim that $t$ necessarily obeys 
\begin{align}
	2 c_6 < \theta_t(y_t) \leq 4 c_6.   
	\label{eq:thetat-c6-2-4}
\end{align}
Assuming the validity of Claim~\eqref{eq:thetat-c6-2-4} for the moment, 
it necessarily satisfies 
\[
	\theta_1(y_1),\theta_2(y_2),\cdots,\theta_t(y_t)\in [c_6,4c_6].
\]
According to the relations~\eqref{eq:crude-ratio-qt-1-qt} and \eqref{eq:Sk-yT-all-Omegai}, we derive
\begin{align*}
	c_{6}=\theta_{1}(y_{1}) & \le\theta_{t}(y_{t})-\theta_{1}(y_{1})=-\frac{\log q_{t}(y_{t})}{d\log T}-\theta_{1}(y_{1})\leq\frac{-\log q_{t}(y_{t})+\log q_{1}(y_{1})}{d\log T}\\
 & =\frac{1}{d\log T}\sum_{j=1}^{t-1}\big(\log q_{j}(y_{j})-\log q_{j+1}(y_{j+1})\big)\\
	& \leq2c_{1}+C_{10}\left\{ \frac{d\log^{3}T}{T}+\frac{S_{\tau(y_T)-1}(y_{T})}{d\log T}\right\} < 3c_{1}  
\end{align*}
under our sample size condition. This, however, cannot possibly hold if $c_6 \geq 3c_1$ as assumed for Lemma~\ref{lem:q1-large-qk-large}.

To finish up, it suffices to justify Claim~\eqref{eq:thetat-c6-2-4}.  
In order to see this, suppose instead that $\theta_t(y_t) > 4 c_6$. 
Given relation~\eqref{eq:Sk-yT-all-Omegai} that $S_k(y_T)\leq c_{14}$,  it can be readily seen from 
\eqref{eq:xt_up}, \eqref{eq:Sk-yT-all-Omegai} as well as the learning rate properties \eqref{eqn:properties-alpha-proof} that 
\begin{align*}
\theta_{t-1}(y_{t-1}) & =\theta_{t}(y_{t})+\theta_{t-1}(y_{t-1})-\theta_{t}(y_{t})\\
 & =\theta_{t}(y_{t})+\theta_{t-1}(y_{t-1})+\frac{\log p_{t}(y_{t})}{d\log T}\geq\theta_{t}(y_{t})-\frac{\log p_{t-1}(y_{t-1})-\log p_{t}(y_{t})}{d\log T}\\
 & \geq\theta_{t}(y_{t})-\frac{4c_{1}\Big(5\varepsilon_{\score,t}(y_{t})\sqrt{\theta_{t}(y_{t})d\log T}+60\theta_{t}(y_{t})d\log T\Big)}{dT}-\frac{\log2}{d\log T}\\
 & \geq\theta_{t}(y_{t})-\frac{4c_{1}\Big(5\varepsilon_{\score,t}(y_{t})\sqrt{d\log T}+60d\log T\Big)}{dT}\theta_{t}(y_{t})-\frac{\log2}{d\log T}\\
 & >\frac{1}{2}\theta_{t}(y_{t})>2c_{6},
\end{align*}
which is contradictory with the assumption that $t$ is the smallest step obeying $\theta_t(y_t) > 2 c_6$. Thus, we complete the proof of relation~\eqref{eq:q_k_yk_UB} as required. 
%
%
%
%
%
%
%

\subsection{Proof of Lemma~\ref{lem:density-ratio-tau}}
\label{sec:proof-lem-density-ratio-tau}

Next, consider any $y_T$, with $\{y_{T-1},\cdots,y_1\}$ being the associated deterministic sequence (cf.~\eqref{eq:defn-yt-sequence-proof})).  
As an immediate consequence of Lemma~\ref{lem:q1-large-qk-large} and the definition \eqref{eqn:choice-y} of $\theta_t(\cdot)$, one has 
\begin{equation}
	\theta_t(y_t)\leq 2c_6, \qquad \forall  t < \tau(y_T) 
	\label{eq:theta-t-all-small-ST}
\end{equation}
We then intend to invoke Lemma~\ref{lem:main-ODE} to control the term of interest. 
To do so, note that Lemma~\ref{lem:x0}, \eqref{eqn:properties-alpha-proof} and the definition \eqref{eq:defn-tao-i} of $\tau(y_T)$ taken together reveal that: 
for all $t<\tau(y_T)$ one has 
\[
	\frac{d(1-\alpha_{t})}{2(\alpha_{t}-\overline{\alpha}_{t})}\lesssim\frac{d\log T}{T}=o(1) ,
\]
\begin{align*}
 & \theta_{t}(y_{t})^{2}d^{2}\Big(\frac{1-\alpha_{t}}{\alpha_{t}-\overline{\alpha}_{t}}\Big)^{2}\log^{2}T+\varepsilon_{\score,t}(y_{t})\sqrt{\theta_{t}(y_{t})d\log T}\Big(\frac{1-\alpha_{t}}{\alpha_{t}-\overline{\alpha}_{t}}\Big)+\theta_{t}(y_{t})^{3}d^{6}\Big(\frac{1-\alpha_{t}}{\alpha_{t}-\overline{\alpha}_{t}}\Big)^{3}\log^{3}T+(1-\alpha_{t})d\varepsilon_{\Jacobi,t}(y_{t})\\
 & \qquad\lesssim\frac{d^{2}\log^{4}T}{T^{2}}+\frac{d^{6}\log^{6}T}{T^{3}}+\frac{\varepsilon_{\score,t}(y_{t})\sqrt{d\log^{3}T}}{T}+\frac{d\varepsilon_{\Jacobi,t}(y_{t})\log T}{T}=o(1),
\end{align*}
and
\begin{align*}
 & \left|\frac{(1-\alpha_{t})\Big(\big\|\int\mathbb{E}\big[X_{t}-\sqrt{\overline{\alpha}_{t}}X_{0}\mymid X_{t}=y_t\big]\big\|_{2}^{2}-\int\mathbb{E}\big[\big\| X_{t}-\sqrt{\overline{\alpha}_{t}}X_{0}\big\|_{2}^{2}\mymid X_{t}=y_t\big]\Big)}{(\alpha_{t}-\overline{\alpha}_{t})(1-\overline{\alpha}_{t})}\right|\\
 & \qquad\leq\left|\frac{(1-\alpha_{t})\int\mathbb{E}\big[\big\| X_{t}-\sqrt{\overline{\alpha}_{t}}X_{0}\big\|_{2}^{2}\mymid X_{t}=y_t\big]}{(\alpha_{t}-\overline{\alpha}_{t})(1-\overline{\alpha}_{t})}\right|\lesssim\frac{(1-\alpha_{t})d\log T}{\alpha_{t}-\overline{\alpha}_{t}}\lesssim\frac{d\log^2 T}{T}=o(1) .
\end{align*}
With these bounds in mind, applying relations~\eqref{eq:xt} and~\eqref{eq:yt} in Lemma~\ref{lem:main-ODE} leads to
\begin{align*}
 & \frac{p_{\sqrt{\alpha_{t}}Y_{t-1}}\big(\phi_{t}(y_{t})\big)}{p_{Y_{t}}(y_{t})}\bigg(\frac{p_{\sqrt{\alpha_{t}}X_{t-1}}\big(\phi_{t}(y_{t})\big)}{p_{X_{t}}(y_{t})}\bigg)^{-1}=\frac{p_{\phi_{t}(Y_{t})}\big(\phi_{t}(y_{t})\big)}{p_{Y_{t}}(y_{t})}\bigg(\frac{p_{\sqrt{\alpha_{t}}X_{t-1}}\big(\phi_{t}(y_{t})\big)}{p_{X_{t}}(y_{t})}\bigg)^{-1}\\
 & \qquad=1+O\Bigg(\frac{d^{2}\log^{4}T}{T^{2}}+\frac{d^{6}\log^{6}T}{T^{3}}+\frac{\varepsilon_{\score,t}(y_{t})\sqrt{d\log^{3}T}}{T}+\frac{d\varepsilon_{\Jacobi,t}(y_{t})\log T}{T}\Bigg)
\end{align*}
for all $t<\tau(y_T)$. 
Using the fact that $y_{t-1}=\frac{1}{\sqrt{\alpha_t}} \phi_t(y_t)$ and invoking the relation \eqref{eq:recursion},  
we arrive at
\begin{align*}
\frac{p_{t-1}(y_{t-1})}{q_{t-1}(y_{t-1})} & =\left\{ 1+O\Bigg(\frac{d^{2}\log^{4}T}{T^{2}}+\frac{d^{6}\log^{6}T}{T^{3}}+\frac{\varepsilon_{\score,t}(y_{t})\sqrt{d\log^{3}T}}{T}+\frac{d\varepsilon_{\Jacobi,t}(y_{t})\log T}{T}\Bigg) \right\} \frac{p_{t}(y_{t})}{q_{t}(y_{t})}
\end{align*}
for any $t<\tau(y_T)$. 
By abbreviating $\tau=\tau(y_{T})$ for notational simplicity, 
we reach
\begin{subequations}
	\label{eq:pt-qt-equiv-ODE-St-temp}
\begin{align}
\frac{q_{1}(y_{1})}{p_{1}(y_{1})} & =\left\{ 1+O\Bigg(\frac{d^{2}\log^{4}T}{T}+\frac{d^{6}\log^{6}T}{T^{2}}+S_{\tau-1}(y_{\tau-1})\Bigg)\right\} \frac{q_{\tau-1}(y_{\tau-1})}{p_{\tau-1}(y_{\tau-1})}\notag\\
 & \in\left[\frac{p_{\tau-1}(y_{\tau-1})}{2q_{\tau-1}(y_{\tau-1})},\frac{2p_{\tau-1}(y_{\tau-1})}{q_{\tau-1}(y_{\tau-1})}\right],	
	\label{eq:pt-qt-equiv-ODE-St-taui-temp}
\end{align}
and similarly, 
\begin{align}
	\frac{q_{k}(y_{k})}{2p_{k}(y_{k})} \leq \frac{q_{1}(y_{1})}{p_{1}(y_{1})} \leq 2 \frac{q_{k}(y_{k})}{p_{k}(y_{k})}, \qquad \forall k < \tau. 
	\label{eq:pt-qt-equiv-ODE-St-k-temp}
\end{align}
\end{subequations}

\subsection{Proof of Lemma~\ref{lem:I2-I3-I4-bound}}
\label{sec:proof-lem:I2-I3-I4-bound}

In the following, we shall tackle $\mathcal{I}_{2}$, $\mathcal{I}_{3}$ and $\mathcal{I}_{4}$ separately. 
Throughout this proof, we shall abbreviate $\tau=\tau(Y_T)$ (cf.~\eqref{eq:defn-tao-i}) whenever it is clear from the context.

\paragraph{The sub-collection in $\mathcal{I}_{2}$.}
By virtue of the definition~\eqref{eq:defn-I2-I3-I4-ode-I2} of $\mathcal{I}_{2}$, we make the observation that
\begin{align}
 &  \mathop{\mathbb{E}}_{Y_{T}\sim p_{T}}\bigg[\frac{q_{1}(Y_{1})}{p_{1}(Y_{1})}\ind\left\{ Y_{1}\in\mathcal{E},Y_{T}\in\mathcal{I}_2\right\} \bigg]
	\overset{\mathrm{(i)}}{\leq}  \mathop{\mathbb{E}}_{Y_{T}\sim p_{T}}\bigg[\frac{q_{1}(Y_{1})}{p_{1}(Y_{1})}\ind\left\{ Y_{1}\in\mathcal{E},Y_{T}\in\mathcal{I}_2\right\} \frac{S_{\tau}(Y_{T})}{c_{14}}\bigg]\nonumber \\
	& \quad \overset{\mathrm{(ii)}}{=}\frac{\log T}{c_{14}T} \sum_{t=2}^{\tau}\mathop{\mathbb{E}}_{Y_{T}\sim p_{T}}\bigg[\frac{q_{1}(Y_{1})}{p_{1}(Y_{1})}\ind\left\{ Y_{1}\in\mathcal{E},Y_{T}\in\mathcal{I}_2\right\} \left(d\varepsilon_{\Jacobi,t}(Y_{t})+\sqrt{d\log T}\varepsilon_{\score,t}(Y_{t})\right)\bigg]\nonumber \\
 & \quad\overset{\mathrm{(iii)}}{\leq} \frac{2\log T}{c_{14}T} \sum_{t=2}^{\tau}\mathop{\mathbb{E}}_{Y_{T}\sim p_{T}}\bigg[\frac{q_{t}(Y_{t})}{p_{t}(Y_{t})}\ind\left\{ Y_{1}\in\mathcal{E},Y_{T}\in\mathcal{I}_2\right\} \left(d\varepsilon_{\Jacobi,t}(Y_{t})+\sqrt{d\log T}\varepsilon_{\score,t}(Y_{t})\right)\bigg]\nonumber \\
 & \quad=\frac{2\log T}{c_{14}T}\sum_{t=2}^{T}\sum_{i\in\mathcal{I}_{2},\tau\geq t}\mathop{\mathbb{E}}_{Y_{T}\sim p_{T}}\bigg[\frac{q_{t}(Y_{t})}{p_{t}(Y_{t})}\ind\left\{ Y_{1}\in\mathcal{E},Y_{T}\in\mathcal{I}_2\right\} \left(d\varepsilon_{\Jacobi,t}(Y_{t})+\sqrt{d\log T}\varepsilon_{\score,t}(Y_{t})\right)\bigg]\nonumber \\
 & \quad\leq\frac{2\log T}{c_{14}T}\sum_{t=2}^{T}\mathop{\mathbb{E}}_{Y_{T}\sim p_{T}}\bigg[\frac{q_{t}(Y_{t})}{p_{t}(Y_{t})}\left(d\varepsilon_{\Jacobi,t}(Y_{t})+\sqrt{d\log T}\varepsilon_{\score,t}(Y_{t})\right)\bigg]\nonumber \\
 & \quad=\frac{2\log T}{c_{14}T}\sum_{t=2}^{T}\mathop{\mathbb{E}}_{Y_{t}\sim p_{t}}\bigg[\frac{q_{t}(Y_{t})}{p_{t}(Y_{t})}\left(d\varepsilon_{\Jacobi,t}(Y_{t})+\sqrt{d\log T}\varepsilon_{\score,t}(Y_{t})\right)\bigg]\nonumber \\
 & \quad=\frac{2\log T}{c_{14}T}\sum_{t=2}^{T}\mathop{\mathbb{E}}_{Y_{t}\sim q_{t}}\Big[d\varepsilon_{\Jacobi,t}(Y_{t})+\sqrt{d\log T}\varepsilon_{\score,t}(Y_{t})\Big]\nonumber \\
	& \quad \overset{\mathrm{(iv)}}{\lesssim} \big(d\varepsilon_{\Jacobi}+\sqrt{d\log T}\varepsilon_{\score}\big)\log T.\label{eq:UB-I2-ode}
\end{align}
Here, (i) follows since $S_{\tau}\big(y_{T}\big)\geq c_{14}$ in $\mathcal{I}_2$ (see \eqref{eq:defn-I2-I3-I4-ode-I2}); 
(ii) comes from the definition of $S_{t}(\cdot)$ (see \eqref{eq:defn-xik-Stk-proof}); 
(iii) holds since (by repeating the same proof arguments as for \eqref{eq:pt-qt-equiv-ODE-St} as long as $2c_{14}$ is small enough)
\[
	\frac{p_{1}(y_{1})}{q_{1}(y_{1})}\leq  \frac{2p_{t}(y_{t})}{q_{t}(y_{t})}, \qquad \forall t\le \tau ;
\]
%
and (iv) arises from \eqref{eq:score-assumptions-equiv}.

\paragraph{The sub-collection in $\mathcal{I}_{3}$.}
With regards to  $\mathcal{I}_{3}$ (cf.~\eqref{eq:defn-I2-I3-I4-ode-I3}), we can derive the following bound in a way similar to \eqref{eq:UB-I2-ode}: 
\begin{align}
 & \mathop{\mathbb{E}}_{Y_{T}\sim p_{T}}\bigg[\frac{q_{1}(Y_{1})}{p_{1}(Y_{1})}\ind\left\{ Y_{1}\in\mathcal{E},Y_{T}\in \mathcal{I}_3\right\} \bigg]\overset{\mathrm{(i)}}{\leq} \mathop{\mathbb{E}}_{Y_{T}\sim p_{T}}\bigg[\frac{q_{1}(Y_{1})}{p_{1}(Y_{1})}\ind\left\{ Y_{1}\in\mathcal{E},Y_{T}\in \mathcal{I}_3\right\} \frac{\xi_{\tau}(Y_{T})}{c_{14}}\bigg]\nonumber\\
 & =\frac{\log T}{c_{14}T}\mathop{\mathbb{E}}_{Y_{T}\sim p_{T}}\bigg[\frac{q_{1}(Y_{1})}{p_{1}(Y_{1})}\ind\left\{ Y_{1}\in\mathcal{E},Y_{T}\in \mathcal{I}_3\right\} \left(d\varepsilon_{\Jacobi,\tau}(Y_{\tau})+\sqrt{d\log T}\varepsilon_{\score,\tau}(Y_{\tau})\right)\bigg]\nonumber\\
 & \overset{\mathrm{(ii)}}{\leq}\frac{2\log T}{c_{14}T}\mathop{\mathbb{E}}_{Y_{T}\sim p_{T}}\bigg[\frac{q_{\tau-1}(Y_{\tau-1})}{p_{\tau-1}(Y_{\tau-1})}\ind\left\{ Y_{1}\in\mathcal{E},Y_{T}\in \mathcal{I}_3\right\} \left(d\varepsilon_{\Jacobi,\tau}(Y_{\tau})+\sqrt{d\log T}\varepsilon_{\score,\tau}(Y_{\tau})\right)\bigg]\notag\\
 & =\frac{2\log T}{c_{14}T}\sum_{t=2}^{T}
	\mathop{\mathbb{E}}_{Y_{T}\sim p_{T}}\bigg[\frac{q_{t-1}(Y_{t-1})}{p_{t-1}(Y_{t-1})}\ind\left\{ Y_{1}\in\mathcal{E},Y_{T}\in \mathcal{I}_3\right\} \left(d\varepsilon_{\Jacobi,t}(Y_{t})+\sqrt{d\log T}\varepsilon_{\score,t}(Y_{t})\right) \ind\{\tau=t\} \bigg]
	\label{eq:E-I3-UB-579}\\
 & \overset{\mathrm{(iii)}}{\leq}\frac{16\log T}{c_{14}T}\mathop{\mathbb{E}}_{Y_{T}\sim p_{T}}\bigg[\frac{q_{\tau}(Y_{\tau})}{p_{\tau}(Y_{\tau})}\ind\left\{ Y_{1}\in\mathcal{E},Y_{T}\in\mathcal{I}_3\right\} \left(d\varepsilon_{\Jacobi,\tau}(Y_{\tau})+\sqrt{d\log T}\varepsilon_{\score,\tau}(Y_{\tau})\right)\bigg]\notag\\
 & \leq\frac{16\log T}{c_{14}T}\sum_{t=2}^{T}\mathop{\mathbb{E}}_{Y_{T}\sim p_{T}}\bigg[\frac{q_{t}(Y_{t})}{p_{t}(Y_{t})}\left(d\varepsilon_{\Jacobi,t}(Y_{t})+\sqrt{d\log T}\varepsilon_{\score,t}(Y_{t})\right)\bigg]\notag\\
 & =\frac{16\log T}{c_{14}T}\sum_{t=2}^{T}\mathop{\mathbb{E}}_{Y_{t}\sim p_{t}}\bigg[\frac{q_{t}(Y_{t})}{p_{t}(Y_{t})}\left(d\varepsilon_{\Jacobi,t}(Y_{t})+\sqrt{d\log T}\varepsilon_{\score,t}(Y_{t})\right)\bigg]\notag\\
 & =\frac{16\log T}{c_{14}T}\sum_{t=2}^{T}\mathop{\mathbb{E}}_{Y_{t}\sim q_{t}}\Big[d\varepsilon_{\Jacobi,t}(Y_{t})+\sqrt{d\log T}\varepsilon_{\score,t}(Y_{t})\Big]\notag\\
 & \lesssim\left(d\varepsilon_{\Jacobi}+\sqrt{d\log T}\varepsilon_{\score}\right)\log T.
	\label{eq:UB-I3-ode}
\end{align}
Here, (i) comes from \eqref{eq:defn-I2-I3-I4-ode-I3}, 
(ii) arises from \eqref{eq:pt-qt-equiv-ODE-St-k}, 
whereas (iii) is a consequence of \eqref{eq:defn-I2-I3-I4-ode-I3}.

\paragraph{The sub-collection in $\mathcal{I}_{4}$.}

We now turn attention to $\mathcal{I}_4$ (cf.~\eqref{eq:defn-I2-I3-I4-ode-I4}), towards which we find it helpful to define
\begin{subequations}
	\label{eq:defn-J1t-J2t-J3t}
\begin{align}
\mathcal{J}_{1,t} & \coloneqq \Big\{ y_T :\xi_{t}\big(y_T\big)<c_{14}\Big\} \label{eq:defn-J1t-J2t-J3t-J1t}\\
\mathcal{J}_{2,t} & \coloneqq \bigg\{ y_T :\xi_{t}\big(y_T\big)\geq c_{14},\frac{q_{t-1}(y_{t-1})}{p_{t-1}(y_{t-1})}\leq\frac{8 q_{t}(y_t)}{p_{t}(y_t)}\bigg\} \label{eq:defn-J1t-J2t-J3t-J2t}\\
\mathcal{J}_{3,t} & \coloneqq \bigg\{ y_T :\xi_{t}\big(y_T\big)\geq c_{14},\frac{q_{t-1}(y_{t-1})}{p_{t-1}(y_{t-1})}>\frac{8 q_{t}(y_t)}{p_{t}(y_{t})}\bigg\}
\label{eq:defn-J1t-J2t-J3t-J3t}
\end{align}
\end{subequations}
for each $2\leq t\leq T$. 
%
Equipped with the above definitions, we first make the observation that 
\begin{align}
\mathop{\mathbb{E}}_{Y_{T}\sim p_{T}}\bigg[\frac{q_{1}(Y_{1})}{p_{1}(Y_{1})}\ind\left\{ Y_{1}\in\mathcal{E},Y_{T}\in\mathcal{I}_{4}\right\} \bigg] 
	& \leq 2 \mathop{\mathbb{E}}_{Y_{T}\sim p_{T}}\bigg[\frac{q_{\tau-1}(Y_{1})}{p_{\tau-1}(Y_{1})}\ind\left\{ Y_{1}\in\mathcal{E},Y_{T}\in\mathcal{I}_{4}\right\} \bigg]\notag\\
 	& = 2\sum_{t=2}^{T}\mathop{\mathbb{E}}_{Y_{T}\sim p_{T}}\bigg[\frac{q_{t-1}(Y_{t-1})}{p_{t-1}(Y_{t-1})}\ind\left\{ Y_{1}\in\mathcal{E},Y_{T}\in\mathcal{I}_{4}\right\}
	\ind\{\tau=t\} \bigg]\notag\\
	& \leq 2\sum_{t=2}^{T}\mathop{\mathbb{E}}_{Y_{T}\sim p_{T}}\bigg[\frac{q_{t-1}(Y_{t-1})}{p_{t-1}(Y_{t-1})}\ind\left\{Y_{1}\in\mathcal{E}, Y_{T}\in \mathcal{J}_{3,t}\right\} \bigg], 
	\label{eq:sum-I4-UB-13579}
\end{align}
where the first inequality follows from \eqref{eq:pt-qt-equiv-ODE-St-k}, 
and the last line comes from the definition of $\mathcal{I}_4$ (cf.~\eqref{eq:defn-I2-I3-I4-ode-I4}) and $\mathcal{J}_{3,t}$ (cf.~\eqref{eq:defn-J1t-J2t-J3t-J3t}).   
For notational simplicity, let us define, for $2\leq t\leq T$, 
\begin{align*}
h_{t} & \coloneqq\frac{q_{t}(Y_{t})}{p_{t}(Y_{t})}.
\end{align*}
In view of the second inequality in \eqref{eq:defn-J1t-J2t-J3t-J3t}, one has 
$h_{t-1} > 8h_{t}$ as long as $y_T \in \mathcal{J}_{3,t}$. 
Consequently, 
\begin{align*}
& \sum_{t=2}^{T}h_{t-1}\ind\left\{Y_T \in \mathcal{J}_{3,t}\right\} \\
& < 
\sum_{t=2}^{T}h_{t-1}\ind\left\{Y_T \in \mathcal{J}_{3,t}\right\} + \frac{1}{7}\sum_{t=2}^{T}h_{t-1}\ind\left\{Y_T \in \mathcal{J}_{3,t}\right\} -\frac{8}{7}\sum_{t=2}^{T}h_{t}\ind\left\{Y_T \in \mathcal{J}_{3,t}\right\}\\
& =\frac{8}{7}\sum_{t=2}^{T}
\bigg(
	\Big(h_{t-1} - h_{t-1}\ind\left\{Y_T \in \mathcal{J}_{1,t}\right\}  - h_{t-1}\ind\left\{Y_T \in \mathcal{J}_{2,t}\right\} \Big)
-
 \Big(h_{t} - h_{t}\ind\left\{Y_T \in \mathcal{J}_{1,t}\right\}  - h_{t}\ind\left\{Y_T \in \mathcal{J}_{2,t} \right\}\Big) 
\bigg)\\
&= \frac{8}{7}\sum_{t=2}^{T}\big(h_{t}-{h}_{t-1}\big)\ind\left\{Y_T \in \mathcal{J}_{1,t}\cup \mathcal{J}_{2,t}\right\}
+
	\frac{8}{7}\sum_{t=2}^{T} \big(h_{t-1} - h_{t}\big).
\end{align*}
Here, the second line holds true since, for all $t$, one has (i) $\mathcal{J}_{1,t}\cup \mathcal{J}_{2,t} \cup \mathcal{J}_{3,t}= \mathbb{R}^d$, and (ii) $\mathcal{J}_{1,t}$, 
$\mathcal{J}_{2,t}$ and $\mathcal{J}_{3,t}$ are disjoint. 
Substituting this into \eqref{eq:sum-I4-UB-13579}, we arrive at
\begin{align}
	&\mathop{\mathbb{E}}_{Y_{T}\sim p_{T}}\bigg[\frac{q_{1}(Y_{1})}{p_{1}(Y_{1})}
	\ind\left\{ Y_{1}\in\mathcal{E}, Y_T \in \mathcal{J}_{3,t}\right\} \bigg]  
	\leq 2\sum_{t=2}^{T}\mathop{\mathbb{E}}_{Y_{T}\sim p_{T}}\big[h_{t-1}\ind\left\{Y_T \in \mathcal{J}_{3,t}\right\}\big] \notag\\
 & \quad \leq\frac{8}{7}\sum_{t=2}^{T}\Big(\mathop{\mathbb{E}}_{Y_{T}\sim p_{T}}\big[h_{t}\ind\left\{Y_T \in \mathcal{J}_{1,t}\cup \mathcal{J}_{2,t}\right\}\big] - \mathop{\mathbb{E}}_{Y_{T}\sim p_{T}}\big[h_{t-1}\ind\left\{Y_T \in \mathcal{J}_{1,t}\cup \mathcal{J}_{2,t}\right\}\big]\Big) \notag\\
	&\qquad\qquad + \frac{8}{7}\sum_{t=2}^{T}\Big(\mathop{\mathbb{E}}_{Y_{T}\sim p_{T}}\big[h_{t-1}\big]-\mathop{\mathbb{E}}_{Y_{T}\sim p_{T}}\big[h_{t}\big]\Big) . 
	\label{eq:sum-I4-UB-7924}
\end{align}

In order to further bound \eqref{eq:sum-I4-UB-7924}, 
we make note of a few basic facts. Firstly, the identity below holds: 
\begin{align*}
\mathop{\mathbb{E}}_{Y_{T}\sim p_{T}}\big[h_{t}\big] 
& =\mathop{\mathbb{E}}_{Y_{T}\sim p_{T}}\bigg[\frac{q_{t}(Y_{t})}{p_{t}(Y_{t})} \bigg]
	=\mathop{\mathbb{E}}_{Y_{t}\sim p_{t}}\bigg[\frac{q_{t}(Y_{t})}{p_{t}(Y_{t})} \bigg]
	= 1, \qquad  2\leq t\leq T.
\end{align*}
Secondly, by defining the set
\begin{align}
	\mathcal{E}_t &\coloneqq \Big\{y : q_{t}(y) >  \exp\big(- c_{6} d\log T \big)  \Big\}, \qquad 2\leq t\leq T, 
\end{align}
we can show that
\begin{align*}
	&\sum_{t=2}^{T}\mathop{\mathbb{E}}_{Y_{T}\sim p_{T}}\big[h_{t}\ind\left\{ Y_{t}\notin\mathcal{E}_{t}, Y_T\in\mathcal{J}_{1,t}\right\} \big] 
 \le\sum_{t=2}^{T}\mathop{\mathbb{E}}_{Y_{T}\sim p_{T}}\bigg[\frac{q_{t}(Y_{t})}{p_{t}(Y_{t})}\ind\left\{ Y_{t}\notin\mathcal{E}_{t}\right\} \bigg]
=\sum_{t=2}^{T}\mathop{\mathbb{E}}_{Y_{t}\sim p_{t}}\bigg[\frac{q_{t}(Y_{t})}{p_{t}(Y_{t})}\ind\left\{ Y_{t}\notin\mathcal{E}_{t}\right\} \bigg]\\
 & \qquad =\sum_{t=2}^{T}\mathbb{P}_{Y_{t}\sim q_{t}}\left\{ Y_{t}\notin\mathcal{E}_{t} \right\} =\sum_{t=2}^{T}\mathbb{P}_{X_{t}\sim q_{t}}\left\{ X_{t}\notin\mathcal{E}_{t}\right\} \\
 & \qquad\leq\sum_{t=2}^{T}\mathbb{P}_{X_{t}\sim q_{t}}\left\{ X_{t}\notin\mathcal{E}_{t} \text{ and }\|X_{t}\|_{2}\leq T^{2c_{R}+2}\right\} +\sum_{t=2}^{T}\mathbb{P}_{X_{t}\sim q_{t}}\left\{ \|X_{t}\|_{2}>T^{2c_{R}+2}\right\} \\
 & \qquad\leq\sum_{t=2}^{T}{\displaystyle \int}_{x_{t}:q_{t}(x_{t})\leq\exp(-c_{6}d\log T),\|x_{t}\|_{2}\leq T^{2c_{R}+2}}q_{t}(x_{t})\mathrm{d}x_{t}+T\exp\big(-c_{6}d\log T\big)\\
 & \qquad\leq T\big(2T^{2c_{R}+2}\big)^{d}\exp(-c_{6}d\log T)+T\exp\big(-c_{6}d\log T\big)\le\exp\Big(-\frac{c_{6}}{2}d\log T\Big),
\end{align*}
where the penultimate line comes from \eqref{eq:Xt-2range-ODE}, and
the last inequality holds true as long as $c_{6}$ is large enough. 
Plugging the preceding two results into \eqref{eq:sum-I4-UB-7924}, we reach
\begin{align}
	&\mathop{\mathbb{E}}_{Y_{T}\sim p_{T}}\bigg[\frac{q_{1}(Y_{1})}{p_{1}(Y_{1})}\ind\left\{ Y_{1}\in\mathcal{E},Y_{T}\in\mathcal{I}_{4}\right\} \bigg]  
	 \le \frac{8}{7}\sum_{t=2}^{T}
	\mathop{\mathbb{E}}_{Y_{T}\sim p_{T}}\big[ (h_{t}-h_{t-1}) \ind\left\{y_{t}\in\mathcal{E}_t, Y_T \in \mathcal{J}_{1,t}\right\}\big] 
	\notag\\
	&\quad\quad\qquad\qquad + \frac{8}{7}\sum_{t=2}^{T}\mathop{\mathbb{E}}_{Y_{T}\sim p_{T}}\big[h_{t}\ind\left\{Y_T \in \mathcal{J}_{2,t}\right\}\big] + \exp\Big(-\frac{c_{6}}{2}d\log T\Big).
	\label{eq:sum-I4-UB-7935}
\end{align}
%



As it turns out, the sum w.r.t.~the set $\mathcal{J}_{1,t}$ and the sum w.r.t.~the set $\mathcal{J}_{2,t}$ 
in \eqref{eq:sum-I4-UB-7935}  can be controlled respectively using the same arguments as for $\mathcal{I}_1$ and $\mathcal{I}_3$ to derive
\begin{align*}
\sum_{t=2}^{T}\mathop{\mathbb{E}}_{Y_{T}\sim p_{T}}\big[(h_{t}-h_{t-1})\ind\left\{ y_{t}\in\mathcal{E}_{t},Y_{T}\in\mathcal{J}_{1,t}\right\} \big] & \lesssim\frac{d^{2}\log^{4}T}{T}+\frac{d^{6}\log^{6}T}{T^{2}}+\sqrt{d\log^{3}T}\varepsilon_{\score}+(d\log T)\varepsilon_{\Jacobi},\\
\sum_{t=2}^{T}\mathop{\mathbb{E}}_{Y_{T}\sim p_{T}}\big[h_{t}\ind\left\{ Y_{T}\in\mathcal{J}_{2,t}\right\} \big] & \lesssim\sqrt{d\log^{3}T}\varepsilon_{\score}+(d\log T)\varepsilon_{\Jacobi};
\end{align*}
we omit the arguments here for the sake of brevity.  
Therefore, we have proven that
\begin{align}
\mathop{\mathbb{E}}_{Y_{T}\sim p_{T}}\bigg[\frac{q_{1}(Y_{1})}{p_{1}(Y_{1})}\ind\left\{ Y_{1}\in\mathcal{E},Y_{T}\in\mathcal{I}_{4}\right\} \bigg] & \lesssim\frac{d^{2}\log^{4}T}{T}+\frac{d^{6}\log^{6}T}{T^{2}}+\sqrt{d\log^{3}T}\varepsilon_{\score}+(d\log T)\varepsilon_{\Jacobi}.  
	\label{eq:UB-I4-ode}
\end{align}

\paragraph{Putting all this together.} 
Taking \eqref{eq:UB-I2-ode}, \eqref{eq:UB-I3-ode} and \eqref{eq:UB-I4-ode} together, we establish the advertised result. 

\section{Proofs of auxiliary lemmas for the DDPM-type sampler} 
\label{sec:proof-lemmas-sde-sdeR}

\subsection{Proof of Lemma~\ref{lem:sde}}
\label{sec:proof-lem:sde}

For notational simplicity, we find it helpful to define, for any constant $\gamma \in [0,1]$, 
\begin{align}
	x_{t}(\gamma) \coloneqq \gamma x_{t-1} + (1-\gamma) \widehat{x}_t
	\qquad
	\text{and}\qquad \widehat{x}_t \coloneqq \frac{1}{\sqrt{\alpha_t}} x_t.
\end{align}

\paragraph{Step 1: decomposing the target distribution $p_{X_{t-1}\mymid X_{t}}(x_{t-1}\mymid x_{t})$.}
With this piece of notation in mind, we can recall the forward process~\eqref{eq:forward-process} and calculate: 
for any $x_{t-1},x_t\in \mathbb{R}^d$,  
\begin{align}
\notag & p_{X_{t-1}\mymid X_{t}}(x_{t-1}\mymid x_{t})\\
\notag & =\frac{1}{p_{X_{t}}(x_{t})}p_{X_{t-1},X_{t}}(x_{t-1},x_{t})=\frac{1}{p_{X_{t}}(x_{t})}\exp\Big(\log p_{X_{t-1}}(x_{t-1})+\log p_{X_{t}\mymid X_{t-1}}(x_{t}\mymid x_{t-1})\Big)\notag\\
 & =\frac{1}{p_{X_{t}}(x_{t})}\exp\bigg(\log p_{X_{t-1}}(\widehat{x}_{t})+\int_{0}^{1}\Big[\nabla\log p_{X_{t-1}}\big(x_{t}(\gamma)\big)\Big]^{\top}(x_{t-1}-\widehat{x}_{t})\mathrm{d}\gamma+\log p_{X_{t}\mymid X_{t-1}}(x_{t}\mymid x_{t-1})\bigg)\notag\\
 & =\frac{p_{X_{t-1}}(\widehat{x}_{t})}{p_{X_{t}}(x_{t})}\exp\bigg((x_{t-1}-\widehat{x}_{t})^{\top}\int_{0}^{1}\mathrm{d}\gamma\int_{x_{0}}\frac{\nabla p_{X_{t-1}\mymid X_{0}}\big(x_{t}(\gamma)\mymid x_{0}\big)p_{X_{0}}(x_{0})}{p_{X_{t-1}}\big(x_{t}(\gamma)\big)}\mathrm{d}x_{0}+\log p_{X_{t}\mymid X_{t-1}}(x_{t}\mymid x_{t-1})\bigg),\label{eqn:sde-cond}
\end{align}
where the penultimate line comes from the fundamental theorem of calculus. 
In particular, the exponent in \eqref{eqn:sde-cond} consists of a term that satisfies
\begin{align}
\notag & (x_{t-1}-\widehat{x}_{t})^{\top}\int_{0}^{1}\mathrm{d}\gamma\int_{x_{0}}\frac{\nabla p_{X_{t-1}\mymid X_{0}}\big(x_{t}(\gamma)\mymid x_{0}\big)p_{X_{0}}(x_{0})}{p_{X_{t-1}}\big(x_{t}(\gamma)\big)}\mathrm{d}x_{0}\\
\notag & =(x_{t-1}-\widehat{x}_{t})^{\top}\int_{0}^{1}\mathrm{d}\gamma\int_{x_{0}}\frac{\nabla p_{X_{t-1}\mymid X_{0}}\big(x_{t}(\gamma)\mymid x_{0}\big)}{p_{X_{t-1}\mymid X_{0}}\big(x_{t}(\gamma)\mymid x_{0}\big)}p_{X_{0}\mymid X_{t-1}}\big(x_{0}\mymid x_{t}(\gamma)\big)\mathrm{d}x_{0}\\
 & =-(x_{t-1}-\widehat{x}_{t})^{\top}\int_{0}^{1}\mathrm{d}\gamma\int_{x_{0}}\frac{x_{t}(\gamma)-\sqrt{\overline{\alpha}_{t-1}}x_{0}}{1-\overline{\alpha}_{t-1}}p_{X_{0}\mymid X_{t-1}}\big(x_{0}\mymid x_{t}(\gamma)\big)\mathrm{d}x_{0} \notag\\
	& \eqqcolon -\frac{1}{1-\overline{\alpha}_{t-1}}(x_{t-1}-\widehat{x}_{t})^{\top}\int_{0}^{1}g_{t-1}\big(x_{t}(\gamma)\big)\mathrm{d}\gamma, \label{eqn:schubert-12}
\end{align}
where the second line holds since 
$p_{X_{0}\mymid X_{t-1}}\big(x_{0}\mymid x_{t}(\gamma)\big)p_{X_{t-1}}\big(x_{t}(\gamma)\big)=p_{X_{t-1}\mymid X_{0}}\big(x_{t}(\gamma)\mymid x_{0}\big)p_{X_{0}}\big(x_{0}\big)$, 
and we remind the reader of the definition of $g_t(\cdot)$ in \eqref{eq:st-MMSE-expression}.

%
%

To continue, it is then seen from the fundamental
theorem of calculus that
\begin{align*}
g_{t-1}\big(x_{t}(\gamma)\big) & =g_{t-1}\big(\widehat{x}_{t}\big)+\int_{0}^{1}J_{t-1}\Big((1-\tau)\widehat{x}_{t}+\tau x_{t}(\gamma)\Big)\big(x_{t}(\gamma)-\widehat{x}_{t}\big)\mathrm{d}\tau, 
\end{align*}
where $J_{t-1}(x)\coloneqq\frac{\partial g_{t-1}(x)}{\partial x}\in\mathbb{R}^{d\times d}$
is the associated Jacobian matrix. 
 As a consequence, we can show that
\begin{align}
\notag & (x_{t-1}-\widehat{x}_{t})^{\top}\int_{0}^{1}\mathrm{d}\gamma\int_{x_{0}}\frac{\nabla p_{X_{t-1}\mymid X_{0}}\big(x_{t}(\gamma)\mymid x_{0}\big)p_{X_{0}}(x_{0})}{p_{X_{t-1}}\big(x_{t}(\gamma)\big)}\mathrm{d}x_{0}\\
	\notag & =-\frac{1}{1-\overline{\alpha}_{t-1}}\left\{ (x_{t-1}-\widehat{x}_{t})^{\top}g_{t-1}\big(\widehat{x}_{t}\big)+(x_{t-1}-\widehat{x}_{t})^{\top}\int_{0}^{1}\int_{0}^{1}J_{t-1}\Big((1-\tau)\widehat{x}_{t}+\tau x_{t}(\gamma)\Big)\big(x_{t}(\gamma)-\widehat{x}_{t}\big)\mathrm{d}\tau\mathrm{d}\gamma\right\} \\
	& =-\frac{1}{1-\overline{\alpha}_{t-1}}\left\{ (x_{t-1}-\widehat{x}_{t})^{\top}g_{t-1}\big(\widehat{x}_{t}\big)+(x_{t-1}-\widehat{x}_{t})^{\top}\left[\int_{0}^{1}\int_{0}^{1}\gamma J_{t-1}\Big((1-\tau)\widehat{x}_{t}+\tau x_{t}(\gamma)\Big)\mathrm{d}\tau\mathrm{d}\gamma\right]\big(x_{t-1}-\widehat{x}_{t}\big)\right\}  .
\label{eqn:schubert}
\end{align}
Combining \eqref{eqn:sde-cond} and \eqref{eqn:schubert} allows us to rewrite the target quantity $p_{X_{t-1}\mymid X_{t}}(x_{t-1}\mymid x_{t})$ as: 
\begin{align}
 & p_{X_{t-1}\mymid X_{t}}(x_{t-1}\mymid x_{t})\notag\\
 & =\frac{p_{X_{t-1}}(\widehat{x}_{t})}{p_{X_{t}}(x_{t})}\exp\Bigg((x_{t-1}-\widehat{x}_{t})^{\top}\int_{0}^{1}\mathrm{d}\gamma\int_{x_{0}}\frac{\nabla p_{X_{t-1}\mymid X_{0}}(\widetilde{x}_{t}\mymid x_{0})p_{X_{0}}(x_{0})}{p_{X_{t-1}}(\widetilde{x}_{t})}\diff x_{0}+\log p_{X_{t}\mymid X_{t-1}}(x_{t}\mymid x_{t-1})\Bigg)\notag\\
	& =\frac{p_{X_{t-1}}(\widehat{x}_{t})}{p_{X_{t}}(x_{t})}\exp\Bigg(-\frac{(x_{t-1}-\widehat{x}_{t})^{\top}g_{t-1}\big(\widehat{x}_{t}\big)+(x_{t-1}-\widehat{x}_{t})^{\top}\left[\int_{0}^{1}\int_{0}^{1}\gamma J_{t-1}\Big((1-\tau)\widehat{x}_{t}+\tau x_{t}(\gamma)\Big)\mathrm{d}\tau\mathrm{d}\gamma\right]\big(x_{t-1}-\widehat{x}_{t}\big)}{1-\overline{\alpha}_{t-1}}\notag\\
	& \qquad\qquad-\frac{\alpha_t\|x_{t-1} - \widehat{x}_{t}\|_{2}^{2}}{2(1-\alpha_{t})} -\frac{d}{2}\log\big(2\pi(1-\alpha)\big) \Bigg),\label{eqn:allegro-here}
\end{align}
where we have also used the fact that conditional on  $X_{t}\mymid X_{t-1}=x_{t-1} \sim \mathcal{N}\big(\sqrt{\overline{\alpha}_{t}}x_{t-1},(1-\overline{\alpha}_{t})I_{d}\big).$ 
Note that the pre-factor $\frac{p_{X_{t-1}}(\widehat{x}_{t})}{p_{X_{t}}(x_{t})}$ in the above display is independent from the specific value of $x_{t-1}$.

\paragraph{Step 2: controlling the exponent in \eqref{eqn:allegro-here}.}  
Consider now any $(x_t, x_{t-1})\in \mathcal{E}$ (cf.~\eqref{eqn:eset}). 
In order to further simplify the exponent in the display \eqref{eqn:allegro-here}, 
we make the following claims: 
\begin{itemize}
	\item[(a)] for any $x$ that can be written as $x=w x_{t-1} + (1-w) x_t / \sqrt{\alpha_t}$ for some $w\in[0,1]$, 
		the Jacobian matrix $J_{t-1}(x)=\frac{\partial g_{t-1}(x)}{\partial x}$ obeys
\begin{subequations}
	\label{eq:two-claims-Jacobi-approx}
\begin{align}
\label{eq:Jacobi-a}
	\big\| J_{t-1}(x) - I \big\| \lesssim d\log T ;
\end{align}
	\item[(b)] in addition, one has
\begin{align}
	\frac{1}{1-\overline{\alpha}_{t}}\Big\|\big(x_{t-1}-\widehat{x}_{t}\big)g_{t}(x_{t})-\sqrt{\alpha_{t}}\big(x_{t-1}-\widehat{x}_{t}\big)g_{t}(x_{t})\Big\|_{2} & \lesssim\frac{d\log^{2}T}{T^{3/2}}.
\label{eq:approx-t-bbb}
\\
\bigg\|\frac{g_{t-1}(\widehat{x}_{t})}{1-\overline{\alpha}_{t-1}}-\frac{g_{t}(x_{t})}{1-\overline{\alpha}_{t}}\bigg\|_{2} & \lesssim(1-\alpha_{t})\Big(\frac{d\log T}{\alpha_{t}-\overline{\alpha}_{t}}\Big)^{3/2}, 
	\label{eq:approx-t-a}  
\end{align}
\end{subequations}

\end{itemize}
Assuming the validity of these claims (which will be established in Appendix~\ref{sec:proof-claims-Jacobi-approx}) and recalling the definition of $\mu_t(\cdot)$ in \eqref{eqn:nu-t-2}, we can use \eqref{eqn:allegro-here} together with a little algebra to obtain
\begin{align}
	p_{X_{t-1}\mymid X_{t}}(x_{t-1}\mymid x_{t}) & =f_{0}(x_{t})\exp\bigg(-\frac{\alpha_{t}\|x_{t-1}-\mu_{t}^{\star}(x_t)\|_{2}^{2}}{2(1-\alpha_{t})}+\zeta_{t}(x_{t-1},x_t)\bigg)
	\label{eq:density-approx-condition-X-f0}
\end{align}
for some function $f_{0}(\cdot)$ and some residual term $\zeta_{t}(x_{t-1},x_t)$ obeying
\begin{align*}
|\zeta_{t}(x_{t-1},x_t)| & \lesssim\big\| x_{t-1}-\widehat{x}_{t}\big\|_{2}^{2}\sup_{x:\,\exists w\in[0,1]\text{ s.t. }x=(1-w)\widehat{x}_{t}+wx_{t-1}}\big\| J_{t-1}(x)\big\|\\
 & \quad+\big\| x_{t-1}-\widehat{x}_{t}\big\|_{2}\Bigg\|\frac{\int_{x_{0}}p_{X_{0}\mymid X_{t-1}}(x_{0}\mymid\widehat{x}_{t})(\widehat{x}_{t}-\sqrt{\overline{\alpha}_{t-1}}x_{0})\mathrm{d}x_{0}}{1-\overline{\alpha}_{t-1}}-\frac{\int_{x_{0}}p_{X_{0}\mymid X_{t}}(x_{0}\mymid x_{t})(x_{t}-\sqrt{\overline{\alpha}_{t}}x_{0})\mathrm{d}x_{0}}{1-\overline{\alpha}_{t}}\Bigg\|_{2}\\
 & \lesssim\frac{\big(d(1-\alpha_{t})\log T\big)d\log T}{1-\overline{\alpha}_{t-1}}+\sqrt{d(1-\alpha_{t})\log T}\,(1-\alpha_{t})\bigg(\frac{d\log T}{\alpha_{t}-\overline{\alpha}_{t}}\bigg)^{3/2}\\
 & \lesssim\frac{(1-\alpha_{t})d^{2}\log^{2}T}{1-\overline{\alpha}_{t-1}}+d^{2}\bigg(\frac{1-\alpha_{t}}{\alpha_{t}-\overline{\alpha}_{t}}\bigg)^{3/2}\log^{2}T\asymp\frac{(1-\alpha_{t})d^{2}\log^{2}T}{\alpha_{t}-\overline{\alpha}_{t}},
\end{align*}
where the penultimate line makes use of the assumption $(x_t,x_{t-1})\in \mathcal{E}$ (cf.~\eqref{eqn:eset}), 
and the last inequality holds since $\alpha_t\geq 1/2$ (cf.~\eqref{eqn:alpha-t}).

\paragraph{Step 3: approximating the function $f_0(x_t)$.} 
To finish up, it remains to quantify the function $f_0(\cdot)$ in \eqref{eq:density-approx-condition-X-f0}. 
Note that for any $x_{t}$ obeying $p_{X_{t}}(x_{t})\geq\exp\big(-\frac{1}{2}c_6d\log T\big)$,
it is easily seen that
\begin{align*}
	&\int_{x_{t-1}:(x_{t},x_{t-1})\notin\mathcal{E}}p_{X_{t-1}\mymid X_{t}}(x_{t-1}\mymid x_{t})\mathrm{d}x_{t-1}  
	=\frac{\int_{x_{t-1}:(x_{t},x_{t-1})\notin\mathcal{E}}p_{X_{t}\mymid X_{t-1}}(x_{t}\mymid x_{t-1})p_{X_{t-1}}(x_{t-1})\mathrm{d}x_{t-1}}{p_{X_{t}}(x_{t})}\\
	& \qquad \leq\frac{\frac{1}{(2\pi(1-\alpha_{t}))^{d/2}}\int_{x_{t-1}:(x_{t},x_{t-1})\notin\mathcal{E}}\exp\big(-\frac{\|x_{t}-\sqrt{\overline{\alpha}_{t}}x_{t-1}\|_{2}^{2}}{2(1-\alpha_{t})}\big)\mathrm{d}x_{t-1}}{\exp\big(-\frac{1}{2}c_6d\log T\big)}\\
 & \qquad \leq \frac{\exp\big(-c_{3}d\log T\big)}{\exp\big(-\frac{1}{2}c_{6}d\log T\big)}\leq\exp\Big(-\frac{1}{4}c_{6}d\log T\Big), 
\end{align*}
provided that $c_3\geq 3c_6/4$. 
This means that
\begin{equation}
	1\geq \int_{x_{t-1}:(x_{t},x_{t-1})\in\mathcal{E}}p_{X_{t-1}\mymid X_{t}}(x_{t-1}\mymid x_{t})\mathrm{d}x_{t-1}\geq1-\exp\left(-\frac{1}{4}c_6d\log T\right).
	\label{eq:sanwidch-bound-135}
\end{equation}
Moreover, for any $(x_{t},x_{t-1})\in\mathcal{E}$, one has
\begin{align*}
\sqrt{\alpha_{t}}\big\| x_{t-1}-\mu_{t}^{\star}(x_t)\big\|_{2} & =\sqrt{\alpha_{t}}\bigg\| x_{t-1}-\widehat{x}_{t}-\frac{1-\alpha_{t}}{\sqrt{\alpha_{t}}(1-\overline{\alpha}_{t})}\mathbb{E}\Big[x_{t}-\sqrt{\overline{\alpha}_{t}}X_{0}\mid X_{t}=x_{t}\Big]\bigg\|_{2}\\
 & \geq\sqrt{\alpha_{t}}\big\| x_{t-1}-\widehat{x}_{t}\big\|_{2}-\frac{1-\alpha_{t}}{(1-\overline{\alpha}_{t})}\mathbb{E}\Big[\big\| x_{t}-\sqrt{\overline{\alpha}_{t}}X_{0}\big\|_{2}\mid X_{t}=x_{t}\Big]\\
 & \geq c_{3}\sqrt{\alpha_{t}}\cdot\sqrt{d(1-\alpha_{t})\log T}-\frac{1-\alpha_{t}}{1-\overline{\alpha}_{t}}6\overline{c}_{5}\sqrt{d(1-\overline{\alpha}_{t})\log T}\\
 & \geq\frac{1}{2}c_{3}\sqrt{d(1-\alpha_{t})\log T},
\end{align*}
where the first identity comes from the definition \eqref{eqn:nu-t-2}
of $\mu_{t}^{\star}(x_t)$, the penultimate line makes use of the result
\eqref{eq:E-xt-X0} in Lemma~\ref{lem:x0}, and the last inequality
is valid as long as $c_{3}$ is sufficiently large. 
This in turn allows one to derive
\begin{align}
\frac{1}{\big(2\pi\frac{1-\alpha_{t}}{\alpha_{t}}\big)^{d/2}}\int_{x_{t-1}:(x_{t},x_{t-1})\in\mathcal{E}}\exp\bigg(-\frac{\alpha_{t}\|x_{t-1}-\mu_{t}^{\star}(x_t)\|_{2}^{2}}{2(1-\alpha_{t})}\bigg)\mathrm{d}x_{t-1}\geq1-\exp\big(-c_{3}d\log T\big).
\label{eq:sanwidch-bound-1357}
\end{align}

In addition, by virtue of \eqref{eq:density-approx-condition-X-f0}, the integral in \eqref{eq:sanwidch-bound-135} can be respectively bounded from above and from below as follows:
\begin{align*}
	& \int_{x_{t-1}:(x_{t},x_{t-1})\in\mathcal{E}}p_{X_{t-1}\mymid X_{t}}(x_{t-1}\mymid x_{t})\mathrm{d}x_{t-1} \\
	& \qquad =\frac{f_{0}(x_{t})\int_{x_{t-1}:(x_{t},x_{t-1})\in\mathcal{E}}\exp\Big(-\frac{\alpha_{t}}{2(1-\alpha_{t})}\big\| x_{t-1}-\mu_{t}^{\star}(x_t)\big\|_{2}^{2}+\zeta_{t}(x_{t-1},x_{t})\Big)\mathrm{d}x_{t-1}}{\big(2\pi\frac{1-\alpha_{t}}{\alpha_{t}}\big)^{-d/2}\int_{x_{t-1}}\exp\Big(-\frac{\alpha_{t}}{2(1-\alpha_{t})}\big\| x_{t-1}-\mu_{t}^{\star}(x_t)\big\|_{2}^{2}\Big)\mathrm{d}x_{t-1}}\\
 & \qquad \leq\frac{f_{0}(x_{t})\int_{x_{t-1}:(x_{t},x_{t-1})\in\mathcal{E}}\exp\Big(-\frac{\alpha_{t}}{2(1-\alpha_{t})}\big\| x_{t-1}-\mu_{t}^{\star}(x_t)\big\|_{2}^{2}+\zeta_{t}(x_{t-1},x_{t})\Big)\mathrm{d}x_{t-1}}{\big(2\pi\frac{1-\alpha_{t}}{\alpha_{t}}\big)^{-d/2}\int_{x_{t-1}:(x_{t},x_{t-1})\in\mathcal{E}}\exp\Big(-\frac{\alpha_{t}}{2(1-\alpha_{t})}\big\| x_{t-1}-\mu_{t}^{\star}(x_t)\big\|_{2}^{2}\Big)\mathrm{d}x_{t-1}}\\
 & \qquad \lesssim\frac{f_{0}(x_{t})}{\big(2\pi\frac{1-\alpha_{t}}{\alpha_{t}}\big)^{-d/2}}\exp\Bigg\{ O\bigg(d^{2}\Big(\frac{1-\alpha_{t}}{\alpha_{t}-\overline{\alpha}_{t}}\Big)\log^{2}T\bigg)\Bigg\}
\end{align*}
and 
\begin{align*}
 & \int_{x_{t-1}:(x_{t},x_{t-1})\in\mathcal{E}}p_{X_{t-1}\mymid X_{t}}(x_{t-1}\mymid x_{t})\mathrm{d}x_{t-1}\\
 & \quad\geq\big(1-\exp\left(-c_3d\log T\right)\big)\frac{f_{0}(x_{t})\int_{x_{t-1}:(x_{t},x_{t-1})\in\mathcal{E}}\exp\Big(-\frac{\alpha_{t}}{2(1-\alpha_{t})}\big\| x_{t-1}-\mu_{t}^{\star}(x_t)\big\|_{2}^{2}+\zeta_{t}(x_{t-1},x_{t})\Big)\mathrm{d}x_{t-1}}{\big(2\pi\frac{1-\alpha_{t}}{\alpha_{t}}\big)^{-d/2}\int_{x_{t-1}:(x_{t},x_{t-1})\in\mathcal{E}}\exp\Big(-\frac{\alpha_{t}}{2(1-\alpha_{t})}\big\| x_{t-1}-\mu_{t}^{\star}(x_t)\big\|_{2}^{2}\Big)\mathrm{d}x_{t-1}}\\
 & \quad\geq\big(1-\exp\left(-c_3d\log T\right)\big)\frac{f_{0}(x_{t})}{\big(2\pi\frac{1-\alpha_{t}}{\alpha_{t}}\big)^{-d/2}}\exp\Bigg\{-O\bigg(d^{2}\Big(\frac{1-\alpha_{t}}{\alpha_{t}-\overline{\alpha}_{t}}\Big)\log^{2}T\bigg)\Bigg\}.
\end{align*}
These taken collectively with \eqref{eq:sanwidch-bound-135} allow one to demonstrate that
\[
\max\Bigg\{\frac{f_{0}(x_{t})}{\big(2\pi\frac{1-\alpha_{t}}{\alpha_{t}}\big)^{-d/2}},\frac{\big(2\pi\frac{1-\alpha_{t}}{\alpha_{t}}\big)^{-d/2}}{f_{0}(x_{t})}\Bigg\}
= \exp\Bigg\{ O\bigg(d^{2}\Big(\frac{1-\alpha_{t}}{\alpha_{t}-\overline{\alpha}_{t}}\Big)\log^{2}T\bigg)\Bigg\}=1+O\bigg(d^{2}\Big(\frac{1-\alpha_{t}}{\alpha_{t}-\overline{\alpha}_{t}}\Big)\log^{2}T\bigg),
\]
with the proviso that $d^{2}\big(\frac{1-\alpha_{t}}{\alpha_{t}-\overline{\alpha}_{t}}\big)\log^{2}T\lesssim 1$.

Combining this with \eqref{eq:density-approx-condition-X-f0} concludes the proof of Lemma~\ref{lem:sde}, as long as the two claims in \eqref{eq:two-claims-Jacobi-approx} are valid (to be justified in Appendix~\ref{sec:proof-claims-Jacobi-approx}).

\subsubsection{Proof of auxiliary claims \eqref{eq:two-claims-Jacobi-approx} in Lemma~\ref{lem:sde}}
\label{sec:proof-claims-Jacobi-approx}

\paragraph{Proof of relation~\eqref{eq:Jacobi-a}.} 
%
%
Recall from \eqref{eq:Jt-x-expression-ij-23} that
\begin{align*}
J_{t-1}(x)-I & =\frac{1}{1-\overline{\alpha}_{t-1}}\mathbb{E}\left[x-\sqrt{\overline{\alpha}_{t-1}}X_{0}\mid X_{t-1}=x\right]\Big(\mathbb{E}\left[x-\sqrt{\overline{\alpha}_{t-1}}X_{0}\mid X_{t-1}=x\right]\Big)^{\top}\\
 & \qquad-\frac{1}{1-\overline{\alpha}_{t-1}}\mathbb{E}\left[\big(x-\sqrt{\overline{\alpha}_{t-1}}X_{0}\big)\big(x-\sqrt{\overline{\alpha}_{t-1}}X_{0}\big)^{\top}\mid X_{t-1}=x\right].
\end{align*}
Recognizing that 
\begin{align*}
\left\Vert \mathbb{E}\big[ZZ^{\top}\big]-\mathbb{E}[Z]\mathbb{E}[Z]^{\top}\right\Vert  & =\left\Vert \mathbb{E}\Big[\big(Z-\mathbb{E}[Z]\big)\big(Z-\mathbb{E}[Z]\big)^{\top}\Big]\right\Vert \leq\left\Vert \mathbb{E}\big[ZZ^{\top}\big]\right\Vert 
  \leq\mathbb{E}\big[\left\Vert ZZ^{\top}\right\Vert \big]=\mathbb{E}\big[\|Z\|_{2}^{2}\big]
\end{align*}
for any random vector $Z$, we can readily obtain
\begin{align}
\big\| J_{t-1}(x)-I\big\| & \leq\frac{1}{1-\overline{\alpha}_{t-1}}\mathbb{E}\left[\big\| x-\sqrt{\overline{\alpha}_{t-1}}X_{0}\big\|_{2}^{2}\mid X_{t-1}=x\right].
	\label{eq:Jg-x-bound-temp}
\end{align}

When $(x_{t},x_{t-1})\in \mathcal{E}$, 
it follows from Lemma~\ref{lem:river} that 
\begin{equation}
	-\log p_{X_{t-1}}(x) \leq c_6 d \log T
\end{equation}
for any $x$ lying in the line segment connecting $x_{t-1}$ and $\widehat{x}_t$. 
With this result in place, 
taking \eqref{eq:Jg-x-bound-temp} together with the bound \eqref{eq:E2-xt-X0} in Lemma~\ref{lem:x0} immediately leads to
\begin{align*}
\big\| J_{t-1}(x)-I\big\| & \lesssim\frac{1}{1-\overline{\alpha}_{t-1}}\cdot\left\{ d\big(1-\overline{\alpha}_{t-1}\big)\log T\right\} \asymp d\log T
\end{align*}
for any $x$ lying within the line segment between $x_{t-1}$ and
$x_{t}/\sqrt{\alpha_{t}}$, as claimed.

\paragraph{Proof of relation~\eqref{eq:approx-t-bbb}.} 
To prove this result, we observe that
\begin{align*}
	& \left|\left(1-\sqrt{\alpha_{t}}\right)\frac{(x_{t-1}-\widehat{x}_{t})^{\top}g_t(x_t)}{1-\overline{\alpha}_{t}}\right|\leq\frac{1-\alpha_{t}}{1+\sqrt{\alpha_{t}}}\frac{\|x_{t-1}-\widehat{x}_{t}\|_{2}\mathbb{E}\big[\|x_{t}-\sqrt{\overline{\alpha}_{t}}X_{0}\|_{2}\mid X_{t}=x_{t}\big]}{1-\overline{\alpha}_{t}}\\
 & \qquad\lesssim\frac{\log T}{T}\cdot\sqrt{d(1-\alpha_{t})\log T}\cdot\sqrt{d(1-\overline{\alpha}_{t})\log T}\lesssim\frac{d\log^{2}T}{T^{3/2}}, 
\end{align*}
where the last line comes from Lemma~\ref{lem:x0} as well as the basic property \eqref{eqn:properties-alpha-proof-00} about the learning rates.

\paragraph{Proof of relation~\eqref{eq:approx-t-a}.} 
 To begin with, the triangle inequality together with the fact $\overline{\alpha}_t=\prod_{k=1}^t \alpha_k$ gives
\begin{align}
\notag & \bigg\|\frac{\int_{x_{0}}p_{X_{0}\mymid X_{t-1}}(x_{0}\mymid\widehat{x}_{t})(\widehat{x}_{t}-\sqrt{\overline{\alpha}_{t-1}}x_{0})\mathrm{d}x_{0}}{1-\overline{\alpha}_{t-1}}-\frac{\int_{x_{0}}p_{X_{0}\mymid X_{t}}(x_{0}\mymid x_{t})(x_{t}-\sqrt{\overline{\alpha}_{t}}x_{0})\mathrm{d}x_{0}}{1-\overline{\alpha}_{t}}\bigg\|_{2}\\
\notag & \le\bigg\|\frac{\int_{x_{0}}p_{X_{0}\mymid X_{t-1}}(x_{0}\mymid\widehat{x}_{t})(x_{t}-\sqrt{\overline{\alpha}_{t}}x_{0})\mathrm{d}x_{0}-\int_{x_{0}}p_{X_{0}\mymid X_{t}}(x_{0}\mymid x_{t})(x_{t}-\sqrt{\overline{\alpha}_{t}}x_{0})\mathrm{d}x_{0}}{\sqrt{\alpha_{t}}(1-\overline{\alpha}_{t-1})}\bigg\|_{2}\\
 & \qquad+\bigg\|\Big(\frac{1}{\sqrt{\alpha_{t}}(1-\overline{\alpha}_{t-1})}-\frac{1}{1-\overline{\alpha}_{t}}\Big)\int_{x_{0}}p_{X_{0}\mymid X_{t}}(x_{0}\mymid x_{t})(x_{t}-\sqrt{\overline{\alpha}_{t}}x_{0})\mathrm{d}x_{0}\bigg\|_{2}.\label{eqn:adagio}
\end{align}

Let us first consider the last term in \eqref{eqn:adagio}. 
According to Lemma~\ref{lem:x0}, given that $-\log p_{X_t}(x_t) \leq  \frac{1}{2}c_6 d\log T$ for some constant $c_6>0$,  one has 
\begin{align}
\int_{x_0}  p_{X_0\mymid X_{t}}(x_0\mymid x_t) \big\|x_t - \sqrt{\overline{\alpha}_{t}} x_0 \big\|_2\mathrm{d} x_0 
%
	%
	\lesssim & \sqrt{d(1 - \overline{\alpha}_t)\log T}. 
	\label{eqn:rondo}
\end{align}
This in turn reveals that
\begin{align}
 & \bigg\|\bigg(\frac{1}{\sqrt{\alpha_{t}}(1-\overline{\alpha}_{t-1})}-\frac{1}{1-\overline{\alpha}_{t}}\bigg)\int_{x_{0}}p_{X_{0}\mymid X_{t}}(x_{0}\mymid x_{t})(x_{t}-\sqrt{\overline{\alpha}_{t}}x_{0})\mathrm{d}x_{0}\bigg\|_{2}\nonumber\\
 & \qquad\lesssim\bigg|\frac{1}{\sqrt{\alpha_{t}}(1-\overline{\alpha}_{t-1})}-\frac{1}{1-\overline{\alpha}_{t}}\bigg|\cdot\sqrt{d(1-\overline{\alpha}_{t})\log T}\nonumber\\
 & \qquad\asymp\frac{(1-\sqrt{\alpha_{t}})(1+\overline{\alpha}_{t-1}\sqrt{\alpha_{t}})}{\sqrt{\alpha_{t}}(1-\overline{\alpha}_{t-1})(1-\overline{\alpha}_{t})}\sqrt{d(1-\overline{\alpha}_{t})\log T} \notag\\
 & \qquad\asymp\frac{1-\alpha_{t}}{(\alpha_{t}-\overline{\alpha}_{t})^{3/2}}\sqrt{d\log T},\label{eq:bound-last-term-135}
\end{align}
where the last inequality makes use of the properties \eqref{eqn:properties-alpha-proof}.

Next, we turn attention to the first term in \eqref{eqn:adagio}, which relies on the following claim. 
\begin{claim}
\label{lem:claim-pXt-1-pXt-middle}
Consider any point $x_t$ obeying $-\log p_{X_t}(x_t) \leq c_6 d\log T$ for some large constant $c_6>0$. 
One has
\begin{subequations}
\label{eq:claim-pXt-1-pXt-middle}
\begin{align}
	p_{X_{t-1}}(\widehat{x}_t) 
	&= \bigg(1 + O\Big(\frac{d(1-\alpha_t)\log T}{1-\overline{\alpha}_{t-1}}\Big)\bigg)p_{X_{t}}(x_t),
\end{align}
In addition, by defining the following set 
$$ \mathcal{E}_1 \coloneqq \big\{ x: \|x_t - \sqrt{\overline{\alpha}_{t}}x\|_2 \leq c_4 \sqrt{d(1-\overline{\alpha}_t)\log T} \big\}$$ 
for some large enough constant $c_4>0$, we have
\begin{align}
\frac{p_{X_{t-1}\mymid X_{0}}(\widehat{x}_{t}\mymid x_{0})}{p_{X_{t}\mymid X_{0}}(x_{t}\mymid x_{0})} & =1+O\bigg(\frac{d(1-\alpha_{t})\log T}{1-\overline{\alpha}_{t-1}}\bigg),\qquad
	&&\text{if }x_{0}\in\mathcal{E}_{1},\\
	\frac{p_{X_{t-1}\mymid X_{0}}(\widehat{x}_{t}\mymid x_{0})}{p_{X_{t}\mymid X_{0}}(x_{t}\mymid x_{0})} & \leq\exp\Bigg(\frac{16c_{1}\big\| x_{t}-\sqrt{\overline{\alpha}_{t}}x_{0}\big\|_{2}^{2}\log T}{(1-\overline{\alpha}_{t})T}\Bigg),\qquad&&\text{if }x_{0}\notin\mathcal{E}_{1}.
\end{align}
\end{subequations}
\end{claim}
As an immediate consequence of this claim, one has
\begin{align}
\frac{p_{X_{0}\mymid X_{t-1}}(x_{0}\mymid\widehat{x}_{t})}{p_{X_{0}\mymid X_{t}}(x_{0}\mymid x_{t})} & =\frac{p_{X_{t-1}\mymid X_{0}}(\widehat{x}_{t}\mymid x_{0})}{p_{X_{t}\mymid X_{0}}(x_{t}\mymid x_{0})}\cdot\frac{p_{X_{t}}(x_{t})}{p_{X_{t-1}}(\widehat{x}_{t})}=\left(1+O\bigg(\frac{d(1-\alpha_{t})\log T}{1-\overline{\alpha}_{t-1}}\bigg)\right)\frac{p_{X_{t-1}\mymid X_{0}}(\widehat{x}_{t}\mymid x_{0})}{p_{X_{t}\mymid X_{0}}(x_{t}\mymid x_{0})} \notag\\
 & \begin{cases}
=1+O\Big(\frac{d(1-\alpha_{t})\log T}{1-\overline{\alpha}_{t-1}}\Big), & \text{if }x_{0}\in\mathcal{E}_{1},\\
\leq \exp\Big(\frac{16c_{1} \| x_{t}-\sqrt{\overline{\alpha}_{t}}x_{0}\|_{2}^{2}\log T}{(1-\overline{\alpha}_{t})T}\Big),\quad & \text{if }x_{0}\notin\mathcal{E}_{1}.
\end{cases}
\label{eq:ratio-12345}
\end{align}
In turn, this allows one to deduce that
\begin{align*}
\notag & \bigg\|\int_{x_{0}}p_{X_{0}\mymid X_{t-1}}(x_{0}\mymid\widehat{x}_{t})(x_{t}-\sqrt{\overline{\alpha}_{t}}x_{0})\mathrm{d}x_{0}-\int_{x_{0}}p_{X_{0}\mymid X_{t}}(x_{0}\mymid x_{t})(x_{t}-\sqrt{\overline{\alpha}_{t}}x_{0})\mathrm{d}x_{0}\bigg\|_{2}\\
\notag & \quad\le\int_{x_{0}\in\mathcal{E}_{1}}\big|p_{X_{0}\mymid X_{t-1}}(x_{0}\mymid\widehat{x}_{t})-p_{X_{0}\mymid X_{t}}(x_{0}\mymid x_{t})\big|\cdot\big\| x_{t}-\sqrt{\overline{\alpha}_{t}}x_{0}\big\|_{2}\mathrm{d}x_{0}\\
 & \qquad\qquad+\int_{x_{0}\notin\mathcal{E}_{1}}\big|p_{X_{0}\mymid X_{t-1}}(x_{0}\mymid\widehat{x}_{t})-p_{X_{0}\mymid X_{t}}(x_{0}\mymid x_{t})\big|\cdot\big\| x_{t}-\sqrt{\overline{\alpha}_{t}}x_{0}\big\|_{2}\mathrm{d}x_{0}\\
 & \quad\le\int_{x_{0}}O\bigg(\frac{d(1-\alpha_{t})\log T}{1-\overline{\alpha}_{t-1}}\bigg)p_{X_{0}\mymid X_{t}}(x_{0}\mymid x_{t})\big\| x_{t}-\sqrt{\overline{\alpha}_{t}}x_{0}\big\|_{2}\mathrm{d}x_{0}\\
 & \qquad\qquad+\int_{x_{0}\notin\mathcal{E}_{1}}\exp\Bigg(\frac{16c_{1}\big\| x_{t}-\sqrt{\overline{\alpha}_{t}}x_{0}\big\|_{2}^{2}\log T}{(1-\overline{\alpha}_{t})T}\Bigg)p_{X_{0}\mymid X_{t}}(x_{0}\mymid x_{t})\big\| x_{t}-\sqrt{\overline{\alpha}_{t}}x_{0}\big\|_{2}\mathrm{d}x_{0}\\
 & \quad\le O\bigg(\frac{d(1-\alpha_{t})\log T}{1-\overline{\alpha}_{t-1}}\bigg)\mathbb{E}\left[\big\| x_{t}-\sqrt{\overline{\alpha}_{t}}X_{0}\big\|_{2}\mid X_{t}=x_{0}\right]\\
 & \qquad\qquad+\int_{x_{0}\notin\mathcal{E}_{1}}\exp\Bigg(\frac{20c_{1}\big\| x_{t}-\sqrt{\overline{\alpha}_{t}}x_{0}\big\|_{2}^{2}\log T}{(1-\overline{\alpha}_{t})T}\Bigg)p_{X_{0}\mymid X_{t}}(x_{0}\mymid x_{t})\mathrm{d}x_{0}.
\end{align*}
Invoking \eqref{eq:E-xt-X0} in Lemma~\ref{lem:x0} as well as
similar arugments for \eqref{eqn:rachmaninoff-2} (assuming that $c_{4}$
is sufficiently large), we arrive at
\begin{align}
\notag & \frac{1}{\sqrt{\alpha_{t}}(1-\overline{\alpha}_{t-1})}\bigg\|\int_{x_{0}}p_{X_{0}\mymid X_{t-1}}(x_{0}\mymid\widehat{x}_{t})(x_{t}-\sqrt{\overline{\alpha}_{t}}x_{0})\mathrm{d}x_{0}-\int_{x_{0}}p_{X_{0}\mymid X_{t}}(x_{0}\mymid x_{t})(x_{t}-\sqrt{\overline{\alpha}_{t}}x_{0})\mathrm{d}x_{0}\bigg\|_{2}\\
 & \quad\quad\lesssim\frac{d^{3/2}(1-\alpha_{t})\sqrt{1-\overline{\alpha}_{t}}\log^{3/2}T}{\sqrt{\alpha_{t}}(1-\overline{\alpha}_{t-1})^{2}}+\frac{1}{\sqrt{\alpha_{t}}(1-\overline{\alpha}_{t-1})}\cdot\frac{1}{T^{c_{0}}} \notag\\
 & \quad\quad\asymp\frac{d^{3/2}(1-\alpha_{t})\log^{3/2}T}{(1-\overline{\alpha}_{t-1})^{3/2}},
	\label{eq:case2-1359}
\end{align}
where $c_0$ is some large enough constant, and we have also made use of the properties in \eqref{eqn:properties-alpha-proof}.

Substituting \eqref{eq:bound-last-term-135} and \eqref{eq:case2-1359} into \eqref{eqn:adagio} readily concludes the proof.


\begin{proof}[Proof of Claim~\ref{lem:claim-pXt-1-pXt-middle}]
%
%
%
Let us make the following observations: for any $x_0\in \mathcal{E}_1$, one has
\begin{align}
\label{eqn:lotus}
\notag \frac{p_{X_{t-1}\mymid X_{0}}(\widehat{x}_{t}\mymid x_{0})}{p_{X_{t}\mymid X_{0}}(x_{t}\mymid x_{0})} & =\bigg(\frac{1-\overline{\alpha}_{t}}{1-\overline{\alpha}_{t-1}}\bigg)^{d/2}\exp\bigg(\frac{\big\| x_{t}-\sqrt{\overline{\alpha}_{t}}x_{0}\big\|_{2}^{2}}{2(1-\overline{\alpha}_{t})}-\frac{\big\|\widehat{x}_{t}-\sqrt{\overline{\alpha}_{t-1}}x_{0}\big\|_{2}^{2}}{2(1-\overline{\alpha}_{t-1})}\bigg)\\
\notag & =\exp\Bigg\{\big(1+o(1)\big)\frac{\overline{\alpha}_{t-1}(1-\alpha_{t})}{1-\overline{\alpha}_{t-1}}\cdot\frac{d}{2}\Bigg\}\exp\bigg(-\frac{(1-\alpha_{t})\big\| x_{t}-\sqrt{\overline{\alpha}_{t}}x_{0}\big\|_{2}^{2}}{2(1-\overline{\alpha}_{t})(\alpha_{t}-\overline{\alpha}_{t})}\bigg)\\
\notag & =\exp\Bigg\{ O\bigg(\frac{d(1-\alpha_{t})}{1-\overline{\alpha}_{t-1}}+\frac{(1-\alpha_{t})\big\| x_{t}-\sqrt{\overline{\alpha}_{t}}x_{0}\big\|_{2}^{2}}{(1-\overline{\alpha}_{t-1})(1-\overline{\alpha}_{t})}\bigg)\Bigg\}\\
\notag & =\exp\Bigg\{ O\bigg(\frac{d(1-\alpha_{t})}{1-\overline{\alpha}_{t-1}}+\frac{d(1-\alpha_{t})\log T}{1-\overline{\alpha}_{t-1}}\bigg)\Bigg\}\\
 & =1+O\bigg(\frac{d(1-\alpha_{t})\log T}{1-\overline{\alpha}_{t-1}}\bigg), 
\end{align}
where the second line holds since 
\[
\log\frac{1-\overline{\alpha}_{t}}{1-\overline{\alpha}_{t-1}}=\log\bigg(1+\frac{\overline{\alpha}_{t-1}(1-\alpha_{t})}{1-\overline{\alpha}_{t-1}}\bigg)=\big(1+o(1)\big)\frac{\overline{\alpha}_{t-1}(1-\alpha_{t})}{1-\overline{\alpha}_{t-1}}, 
\]
and we have also made use of \eqref{eqn:properties-alpha-proof} (so that $\frac{d(1-\alpha_{t})\log T}{1-\overline{\alpha}_{t-1}}=o(1)$ under our assumption on $T$). 
Additionally, for any $x_0 \notin \mathcal{E}_1$, it follows from the argument for \eqref{eq:density-ratio-eqn-river} that
\begin{align*}
\frac{p_{X_{t-1}\mymid X_{0}}(\widehat{x}_{t}\mymid x_{0})}{p_{X_{t}\mymid X_{0}}(x_{t}\mymid x_{0})} & \leq\exp\Bigg(\frac{d(1-\alpha_{t})}{2(1-\overline{\alpha}_{t-1})}+\frac{(1-\alpha_{t})\big\| x_{t}-\sqrt{\overline{\alpha}_{t}}x_{0}\big\|_{2}^{2}}{(1-\overline{\alpha}_{t-1})^{2}}\Bigg)\\
 & \leq\exp\Bigg(\frac{(1-\alpha_{t})\big\| x_{t}-\sqrt{\overline{\alpha}_{t}}x_{0}\big\|_{2}^{2}}{2c_{4}^{2}(1-\overline{\alpha}_{t-1})^{2}\log T}+\frac{(1-\alpha_{t})\big\| x_{t}-\sqrt{\overline{\alpha}_{t}}x_{0}\big\|_{2}^{2}}{(1-\overline{\alpha}_{t-1})^{2}}\Bigg)\\
 & \leq\exp\Bigg(\frac{8c_{1}\big\| x_{t}-\sqrt{\overline{\alpha}_{t}}x_{0}\big\|_{2}^{2}\log T}{(1-\overline{\alpha}_{t-1})T}\Bigg)
	\leq\exp\Bigg(\frac{16c_{1}\big\| x_{t}-\sqrt{\overline{\alpha}_{t}}x_{0}\big\|_{2}^{2}\log T}{(1-\overline{\alpha}_{t})T}\Bigg), 
\end{align*}
where the last line comes from \eqref{eqn:properties-alpha-proof-1} and \eqref{eqn:properties-alpha-proof-3}.  
Taking the above bounds together and invoking the same calculation as in \eqref{eqn:rachmaninoff}, \eqref{eqn:rachmaninoff-2} and \eqref{eqn:prokofiev-2}, we reach
\begin{align*}
	p_{X_{t-1}}(\widehat{x}_t) 
&= \bigg(1 + O\Big(\frac{d(1-\alpha_t)\log T}{1-\overline{\alpha}_{t-1}}\Big)\bigg)p_{X_{t}}(x_t), 
\end{align*}
thus concluding the proof of  Claim~\ref{lem:claim-pXt-1-pXt-middle}. 
\end{proof}




\subsection{Proof of Lemma~\ref{lem:sde-full}}
\label{sec:proof-lem:sde-full}


First, recognizing that $X_t$ follows a Gaussian distribution when conditioned on $X_{t-1}=x_{t-1}$, we can derive
\begin{align*}
\log p_{X_{t-1}\mymid X_{t}}(x_{t-1}\mymid x_{t}) & =\log\frac{p_{X_{t}\mymid X_{t-1}}(x_{t}\mymid x_{t-1})p_{X_{t-1}}(x_{t-1})}{p_{X_{t}}(x_{t})}\notag\\
 & =\log\frac{p_{X_{t-1}}(x_{t-1})}{p_{X_{t}}(x_{t})}+\frac{\|x_{t}-\sqrt{{\alpha_{t}}}x_{t-1}\|_{2}^{2}}{2(1-\alpha_{t})}-\frac{d}{2}\log\big(2\pi(1-\alpha_{t})\big)\\
 & \leq\log\frac{p_{X_{t-1}}(x_{t-1})}{p_{X_{t}}(x_{t})}+T\big(\|x_{t-1}-\widehat{x}_{t}\|_{2}^{2}+1\big),
\end{align*}
where the last inequality makes use of the properties \eqref{eqn:properties-alpha-proof} about $\alpha_t$ 
(recall that $\widehat{x}_{t}$ has been defined in \eqref{eq:defn-xt-proof-thm3}). 
Some direct calculations then yield 
\begin{align*}
\log\frac{p_{X_{t-1}}(x_{t-1})}{p_{X_{t}}(x_{t})} & =\log\frac{\int_{x_{0}}p_{X_{0}}(x_{0})\exp\Big(-\frac{\|x_{t-1}-\sqrt{\overline{\alpha}_{t-1}}x_{0}\|_{2}^{2}}{2(1-\overline{\alpha}_{t-1})}\Big)\mathrm{d}x_{0}}{\int_{x_{0}}p_{X_{0}}(x_{0})\exp\Big(-\frac{\|x_{t}-\sqrt{\overline{\alpha}_{t}}x_{0}\|_{2}^{2}}{2(1-\overline{\alpha}_{t})}\Big)\mathrm{d}x_{0}}-\frac{d}{2}\log\Big(\frac{1-\overline{\alpha}_{t-1}}{1-\overline{\alpha}_{t}}\Big)\\
 & \leq\sup_{x_{0}:\|x_{0}\|_{2}\leq T^{c_{R}}}\Bigg\{\frac{\|x_{t}-\sqrt{\overline{\alpha}_{t}}x_{0}\|_{2}^{2}}{2(1-\overline{\alpha}_{t})}-\frac{\|x_{t-1}-\sqrt{\overline{\alpha}_{t-1}}x_{0}\|_{2}^{2}}{2(1-\overline{\alpha}_{t-1})}\Bigg\}+\frac{d}{2}\log2\\
 & \leq\sup_{x_{0}:\|x_{0}\|_{2}\leq T^{c_{R}}}\frac{\|x_{t}-\sqrt{\overline{\alpha}_{t}}x_{0}\|_{2}^{2}}{2(1-\overline{\alpha}_{t})}+\frac{d}{2}\log2\\
 & \leq\sup_{x_{0}:\|x_{0}\|_{2}\leq T^{c_{R}}}\frac{\|x_{t}\|_{2}^{2}+\|x_{0}\|_{2}^{2}}{1-\overline{\alpha}_{t}}+\frac{d}{2}\log2\\
 & \leq2T\left(\|x_{t}\|_{2}^{2}+T^{2c_{R}}\right),
\end{align*}
where the second line makes use of the properties in \eqref{eqn:properties-alpha-proof}. 
Combining the above two relations, we arrive at 
\begin{align}
\label{eqn:x-condi}
\log p_{X_{t-1} \mymid X_t}(x_{t-1}\mymid x_t) 
&\leq 2T \big( \| x_t\|_2^2 + \|x_{t-1} - \widehat{x}_t\|_2^2 + T^{2c_{R}} \big) .
\end{align}

We then turn attention to the conditional distribution $p_{Y^{\star}_{t-1} \mymid Y_t}$. 
First, it follows from \eqref{eq:dist-Ytstar} that
%
%
%
\begin{align}
\label{eqn:y-condi}
\log\frac{1}{p_{Y_{t-1}^{\star}\mymid Y_{t}}(x_{t-1}\mymid x_{t})} & =\frac{\alpha_{t}\|x_{t-1}-\mu_{t}^{\star}(x_t)\|_{2}^{2}}{2(1-\alpha_{t})}+\frac{d}{2}\log\bigg(2\pi\frac{1-\alpha_{t}}{\alpha_{t}}\bigg)\\
 & \leq T\left(\|x_{t-1}-\widehat{x}_{t}\|_{2}^{2}+\|x_{t}\|_{2}^{2}+T^{2c_{R}}\right)+d\log T,  
\end{align}
with $\mu_t(x_t)$ defined in \eqref{eqn:nu-t-2}. 
To see why the last inequality follows, we recall from \eqref{eqn:nu-t-2} that
\begin{align}
\|x_{t-1}-\mu_{t}^{\star}(x_t)\|_{2}^{2} & \leq2\|x_{t-1}-\widehat{x}_{t}\|_{2}^{2}+2\|\widehat{x}_{t}-\mu_{t}^{\star}(x_t)\|_{2}^{2} \notag\\
 & =2\|x_{t-1}-\widehat{x}_{t}\|_{2}^{2}+2\bigg(\frac{1-\alpha_{t}}{\sqrt{\alpha_{t}}(1-\overline{\alpha}_{t})}\bigg)^{2}\Big\|\int_{x_{0}}p_{X_{0}\mymid X_{t}}(x_{0}\mymid x_{t})\big(x_{t}-\sqrt{\overline{\alpha}_{t}}x_{0}\big)\mathrm{d}x_{0}\Big\|_{2}^{2}\notag\\
 & \leq2\|x_{t-1}-\widehat{x}_{t}\|_{2}^{2}+\frac{2(1-\alpha_{t})^{2}}{\alpha_{t}(1-\overline{\alpha}_{t-1})^{2}}\sup_{x_{0}:\|x_{0}\|_{2}\leq T^{c_{R}}}\|x_{t}-\sqrt{\overline{\alpha}_{t}}x_{0}\|_{2}^{2}\notag\\
 & \leq2\|x_{t-1}-\widehat{x}_{t}\|_{2}^{2}+\frac{64c_{1}^{2}\log^{2}T}{T^{2}}\bigg(2\|x_{t}\|_{2}^{2}+2\overline{\alpha}_{t}T^{2c_{R}}\bigg)\notag\\
 & \leq2\|x_{t-1}-\widehat{x}_{t}\|_{2}^{2}+\|x_{t}\|_{2}^{2}+T^{2c_{R}},
	\label{eq:dist-xt-mu-xt-UB-crude}
\end{align}
where we have invoked the properties \eqref{eqn:properties-alpha-proof}.

Putting everything together, we arrive at the advertised crude bound: 
\[
\log\frac{p_{X_{t-1}\mymid X_{t}}(x_{t-1}\mymid x_{t})}{p_{Y^{\star}_{t-1}\mymid Y_{t}}(x_{t-1}\mymid x_{t})}\leq2T\left(\|x_{t-1}-\widehat{x}_{t}\|_{2}^{2}+\|x_{t}\|_{2}^{2}+T^{2c_{R}}\right).
\]

\subsection{Proof of Lemma~\ref{lem:influence-error-KL}} 
\label{sec:proof-lem:influence-error-KL}

To begin with, we make the observation that
\begin{align}
 & \mathbb{E}_{x_{t}\sim q_{t}}\Big[\mathsf{KL}\Big(p_{X_{t-1}\mymid X_{t}}(\cdot\mymid x_{t})\parallel p_{Y_{t-1}\mymid Y_{t}}(\cdot\mymid x_{t})\Big)\Big] - \mathbb{E}_{x_{t}\sim q_{t}}\Big[\mathsf{KL}\Big(p_{X_{t-1}\mymid X_{t}}(\cdot\mymid x_{t})\parallel p_{Y_{t-1}^{\star}\mymid Y_{t}}(\cdot\mymid x_{t})\Big)\Big] \notag\\
 & = \int p_{X_{t}}(x_{t})p_{X_{t-1}\mymid X_{t}}(x_{t-1}\mymid x_{t})\log\frac{p_{Y_{t-1}^{\star}\mymid Y_{t}}(x_{t-1}\mymid x_{t})}{p_{Y_{t-1}\mymid Y_{t}}(x_{t-1}\mymid x_{t})}\mathrm{d}x_{t-1}\mathrm{d}x_{t} \notag\\
 & = \int p_{X_{t}}(x_{t})p_{X_{t-1}\mymid X_{t}}(x_{t-1}\mymid x_{t})\frac{\alpha_{t}}{2(1-\alpha_{t})}\Big(\big\| x_{t-1}-\mu_{t}(x_t)\big\|_{2}^{2} - \big\| x_{t-1}-\mu_{t}^{\star}(x_t)\big\|_{2}^{2}\Big)\mathrm{d}x_{t-1}\mathrm{d}x_{t} \notag\\
	& = \underset{\eqqcolon\, \mathcal{H}_t}{\underbrace{ \int p_{X_{t}}(x_{t})p_{X_{t-1}\mymid X_{t}}(x_{t-1}\mymid x_{t})\bigg(\frac{1-\alpha_t}{2}\varepsilon_{\score, t}(x_t)^2 - \sqrt{\alpha_t}\big(x_{t-1} - \mu_{t}^{\star}(x_t)\big)^{\top}\big(s_t(x) - s_t^{\star}(x_t)\big)\bigg)\mathrm{d}x_{t-1}\mathrm{d}x_{t} }} , \label{eqn:tmp-sde-brahms} 
\end{align}
where we have used \eqref{eq:dist-Yt-Yt-1} and \eqref{eq:dist-Ytstar} as well as the definition \eqref{eq:pointwise-epsilon-score-J} of $\varepsilon_{\score, t}(x)$. 
%
Everything comes down to proving that
\begin{align}
	\mathcal{H}_t
	\lesssim \exp\Big(- \frac{\min\{c_{3},c_{6}\}}{8} d\log T \Big) + \frac{d\log^3 T}{T}\mathbb{E}_{X_t\sim q_t}\big[\varepsilon_{\score, t}(X_t)^2\big]. 
	\label{eq:H-upper-bound-DDPM}
\end{align}

Towards this end, one first notes that the first component of $\mathcal{H}_t$ obeys 
\begin{align}
	\int p_{X_{t}}(x_{t})p_{X_{t-1}\mymid X_{t}}(x_{t-1}\mymid x_{t})\frac{1-\alpha_{t}}{2}\varepsilon_{\score,t}(x_{t})^{2}\mathrm{d}x_{t-1}\mathrm{d}x_{t} & =\frac{1-\alpha_{t}}{2}\mathbb{E}_{X_{t}\sim q_{t}}\big[\varepsilon_{\score,t}(X_{t})^{2}\big] \notag\\
 & \lesssim\frac{\log T}{T}\mathbb{E}_{X_{t}\sim q_{t}}\big[\varepsilon_{\score,t}(X_{t})^{2}\big],
	\label{eq:Ht-first-component-bound}
\end{align}
where the last inequality comes from \eqref{eqn:properties-alpha-proof}. 
We then switch attention to the second component of $\mathcal{H}_t$. 
In view of the distribution of $Y_{t-1}^{\star}\mymid Y_{t}$ in \eqref{eq:dist-Ytstar}, we obtain 
\begin{align}
 & \int p_{X_{t}}(x_{t})p_{X_{t-1}\mymid X_{t}}(x_{t-1}\mymid x_{t})\big(x_{t-1}-\mu_{t}^{\star}(x_{t})\big)^{\top}\big(s_{t}(x_{t})-s_{t}^{\star}(x_{t})\big)\mathrm{d}x_{t-1}\mathrm{d}x_{t}\notag\\
 & =\int p_{X_{t}}(x_{t})\big(p_{X_{t-1}\mymid X_{t}}(x_{t-1}\mymid x_{t})-p_{Y_{t-1}^{\star}\mymid Y_{t}}(x_{t-1}\mymid x_{t})\big)\big(x_{t-1}-\mu_{t}^{\star}(x_{t})\big)^{\top}\big(s_{t}(x_{t})-s_{t}^{\star}(x_{t})\big)\mathrm{d}x_{t-1}\mathrm{d}x_{t}\notag\\
 & \quad+\int_{x_{t}}p_{X_{t}}(x_{t})\left\{ \int_{x_{t-1}}p_{Y_{t-1}^{\star}\mymid Y_{t}}\big(x_{t-1}\mymid x_{t}\big)\big(x_{t-1}-\mu_{t}^{\star}(x_{t})\big)^{\top}\mathrm{d}x_{t-1}\right\} \big(s_{t}(x_{t})-s_{t}^{\star}(x_{t})\big)\mathrm{d}x_{t}\notag\\
 & \stackrel{(\text{i})}{=} \int p_{X_{t}}(x_{t})p_{X_{t-1}\mymid X_{t}}(x_{t-1}\mymid x_{t})\bigg(1-\frac{p_{Y_{t-1}^{\star}\mymid Y_{t}}(x_{t-1}\mymid x_{t})}{p_{X_{t-1}\mymid X_{t}}(x_{t-1}\mymid x_{t})}\bigg)\big(x_{t-1}-\mu_{t}^{\star}(x_{t})\big)^{\top}\big(s_{t}(x_{t})-s_{t}^{\star}(x_{t})\big)\mathrm{d}x_{t-1}\mathrm{d}x_{t}\notag\\
 & \leq\underset{\eqqcolon\,\mathcal{K}_{1}}{\underbrace{\int_{\mathcal{E}}p_{X_{t}}(x_{t})p_{X_{t-1}\mymid X_{t}}(x_{t-1}\mymid x_{t})\bigg(1-\frac{p_{Y_{t-1}^{\star}\mymid Y_{t}}(x_{t-1}\mymid x_{t})}{p_{X_{t-1}\mymid X_{t}}(x_{t-1}\mymid x_{t})}\bigg)\big(x_{t-1}-\mu_{t}^{\star}(x_{t})\big)^{\top}\big(s_{t}(x_{t})-s_{t}^{\star}(x_{t})\big)\mathrm{d}x_{t-1}\mathrm{d}x_{t}}}\notag\\
 & \quad+\underset{\eqqcolon\,\mathcal{K}_{2}}{\underbrace{\int_{\mathcal{E}^{\mathrm{c}}}p_{X_{t}}(x_{t})\left\{ p_{X_{t-1}\mymid X_{t}}(x_{t-1}\mymid x_{t})+p_{Y_{t-1}^{\star}\mymid Y_{t}}(x_{t-1}\mymid x_{t})\right\} \big\| x_{t-1}-\mu_{t}^{\star}(x_{t})\big\|_{2}\big\| s_{t}(x_{t})-s_{t}^{\star}(x_{t})\big\|_{2}\mathrm{d}x_{t-1}\mathrm{d}x_{t}}}, 
	\label{eq:int-Xt-Xt-1-792}
\end{align}
where $(\text{i})$ holds since $\mathbb{E}\big[Y_{t-1}^{\star}-\mu_{t}^{\star}(Y_t)\mid Y_t\big]=0$. 
This leaves us with two terms to cope with. 
\begin{itemize}
	\item
Regarding the term $\mathcal{K}_1$, we can bound
\begin{align}
	\mathcal{K}_1
	&\overset{\mathrm{(i)}}{\lesssim} \frac{d^2\log^3 T}{T}\int_{\mathcal{E}} p_{X_{t}}(x_{t})p_{X_{t-1}\mymid X_{t}}(x_{t-1}\mymid x_{t}) \big\|x_{t-1} - \mu_{t}^{\star}(x_t)\big\|_2\varepsilon_{\score, t}(x_t)\mathrm{d}x_{t-1}\mathrm{d}x_{t} 
	\notag\\
&\overset{\mathrm{(ii)}}{\lesssim} \frac{d^3\log^3 T}{T}\int_{\mathcal{E}} p_{X_{t}}(x_{t})p_{X_{t-1}\mymid X_{t}}(x_{t-1}\mymid x_{t})
\big\|x_{t-1} - \mu_{t}^{\star}(x_t)\big\|_2^2\mathrm{d}x_{t-1}\mathrm{d}x_{t} 
	+ \frac{d\log^3 T}{T}
\mathbb{E}_{X_t\sim q_t} \big[\varepsilon^2_{\score, t}(X_t)\big] 
	\notag\\
&= \frac{d^3\log^3 T}{T}\int_{\mathcal{E}} p_{X_{t}}(x_{t})\frac{p_{X_{t-1}\mymid X_{t}}(x_{t-1}\mymid x_{t})}{p_{Y_{t-1}^{\star}\mymid Y_{t}}(x_{t-1}\mymid x_{t})}p_{Y_{t-1}^{\star}\mymid Y_{t}}(x_{t-1}\mymid x_{t})
\big\|x_{t-1} - \mu_{t}^{\star}(x_t)\big\|_2^2\mathrm{d}x_{t-1}\mathrm{d}x_{t} \notag\\
&\qquad\qquad + 
 \frac{d\log^3 T}{T}
\mathbb{E}_{X_t\sim q_t}\big[\varepsilon^2_{\score, t}(X_t)\big] 
	\notag\\
	&\overset{\mathrm{(iii)}}{\lesssim} \frac{d^3\log^3 T}{T}\bigg(1+O\Big(\frac{d^2\log^3 T}{T}\Big)\bigg)\frac{d(1 - \alpha_t)}{\alpha_t} + 
	\frac{d\log^3 T}{T} \mathbb{E}_{X_t\sim q_t}\big[\varepsilon^2_{\score, t}(X_t)\big] + \exp\Big(- \frac{c_3}{4} d\log T \Big) \notag\\
	&\overset{\mathrm{(iv)}}{\lesssim} \frac{d^{4}\log^{4}T}{T^{2}} + \frac{d\log^3 T}{T}\mathbb{E}_{X_t\sim q_t}\big[\varepsilon_{\score, t}(X_t)^2\big]. 
	\label{eq:final-bound-K1}
\end{align}
Here, (i) comes from \eqref{eq:ratio-p-X-Y-range}, 
(ii) is due to the elementary inequality $2ab\leq a^2+b^2$, 
(iii) invokes the relation~\eqref{eq:ratio-p-X-Y-range} as well as the Gaussian distribution of $Y_{t-1}^{\star}\mid Y_t$ in \eqref{eq:dist-Ytstar}, 
whereas (iv) applies \eqref{eqn:properties-alpha-proof}. 

	\item With regards to the term $\mathcal{K}_2$, 
	it follows from the Cauchy-Schwarz inequality that
\[
\mathcal{K}_{2}\lesssim\sqrt{\mathcal{K}_{3}\mathcal{K}_{4}},
\]
where
\begin{subequations}
\label{eq:defn-J3-J4-DDPM}
\begin{align}
\mathcal{K}_{3} & \coloneqq\int_{\mathcal{E}^{\mathrm{c}}}p_{X_{t}}(x_{t})\left\{ p_{X_{t-1}\mymid X_{t}}(x_{t-1}\mymid x_{t})+p_{Y_{t-1}^{\star}\mymid Y_{t}}(x_{t-1}\mymid x_{t})\right\} \big\| x_{t-1}-\mu_{t}^{\star}(x_{t})\big\|_{2}^{2}\mathrm{d}x_{t-1}\mathrm{d}x_{t};
	\label{eq:defn-J3-J4-DDPM-K3}\\
\mathcal{K}_{4} & \coloneqq\int_{\mathcal{E}^{\mathrm{c}}}p_{X_{t}}(x_{t})\left\{ p_{X_{t-1}\mymid X_{t}}(x_{t-1}\mymid x_{t})+p_{Y_{t-1}^{\star}\mymid Y_{t}}(x_{t-1}\mymid x_{t})\right\} \big\| s_{t}(x_{t})-s_{t}^{\star}(x_{t})\big\|_{2}^{2}\mathrm{d}x_{t-1}\mathrm{d}x_{t} \notag\\
 & \leq2\int p_{X_{t}}(x_{t})\big\| s_{t}(x_{t})-s_{t}^{\star}(x_{t})\big\|_{2}^{2}\mathrm{d}x_{t}=2\mathbb{E}_{X_{t}\sim q_{t}}\left[\varepsilon_{\score,t}(X_{t})^{2}\right].
\end{align}
\end{subequations}
Consequently, it suffices to look at the term $\mathcal{K}_3$. 
According to \eqref{eq:x-tminus1-mu-2norm}, we have 
\begin{align*}
\big\|x_{t-1}-\mu_{t}^{\star}(x_{t})\big\|_{2}^{2} & \lesssim\bigg\| x_{t-1}-\frac{1}{\sqrt{\alpha_{t}}}x_{t}\bigg\|_{2}^{2}+\left(\frac{1-\alpha_{t}}{\sqrt{\alpha_{t}}(1-\overline{\alpha}_{t})}\right)^{2}\left(\mathbb{E}\Big[\big\| x_{t}-\sqrt{\overline{\alpha}_{t}}X_{0}\big\|_{2}\mid X_{t}=x_{t}\Big]\right)^{2}\\
 & \lesssim\bigg\| x_{t-1}-\frac{1}{\sqrt{\alpha_{t}}}x_{t}\bigg\|_{2}^{2}+\frac{\log^{2}T}{T^{2}}\left(\|x_{t}\|_{2}^{2}+T^{2c_{R}}\right),
\end{align*}
where the last inequality makes use of \eqref{eqn:properties-alpha-proof}
and the assumption~\eqref{eq:assumption-data-bounded}. 
Combining this with \eqref{eq:P-Xt-X-t-1-not-E-246} gives
\begin{align*}
 & \int_{\mathcal{E}^{\mathrm{c}}}p_{X_{t}}(x_{t})p_{X_{t-1}\mymid X_{t}}(x_{t-1}\mymid x_{t})\big\| x_{t-1}-\mu_{t}^{\star}(x_{t})\big\|_{2}^{2}\mathrm{d}x_{t-1}\mathrm{d}x_{t}\\
 & \quad\lesssim\int_{\mathcal{E}^{\mathrm{c}}}p_{X_{t}}(x_{t})p_{X_{t-1}\mymid X_{t}}(x_{t-1}\mymid x_{t})\bigg\| x_{t-1}-\frac{1}{\sqrt{\alpha_{t}}}x_{t}\bigg\|_{2}^{2}\mathrm{d}x_{t-1}\mathrm{d}x_{t}\\
 & \qquad\qquad+\frac{\log^{2}T}{T^{2}}\int_{\mathcal{E}^{\mathrm{c}}}p_{X_{t}}(x_{t})p_{X_{t-1}\mymid X_{t}}(x_{t-1}\mymid x_{t})\left(\|x_{t}\|_{2}^{2}+T^{2c_{R}}\right)\mathrm{d}x_{t-1}\mathrm{d}x_{t}\\
 & \quad\lesssim\exp\left(-c_{3}d\log T\right).
\end{align*}
Additionally, let us decompose
\begin{align*}
 & \int_{\mathcal{E}^{\mathrm{c}}}p_{X_{t}}(x_{t})p_{Y_{t-1}^{\star}\mymid Y_{t}}(x_{t-1}\mymid x_{t})\big\| x_{t-1}-\mu_{t}^{\star}(x_{t})\big\|_{2}^{2}\mathrm{d}x_{t-1}\mathrm{d}x_{t}\\
 & \leq\underset{\eqqcolon\,\mathcal{K}_{5}}{\underbrace{\int_{(x_{t},x_{t-1}):p_{X_{t}}(x_{t})<e^{-0.5c_{6}d\log T}}p_{X_{t}}(x_{t})p_{Y_{t-1}^{\star}\mymid Y_{t}}(x_{t-1}\mymid x_{t})\big\| x_{t-1}-\mu_{t}^{\star}(x_{t})\big\|_{2}^{2}\mathrm{d}x_{t-1}\mathrm{d}x_{t}}}\\
 & \quad+\underset{\eqqcolon\,\mathcal{K}_{6}}{\underbrace{\int_{(x_{t},x_{t-1}):p_{X_{t}}(x_{t})\geq e^{-0.5c_{6}d\log T},\|x_{t}-\widehat{x}_{t}\|_{2}>c_{3}\sqrt{d(1-\alpha_{t})\log T}}p_{X_{t}}(x_{t})p_{Y_{t-1}^{\star}\mymid Y_{t}}(x_{t-1}\mymid x_{t})\big\| x_{t-1}-\mu_{t}^{\star}(x_{t})\big\|_{2}^{2}\mathrm{d}x_{t-1}\mathrm{d}x_{t}}}.
\end{align*}
Substitution into \eqref{eq:defn-J3-J4-DDPM-K3} yields
\begin{align}
	\mathcal{K}_{2} & \lesssim\sqrt{\mathcal{K}_{3}\mathcal{K}_{4}}\lesssim
	\sqrt{\big\{\exp\left(-c_{3}d\log T\right)+\mathcal{K}_{5}+\mathcal{K}_{6}\big\}\mathbb{E}_{X_{t}\sim q_{t}}\left[\varepsilon_{\score,t}(X_{t})^{2}\right]} ,
	\label{eq:K2-UB-K3K4K5K6}
\end{align}
thus motivating us to bound $\mathcal{K}_5$ and $\mathcal{K}_6$ separately. 

\begin{itemize}
	\item
Regarding $\mathcal{K}_6$, we can use \eqref{eq:x-tminus1-mu-2norm-479} to demonstrate that
\begin{align*}
	\mathcal{K}_{6} & \leq\int p_{X_{t}}(x_{t})\mathbb{E}\left[\big\| Y_{t-1}^{\star}-\mu_{t}^{\star}(Y_{t})\big\|_{2}^{2}\mathds{1}\Big\{\big\| Y_{t-1}^{\star}-\mu_{t}^{\star}(Y_{t})\big\|_{2}\geq\frac{c_{3}}{2}\sqrt{d(1-\alpha_{t})\log T}\Big\}\mid Y_{t}=x_{t}\right]\mathrm{d}x_{t}\\
 & \lesssim\exp\left(-\frac{c_{3}}{4}d\log T\right).
\end{align*}

\item
Turning to $\mathcal{K}_5$, one can invoke \eqref{eq:Xt-2range-ODE} to derive
\begin{align*}
\mathcal{K}_{5} & \leq\int_{x_{t}:p_{X_{t}}(x_{t})<e^{-0.5c_{6}d\log T}}p_{X_{t}}(x_{t})\mathbb{E}\left[\big\| Y_{t-1}^{\star}-\mu_{t}^{\star}(Y_{t})\big\|_{2}^{2}\mid Y_{t}=x_{t}\right]\mathrm{d}x_{t}\\
 & \leq d\int_{x_{t}:p_{X_{t}}(x_{t})<e^{-0.5c_{6}d\log T}}p_{X_{t}}(x_{t})\mathrm{d}x_{t}\\
 & \leq d\int_{x_{t}:p_{X_{t}}(x_{t})<e^{-0.5c_{6}d\log T},\|x_{t}\|_{2}\leq T^{c_{R}+2}}p_{X_{t}}(x_{t})\mathrm{d}x_{t}+\mathbb{P}\left\{ \|X_{t}\|_{2}\geq T^{c_{R}+2}\right\} \\
 & \leq d\exp\Big(-\frac{1}{2}c_{6}d\log T\Big)\cdot(2T^{c_{R}+2})^{d}+\exp(-c_{6}d\log T)\\
 & \lesssim\exp\Big(-\frac{1}{4}c_{6}d\log T\Big),
\end{align*}
provided that $c_{6}>0$ is large enough. 
\end{itemize}
Combining the above results with \eqref{eq:K2-UB-K3K4K5K6}, we arrive at
\begin{align}
\mathcal{K}_{2} & \lesssim\sqrt{\Big\{\exp\left(-\frac{c_{3}}{4}d\log T\right)+\exp\left(-\frac{c_{6}}{4}d\log T\right)\Big\}\mathbb{E}_{X_{t}\sim q_{t}}\left[\varepsilon_{\score,t}(X_{t})^{2}\right]} \notag\\
	& \lesssim\exp\Big(-\frac{\min\{c_{3},c_{6}\}}{4}d\log T\Big)+\exp\Big(-\frac{\min\{c_{3},c_{6}\}}{4}d\log T\Big)\mathbb{E}_{X_{t}\sim q_{t}}\left[\varepsilon_{\score,t}(X_{t})^{2}\right]. \label{eq:final-bound-K2}
\end{align}
%

%

\end{itemize}

Taking the above bounds \eqref{eq:Ht-first-component-bound}, \eqref{eq:final-bound-K1} and \eqref{eq:final-bound-K2} together with \eqref{eqn:tmp-sde-brahms} immediately finishes the proof.

\section{Analysis for the accelerated deterministic sampler (Theorem~\ref{thm:main-ODE-fast})}

The aim of this section is to establish Theorem~\ref{thm:main-ODE-fast}. 
The proof follows similar arguments as that of Theorem~\ref{thm:main-ODE}.

\subsection{Proof of Theorem~\ref{thm:main-ODE-fast}}
\label{sec:pf-theorem-ode-fast}

\paragraph{Auxiliary vectors and their properties.}
Before embarking on the proof, let us introduce several pieces of notation: 
\begin{subequations}
\label{eqn:varphi-fast}
\begin{align}
\varphi_{t}(x) & \defn x-u_{t}(x), 
\end{align}
where
\begin{align}
u_{t}(x) & \defn\bigg\{\frac{1-\alpha_{t}}{2(1-\overline{\alpha}_{t})}+\frac{(1-\alpha_{t})^{2}}{8(1-\overline{\alpha}_{t})^{2}}-\frac{(1-\alpha_{t})^{2}}{8(1-\overline{\alpha}_{t})^{3}}\|g_{t}(x)\|_{2}^{2}\bigg\} g_{t}(x)\notag \\
 & \quad\quad+\frac{(1-\alpha_{t})^{2}}{8(1-\overline{\alpha}_{t})^{3}}\mathbb{E}\left[\big(X_{t}-\sqrt{\overline{\alpha}_{t}}X_{0}\big)\big(X_{t}-\sqrt{\overline{\alpha}_{t}}X_{0}\big)^{\top}\mymid X_{t}=x\right]g_{t}(x)\notag \\
 & \quad\quad-\frac{(1-\alpha_{t})^{2}}{8(1-\overline{\alpha}_{t})^{3}}\mathbb{E}\left[\big\| X_{t}-\sqrt{\overline{\alpha}_{t}}X_{0}\big\|_{2}^{2}\big(X_{t}-\sqrt{\overline{\alpha}_{t}}X_{0}-g_{t}(X_{t})\big)\mymid X_{t}=x\right]\label{eq:defn-ut-x-thm2}\\
 & =-\bigg\{\frac{1-\alpha_{t}}{2}+\frac{(1-\alpha_{t})^{2}}{8(1-\overline{\alpha}_{t})}-\frac{(1-\alpha_{t})^{2}}{8}\|s_{t}^{\star}(x)\|_{2}^{2}\bigg\} s_{t}^{\star}(x) - \frac{(1-\alpha_{t})^{2}}{8(1-\overline{\alpha}_{t})}\mathbb{E}\left[\overline{W}_{t}\overline{W}_{t}^{\top}s_{t}^{\star}(X_{t})\mymid X_{t}=x\right]\nonumber \\
 & \qquad-\frac{(1-\alpha_{t})^{2}}{8(1-\overline{\alpha}_{t})}\mathbb{E}\left[\big\|\overline{W}_{t}\big\|_{2}^{2}\bigg(\frac{1}{\sqrt{1-\overline{\alpha}_{t}}}\overline{W}_{t}+(1-\overline{\alpha}_{t})s_{t}^{\star}(X_{t})\bigg)\mymid X_{t}=x\right].
\end{align}
\end{subequations}
Here, we recall that $g_t(x)=-(1-\overline{\alpha}_t)s_t^{\star}(x)$  has been defined in \eqref{eq:st-MMSE-expression} 
and that $X_t = \sqrt{\overline{\alpha}_t}X_0 + \sqrt{1-\overline{\alpha}_t}\,\overline{W}_t$ with $\overline{W}_t\sim \mathcal{N}(0,I_d)$ (see \eqref{eqn:Xt-X0}). 
In view of Assumption~\ref{assumption:ODE-estimate} and the fact that the MMSE estimator is the conditional expectation \citep[Section~3.3.1]{hajek2015random}, 
we have
\begin{align*}
w_{t}(x) & =\mathbb{E}\left[\big\|\overline{W}_{t}\big\|_{2}^{2}\bigg(\frac{1}{\sqrt{1-\overline{\alpha}_{t}}}\overline{W}_{t}+(1-\overline{\alpha}_{t})s_{t}^{\star}(X_{t})\bigg) 
+\frac{1}{1-\overline{\alpha}_{t}}\overline{W}_{t}\overline{W}_{t}^{\top}s_{t}^{\star}(X_{t})\mid X_{t}=x\right],
\end{align*}
which taken together with \eqref{eqn:varphi-fast} and \eqref{eqn:ode-sampling-R} confirms the following equivalent expression for the sampler \eqref{eqn:ode-sampling-R}: 
\begin{equation}
	Y_{t-1} = \frac{1}{\sqrt{\alpha_{t}}}\varphi_{t}\big(Y_{t}\big) = \frac{1}{\sqrt{\alpha_{t}}} Y_t - \frac{1}{\sqrt{\alpha_{t}}} u_t(Y_t).
	\label{eq:Yt-recursion-accelerate-proof}
\end{equation}
%
%
%

 We now single out a useful property about $u_t(\cdot)$. For any point $x_t\in \mathbb{R}^d$ obeying $-\log p_{X_t}(x_t) \leq c_6 d \log T$ for some large constant $c_6>0$ (see Lemma~\ref{lem:x0}), one has
 \begin{subequations}
	\label{eqn:u-prop-2323}
\begin{align}
	\big\| u_t(x_t) \big\|_2 \lesssim (1-\alpha_{t})\Big(\frac{d\log T}{1-\overline{\alpha}_{t}}\Big)^{1/2}, \label{eq:ut-properties-thm2-crude}
\end{align}
and one can also write
\begin{align}
\label{eqn:u-prop-2323-12}
	u_t(x_t) &= \frac{1-\alpha_{t}}{2(1-\overline{\alpha}_{t})}g_t(x_t) + \xi_t(x_t)
\end{align}
for some residual term $\xi_t(x_t)$ obeying
\begin{align}
\big\|\xi_{t}(x_{t})\big\|_{2} & \lesssim(1-\alpha_{t})^{2}\bigg(\frac{d\log T}{1-\overline{\alpha}_{t}}\bigg)^{3/2}.\label{eq:zeta-t-UB-thm2}
\end{align} 
 \end{subequations}
To streamline presentation, we leave the proof of \eqref{eqn:u-prop-2323} to the end of this subsection.

%

%
%

\paragraph{Main steps of the proof.} 
Akin to the proof of Theorem~\ref{thm:main-ODE}, the key idea lies in understanding the transformation $\Phi_t$ (cf.~\eqref{eqn:ode-sampling-R}), or equivalently, $\varphi_t$ (cf.~\eqref{eqn:varphi-fast}).
There are several objects that play an important role in the analysis, which we single out as follows:    
\begin{subequations}
\label{eqn:ODE-constant}
\begin{align}
A_{t} & \defn\frac{1}{1-\overline{\alpha}_{t}}\int p_{X_{0}\mymid X_{t}}(x_{0}\mymid x)\big\| x-\sqrt{\overline{\alpha}_{t}}x_{0}\big\|_{2}^{2}\mathrm{d}x_{0};\\
B_{t} & \defn\frac{1}{1-\overline{\alpha}_{t}}\bigg\|\int p_{X_{0}\mymid X_{t}}(x_{0}\mymid x)\big(x-\sqrt{\overline{\alpha}_{t}}x_{0}\big)\mathrm{d}x_{0}\bigg\|_{2}^{2};\\
C_{t} & \defn\frac{1}{(1-\overline{\alpha}_{t})^{2}}\int p_{X_{0}\mymid X_{t}}(x_{0}\mymid x)\big\| x-\sqrt{\overline{\alpha}_{t}}x_{0}\big\|_{2}^{4}\mathrm{d}x_{0};\\
D_{t} & \defn\frac{1}{(1-\overline{\alpha}_{t})^{2}}\int p_{X_{0}\mymid X_{t}}(x_{0}\mymid x)\Big(\big\langle g_{t}(x),\,x-\sqrt{\overline{\alpha}_{t}}x_{0}\big\rangle\Big)^{2}\mathrm{d}x_{0};\\
E_{t} & \defn\frac{1}{(1-\overline{\alpha}_{t})^{2}}\int p_{X_{0}\mymid X_{t}}(x_{0}\mymid x)\big\| x-\sqrt{\overline{\alpha}_{t}}x_{0}\big\|_{2}^{2}\big\langle g_{t}(x),\,x-\sqrt{\overline{\alpha}_{t}}x_{0}\big\rangle\mathrm{d}x_{0}.
\end{align}
\end{subequations}
Here, we suppress the dependency on $x$ in the above five objects to simplify notation whenever it is clear from the context. In view of Lemma~\ref{lem:x0} and the properties  \eqref{eqn:properties-alpha-proof}, 
we have the following bounds: 
\begin{subequations}
\label{eqn:ODE-constant-bounds}
\begin{align}
\big|B_{t}\big|\leq\big|A_{t}\big| & \lesssim\frac{1}{1-\overline{\alpha}_{t}}\cdot d(1-\overline{\alpha}_{t})\log T\asymp d\log T;\\
\big|C_{t}\big| & \lesssim\frac{1}{(1-\overline{\alpha}_{t})^{2}}d^{2}(1-\overline{\alpha}_{t})^{2}\log^{2}T\asymp d^{2}\log^{2}T;\\
\big|D_{t}\big| & \leq\frac{\|g_{t}(x)\|_{2}^{2}}{(1-\overline{\alpha}_{t})^{2}}\int p_{X_{0}\mymid X_{t}}(x_{0}\mymid x)\big\| x-\sqrt{\overline{\alpha}_{t}}x_{0}\big\|_{2}^{2}\mathrm{d}x_{0}\lesssim d^{2}\log^{2}T;\\
\big|E_{t}\big| & \leq\frac{\|g_{t}(x)\|_{2}^{2}}{(1-\overline{\alpha}_{t})^{2}}\int p_{X_{0}\mymid X_{t}}(x_{0}\mymid x)\big\| x-\sqrt{\overline{\alpha}_{t}}x_{0}\big\|_{2}^{3}\mathrm{d}x_{0}\lesssim d^{2}\log^{2}T.
\end{align}
\end{subequations}

As it turns out, Theorem~\ref{thm:main-ODE-fast} can be established in a very similar way as in the proof of Theorem~\ref{thm:main-ODE}. 
In essence, the only step that needs to be changed is to replace~\eqref{eq:ODE} in Lemma~\ref{lem:main-ODE} with~\eqref{eq:ODE-fast} in the lemma below.

\begin{lems} 
\label{lem:main-ODE-fast}
Suppose that $ \frac{d^2(1-\alpha_{t})\log T}{\alpha_{t}-\overline{\alpha}_{t}}  \lesssim 1$. For every $x \in \real$ obeying $-\log p_{X_{t}}(x) \leq c_6 d\log T$ for some large enough constant $c_6>0$,  
we have
\begin{subequations}
\label{eq:ODE-fast}
\begin{align} 
\label{eq:xt-higher}
\frac{p_{\sqrt{\alpha_t}X_{t-1}}(\varphi_t(x))}{p_{X_{t}}(x)} &= 1 + \frac{(1-\alpha_{t})(d+B_t-A_t)}{2(1-\overline{\alpha}_{t})} + O\Big(d^3\Big(\frac{1-\alpha_{t}}{\alpha_{t}-\overline{\alpha}_{t}}\Big)^3\log^3 T\Big) \notag\\
&\qquad+ \frac{(1-\alpha_{t})^2}{8(1-\overline{\alpha}_{t})^2}\big[d(d+2) + (4+2d)(B_t-A_t) - B_t^2 + C_t + 2D_t - 3E_t + A_tB_t\big].
\end{align}
Moreover, for any random vector $Y$, one has
\begin{align} \label{eq:yt-higher}
\frac{p_{\varphi_t(Y)}(\varphi_t(x))}{p_{Y}(x)} &= 1 + \frac{(1-\alpha_{t})(d+B_t-A_t)}{2(1-\overline{\alpha}_{t})} + O\Big(d^6\Big(\frac{1-\alpha_{t}}{\alpha_{t}-\overline{\alpha}_{t}}\Big)^3\log^3 T\Big) \notag\\
&\qquad+ \frac{(1-\alpha_{t})^2}{8(1-\overline{\alpha}_{t})^2}\big[d(d+2) + (4+2d)(B_t-A_t) - B_t^2 + C_t + 2D_t - 3E_t + A_tB_t\big]. 
\end{align}
\end{subequations}
Here, the quantities $A_t,\ldots,E_t$ are defined in~\eqref{eqn:ODE-constant}.
\end{lems}
\noindent 
The proof of this lemma can be found in Appendix~\ref{sec:proof-lem:main-ODE-fast}. 
Crucially, the terms in \eqref{eq:xt-higher} and those in \eqref{eq:yt-higher} coincide (except for the residual terms 
$O\big(d^3\big(\frac{1-\alpha_{t}}{\alpha_{t}-\overline{\alpha}_{t}}\big)^3\log^3 T\big)$ and $O\big(d^6\big(\frac{1-\alpha_{t}}{\alpha_{t}-\overline{\alpha}_{t}}\big)^3\log^3 T\big)$), 
which will cancel each other during the subsequent proof.

With Lemma~\ref{lem:main-ODE-fast} in place, we claim that for every $t \geq 1$,
\begin{align} \label{eq:claim}
	\mathbb{P}\big(Y_t \in {\mathcal{E}}_t\big) \ge 1 - (T-t+1)\exp(-c_3d\log T)
\end{align}
for some constant $c_{3} > 0$, where the set  $\mathcal{E}_t$ is defined as
\begin{align*}
	{\mathcal{E}}_t \defn \bigg\{y : \Big|\frac{p_{Y_t}(y)}{p_{X_t}(y)} - 1\Big| \leq c_5 (T-t+1)\frac{d^6\log^6 T}{T^3}\bigg\}
\end{align*}
for some constant $c_5>0$. We leave its proof to the end of this section. 

Suppose  for the moment that the above claim \eqref{eq:claim} is valid. 
Then taking $t=1$ leads to 
\begin{align*}
 \mathbb{P}(Y_{1}\in\mathcal{E}_{1})
	=\mathbb{P}\Bigg(\bigg|\frac{q_{1}(Y_{1})}{p_{1}(Y_{1})}-1\bigg|\leq \frac{c_{5}d^{6}\log^{6}T}{T^{2}}\Bigg)\ge1- T\exp\big(- c_3 d\log T\big) , 
\end{align*}
which also reveals that
$
	\int_{y\notin \mathcal{E}_{1}} p_1(y) \mathrm{d}y \leq  T\exp\big(- c_3 d\log T\big)	.
$
Additionally, we make the observation that
\begin{align*}
{\displaystyle \int}_{y\notin\mathcal{E}_{1}}\big|p_{1}(y)-q_{1}(y)\big|\mathrm{d}y & \leq{\displaystyle \int}_{y\notin\mathcal{E}_{1}}p_{1}(y)\mathrm{d}y+{\displaystyle \int}_{y\notin\mathcal{E}_{1}}q_{1}(y)\mathrm{d}y={\displaystyle \int}_{y\notin\mathcal{E}_{1}}p_{1}(y)\mathrm{d}y+1-{\displaystyle \int}_{y\in\mathcal{E}_{1}}q_{1}(y)\mathrm{d}y\\
 & \leq{\displaystyle \int}_{y\notin\mathcal{E}_{1}}p_{1}(y)\mathrm{d}y+1-{\displaystyle \int}_{y\in\mathcal{E}_{1}}p_{1}(y)\mathrm{d}y+{\displaystyle \int}_{y\in\mathcal{E}_{1}}\big|p_{1}(y)-q_{1}(y)\big|\mathrm{d}y\\
 & = 2{\displaystyle \int}_{y\notin\mathcal{E}_{1}}p_{1}(y)\mathrm{d}y+{\displaystyle \int}_{y\in\mathcal{E}_{1}}\big|p_{1}(y)-q_{1}(y)\big|\mathrm{d}y.
\end{align*}
The above results together with the definition of the total variation distance gives 
\begin{align}
\notag\mathsf{TV}\big(q_{1},p_{1}\big) 
& =\frac{1}{2}\int_{y\in\mathcal{E}_1}\big|q_{1}(y)-p_{1}(y)\big|\diff y+\frac{1}{2}\int_{y\notin\mathcal{E}_1}\big|q_{1}(y)-p_{1}(y)\big|\diff y\\
 & \leq\int_{y\in\mathcal{E}_1}\big|q_{1}(y)-p_{1}(y)\big|\diff y+{\displaystyle \int}_{y\notin\mathcal{E}_{1}}p_{1}(y)\mathrm{d}y\notag\\
 & =\Exs_{Y_{1}\sim p_{1}}\bigg[\Big|\frac{q_{1}(Y_{1})}{p_{1}(Y_{1})}-1\Big|\cdot\ind\left\{ Y_{1}\in\mathcal{E}_{1}\right\} \bigg]+{\displaystyle \int}_{y\notin\mathcal{E}_{1}}p_{1}(y)\mathrm{d}y\notag\\
 & \leq \frac{c_{5}d^{6}\log^{6}T}{T^{2}}+\exp\big(-c_{3}d\log T\big) \notag\\
 & \asymp \frac{d^{6}\log^{6}T}{T^{2}}.
\label{eqn:ode-tv-1}
\end{align}
This establishes the advertised result in Theorem~\ref{thm:main-ODE-fast}, 
provided that Claim~\eqref{eq:claim} can be verified. 


%
\begin{proof}[Proof of properties~\eqref{eqn:u-prop-2323}]
To justify the above results \eqref{eqn:u-prop-2323}, note that Lemma~\ref{lem:x0} implies that
\begin{subequations}
\label{eq:gt-UB-all-thm2}
\begin{align}
	\big\| g_{t}(x_{t})\big\|_{2}  \leq\mathbb{E}\left[\big\| X_{t}-\sqrt{\overline{\alpha}_{t}}X_{0}\big\|_{2}\mymid X_{t}=x_{t}\right] &\lesssim\sqrt{d(1-\overline{\alpha}_{t})\log T},\label{eq:gt-UB-typical-thm2}\\
	\frac{1-\alpha_{t}}{1-\overline{\alpha}_{t}}\big\| g_{t}(x_{t})\big\|_{2}\lesssim\frac{1-\alpha_{t}}{1-\overline{\alpha}_{t}}\sqrt{d(1-\overline{\alpha}_{t})\log T} &\asymp(1-\alpha_{t})^{2}\sqrt{\frac{d\log T}{1-\overline{\alpha}_{t}}}, \\
	\frac{(1-\alpha_{t})^{2}}{(1-\overline{\alpha}_{t})^{2}}\big\| g_{t}(x_{t})\big\|_{2}  \lesssim\frac{(1-\alpha_{t})^{2}}{(1-\overline{\alpha}_{t})^{2}}\sqrt{d(1-\overline{\alpha}_{t})\log T} &\asymp(1-\alpha_{t})^{2}\frac{\sqrt{d\log T}}{(1-\overline{\alpha}_{t})^{3/2}},\label{eq:gt-UB-typical-thm2-w1}\\
	\frac{(1-\alpha_{t})^{2}}{(1-\overline{\alpha}_{t})^{3}}\|g_{t}(x_{t})\|_{2}^{3}  \lesssim\frac{(1-\alpha_{t})^{2}}{(1-\overline{\alpha}_{t})^{3}}\left(d(1-\overline{\alpha}_{t})\log T\right)^{3/2} &\asymp(1-\alpha_{t})^{2}\bigg(\frac{d\log T}{1-\overline{\alpha}_{t}}\bigg)^{3/2},\label{eq:eq:gt-UB-typical-thm2-w2}\\
	\frac{(1-\alpha_{t})^{2}}{(1-\overline{\alpha}_{t})^{3}}\left\Vert \mathbb{E}\left[\big(x_{t}-\sqrt{\overline{\alpha}_{t}}X_{0}\big)\big(x_{t}-\sqrt{\overline{\alpha}_{t}}X_{0}\big)^{\top}\mymid X_{t}=x_t\right]g_{t}(x_{t})\right\Vert _{2} & \lesssim\frac{(1-\alpha_{t})^{2}\big\| g_{t}(x_{t})\big\|_{2}}{(1-\overline{\alpha}_{t})^{3}}\mathbb{E}\left[\big\| x_{t}-\sqrt{\overline{\alpha}_{t}}X_{0}\big\|_{2}^{2}\mymid X_{t}=x_t\right]\nonumber \\
 & \lesssim(1-\alpha_{t})^{2}\bigg(\frac{d\log T}{1-\overline{\alpha}_{t}}\bigg)^{3/2},\label{eq:eq:gt-UB-typical-thm2-w3}
\end{align}
and 
\begin{align}
 & \frac{(1-\alpha_{t})^{2}}{(1-\overline{\alpha}_{t})^{3}}\left\Vert \mathbb{E}\left[\big\| x_{t}-\sqrt{\overline{\alpha}_{t}}X_{0}\big\|_{2}^{2}\big(x_{t}-\sqrt{\overline{\alpha}_{t}}x_{0}-g_{t}(X_{t})\big)\mymid X_{t}=x_{t}\right]\right\Vert _{2} \notag\\
 & \leq\frac{(1-\alpha_{t})^{2}}{(1-\overline{\alpha}_{t})^{3}}\left\{ \left\Vert \mathbb{E}\left[\big\| x_{t}-\sqrt{\overline{\alpha}_{t}}X_{0}\big\|_{2}^{3}\mymid X_{t}=x_{t}\right]\right\Vert _{2}+\big\| g_{t}(x_{t})\big\|_{2}\left\Vert \mathbb{E}\left[\big\| x_{t}-\sqrt{\overline{\alpha}_{t}}X_{0}\big\|_{2}^{2}\mymid X_{t}=x_{t}\right]\right\Vert _{2}\right\} \notag\\
 & \lesssim(1-\alpha_{t})^{2}\bigg(\frac{d\log T}{1-\overline{\alpha}_{t}}\bigg)^{3/2}.
\end{align}
\end{subequations}
Substituting the bounds \eqref{eq:gt-UB-all-thm2} into \eqref{eq:defn-ut-x-thm2} immediately establishes \eqref{eqn:u-prop-2323}. 
\end{proof}

\begin{proof}[Proof of the claim~\eqref{eq:claim}]
We would like to prove this claim by induction, 
for which we start with the base case with $t = T$. 
Recall that $X_T \overset{\mathrm{d}}{=} \sqrt{\overline{\alpha}_{T}}X_0 + \sqrt{1 - \overline{\alpha}_{T}} B$ and $Y_T \overset{\mathrm{d}}{=} B$ 
with $B\sim\mathcal{N}(0, I_d)$ independent of $X_{0}$, 
and that $\|X_0\|_2\leq R$ with $R=T^{c_R}$ for some constant $c_R>0$.  
For large enough $T$, it immediately follows from \eqref{eqn:properties-alpha-proof-alphaT} that 
\begin{align} \label{eq:claim-Py-1}
\mathbb{P}(Y_T \in {\mathcal{E}}_t) \ge 1 - \exp(-c_3d\log T)
\end{align}
for some constant $c_3>0$ large enough.

Suppose now that the claim~\eqref{eq:claim} holds for some $t\geq 2$, and we wish to prove the claim for $t-1.$ 
We would first like to claim that with probability at least $1 - (T-t)\exp(-c_{3}d\log T)$ for some constant $c_3>0$, 
one has
\begin{align}
\label{eqn:x-paganini}
	q_t(Y_t) \geq \exp\big(- c_6 d\log T \big).
\end{align}
With \eqref{eqn:x-paganini} in place, one sees that Lemma~\ref{lem:main-ODE-fast} is applicable to $x = Y_t$ with high probability.

Next, consider any $y$ obeying $-\log p_{X_t}(y)\leq c_6d \log T$. 
With these bounds~\eqref{eqn:ODE-constant-bounds} in mind, applying relations~\eqref{eq:xt-higher} and~\eqref{eq:yt-higher} in Lemma~\ref{lem:main-ODE-fast} leads to
\begin{align*}
\frac{p_{\sqrt{\alpha_{t}}Y_{t-1}}\big(\phi_{t}(y)\big)}{p_{Y_{t}}(y)}\bigg(\frac{p_{\sqrt{\alpha_{t}}X_{t-1}}\big(\phi_{t}(y)\big)}{p_{X_{t}}(y)}\bigg)^{-1} & =\frac{p_{\phi_{t}(Y_{t})}\big(\phi_{t}(y)\big)}{p_{Y_{t}}(y)}\bigg(\frac{p_{\sqrt{\alpha_{t}}X_{t-1}}\big(\phi_{t}(y)\big)}{p_{X_{t}}(y)}\bigg)^{-1}\\
 & =1+O\Bigg(d^{6}\Big(\frac{1-\alpha_{t}}{\alpha_{t}-\overline{\alpha}_{t}}\Big)^{3}\log^{3}T\Bigg).
\end{align*}
Replacing $y$ with $Y_t$ in the above display, using the fact that $Y_{t-1}=\frac{1}{\sqrt{\alpha_t}} \phi_t(Y_t)$, and invoking the relation \eqref{eq:recursion},  
we immediately arrive at
\begin{align*}
\frac{p_{t-1}(Y_{t-1})}{q_{t-1}(Y_{t-1})} & =\left\{ 1+O\Bigg(d^{6}\Big(\frac{1-\alpha_{t}}{\alpha_{t}-\overline{\alpha}_{t}}\Big)^{3}\log^{3}T\Bigg)\right\} \cdot\frac{p_{t}(Y_{t})}{q_{t}(Y_{t})}
\end{align*}
with probability exceeding  $1 - (T-t)\exp(-c_{3}d\log T)$. This concludes the proof of Claim~\eqref{eq:claim} via standard induction arguments. 
\end{proof}

\subsection{Proof of Lemma~\ref{lem:main-ODE-fast}}
\label{sec:proof-lem:main-ODE-fast}

The proof of Lemma~\ref{lem:main-ODE-fast} is derived in a very similar way to the proof of Lemma~\ref{lem:main-ODE}, 
as detailed below. 
For notational simplicity, we shall abbreviate 
\begin{align}
	u = u_t(x) \qquad \text{and} \qquad z = g_t(x)
\end{align}
whenever it is clear from the context. 

%
%

\subsubsection{Proof of relation~\eqref{eq:xt-higher}}
Through direct calculations as shown in \eqref{eqn:fei}, we can obtain 
\begin{align}
	&p_{\sqrt{\alpha_t}X_{t-1}}\big(\varphi_t(x)\big) 
= p_{X_{t}}(x)\bigg(\frac{1-\overline{\alpha}_{t}}{\alpha_t-\overline{\alpha}_t}\bigg)^{d/2} \notag\\
	&\qquad\qquad \cdot\int_{x_0}  p_{X_0 \mymid X_{t}}(x_0 \mymid x)\exp\bigg(-\frac{(1-\alpha_{t})\big\|x - \sqrt{\overline{\alpha}_{t}}x_0\big\|_2^2}{2(\alpha_t-\overline{\alpha}_{t})(1-\overline{\alpha}_{t})}-\frac{\|u\|_2^2 -2u^{\top}\big(x - \sqrt{\overline{\alpha}_{t}}x_0\big)}{2(\alpha_t-\overline{\alpha}_{t})}\bigg)\mathrm{d} x_0. 
\end{align}
Moreover, similar to the analysis of Lemma~\ref{lem:main-ODE},  we focus our attention on the set given $x$: 
\begin{align}
\mathcal{E} \defn \big\{x_0 : \big\|x - \sqrt{\overline{\alpha}_{t}}x_0\big\|_2 \leq 5c_5 \sqrt{d(1 - \overline{\alpha}_{t})\log T}\big\}, 
\end{align}
which allows us to derive, for some numerical constant $c_8>0$, 
\begin{align}
&\int p_{X_0 \mymid X_{t}}(x_0 \mymid x)\exp\Big(-\frac{(1-\alpha_{t})\big\|x - \sqrt{\overline{\alpha}_{t}}x_0\big\|_2^2}{2(\alpha_t-\overline{\alpha}_{t})(1-\overline{\alpha}_{t})}-\frac{\|u\|_2^2 -2u^{\top}\big(x - \sqrt{\overline{\alpha}_{t}}x_0\big)}{2(\alpha_t-\overline{\alpha}_{t})}\Big) \mathrm{d} x_0 \notag \\
	\notag &= O\left( \exp(- c_8 c_5^2 d\log T ) \right) + \int_{x_0 \in \mathcal{E}} p_{X_0 \mymid X_{t}}(x_0 \mymid x)\exp\bigg(-\frac{(1-\alpha_{t})\big\|x - \sqrt{\overline{\alpha}_{t}}x_0\big\|_2^2}{2(\alpha_t-\overline{\alpha}_{t})(1-\overline{\alpha}_{t})}-\frac{\|u\|_2^2 -2u^{\top}\big(x - \sqrt{\overline{\alpha}_{t}}x_0\big)}{2(\alpha_t-\overline{\alpha}_{t})}\bigg) \mathrm{d} x_0\\
& \eqqcolon \text{RHS} 
	\label{eqn:tmp-rhs}
\end{align}
To further control the right-hand side above, recall that the learning rates are selected such that $\frac{1-\alpha_{t}}{1-\overline{\alpha}_{t-1}} \le \frac{4c_1\log T}{T}$ for $1 < t\leq T$ (see \eqref{eqn:properties-alpha-proof-1}). In view of the Taylor expansion
 $e^{- x} = 1 - x + \frac{1}{2}x^2 + O(x^3)$ for $x \leq 1/2$, we can derive 
\begin{align}
\label{eqn:ode-lemma-fast-rhs}
&\text{RHS}= O\left( \exp(- c_8 c_5^2 d\log T ) \right)  \notag\\
	&\quad+\int_{x_0\in \mathcal{E}} p_{X_0 \mymid X_{t}}(x_0 \mymid x)\bigg\{
1-\frac{(1-\alpha_{t})\big\|x - \sqrt{\overline{\alpha}_{t}}x_0\big\|_2^2}{2(\alpha_t-\overline{\alpha}_{t})(1-\overline{\alpha}_{t})}-\frac{\frac{(1-\alpha_{t})^2}{4(1-\overline{\alpha}_{t})^2}\|z\|_2^2 -2u^{\top}\big(x - \sqrt{\overline{\alpha}_{t}}x_0\big)}{2(\alpha_t-\overline{\alpha}_{t})} \notag\\
&\quad+\frac{(1-\alpha_{t})^2}{8(\alpha_t-\overline{\alpha}_{t})^2(1-\overline{\alpha}_{t})^2}\Big(\big\|x - \sqrt{\overline{\alpha}_{t}}x_0\big\|_2^2-z^{\top}\big(x - \sqrt{\overline{\alpha}_{t}}x_0\big)\Big)^2 
+ O\Big(d^3\Big(\frac{1-\alpha_{t}}{\alpha_{t}-\overline{\alpha}_{t}}\Big)^3\log^3 T\Big)
	\bigg\} \mathrm{d} x_0. 
\end{align}
In order to see this, we recall the property of $u$ (cf.~\eqref{eqn:u-prop-2323}) as 
\begin{align}
\left\| u - \frac{1-\alpha_{t}}{2(1-\overline{\alpha}_{t})}z \right\|_2 \leq O\bigg((1-\alpha_{t})^2\Big(\frac{d\log T}{1-\overline{\alpha}_{t}}\Big)^{3/2}\bigg). 
\end{align}
As a consequence, for any $x_0 \in \mathcal{E}$ we have
\begin{align*}
	\frac{(1-\alpha_{t})\big\|x - \sqrt{\overline{\alpha}_{t}}x_0\big\|_2^2}{2(\alpha_t-\overline{\alpha}_{t})(1-\overline{\alpha}_{t})} &= O\Big(d\Big(\frac{1-\alpha_{t}}{\alpha_{t}-\overline{\alpha}_{t}}\Big)\log T\Big)
\end{align*}
and
\begin{align*}
\frac{\|u\|_2^2 -2u^{\top}\big(x - \sqrt{\overline{\alpha}_{t}}x_0\big)}{2(\alpha_t-\overline{\alpha}_{t})} &= \frac{\frac{(1-\alpha_{t})^2}{4(1-\overline{\alpha}_{t})^2}\|z\|_2^2 -2u^{\top}\big(x - \sqrt{\overline{\alpha}_{t}}x_0\big)}{2(\alpha_t-\overline{\alpha}_{t})} + O\bigg(d^2\Big(\frac{1-\alpha_{t}}{\alpha_{t}-\overline{\alpha}_{t}}\Big)^3\log^2 T\bigg) \\
&= \frac{z^{\top}\big(x - \sqrt{\overline{\alpha}_{t}}x_0\big)}{2(\alpha_t-\overline{\alpha}_{t})(1-\overline{\alpha}_{t})} + O\bigg(d^2\Big(\frac{1-\alpha_{t}}{\alpha_{t}-\overline{\alpha}_{t}}\Big)^2\log^2 T\bigg) \\
&= O\bigg(d\Big(\frac{1-\alpha_{t}}{\alpha_{t}-\overline{\alpha}_{t}}\Big)\log T\bigg),
\end{align*}
where we have invoked the properties \eqref{eqn:properties-alpha-proof}. 
Taking the above results together and using the following basic properties regarding quantities $A_{t},\ldots, E_{t}$
 (defined in \eqref{eqn:ODE-constant}) 
\begin{align*}
&\int p_{X_0 \mymid X_{t}}(x_0 \mymid x)\big\|x - \sqrt{\overline{\alpha}_{t}}x_0\big\|_2^2 \mathrm{d} x_0 = (1-\overline{\alpha}_{t})A_t, \\
&\int p_{X_0 \mymid X_{t}}(x_0 \mymid x)\|z\|_2^2 \mathrm{d} x_0 = (1-\overline{\alpha}_{t})B_t, \\
&\int p_{X_0 \mymid X_{t}}(x_0 \mymid x)u^{\top}\big(x - \sqrt{\overline{\alpha}_{t}}x_0\big) \mathrm{d} x_0 = \frac{1-\alpha_{t}}{2}B_t + \frac{(1-\alpha_{t})^2}{8(1-\overline{\alpha}_{t})}\big[B_t - B_t^2 + D_t - E_t + A_tB_t\big], \\
&\int p_{X_0 \mymid X_{t}}(x_0 \mymid x)\Big(\big\|x - \sqrt{\overline{\alpha}_{t}}x_0\big\|_2^2-z^{\top}\big(x - \sqrt{\overline{\alpha}_{t}}x_0\big)\Big)^2 \mathrm{d} x_0 = (1-\overline{\alpha}_{t})^2\big[C_t + D_t - 2E_t\big],
\end{align*}
we arrive at 
\begin{align*}
\eqref{eqn:ode-lemma-fast-rhs}
&= 1 - \frac{(1-\alpha_{t})(A_t-B_t)}{2(\alpha_t-\overline{\alpha}_{t})}
+ \frac{(1-\alpha_{t})^2}{8(1-\overline{\alpha}_{t})^2}\big[- B_t^2 + C_t + 2D_t - 3E_t + A_tB_t\big] 
+ O\bigg(d^3\Big(\frac{1-\alpha_{t}}{\alpha_{t}-\overline{\alpha}_{t}}\Big)^3\log^3 T\bigg).
\end{align*}

Once again, we note that integrating over set $\mathcal{E}$ and over all possible $x_0$ only incurs a difference at most as large as $ O\big( \exp(- c_8 c_5^2 d\log T ) \big)$. 
Putting all this together establishes the advertised result~\eqref{eq:xt-higher}.

\subsubsection{Proof of relation~\eqref{eq:yt-higher}}

Consider any random vector $Y$, and let us invoke again the basic transformation
\begin{align*}
	p_{\varphi_t(Y)}(\varphi_t(x)) &= \mathsf{det}\Big(\frac{\partial \varphi_t(x)}{\partial x}\Big)^{-1}p_{Y}(x), 
\end{align*}
where $\frac{\partial \varphi_t(x)}{\partial x}$ denotes the Jacobian matrix. 
We are then left with controlling the quantity $\mathsf{det}\Big(\frac{\partial\varphi_t(x)}{\partial x}\Big)^{-1}$.

Towards this, let us again recall that the determinant of a matrix satisfies  
\begin{align*}
\mathsf{det}(I + A)^{-1} = 1 - \mathsf{Tr}(A) + \frac{1}{2}\big[\mathsf{Tr}(A)^2 + \|A\|_{\mathrm{F}}^2\big] + O\big(d^3\|A\|^3\big),
\end{align*}
provided that $d\|A\| \leq c_{20}$ for some small enough constant $c_{20}>0$.
This relation leads to
\begin{align}
\label{eqn:jude}
p_{\varphi_t(Y)}(\varphi_t(x)) 
&= \mathsf{det}\Big(\frac{\partial \varphi_t(x)}{\partial x}\Big)^{-1}p_{Y}(x) \notag\\
&= \bigg\{1 + \mathsf{Tr}\Big(\frac{\partial u}{\partial x}\Big) + \frac{1}{2}\Big[\mathsf{Tr}\Big(\frac{\partial u}{\partial x}\Big)^2 + \Big\|\frac{\partial u}{\partial x}\Big\|_{\mathrm{F}}^2\Big] 
+ O\Big(d^3 \Big\|\frac{\partial u}{\partial x}\Big\|\Big)\bigg\} p_{Y}(x),
\end{align}
where we invoke the definition in \eqref{eqn:varphi-fast} that 
\begin{align*}
\varphi_t(x) = x - u 
&= x - \Big(\frac{(1-\alpha_{t})}{2(1-\overline{\alpha}_{t})} + \frac{(1-\alpha_{t})^2}{8(1-\overline{\alpha}_{t})^2} - \frac{(1-\alpha_{t})^2}{8(1-\overline{\alpha}_{t})^3}\|z\|_2^2\Big) z \notag\\
&- \frac{(1-\alpha_{t})^2}{8(1-\overline{\alpha}_{t})^3}\int_{x_0} p_{X_0 \mymid X_{t}}(x_0 \mymid x)\big(x - \sqrt{\overline{\alpha}_{t}}x_0\big)\big(x - \sqrt{\overline{\alpha}_{t}}x_0\big)^{\top}z \mathrm{d} x_0 \notag\\
&+ \frac{(1-\alpha_{t})^2}{8(1-\overline{\alpha}_{t})^3}\int_{x_0} p_{X_0 \mymid X_{t}}(x_0 \mymid x)\big\|x - \sqrt{\overline{\alpha}_{t}}x_0\big\|_2^2\big(x - \sqrt{\overline{\alpha}_{t}}x_0 - z\big) \mathrm{d} x_0. 
\end{align*}

To further control the right-hand side above of the above display, let us first make note of several identities. 
Proving these identities only requires elementary calculation regarding Gaussian integration and derivatives, which is omitted here for brevity. 
Specifically, one has
\begin{subequations}
\begin{align}
\frac{\partial z}{\partial x} &= J_t,\\
\frac{\partial}{\partial x} \|z\|_2^2z &= \|z\|_2^2J_t + 2zz^{\top}J_t,
\end{align}
\begin{align}
\notag &\frac{\partial}{\partial x} \int_{x_0} p_{X_0 \mymid X_{t}}(x_0 \mymid x)\big\|x - \sqrt{\overline{\alpha}_{t}}x_0\big\|_2^2\big(x - \sqrt{\overline{\alpha}_{t}}x_{0}\big) \mathrm{d} x_0 \\
\notag &= \int_{x_0} p_{X_0 \mymid X_{t}}(x_0 \mymid x)\big\|x - \sqrt{\overline{\alpha}_{t}}x_0\big\|_2^2 \mathrm{d} x_0 I + 2\int_{x_0} p_{X_0 \mymid X_{t}}(x_0 \mymid x)\big(x - \sqrt{\overline{\alpha}_{t}}x_0\big)\big(x - \sqrt{\overline{\alpha}_{t}}x_0\big)^{\top} \mathrm{d} x_0 \\
\notag &\qquad+ \frac{1}{1-\overline{\alpha}_{t}}\Big(\Big(\int_{x_0} p_{X_0 \mymid X_{t}}(x_0 \mymid x)\big\|x - \sqrt{\overline{\alpha}_{t}}x_0\big\|_2^2\big(x - \sqrt{\overline{\alpha}_{t}}x_{0}\big)\Big)\Big(\int_{x_0} p_{X_0 \mymid X_{t}}(x_0 \mymid x)\big(x - \sqrt{\overline{\alpha}_{t}}x_{0}\big) \mathrm{d} x_0\Big)^{\top} \notag \\
&\qquad-\int_{x_0} p_{X_0 \mymid X_{t}}(x_0 \mymid x)\big\|x - \sqrt{\overline{\alpha}_{t}}x_0\big\|_2^2\big(x - \sqrt{\overline{\alpha}_{t}}x_0\big)\big(x - \sqrt{\overline{\alpha}_{t}}x_0\big)^{\top} \mathrm{d} x_0\Big),\\
\notag &\frac{\partial}{\partial x} \int_{x_0} p_{X_0 \mymid X_{t}}(x_0 \mymid x)\big\|x - \sqrt{\overline{\alpha}_{t}}x_0\big\|_2^2z 
= \int_{x_0} p_{X_0 \mymid X_{t}}(x_0 \mymid x)\big\|x - \sqrt{\overline{\alpha}_{t}}x_0\big\|_2^2 \mathrm{d} x_0 J_t + 2zz^{\top} \\
&\quad+ \frac{1}{1-\overline{\alpha}_{t}}\Big(\Big(\int_{x_0} p_{X_0 \mymid X_{t}}(x_0 \mymid x)\big\|x - \sqrt{\overline{\alpha}_{t}}x_0\big\|_2^2\Big)zz^{\top} -z\Big(\int_{x_0} p_{X_0 \mymid X_{t}}(x_0 \mymid x)\big\|x - \sqrt{\overline{\alpha}_{t}}x_0\big\|_2^2\big(x - \sqrt{\overline{\alpha}_{t}}x_{0}\big) \mathrm{d} x_0\Big)^{\top}\Big),\\
\notag &\frac{\partial}{\partial x} \int_{x_0} p_{X_0 \mymid X_{t}}(x_0 \mymid x)\big(x - \sqrt{\overline{\alpha}_{t}}x_0\big)\big(x - \sqrt{\overline{\alpha}_{t}}x_0\big)^{\top}z \mathrm{d} x_0 = \|z\|_2^2I + zz^{\top} \notag\\
\notag &\qquad\qquad\qquad+\int_{x_0} p_{X_0 \mymid X_{t}}(x_0 \mymid x)\big(x - \sqrt{\overline{\alpha}_{t}}x_{0}\big)\big(x - \sqrt{\overline{\alpha}_{t}}x_{0}\big)^{\top}J_t \mathrm{d} x_0 \notag\\
\notag &\qquad\qquad\qquad+\frac{1}{1-\overline{\alpha}_{t}}\int_{x_0} p_{X_0 \mymid X_{t}}(x_0 \mymid x)\big(z^{\top}\big(x - \sqrt{\overline{\alpha}_{t}}x_{0}\big)\big)\big(x - \sqrt{\overline{\alpha}_{t}}x_{0}\big)z^{\top} \mathrm{d} x_0 \notag\\
&\qquad\qquad\qquad-\frac{1}{1-\overline{\alpha}_{t}}\int_{x_0} p_{X_0 \mymid X_{t}}(x_0 \mymid x)\big(z^{\top}\big(x - \sqrt{\overline{\alpha}_{t}}x_{0}\big)\big)\big(x - \sqrt{\overline{\alpha}_{t}}x_{0}\big)\big(x - \sqrt{\overline{\alpha}_{t}}x_{0}\big)^{\top} \mathrm{d} x_0.
\end{align}
\end{subequations} 
Equipped with the above relations,  we can easily verify that
\begin{subequations}
\begin{align}
\Big\|\frac{\partial u}{\partial x}\Big\| &\lesssim \frac{d(1-\alpha_{t})\log T}{1-\overline{\alpha}_{t}},\\
\mathsf{Tr}\Big(\frac{\partial u}{\partial x}\Big) &= \frac{(1-\alpha_t)\big(d + B_t - A_t\big)}{2(1-\overline{\alpha}_{t})} \notag\\
&\qquad+ \frac{(1-\alpha_t)^2}{8(1-\overline{\alpha}_{t})^2}\big(d -2A_t - A_t^2 + 3A_tB_t + 2B_t - 3B_t^2 + C_t + 4D_t - 3E_t - F_t\big), \\
\Big\|\frac{\partial u}{\partial x}\Big\|_{\mathrm{F}}^2 &= \frac{(1-\alpha_t)^2}{4(1-\overline{\alpha}_{t})^2}\Big\|\frac{\partial z}{\partial x}\Big\|_{\mathrm{F}}^2 + O\Big(d^5\Big(\frac{1-\alpha_{t}}{\alpha_{t}-\overline{\alpha}_{t}}\Big)^3\log^3 T\Big) \notag\\
&= \frac{(1-\alpha_t)^2}{4(1-\overline{\alpha}_{t})^2}\big(d + 2(B_t-A_t) + B_t^2 + F_t - 2D_t\big) + O\Big(d^5\Big(\frac{1-\alpha_{t}}{\alpha_{t}-\overline{\alpha}_{t}}\Big)^3\log^3 T\Big),
\end{align}
as long as $d^2\big(\frac{1-\alpha_{t}}{\alpha_{t}-\overline{\alpha}_{t}}\big)\log T \lesssim 1$, 
where we recall the definition of the quantities $A_{t}$ to $E_{t}$ in \eqref{eqn:ODE-constant}, and
\begin{align}
F_t(x) &\defn \Big\|\frac{1}{1-\overline{\alpha}_{t}}\int_{x_0} p_{X_0 \mymid X_{t}}(x_0 \mymid x)\big(x - \sqrt{\overline{\alpha}_{t}}x_{0}\big)\big(x - \sqrt{\overline{\alpha}_{t}}x_{0}\big)^{\top} \mathrm{d} x_0\Big\|_{\mathrm{F}}^2.
\end{align}
Plugging these results into inequality~\eqref{eqn:jude} leads to 
\begin{align}
p_{\varphi_t(Y)}(\varphi_t(x)) 
&= p_{Y}(x)\bigg\{1 + \frac{(1-\alpha_{t})(d+B_t-A_t)}{2(1-\overline{\alpha}_{t})} + O\Big(d^6\Big(\frac{1-\alpha_{t}}{\alpha_{t}-\overline{\alpha}_{t}}\Big)^3\log^3 T\Big) \notag\\
&\qquad+ \frac{(1-\alpha_{t})^2}{8(1-\overline{\alpha}_{t})^2}\big[d(d+2) + (4+2d)(B_t-A_t) - B_t^2 + C_t + 2D_t - 3E_t + A_tB_t\big] \bigg\}. 
\end{align}
\end{subequations}
We have thus completed the proof of Lemma~\ref{lem:main-ODE-fast}.


\section{Analysis for the accelerated stochastic sampler (Theorems~\ref{thm:main-SDE-R})}
\label{sec:analysis-stochastic-samplers-fast}

The proof of Theorem~\ref{thm:main-SDE-R} follows similar structure as the proof of Theorem~\ref{thm:main-SDE}.  
Throughout the proof, we shall employ the notation $\widehat{x}_t \coloneqq x_t / \sqrt{\alpha_t}$ as before.

\subsection{Proof of Theorem~\ref{thm:main-SDE-R}}
\label{sec:pf-thm-sde-r}

\paragraph{Step 1: expressing the update rule in terms of a Jacobian matrix.}  
Let us recall the Jacobian matrix $J_t(x)=\frac{\partial g_t(x)}{\partial x}$ defined in \eqref{eq:Jacobian-Thm4}. 
In view of the expression \eqref{eq:Jt-x-expression-ij-23} as well as \eqref{eqn:Xt-X0} (i.e., $X_t-\sqrt{\overline{\alpha}_t}X_0 = \sqrt{1-\overline{\alpha}_t}\, \overline{W}_t$ with $ \overline{W}_t\sim \mathcal{N}(0,I_d)$), one can write
\begin{align}
J_{t}(x) 
 & =I_{d}+\mathbb{E}\big[\overline{W}_{t}\mid X_{t}=x\big]\Big(\mathbb{E}\big[\overline{W}_{t}\mid X_{t}=x\big]\Big)^{\top}-\mathbb{E}\big[\overline{W}_{t}\overline{W}_{t}^{\top}\mid X_{t}=x\big]\nonumber \\
 & =I_{d}+\big(1-\overline{\alpha}_{t}\big)s_{t}^{\star}(x)s_{t}^{\star}(x)^{\top}-\mathbb{E}\big[\overline{W}_{t}\overline{W}_{t}^{\top}\mid X_{t}=x\big],\label{eq:Jt-x-expression-ij-1}
\end{align}
where 
the last line makes use of the relation \eqref{eq:st-MMSE-expression}. 
Additionally, 
recall that (i) $v_t(x,z)$ (cf.~\eqref{eqn:training-variance}) is the MMSE estimator for estimating $\overline{W}_t\overline{W}_tz$ given $\sqrt{\overline{\alpha}_t}X_0+ \sqrt{1-\overline{\alpha}_t}\, \overline{W}_t=x$ and $Z_t=z$, and (ii) $Z_t$ is independent from $X_0$ and $\overline{W}_t$. Then this MMSE estimator admits the following expression \citep[Section~3.3.1]{hajek2015random}: 
\begin{equation}
v_{t}(x,z)=\mathbb{E}\big[\overline{W}_{t}\overline{W}_{t}^{\top}z\mid X_{t}=x\big]=\mathbb{E}\big[\overline{W}_{t}\overline{W}_{t}^{\top}\mid X_{t}=x\big]z.
	\label{eq:vt-explicit-expression}
\end{equation}
As a result, the mapping introduced in \eqref{eqn:sde-sampling-R-Psi} can be alternatively expressed as: 
\begin{align}
\Psi_{t}(x,z) & =\frac{1}{\sqrt{\alpha_{t}}}\Big(x+(1-\alpha_{t})s_{t}^{\star}(x)\Big)+\sigma_{t}\left\{ z-\frac{1-\alpha_{t}}{2(1-\overline{\alpha}_{t})}\left[z+(1-\overline{\alpha}_{t})s_{t}^{\star}(x)s_{t}^{\star}(x)^{\top}z-v_{t}(x,z)\right]\right\} \notag\\
 & =\mu_{t}(x)+\sigma_{t}\bigg(I-\frac{1-\alpha_{t}}{2(1-\overline{\alpha}_{t})}J_{t}(x)\bigg)z,
	\label{eqn:sde-sampling-R-Psi-proof}
\end{align}
where we have also used the definition \eqref{eqn:nu-t-2} of $\mu_t(\cdot)$. 
In comparison to the plain DDPM-type sampler (cf.~\eqref{eqn:sde-sampling}), 
the key correction term is the second component on the right-hand side of \eqref{eqn:sde-sampling-R-Psi-proof}, 
which adjusts the covariance of the additive Gaussian noise.

Equipped with the above expression \eqref{eqn:sde-sampling-R-Psi-proof}, 
we can readily express the conditional distribution of $Y_{t-1}$ (cf.~\eqref{eqn:sde-sampling-R-Y}) given $Y_t$ such that: for any points $x_t,x_{t-1}\in \mathbb{R}^d$, 
\begin{align}
p_{Y_{t-1}\mymid Y_{t}}(x_{t-1}\mymid x_{t}) & =\frac{1}{\big(2\pi\frac{1-\alpha_{t}}{\alpha_{t}}\big)^{d/2}\big|\det\big(I-\frac{1-\alpha_{t}}{2(1-\overline{\alpha}_{t})}J_{t}(x_{t})\big)\big|} \notag\\
 & \qquad\cdot\exp\Bigg(-\frac{\alpha_{t}}{2(1-\alpha_{t})}\bigg\|\Big(I-\frac{1-\alpha_{t}}{2(1-\overline{\alpha}_{t})}J_{t}(x_{t})\Big)^{-1}\big(x_{t-1}-\mu_{t}(x_{t})\big)\bigg\|_{2}^{2}\Bigg).
	\label{eq:p-Y-conditional-thm4}
\end{align}

\paragraph{Step 2: controlling the conditional distributions $p_{X_{t-1} \mymid X_t}$ and $p_{Y_{t-1} \mymid Y_t}$.} 
Akin to Step 2 in the proof of Theorem~\ref{thm:main-SDE}, 
we need to look at the conditional distribution $p_{X_{t-1} \mymid X_t}$  when restricted to points from the following set:  
\begin{align}
\label{eqn:eset-acclerated}
	\mathcal{E} \defn \bigg\{(x_t, x_{t-1}) \mymid -\log p_{X_t}(x_t) \leq \frac{1}{2}c_6 d\log T, ~\|x_{t-1} - \widehat{x}_t\|_2 \leq c_3 \sqrt{d(1 - \alpha_t)\log T} \bigg\},
\end{align}
with the numerical constants $c_3,c_6>0$ introduced in Lemma~\ref{lem:river}. 
The following lemma, which is a counterpart of Lemma~\ref{lem:sde} for the accelerated sampler, characterizes $p_{X_{t-1} \mymid X_t}$ in a fairly tight manner over the set $\mathcal{E}$. 
The proof is deferred to Appendix~\ref{sec:proof-lem:sde-R}. 
\begin{lems}
\label{lem:sde-R} 
There exists some large enough numerical constant $c_{\zeta}>0$ such that: 
for every $(x_t, x_{t-1}) \in \mathcal{E}$,  
	\begin{align}
		&p_{X_{t-1}\mymid X_{t}}(x_{t-1}\mymid x_{t})  =\frac{1}{\big(2\pi\frac{1-\alpha_{t}}{\alpha_{t}}\big)^{d/2}\big|\det\big(I-\frac{1-\alpha_{t}}{2(1-\overline{\alpha}_{t})}J_{t}(x_{t})\big)\big|} \notag\\
 & \qquad\qquad\qquad\cdot\exp\bigg(-\frac{\alpha_{t}}{2(1-\alpha_{t})}\bigg\|\bigg(I-\frac{1-\alpha_{t}}{2(1-\overline{\alpha}_{t})}J_{t}(x_{t})\bigg)^{-1}\big(x_{t-1}-\mu_{t}(x_{t})\big)\bigg\|_{2}^{2}+\zeta_{t}(x_{t-1},x_{t})\bigg)		 
		\label{eq:cond-dist-crude-fast}
	\end{align}
holds for some residual term $\zeta_{t}(x_{t-1},x_t)$ obeying 
\begin{equation}
	\big|\zeta_{t}(x_{t-1},x_t) \big|\leq c_{\zeta} \frac{d^{3}\log^{4.5}T}{T^{3/2}}. 
	\label{eq:eq:cond-dist-crude-xit-fast}
\end{equation}
Here, we recall the definition of $\mu_t(\cdot)$ (resp.~$J_t(\cdot)$) in \eqref{eqn:nu-t-2} (resp.~\eqref{eq:Jacobian-Thm4}).
\end{lems}

Moving beyond the set $\mathcal{E}$, we are still in need of bounding the log density ratio $\log \frac{p_{X_{t-1} \mymid X_t}}{p_{Y_{t-1} \mymid Y_t}}$ for all pairs $(x_t,x_{t-1})$ outside $\mathcal{E}$,  
in a way similar to Lemma~\ref{lem:sde-full} in  the proof of Theorem~\ref{thm:main-SDE}. 
Our crude bound towards this end is stated as follows, whose proof is postponed to Appendix~\ref{sec:proof-lem:sde-R-full}.
\begin{lems}
\label{lem:sde-R-full}
For all $(x_t, x_{t-1}) \in \real^d \times \real^d$, we have
\begin{align}
\log \frac{p_{X_{t-1} \mymid X_t}(x_{t-1}\mymid x_t)}{p_{Y_{t-1} \mymid Y_t}(x_{t-1}\mymid x_t)} 
& \leq T^{c_{0}+2c_{R}+2}\left\{ \big\| x_{t-1}-\widehat{x}_{t}\big\|_{2}^{2}+\|x_{t}\|_{2}^{2}+1\right\},
	\label{eq:SDE-ratio-crude-2}
\end{align}
where $c_0$ is defined in \eqref{eqn:alpha-t}. 
\end{lems}

\paragraph{Step 3: bounding the KL divergence of interest.}  
With Lemmas~\ref{lem:sde-R}-\ref{lem:sde-R-full} in place, 
one can repeat the arguments in Step 3 in the proof of Theorem~\ref{thm:main-SDE} to arrive at
\[
\mathbb{E}_{x_{t}\sim X_{t}}\Big[\mathsf{KL}\Big(p_{X_{t-1}\mymid X_{t}}(\cdot\mymid x_{t})\parallel p_{Y_{t-1}\mymid Y_{t}}(\cdot\mymid x_{t})\Big)\Big]\lesssim\left(\frac{d^{3}\log^{4.5}T}{T^{3/2}}\right)^{2}.
\]
Substitution into \eqref{eqn:kl-decomp} and \eqref{eq:Pinsker-thm3} then yields
\[
2\mathsf{TV}(p_{X_{1}},p_{Y_{1}})^{2}\leq\mathsf{KL}(p_{X_{1}}\parallel p_{Y_{1}})\lesssim\mathsf{KL}(p_{X_{T}}\parallel p_{Y_{T}})+\sum_{t\geq2}^{T}\frac{d^{6}\log^{9}T}{T^{3}}\asymp\frac{d^{6}\log^{9}T}{T^{2}},
\]
where the last relation results from \eqref{eqn:KL-T-123}. This completes the proof of Theorem~\ref{thm:main-SDE-R}.

%
%


\subsection{Proof of Lemma~\ref{lem:sde-R}}
\label{sec:proof-lem:sde-R}

Recall that an explicit expression for $p_{X_{t-1}\mymid X_{t}}(x_{t-1}\mymid x_{t})$ has already been established in Lemma~\ref{lem:sde} (see \eqref{eqn:allegro-here})
A little algebra then allows one to write
\begin{align}
p_{X_{t-1}\mymid X_{t}}(x_{t-1}\mymid x_{t}) & =f_{1}(x_{t})\exp\Big(-f_{2}(x_{t},x_{t-1})+\zeta_{t,1}(x_{t},x_{t-1})\Big),
\label{eqn:allegro-here-23}
\end{align}
for some function $f_1(\cdot)$, where
\begin{subequations}
	\label{eq:defn-gt-f2-zeta}
\begin{align}
	g_{t-1}(x) & \coloneqq\int_{x_{0}}\big(x-\sqrt{\overline{\alpha}_{t-1}}x_{0}\big)p_{X_{0}|X_{t-1}}(x_{0}\mymid x)\mathrm{d}x_0, \\
f_{2}(x_{t},x_{t-1})&=\frac{\|x_{t}-\sqrt{\alpha_{t}}x_{t-1}\|_{2}^{2}}{2(1-\alpha_{t})}+\frac{(x_{t-1}-\widehat{x}_{t})^{\top}g_{t-1}\big(\widehat{x}_{t}\big)}{1-\overline{\alpha}_{t-1}}+\frac{\frac{1}{2}(x_{t-1}-\widehat{x}_{t})^{\top}J_{t-1}\big(\widehat{x}_{t}\big)\big(x_{t-1}-\widehat{x}_{t}\big)}{1-\overline{\alpha}_{t-1}},\\
\zeta_{t,1}(x_{t},x_{t-1}) & =(x_{t-1}-\widehat{x}_{t})^{\top}\frac{\int_{0}^{1}\int_{0}^{1}\gamma\left[J_{t-1}\Big((1-\tau)\widehat{x}_{t}+\tau x_{t}(\gamma)\Big)-J_{t-1}\big(\widehat{x}_{t}\big)\right]\mathrm{d}\tau\mathrm{d}\gamma}{1-\overline{\alpha}_{t-1}}\big(x_{t-1}-\widehat{x}_{t}\big),
\end{align}
\end{subequations}
and we remind the readers that $x_t(\gamma) = \gamma x_{t-1} + (1-\gamma) \widehat{x}_t$ with $\widehat{x}_t = x_t / \sqrt{\alpha_t}$.

In order to control \eqref{eqn:allegro-here-23}, we single out two useful facts: for every $(x_t,x_{t-1})\in \mathcal{E}$ (cf.~\eqref{eqn:eset-acclerated}),
\begin{align}
\label{eq:Jacobi-b}
	\big\|J_{t-1}\big(x_t(\gamma)\big) - J_{t-1}\big(\widehat{x}_t\big) \big\| \lesssim d^2\sqrt{\frac{1-\alpha_t}{1-\overline{\alpha}_{t-1}}}\log^2 T,
	\qquad \forall \gamma \in [0,1]
\end{align}
and
\begin{align}
\bigg\|\frac{J_{t-1}(\widehat{x}_t)}{1-\overline{\alpha}_{t-1}} - \frac{J_{t}(x_{t})}{1-\overline{\alpha}_{t}}\bigg\| 
\lesssim \frac{d^2(1-\alpha_t)\log^2 T}{(\alpha_t-\overline{\alpha}_{t})^2}.\label{eq:approx-t-b}
\end{align}
The proofs of these two facts are postponed to Appendix~\ref{sec:proof-lem:sde-R-claims}. 
These facts in turn allow us to bound 
\begin{equation}
\big|\zeta_{t,1}(x_{t},x_{t-1})\big|\leq\frac{\big\| x_{t-1}-\widehat{x}_{t}\big\|_{2}^{2}}{2(1-\overline{\alpha}_{t-1})}\sup_{\gamma\in[0,1]}\big\| J_{t-1}\big(x_{t}(\gamma)\big)-J_{t-1}\big(\widehat{x}_{t}\big)\big\|\lesssim d^{3}\bigg(\frac{1-\alpha_{t}}{1-\overline{\alpha}_{t-1}}\bigg)^{3/2}\log^{3}T\label{eq:zeta-t-lem7-UB}
\end{equation}
and
\begin{align}
 & \left|\frac{(x_{t-1}-\widehat{x}_{t})^{\top}J_{t-1}\big(\widehat{x}_{t}\big)\big(x_{t-1}-\widehat{x}_{t}\big)}{1-\overline{\alpha}_{t-1}}-\frac{(x_{t-1}-\widehat{x}_{t})^{\top}J_{t}(x_{t})\big(x_{t-1}-\widehat{x}_{t}\big)}{1-\overline{\alpha}_{t}}\right|\nonumber \\
 & \qquad\leq\big\| x_{t-1}-\widehat{x}_{t}\big\|_{2}^{2}\bigg\|\frac{J_{t-1}(\widehat{x}_{t})}{1-\overline{\alpha}_{t-1}}-\frac{J_{t}(x_{t})}{1-\overline{\alpha}_{t}}\bigg\|\lesssim\frac{d^{3}(1-\alpha_{t})^{2}\log^{3}T}{(\alpha_{t}-\overline{\alpha}_{t})^{2}}\asymp\frac{d^{3}(1-\alpha_{t})^{2}\log^{3}T}{(1-\overline{\alpha}_{t-1})^{2}}\label{eq:gap-Jt-lem7}
\end{align}
for any $(x_t,x_{t-1})\in \mathcal{E}$. 
Consequently, there exists some function $f_3(\cdot)$ such that
\begin{subequations}
	\label{eq:defn-gt-f3-zeta-34}
\begin{align}
p_{X_{t-1}\mymid X_{t}}(x_{t-1}\mymid x_{t}) & =f_{3}(x_{t})\exp\Big(-f_{4}(x_{t},x_{t-1})+\zeta_{t,2}(x_{t},x_{t-1})\Big),
\label{eqn:allegro-here-234}
\end{align}
where
\begin{align}
	f_{4}(x_{t},x_{t-1})&=\frac{\alpha_t\| \widehat{x}_{t}-x_{t-1}\|_{2}^{2}}{2(1-\alpha_{t})}+\frac{(x_{t-1}-\widehat{x}_{t})^{\top}g_{t-1}\big(\widehat{x}_{t}\big)}{1-\overline{\alpha}_{t-1}}+\frac{\frac{1}{2}(x_{t-1}-\widehat{x}_{t})^{\top}J_{t}(x_{t})\big(x_{t-1}-\widehat{x}_{t}\big)}{1-\overline{\alpha}_{t}}, \label{eqn:allegro-here-235}\\
 \big|\zeta_{t,2}(x_{t},x_{t-1})\big| & \lesssim d^{3}\bigg(\frac{1-\alpha_{t}}{1-\overline{\alpha}_{t-1}}\bigg)^{3/2}\log^{3}T 
	\lesssim \frac{d^3 \log^{4.5}T}{T^{3/2}}. 
\end{align}
\end{subequations}

To continue, we further observe that
\begin{equation}
\left|\frac{(x_{t-1}-\widehat{x}_{t})^{\top}g_{t-1}\big(\widehat{x}_{t}\big)}{1-\overline{\alpha}_{t-1}}-\frac{\sqrt{\alpha_{t}}(x_{t-1}-\widehat{x}_{t})^{\top}g_{t}(x_{t})}{1-\overline{\alpha}_{t}}\right|\lesssim\frac{d^{3}\log^{3.5}T}{T^{3/2}}, 
	\label{eq:gap-gt-lem7}
\end{equation}
which is an immediate consequence of the following two bounds (obtained using \eqref{eq:approx-t-a}, \eqref{eq:approx-t-bbb} and \eqref{eqn:properties-alpha-proof-1}): 
\begin{align*}
\left|\frac{(x_{t-1}-\widehat{x}_{t})^{\top}g_{t-1}\big(\widehat{x}_{t}\big)}{1-\overline{\alpha}_{t-1}}-\frac{(x_{t-1}-\widehat{x}_{t})^{\top}g_{t}(x_{t})}{1-\overline{\alpha}_{t}}\right|
	&\lesssim d^{2}\sqrt{\frac{(1-\alpha_{t})^{3}}{(\alpha_{t}-\overline{\alpha}_{t})^{3}}}\log^{2}T
	\lesssim\frac{d^{3}\log^{3.5}T}{T^{3/2}}  \\
	 \left|\left(1-\sqrt{\alpha_{t}}\right)\frac{(x_{t-1}-\widehat{x}_{t})^{\top}g_t(x_t)}{1-\overline{\alpha}_{t}}\right| 
	 &\lesssim\frac{d\log^{2}T}{T^{3/2}}. 
\end{align*}
%
This bound \eqref{eq:gap-gt-lem7} allows us to replace the second term on the right-hand side of \eqref{eqn:allegro-here-235} with $\frac{\sqrt{\alpha_{t}}(x_{t-1}-\widehat{x}_{t})^{\top}g_{t}(x_{t})}{1-\overline{\alpha}_{t}}$.  
It is also seen  from \eqref{eq:Jacobi-a} and the properties \eqref{eqn:properties-alpha-proof}
that
\begin{equation}
\left\Vert \frac{1-\alpha_{t}}{2(1-\overline{\alpha}_{t})}J_{t}(x_{t})\right\Vert \lesssim\frac{\log T}{T}\cdot d\log T\asymp\frac{d\log^{2}T}{T}=o(1), 
	\label{eq:rescaled-J-norm}
\end{equation}
and therefore, 
\begin{align*}
\left\Vert \left(1-\alpha_{t}\right)\frac{(x_{t-1}-\widehat{x}_{t})^{\top}J_{t}(x_{t})\big(x_{t-1}-\widehat{x}_{t}\big)}{2(1-\overline{\alpha}_{t})}\right\Vert  & \lesssim\frac{d\log^{2}T}{T}\cdot\big\| x_{t-1}-\widehat{x}_{t}\big\|_{2}^{2}\lesssim\frac{\left(1-\alpha_{t}\right)d^{2}\log^{3}T}{T}\lesssim\frac{d^{2}\log^{4}T}{T^{2}}.
\end{align*}
These combined with \eqref{eq:defn-gt-f3-zeta-34} allow us to show that: there exist some functions $f_5(\cdot)$ and $\widetilde{f}_5(\cdot)$ such that
\begin{subequations}
	\label{eq:defn-gt-f3-zeta-3456}
\begin{align}
p_{X_{t-1}\mymid X_{t}}(x_{t-1}\mymid x_{t}) & =f_{5}(x_{t})\exp\Big(-f_{6}(x_{t},x_{t-1})+\zeta_{t,3}(x_{t},x_{t-1})\Big)
\label{eqn:allegro-here-23456}
\end{align}
where
\begin{align}
	f_{6}(x_{t},x_{t-1}) & =\frac{\alpha_{t}\|x_{t-1}-\widehat{x}_{t}\|_{2}^{2}}{2(1-\alpha_{t})}+\frac{\sqrt{\alpha_{t}}(x_{t-1}-\widehat{x}_{t})^{\top}g_{t}(x_{t})}{1-\overline{\alpha}_{t}}+\frac{\frac{1}{2}\alpha_{t}(x_{t-1}-\widehat{x}_{t})^{\top}J_{t}(x_{t})\big(x_{t-1}-\widehat{x}_{t}\big)}{1-\overline{\alpha}_{t}} \notag\\
 & =\frac{\alpha_{t}}{2(1-\alpha_{t})}\left\{ \|x_{t-1}-\mu_{t}(x_{t})\|_{2}^{2}+\frac{1-\alpha_{t}}{1-\overline{\alpha}_{t}}(x_{t-1}-\widehat{x}_{t})^{\top}J_{t}(x_{t})\big(x_{t-1}-\widehat{x}_{t}\big)\right\} +\widetilde{f}_{t}(x_{t}),\label{eqn:allegro-here-23556}\\
\big|\zeta_{t,3}(x_{t},x_{t-1})\big| & \lesssim\frac{d^{3}\log^{4.5}T}{T^{3/2}}. 
\end{align}
\end{subequations}

To further proceed, we make note of another useful fact:
\begin{align*}
 & \frac{\alpha_{t}}{1-\alpha_{t}}\left|\big(x_{t-1}-\widehat{x}_{t}\big)\bigg(\frac{1-\alpha_{t}}{1-\overline{\alpha}_{t}}J_{t}(x_{t})\bigg)\left(\frac{1}{\sqrt{\alpha_{t}}}\frac{1-\alpha_{t}}{1-\overline{\alpha}_{t}}g_{t}(x_{t})\right)\right|\notag\\
 & \quad\leq\frac{\alpha_{t}}{1-\alpha_{t}}\big\| x_{t-1}-\widehat{x}_{t}\big\|_{2}\cdot\left\Vert \frac{1-\alpha_{t}}{2(1-\overline{\alpha}_{t})}J_{t}(x_{t})\right\Vert \cdot\frac{1-\alpha_{t}}{\sqrt{\alpha_{t}}(1-\overline{\alpha}_{t})}\mathbb{E}\left[\big\| x_{t}-\sqrt{\overline{\alpha}_{t}}X_{0}\big\|_{2}\mymid X_{t}=x_{t}\right]\notag\\
 & \quad\lesssim\sqrt{d(1-\alpha_{t})\log T}\cdot\frac{d\log^{2}T}{T}\cdot\frac{\log T}{T}\cdot\sqrt{d(1-\overline{\alpha}_{t})\log T}\notag\\
 & \quad\asymp\frac{d^{2}\log^{4}T}{T^{2}}\sqrt{1-\alpha_{t}}\lesssim\frac{d^{2}\log^{4.5}T}{T^{3/2}},
\end{align*}
where we have made use of the crude bound \eqref{eq:rescaled-J-norm} in conjunction with the properties \eqref{eqn:properties-alpha-proof}. 
Taking this observation together with \eqref{eq:defn-gt-f3-zeta-3456} and \eqref{eqn:nu-t-2}, 
we can apply a little algebra to derive
\begin{subequations}
	\label{eq:defn-gt-f3-zeta-56}
\begin{align}
p_{X_{t-1}\mymid X_{t}}(x_{t-1}\mymid x_{t}) & =f_{7}(x_{t})\exp\Big(-f_{8}(x_{t},x_{t-1})+\zeta_{t,4}(x_{t},x_{t-1})\Big)
\label{eqn:allegro-here-456}
\end{align}
for some function $f_7(\cdot)$, where
\begin{align}
f_{8}(x_{t},x_{t-1}) & =\frac{\alpha_{t}}{2(1-\alpha_{t})}\left\{ \big(x_{t-1}-\mu_{t}(x_{t})\big)^{\top}\bigg(I+\frac{1-\alpha_{t}}{2(1-\overline{\alpha}_{t})}J_{t}(x_{t})\bigg)\big(x_{t-1}-\mu_{t}(x_{t})\big)\right\} ,\\
 \big|\zeta_{t,4}(x_{t},x_{t-1})\big| & \lesssim  \frac{d^3 \log^{4.5}T}{T^{3/2}}. 
\end{align}
\end{subequations}
%

%

Note, however, that the covariance matrix $I+\frac{1-\alpha_{t}}{2(1-\overline{\alpha}_{t})}J_{t}(x_{t})$ 
still differs from the desired one $\big(I-\frac{1-\alpha_{t}}{2(1-\overline{\alpha}_{t})}J_{t}(x_{t})\big)^{-2}$. 
As it turns out, these two matrices are fairly close to each other. 
To see this, we write
\[
\bigg(I-\frac{1-\alpha_{t}}{2(1-\overline{\alpha}_{t})}J_{t}(x_{t})\bigg)^{-2}=I+\frac{1-\alpha_{t}}{1-\overline{\alpha}_{t}}J_{t}(x_{t})+A,
\]
where $A$ is a matrix obeying (see \eqref{eq:rescaled-J-norm})
\[
\|A\|\lesssim\left\Vert \frac{1-\alpha_{t}}{2(1-\overline{\alpha}_{t})}J_{t}(x_{t})\right\Vert ^{2}\lesssim\frac{d^{2}\log^{4}T}{T^{2}}.
\]
Consequently, we can demonstrate that
\begin{align*}
 & \frac{\alpha_{t}}{2(1-\alpha_{t})}\bigg\|\bigg(I-\frac{1-\alpha_{t}}{2(1-\overline{\alpha}_{t})}J_{t}(x_{t})\bigg)^{-1}\big(x_{t-1}-\mu_{t}(x_{t})\big)\bigg\|_{2}^{2}\\
 & =\frac{\alpha_{t}}{2(1-\alpha_{t})}\left\{ \big(x_{t-1}-\mu_{t}(x_{t})\big)^{\top}\bigg(I+\frac{1-\alpha_{t}}{1-\overline{\alpha}_{t}}J_{t}(x_{t})\bigg)\big(x_{t-1}-\mu_{t}(x_{t})\big)\right\} +
	O\bigg( \frac{\alpha_{t}}{2(1-\alpha_{t})}\|A\|\,\big\| x_{t-1}-\mu_{t}(x_{t})\big\|_{2}^{2} \bigg)\\
 & =\frac{\alpha_{t}}{2(1-\alpha_{t})}\big(x_{t-1}-\mu_{t}(x_{t})\big)^{\top}\bigg(I+\frac{1-\alpha_{t}}{1-\overline{\alpha}_{t}}J_{t}(x_{t})\bigg)\big(x_{t-1}-\mu_{t}(x_{t})\big)+O\bigg(\frac{d^{3}\log^{5}T}{T^{2}}\bigg).
\end{align*}
To see why the last line holds, note that (according to Lemma~\ref{lem:x0}
and the properties \eqref{eqn:properties-alpha-proof})
\begin{align*}
\big\| x_{t-1}-\mu_{t}(x_{t})\big\|_{2} & \leq\big\| x_{t-1}-\widehat{x}_{t}\big\|_{2}+\frac{1-\alpha_{t}}{\sqrt{\alpha_{t}}(1-\overline{\alpha}_{t})}\big\|\mathbb{E}\left[\big\| x_{t}-\sqrt{\overline{\alpha}_{t}}X_{0}\big\|_{2}\mymid X_{t}=x_{t}\right]\big\|_{2}\\
 & \lesssim\sqrt{d(1-\alpha_{t})\log T}+\sqrt{\frac{d\log T}{1-\overline{\alpha}_{t}}}\left(1-\alpha_{t}\right)\asymp\sqrt{d(1-\alpha_{t})\log T},
\end{align*}
and hence
\[
\frac{\alpha_{t}}{1-\alpha_{t}}\|A\|\,\big\| x_{t-1}-\mu_{t}(x_{t})\big\|_{2}^{2}\lesssim\frac{d^{3}\log^{5}T}{T^{2}}.
\]
Combining the above bound with \eqref{eq:defn-gt-f3-zeta-56}, we
arrive at
\begin{subequations}
	\label{eq:final-density-ratio-expression}
\begin{equation}
p_{X_{t-1}\mymid X_{t}}(x_{t-1}\mymid x_{t})=f_{9}(x_{t})\exp\Big(-f_{10}(x_{t},x_{t-1})+\zeta_{t,5}(x_{t},x_{t-1})\Big)
\end{equation}
for some function $f_{9}(\cdot)$, where 
\begin{align}
f_{10}(x_{t},x_{t-1}) & =\frac{\alpha_{t}}{2(1-\alpha_{t})}\bigg\|\bigg(I-\frac{1-\alpha_{t}}{2(1-\overline{\alpha}_{t})}J_{t}(x_{t})\bigg)^{-1}\big(x_{t-1}-\mu_{t}(x_{t})\big)\bigg\|_{2}^{2},\\
\big|\zeta_{t,5}(x_{t},x_{t-1})\big| & \lesssim\frac{d^{3}\log^{4.5}T}{T^{3/2}}.
\end{align}
\end{subequations} 

To finish up, repeat Step 3 in the proof of Lemma~\ref{lem:sde} to yield
\[
f_{7}(x_{t})=\left(1+O\left(\frac{d^{3}\log^{4.5}T}{T^{3/2}}\right)\right)\frac{1}{\big(2\pi\frac{1-\alpha_{t}}{\alpha_{t}}\big)^{d/2}\big|\det\big(I-\frac{1-\alpha_{t}}{2(1-\overline{\alpha}_{t})}J_{t}(x_{t})\big)\big|}
\]
as claimed. This combined with \eqref{eq:final-density-ratio-expression} concludes the proof.

%
%

\subsubsection{Proof of auxiliary claims in Lemma~\ref{lem:sde-R}}
\label{sec:proof-lem:sde-R-claims}

\paragraph{Proof of relation~\eqref{eq:Jacobi-b}.}
For any $(x_t,x_{t-1})\in \mathcal{E}$, one necessarily has
\begin{align}
	\|x_{t}(\gamma) - \widehat{x}_t\|_2 \leq \|x_{t-1} - \widehat{x}_t\|_2 \leq c_3 \sqrt{d(1-\alpha_t)\log T} .
\end{align}
Given $x_t$, we define the set
\[
	\mathcal{E}_1 \coloneqq \big\{ x: \|\widehat{x}_t - \sqrt{\overline{\alpha}_{t-1}}x\|_2 \leq c_4 \sqrt{d(1-\overline{\alpha}_{t-1})\log T} \big\}. 
\]
Then for any $x_0\in \mathcal{E}_1$,  one has
\begin{align*}
\big\| x_{t}(\gamma)-\sqrt{\overline{\alpha}_{t-1}}x_{0}\big\|_{2} & \leq\max\left\{ \|\widehat{x}_{t}-\sqrt{\overline{\alpha}_{t-1}}x_{0}\|_{2},\|\widehat{x}_{t}-x_{t-1}\|_{2}\right\} \leq\max\left\{ c_{4}\sqrt{d(1-\overline{\alpha}_{t-1})\log T},c_{3}\sqrt{d(1-\alpha_{t})\log T}\right\} \\
 & =c_{4}\sqrt{d(1-\overline{\alpha}_{t-1})\log T},
\end{align*}
where the last inequality comes from \eqref{eqn:properties-alpha-proof-1}. 
This in turn reveals that
\begin{align*}
 & \left|\frac{p_{X_{t-1}\mymid X_{0}}\big(x_{t}(\gamma)\mymid x_{0}\big)}{p_{X_{t-1}\mymid X_{0}}(\widehat{x}_{t}\mymid x_{0})}-1\right|=\left|\exp\bigg(\frac{\big\|\widehat{x}_{t}-\sqrt{\overline{\alpha}_{t-1}}x_{0}\big\|_{2}^{2}}{2(1-\overline{\alpha}_{t-1})}-\frac{\big\| x_{t}(\gamma)-\sqrt{\overline{\alpha}_{t-1}}x_{0}\big\|_{2}^{2}}{2(1-\overline{\alpha}_{t-1})}\bigg)-1\right|\\
 & \quad\quad\leq \left|\exp\bigg(\frac{\big\|\widehat{x}_{t}-x_{t}(\gamma)\big\|_{2}\big\{\big\|\widehat{x}_{t}-\sqrt{\overline{\alpha}_{t-1}}x_{0}\big\|_{2}+\big\| x_{t}(\gamma)-\sqrt{\overline{\alpha}_{t-1}}x_{0}\big\|_{2}\big\}}{2(1-\overline{\alpha}_{t-1})} \bigg)-1\right|\\
 & \quad\quad\lesssim\frac{\big\|\widehat{x}_{t}-x_{t}(\gamma)\big\|_{2}\big\{\big\|\widehat{x}_{t}-\sqrt{\overline{\alpha}_{t-1}}x_{0}\big\|_{2}+\big\| x_{t}(\gamma)-\sqrt{\overline{\alpha}_{t-1}}x_{0}\big\|_{2}\big\}}{2(1-\overline{\alpha}_{t-1})}\\
 & \qquad\lesssim d\sqrt{\frac{1-\alpha_{t}}{1-\overline{\alpha}_{t-1}}}\log T\lesssim d\sqrt{\frac{\log^{3}T}{T}}=o(1),
\end{align*}
where the second line follows from the elementary relation
\[
\Big|\|a\|_{2}^{2}-\|b\|_{2}^{2}\Big|  =\Big|\|a\|_{2}-\|b\|_{2}\Big|\cdot\left(\|a\|_{2}+\|b\|_{2}\right)\leq\|a-b\|_{2}\left(\|a\|_{2}+\|b\|_{2}\right), 
\]
and the last line relies on \eqref{eqn:properties-alpha-proof} and our assumption on $T$. 
Moreover, repeating the same argument as in \eqref{eqn:rachmaninoff} and \eqref{eqn:prokofiev}, we arrive at
\begin{align}
\label{eqn:cond-xt-1}
\frac{p_{X_{t-1}}\big(x_{t}(\gamma)\big)}{p_{X_{t-1}}(\widehat{x}_t)} 
&= 1 + O\bigg(d\sqrt{\frac{1-\alpha_t}{1-\overline{\alpha}_{t-1}}}\log T\bigg). 
\end{align}
Putting the above results together leads to 
\begin{align*}
\frac{p_{X_{0}\mymid X_{t-1}}\big(x_{0}\mymid x_{t}(\gamma)\big)}{p_{X_{0}\mymid X_{t-1}}(x_{0}\mymid\widehat{x}_{t})} & =\frac{p_{X_{t-1}\mymid X_{0}}\big(x_{t}(\gamma)\mymid x_{0}\big)/p_{X_{t-1}}\big(x_{t}(\gamma)\big)}{p_{X_{t-1}\mymid X_{0}}(\widehat{x}_{t}\mymid x_{0})/p_{X_{t-1}}(\widehat{x}_{t})}=1+O\bigg(d\sqrt{\frac{1-\alpha_{t}}{1-\overline{\alpha}_{t-1}}}\log T\bigg).
\end{align*}

Equipped with the above relation, we can demonstrate that 
\begin{align}
\notag & \Big\|\int p_{X_{0}\mymid X_{t-1}}\big(x_{0}\mymid x_{t}(\gamma)\big)\big(x_{t}(\gamma)-\sqrt{\overline{\alpha}_{t-1}}x_{0}\big)\diff x_{0}-\int p_{X_{0}\mymid X_{t-1}}(x_{0}\mymid\widehat{x}_{t})\big(\widehat{x}_{t}-\sqrt{\overline{\alpha}_{t-1}}x_{0}\big)\diff x_{0}\Big\|_{2}\\
\notag & \qquad\leq O\Big(d\sqrt{\frac{1-\alpha_{t}}{1-\overline{\alpha}_{t-1}}}\log T\Big)\cdot\bigg( \int p_{X_{0}\mymid X_{t-1}}\big(x_{0}\mymid x_{t}(\gamma)\big)\big\|x_{t}(\gamma)-\sqrt{\overline{\alpha}_{t-1}}x_{0}\big\|_2\diff x_{0}\bigg) \\
	& \qquad\qquad+\Big\|\int p_{X_{0}\mymid X_{t-1}}\big(x_{0}\mymid \widehat{x}_{t} \big)\big(x_{t}(\gamma)-\widehat{x}_{t}\big)\diff x_{0}\Big\|_{2}\\
 & \qquad\lesssim\sqrt{d^{3}(1-\alpha_{t})\log^{3}T},
 \label{eqn:cool-cat}
\end{align}
where the last step invokes the property~\eqref{eq:E-xt-X0}.
Following similar arguments (which we omit for brevity), we can also derive 
\begin{align}
 & \Big\|\int p_{X_{0}\mymid X_{t-1}}\big(x_{0}\mymid x_{t}(\gamma)\big)\big(x_{t}(\gamma)-\sqrt{\overline{\alpha}_{t-1}}x_{0}\big)\big(x_{t}(\gamma)-\sqrt{\overline{\alpha}_{t-1}}x_{0}\big)^{\top}\diff x_{0}\notag\\
 & \qquad-\int p_{X_{0}\mymid X_{t-1}}(x_{0}\mymid\widehat{x}_{t})\big(\widehat{x}_{t}-\sqrt{\overline{\alpha}_{t-1}}x_{0}\big)\big(\widehat{x}_{t}-\sqrt{\overline{\alpha}_{t-1}}x_{0}\big)^{\top}\diff x_{0}\Big\| \notag\\
 & \quad\lesssim d^{2}\sqrt{(1-\alpha_{t})(1-\overline{\alpha}_{t-1})}\log^{2}T,\label{eqn:butterfly-123}
\end{align}
where we have made use of the property~\eqref{eq:E2-xt-X0}.  
Taking the above two above perturbation bounds together with the expression~\eqref{eqn:derivative-2} and making use of \eqref{eq:E-xt-X0} immediately lead to the advertised result: 
\begin{align}
	\big\| J_{t-1}\big(x_{t}(\gamma)\big)-J_{t-1}(\widehat{x}_{t})\big\| \lesssim d^{2}\sqrt{\frac{1-\alpha_{t}}{1-\overline{\alpha}_{t-1}}}\log^{2}T.
\label{eqn:butterfly} 
\end{align}

\paragraph{Proof of relation~\eqref{eq:approx-t-b}.} 
To establish this relation, we first apply the triangle inequality: 
\begin{align*}
\bigg\|\frac{J_{t-1}(\widehat{x}_{t})}{1-\overline{\alpha}_{t-1}}-\frac{J_{t}(x_{t})}{1-\overline{\alpha}_{t}}\bigg\| 
& \leq\bigg\|\frac{J_{t-1}(\widehat{x}_{t})-J_{t}(x_{t})}{1-\overline{\alpha}_{t-1}}\bigg\|
+ \bigg\|\bigg(\frac{1}{1-\overline{\alpha}_{t-1}}-\frac{1}{1-\overline{\alpha}_{t}}\bigg)J_{t}(x_{t})\bigg\|.
\end{align*}
Let us first consider the second term 
\begin{align}
\label{eqn:part-doupi}
	\bigg\|\bigg(\frac{1}{1-\overline{\alpha}_{t-1}}-\frac{1}{1-\overline{\alpha}_{t}}\bigg)J_{t}(x_{t})\bigg\|
	&= 
	\frac{\overline{\alpha}_{t-1}(1-\alpha_t)}{(1-\overline{\alpha}_{t-1})(1-\overline{\alpha}_{t})}
	\|J_{t}(x_{t})\|
	\leq \frac{\overline{\alpha}_{t-1}(1-\alpha_t)d\log T}{(1-\overline{\alpha}_{t-1})(1-\overline{\alpha}_{t})} 
	&\lesssim \frac{(1-\alpha_t)d\log T}{(\alpha_t-\overline{\alpha}_{t})^2}.
\end{align}
where the last inequality uses \eqref{eq:Jacobi-a} and the properties \eqref{eqn:properties-alpha-proof}.

Next, we move on to bound the difference $J_{t-1}(\widehat{x}_{t})-J_{t}(x_{t})$.
%
By virtue of the relation~\eqref{eqn:lotus}, one can deduce that
\begin{align*}
 & \left\Vert \int p_{X_{0}\mymid X_{t-1}}\big(x_{0}\mymid\widehat{x}_{t}\big)\big(\widehat{x}_{t}-\sqrt{\overline{\alpha}_{t-1}}x_{0}\big)\mathrm{d}x_{0}-\frac{1}{\sqrt{\alpha_{t}}}\int p_{X_{0}\mymid X_{t}}\big(x_{0}\mymid{x}_{t}\big)\big({x}_{t}-\sqrt{\overline{\alpha}_{t}}x_{0}\big)\mathrm{d}x_{0}\right\Vert _{2}\\
 & \quad=\frac{1}{\sqrt{\alpha_{t}}}\left\Vert \int p_{X_{0}\mymid X_{t-1}}\big(x_{0}\mymid\widehat{x}_{t}\big)\big(x_{t}-\sqrt{\overline{\alpha}_{t}}x_{0}\big)\mathrm{d}x_{0}-\int p_{X_{0}\mymid X_{t}}\big(x_{0}\mymid{x}_{t}\big)\big({x}_{t}-\sqrt{\overline{\alpha}_{t}}x_{0}\big)\mathrm{d}x_{0}\right\Vert \\
 & \quad\leq\frac{1}{\sqrt{\alpha_{t}}}O\bigg(\frac{d(1-\alpha_{t})\log T}{1-\overline{\alpha}_{t-1}}\bigg)\int p_{X_{0}\mymid X_{t}}\big(x_{0}\mymid{x}_{t}\big)\big\|{x}_{t}-\sqrt{\overline{\alpha}_{t}}x_{0}\big\|_{2}\mathrm{d}x_{0}.
\end{align*}
In addition, recognizing that $\frac{1}{\sqrt{\alpha_{t}}}=1+\frac{1-\alpha_{t}}{\sqrt{\alpha_{t}}(1+\sqrt{\alpha_{t}})}=1+O(1-\alpha_{t})$,
we can further invoke the triangle inequality and \eqref{eqn:properties-alpha-proof-00} to obtain
\begin{align*}
 & \left\Vert \int p_{X_{0}\mymid X_{t-1}}\big(x_{0}\mymid\widehat{x}_{t}\big)\big(\widehat{x}_{t}-\sqrt{\overline{\alpha}_{t-1}}x_{0}\big)\mathrm{d}x_{0}-\int p_{X_{0}\mymid X_{t}}\big(x_{0}\mymid{x}_{t}\big)\big({x}_{t}-\sqrt{\overline{\alpha}_{t}}x_{0}\big)\mathrm{d}x_{0}\right\Vert _{2}\\
 & \quad\lesssim \frac{d(1-\alpha_{t})\log T}{1-\overline{\alpha}_{t-1}}\int p_{X_{0}\mymid X_{t}}\big(x_{0}\mymid{x}_{t}\big)\big\|{x}_{t}-\sqrt{\overline{\alpha}_{t}}x_{0}\big\|_{2}\mathrm{d}x_{0}.
\end{align*}
Repeating the same argument also reveals that
\begin{align*}
 & \left\Vert \int p_{X_{0}\mymid X_{t-1}}\big(x_{0}\mymid\widehat{x}_{t}\big)\big(\widehat{x}_{t}-\sqrt{\overline{\alpha}_{t-1}}x_{0}\big)\big(\widehat{x}_{t}-\sqrt{\overline{\alpha}_{t-1}}x_{0}\big)^{\top}\mathrm{d}x_{0}-\int p_{X_{0}\mymid X_{t}}\big(x_{0}\mymid x_{t}\big)\big(x_{t}-\sqrt{\overline{\alpha}_{t}}x_{0}\big)\big(x_{t}-\sqrt{\overline{\alpha}_{t}}x_{0}\big)^{\top}\mathrm{d}x_{0}\right\Vert \\
 & \lesssim\frac{d(1-\alpha_{t})\log T}{1-\overline{\alpha}_{t-1}}\left\Vert \int p_{X_{0}\mymid X_{t}}\big(x_{0}\mymid x_{t}\big)\big(x_{t}-\sqrt{\overline{\alpha}_{t}}x_{0}\big)\big(x_{t}-\sqrt{\overline{\alpha}_{t}}x_{0}\big)^{\top}\mathrm{d}x_{0}\right\Vert .
\end{align*}
In view of  the expression~\eqref{eqn:derivative-2} for $J_t$, combining the preceding two bounds with a little algebra yields
\begin{align*}
 & \frac{1}{1-\overline{\alpha}_{t-1}}\big\| J_{t-1}(\widehat{x}_{t})-J_{t}(x_{t})\big\|\\
 & \lesssim\frac{1}{(1-\overline{\alpha}_{t-1})(1-\overline{\alpha}_{t})}\frac{d(1-\alpha_{t})\log T}{1-\overline{\alpha}_{t-1}}\\
 & \quad\cdot\Bigg\{\bigg\|\int p_{X_{0}\mymid X_{t}}\big(x_{0}\mymid x_{t}\big)\big(x_{t}-\sqrt{\overline{\alpha}_{t}}x_{0}\big)\big(x_{t}-\sqrt{\overline{\alpha}_{t}}x_{0}\big)^{\top}\mathrm{d}x_{0}\bigg\|+\left(\int p_{X_{0}\mymid X_{t}}\big(x_{0}\mymid{x}_{t}\big)\big\|{x}_{t}-\sqrt{\overline{\alpha}_{t}}x_{0}\big\|_{2}\mathrm{d}x_{0}\right)^{2}\Bigg\}\\
 & \lesssim\frac{d^{2}(1-\alpha_{t})\log^{2}T}{(\alpha_{t}-\overline{\alpha}_{t})^{2}},
\end{align*}
where the last line follows from Lemma~\ref{lem:x0} and the properties \eqref{eqn:properties-alpha-proof}.

Putting the above bounds together immediately establishes relation~\eqref{eq:approx-t-b}.

\subsection{Proof of Lemma~\ref{lem:sde-R-full}}
\label{sec:proof-lem:sde-R-full}


According to the expression \eqref{eq:p-Y-conditional-thm4}, one has
\[
Y_{t-1}\mid Y_{t}=x_{t}\ \sim\ \mathcal{N}\Bigg(\mu_{t}(x_{t}),\,\underset{\eqqcolon\,\Sigma(\widehat{x}_{t})}{\underbrace{\frac{1-\alpha_{t}}{\alpha_{t}}\bigg(I-\frac{1-\alpha_{t}}{2(1-\overline{\alpha}_{t})}J_{t}(x_{t})\bigg)^{2}}}\,\Bigg).
\]
In order to quantify the density, we first bound the Jacobian matrix $J_{t}(x)$ defined in \eqref{eq:Jacobian-Thm4}. 
On the one hand, the expression \eqref{eq:Jt-x-expression-ij-23} tells us that $J_{t}({x}) \preceq I_d$ for any $x$,  
given that the term within the curly bracket in \eqref{eq:Jt-x-expression-ij-23} is a negative covariance matrix. 
On the other hand, $J_{t}(x)$ can be lower bounded by
\begin{align*}
J_{t}(x) & \succeq-\frac{1}{1-\overline{\alpha}_{t}}\mathbb{E}\Big[\big(X_{t}-\sqrt{\overline{\alpha}_{t}}X_{0}\big)\big(X_{t}-\sqrt{\overline{\alpha}_{t}}X_{0}\big)^{\top}\mid X_{t}=x\Big]\notag\\
 & \succeq-\frac{\mathbb{E}\Big[\big\| X_{t}-\sqrt{\overline{\alpha}_{t}}X_{0}\big\|_{2}^{2}\mid X_{t}=x\Big]}{1-\overline{\alpha}_{t}}I_{d}\succeq-\frac{2\|x\|_{2}^{2}+2T^{2c_{R}}}{1-\overline{\alpha}_{t}}I_{d}\\
 & \succeq-T^{c_{0}+1}\big(\|x\|_{2}^{2}+T^{2c_{R}}\big)I_{d},
\end{align*}
where the second line applies the assumption that $\|X_{0}\|_{2}\leq T^{c_{R}}$,
and the last line invokes the choice \eqref{eqn:alpha-t}. 
As a consequence, we have
\begin{subequations}
\label{eq:Sigma-UB-LB}
\begin{align}
\Sigma(\widehat{x}_{t}) & \succeq\frac{1-\alpha_{t}}{\alpha_{t}}\bigg(1-\frac{1-\alpha_{t}}{2(1-\overline{\alpha}_{t})}\bigg)^{2}I_{d}=\frac{1-\alpha_{t}}{4\alpha_{t}}\bigg(\frac{1-\overline{\alpha}_{t}+\alpha_{t}-\overline{\alpha}_{t}}{1-\overline{\alpha}_{t}}\bigg)^{2}I_{d}\succeq\frac{1-\alpha_{t}}{4\alpha_{t}}I_{d}\succeq\frac{1-\alpha_{t}}{4}I_{d};\\
\Sigma(\widehat{x}_{t}) & \preceq\frac{1-\alpha_{t}}{\alpha_{t}}T^{2c_{0}+2}\big(2\|\widehat{x}_{t}\|_{2}^{4}+2T^{4c_{R}}\big)I_{d}\preceq4T^{2c_{0}+2}\big(\|\widehat{x}_{t}\|_{2}^{4}+T^{4c_{R}}\big)I_{d}.
\end{align}
\end{subequations}
%

%
%

%
%

With the above relations in mind, we are ready to bound the density function $p_{Y_{t-1} \mymid Y_t}(x_{t-1}\mymid x_t)$ for any $x_t,x_{t-1}\in \mathbb{R}^d$. 
It is seen from \eqref{eq:p-Y-conditional-thm4} that
\begin{align*}
\log\frac{1}{p_{Y_{t-1}\mid Y_{t}}(x_{t-1}\mymid x_{t})} & =\frac{\big(x_{t-1}-\mu_{t}(x_{t})\big){}^{\top}\big(\Sigma(\widehat{x}_{t})\big)^{-1}\big(x_{t-1}-\mu_{t}(x_{t})\big)}{2}+\frac{1}{2}\log\mathsf{det}\big(\Sigma(\widehat{x}_{t})\big)+\frac{d}{2}\log(2\pi)\\
 & \leq\frac{2\big\| x_{t-1}-\mu_{t}(x_{t})\big\|_{2}^{2}}{1-\alpha_{t}}+\frac{d}{2}\log\Big(8\pi T^{2c_{0}+2}\big(\|\widehat{x}_{t}\|_{2}^{4}+T^{4c_{R}}\big)\Big)\\
 & \leq2T^{c_{0}+1}\left\{ 2\big\| x_{t-1}-\widehat{x}_{t}\big\|_{2}^{2}+\|x_{t}\|_{2}^{2}+T^{2c_{R}}\right\} +\frac{d}{2}\log\Big(8\pi T^{2c_{0}+2}\big(\|\widehat{x}_{t}\|_{2}^{4}+T^{4c_{R}}\big)\Big)\\
 & \leq T^{c_{0}+2c_{R}+2}\left\{ \big\| x_{t-1}-\widehat{x}_{t}\big\|_{2}^{2}+\|x_{t}\|_{2}^{2}+1\right\} ,
\end{align*}
where the second inequality results from \eqref{eq:Sigma-UB-LB}, 
and the third inequality makes use of \eqref{eq:dist-xt-mu-xt-UB-crude} and \eqref{eqn:alpha-t}. 
Given that $\log\frac{p_{X_{t-1}\mid X_{t}}(x_{t-1}\mymid x_{t})}{p_{Y_{t-1}\mid Y_{t}}(x_{t-1}\mymid x_{t})}\leq \log\frac{1}{p_{Y_{t-1}\mid Y_{t}}(x_{t-1}\mymid x_{t})}$, 
we have concluded the proof.



\bibliographystyle{apalike}
\bibliography{reference-diffusion}


\end{document}